\documentclass[11pt]{article}
\usepackage[margin=1in]{geometry}
\usepackage[T1]{fontenc}
\usepackage{calligra} 
\usepackage{lmodern}
\usepackage{microtype}
\usepackage{amsmath,amssymb,amsthm,mathtools}
\usepackage{mathrsfs}
\usepackage{graphicx}
\usepackage{float}
\usepackage{booktabs,tabularx,array}
\usepackage{tikz}
\usepackage{tikz-cd}
\usetikzlibrary{arrows.meta,positioning,decorations.pathmorphing,calc,shapes.geometric,fit,patterns}
\usepackage[normalem]{ulem}
\usepackage{enumitem}
\usepackage{authblk}
\usepackage[authoryear,sort&compress]{natbib}
\usepackage[colorlinks=true,linkcolor=blue,citecolor=blue,urlcolor=blue,hypertexnames=false]{hyperref}
\usepackage[most]{tcolorbox}
\usepackage{etoc}
\allowdisplaybreaks

\graphicspath{{./}{/mnt/data/merge_work/}{/mnt/data/neuroai_src/}}

\newtcolorbox{definitionguide}[1]{
  enhanced, breakable, colback=gray!20, colframe=black!55, boxrule=0.5pt, arc=2pt,
  left=7pt, right=7pt, top=6pt, bottom=6pt,
  title=\textbf{#1}, coltitle=black, fonttitle=\normalsize,
  attach boxed title to top left={xshift=7pt,yshift=-2pt},
  boxed title style={colback=yellow, colframe=black!55, boxrule=0.5pt, arc=2pt,
  left=4pt,right=4pt,top=2pt,bottom=2pt}
}

\theoremstyle{plain}
\newtheorem{theorem}{Theorem}
\newtheorem{lemma}{Lemma}
\newtheorem{proposition}{Proposition}
\newtheorem{corollary}{Corollary}
\theoremstyle{definition}
\newtheorem{definition}{Definition}

\theoremstyle{remark}
\newtheorem{remark}{Remark}

\newcommand{\R}{\mathbb{R}}
\newcommand{\C}{\mathbb{C}}
\newcommand{\Z}{\mathbb{Z}}
\newcommand{\E}{\mathbb{E}}
\newcommand{\eps}{\varepsilon}
\newcommand{\sig}{\sigma}
\newcommand{\rhoR}{\rho}

\newcommand{\Var}{\operatorname{Var}}
\newcommand{\rank}{\operatorname{rank}}

\newcommand{\relint}{\operatorname{relint}}
\newcommand{\AxisAlign}{\operatorname{AxisAlign}}
\newcommand{\HeadAlign}{\operatorname{HeadAlign}}
\newcommand{\Sing}{\operatorname{Sing}}

\newcommand{\id}{\operatorname{id}}

\newcommand{\op}{\mathrm{op}}

\newcommand{\gbudget}{m}

\newcommand{\calK}{\mathcal{K}}
\newcommand{\calC}{\mathcal{C}}
\newcommand{\calE}{\mathcal{E}}
\newcommand{\calP}{\mathscr{P}}
\newcommand{\calB}{\mathscr{B}}
\newcommand{\calL}{\mathscr{L}}
\newcommand{\calM}{\mathcal{M}}
\newcommand{\calS}{\mathcal{S}}
\newcommand{\calU}{\mathcal{U}}

\newcolumntype{L}[1]{>{\raggedright\arraybackslash}p{#1}}
\tikzset{flow/.style={-Latex,thick},block/.style={draw,rounded corners,minimum height=7mm,minimum width=16mm,align=center,inner sep=3pt},note/.style={draw,dashed,rounded corners,align=center,inner sep=4pt,font=\small},layer/.style={draw,rounded corners,minimum height=8mm,minimum width=20mm,align=center,inner sep=3pt},bad/.style={draw,double,rounded corners,align=center,inner sep=4pt,font=\small}}

\begin{document}

\title{Contravariance Theory:\\
Strong Alignment for Minimal Solutions to Hard Tasks}
\author[*,1]{Dan Yamins}
\author[*,2]{Aran Nayebi}
\affil[1]{Departments of Computer Science and Psychology, and Wu Tsai Neurosciences Institute, Stanford University, Stanford, CA 94305 USA}
\affil[2]{Machine Learning Department, Neuroscience Institute, and Robotics Institute,\newline Carnegie Mellon University, Pittsburgh, PA 15213 USA}
\affil[*]{Equal contribution}
\date{\today}
\maketitle

\begin{abstract}
A series of results from NeuroAI over the past fifteen years has raised core questions both about how to compare Deep Neural Network (DNN) models to the brain, and about how much convergent evolution to expect between artificial networks and real brain networks. Here, we show that for any two minimal DNN solutions to a sufficiently hard task: (i) ``weak'' alignment of network representations based on affine mappings guarantees ``strong'' alignment of privileged axes, and (ii) alignment ``zippers'' up the network hierarchy, causing the emergence of privileged axes from end-to-end task optimization. These results formalize the notion of contravariance from \citet{cao2024contravariance}, and illustrate important consequences for the theory of NeuroAI: with sufficiently strong tasks, choice of metric for inter-network comparison is not all that sensitive, and that convergent evolution is probably inevitable. 
\end{abstract}
\section{Introduction}
The field of NeuroAI is concerned with the ability of deep neural networks (DNNs) to model the brain~\citep{zador2023catalyzing}. A clear statement of this goal comes in the NeuroAI Turing Test, which asks for models whose behavior and internal representations are indistinguishable from biological systems up to the variability seen across real individuals \citep{feather2025neuroai, thobani2025iatc}.  

In one main style of NeuroAI research toward this end, so-called goal-driven modeling~\citep{yamins2016goal}, these DNN networks are optimized in an end-to-end fashion for some cognitive goal(s) -- or a self-supervised proxy for such goals. Then, fixed by the constraints of the optimization process, the networks' internal activations are compared to brain data to ask to what extent a match has organically emerged from the optimization constraints themselves. This approach has had notable success building principled, quantitatively accurate models of cortical brain areas responsible for vision~\citep{yamins2014performance, khaligh2014deep}, audition~\citep{kell2018task}, somatosensation~\citep{chung2026task}, motor behavior~\citep{sussillo2015neural, michaels2020goal}, memory and navigation~\citep{nayebi2021explaining}, human language~\citep{schrimpf2021neural}, and agentic decision making~\citep{keller2026intrinsic}.  

Many of these results have been obtained by the use of linear mapping techniques~\citep{yamins2014performance}, in which real units from brain data (electrodes or voxels, as the case may be) are fit, using a small amount of brain data, to a linear combination of neural network model units.  In this context, the key results are that the representations of highly-performant deep neural networks -- which are very non-linear functions of their input stimuli -- match those of brain responses up to linear similarity.  In fact, the better the models are at solving the AI task they have been optimized for, the better their internal representations match real neural data (Fig. \ref{fig:motivation}A).  This is nontrivial because the functions of the stimulus that real neurons compute are themselves also highly nonlinear, and finding a match between two sets of non-linear functions up to linear span is statistically extremely unlikely relative to a random functional baseline.\footnote{Note that the untrained network state with randomly-initialized filters in highly-performant architectures results in networks that are very much \emph{not} random functionals, and these ``at-birth'' networks often have substantially above-chance neural predictivity, though typically less than filter-optimized versions.}  

\begin{figure}[t]
    \centering
    \includegraphics[width=\linewidth]{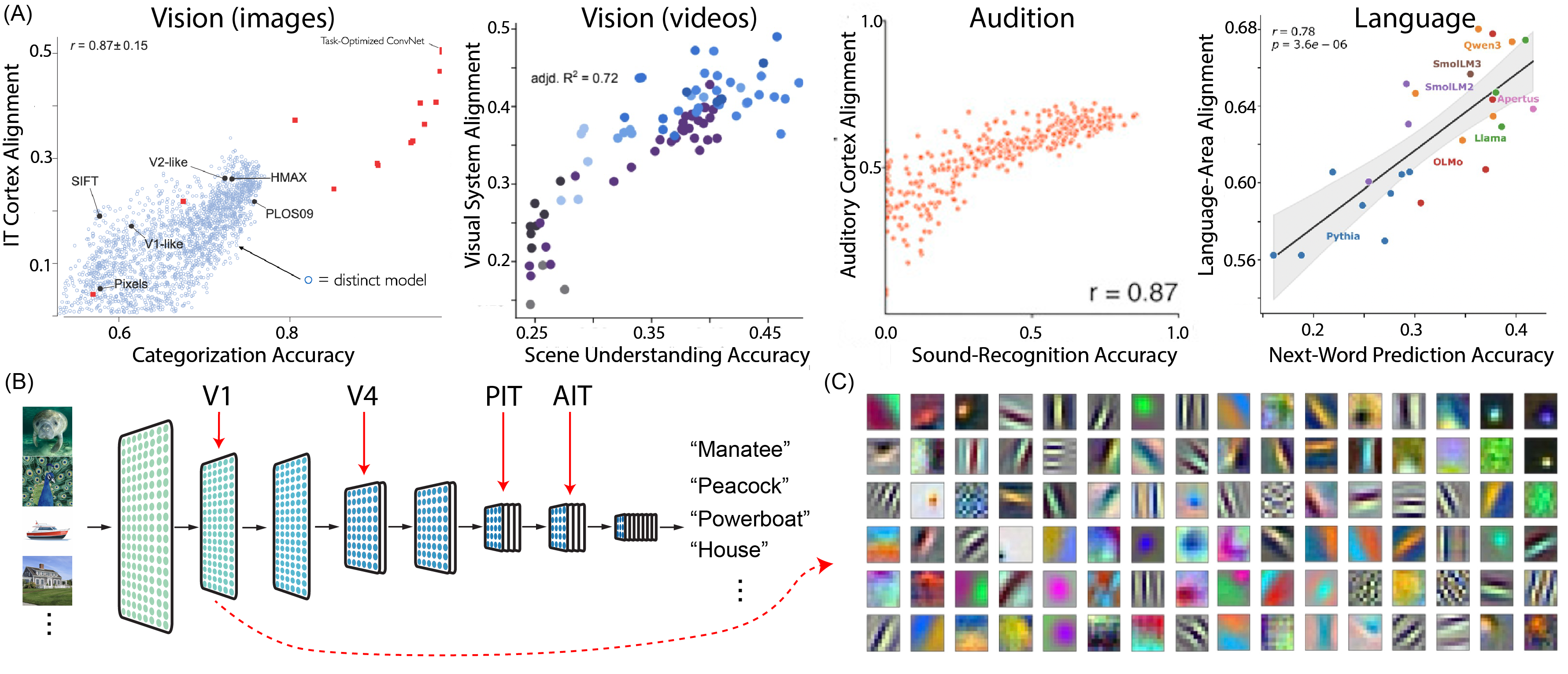}
    \caption{\textbf{Why do these core NeuroAI results arise?  What can we infer from them?} \textbf{(A)} In a variety of domains, from visual cortex responses to static images and movies, auditory cortex responses to sounds and spoken words, to language-area responses to text, there turns out to be a correlation across a wide range of DNN models, between a DNN model's performance on AI tasks and fit of that model to neural data, using linear mapping transforms.  What causes these correlations? (Images reproduced, from left to right, from: \citep{yamins2014performance}, \citep{tang2025diverse}, \citep{kell2018task}, and (right panel) by permission from M. Schrimpf.) \textbf{(B)} Within a model, there is also an emergence of neuroanatomical consistency, with model layers -- such as, in the visual system, visual areas V1, V4, posterior inferior temporal cortex (PIT) and anterior interior temporal cortex (AIT) -- being matched most effectively by different, and anatomically consistently ordered, DNN model layers. Why do the brain system hierarchies emerge in DNN model hierarchies? \textbf{(C)} And, we see that, at a putatively finer-grain of matching than linear similarity, preferred axes (such as the Gabor-like structures observed in the filters shown here, obtained from an ImageNet-optimized AlexNet) reliably emerge deep inside neural networks optimized end-to-end for real-world tasks. Why? In fact, why do preferred axes exist at all?}
    \label{fig:motivation}
\end{figure}

In fact, there appears to be a fairly general (if not totally universal) observation of \emph{neuroanatomical consistency}, with a surprisingly tight match between intra-model and intra-brain hierarchies observed in a variety of domains~\citep{yamins2016goal, kell2018task, chung2026task, michaels2020goal} (but cf. \citep{schrimpf2021neural}). Each different model component (i.e. DNN layer) corresponds to a well-defined brain area and vice versa (Fig. \ref{fig:motivation}B): lower layers of deep networks match early cortical areas, mid-model layers best match intermediate cortical areas, and higher layers of deep networks match high-level cortical areas \citep{yamins2016goal}.  In other words, the neurons in each real cortical area have a linear span covering a distinct characteristic subspace of area-specific functions; and these align in an area-specific and order-preserving fashion with the progression of function spaces created by the layers of the task-optimized neural network model.  This observation serves to underline the non-triviality of the brain-model match, since it rules out deflationary explanations (such as, e.g., the possibility that the brain-model alignment arises merely due to the high dimensionality of the model feature space~\citep{elmoznino2024high}).  

Intriguingly, these types of model-brain similarities also arise at a finer grain of detail than linear similarity. Detailed structures buried deep in the midst of neural processing chains -- even in brain areas far from the sensory or motor periphery -- can emerge merely due to end-to-end constraints. An especially striking result of this kind was the observation that the most characteristic of visual response patterns -- the famous Gabor-like tuning of many V1 neurons -- appears in early layers of deep neural networks optimized for downstream tasks (see Fig. \ref{fig:motivation}C and \citep{krizhevsky2012imagenet}). This observation has been generalized to the notion of \emph{privileged axes}, where it has been shown that brains and neural networks can agree not only in affine-invariant content, but also in native axes: single units, filters, or voxels can have systematically aligned tuning functions~\citep{khosla2024privileged}.

Aside from providing quantitatively accurate models (and thereby potentially enabling  practically useful neural control possibilities~\citep{bashivan2019neural,honarmand2025inducing}), the NeuroAI results of the past decade suggest a deeper conceptual conclusion of \emph{convergent evolution}: that the underlying constraints on task-trained neural networks are reasonably similar to those which must have been acting on real neural systems, posing a causal explanation for why nontrivially-similar representations are observed.  Despite being an attractive hypothesis, it is not obvious how strongly the inference of convergent evolution is licensed from correlational observations of model-brain alignment.

These considerations have occasioned two (or perhaps, more accurately, two-and-a-half) fundamental lines of questions. First, what is the correct method for comparing DNN models to the brain? With what granularity should comparisons be made between DNN structures and brain structures?  Should we seek a strict match at the single unit level, in which specific privileged representational axes are asked to align between model and brain~\citep{khosla2024privileged}?  Or should we be entirely comfortable with looser modes of comparison, e.g., linear transforms between model layers and brain areas that might in principle destroy unit-level privileged basis elements~\citep{yamins2016goal}?  Or something in between~\citep{thobani2025iatc}? 

A second body of questions asks, why do we reliably observe certain privileged axes at certain locations -- such as Gabor-like units in V1 -- in the first place?  And relatedly (our half-question more), why do these structures arise, at least in DNNs, from end-to-end task optimization (see Fig. \ref{fig:motivation}C)?  To what extent do end-to-end constraints force deep internal structures to be a particular way? How many different structurally distinct solutions to an end-to-end task constraint can there be? Are there any underlying theoretical reasons for convergent evolution between natural and artificial brains? 
 
Here, we present two sets of mathematical results that go some distance in resolving the answers to both of these types of questions, and exposing how they are related:
\begin{enumerate}
	\item \textbf{The Weak-Strong Equivalence Theorems:} For any two sufficiently minimal DNNs solving a sufficiently hard task (``minimal'' and ``hard'' in senses to be formalized below), we show that equivalence up to linear similarity (``weak alignment'') in adjacent layers forces equivalence of the privileged axes observed at the single-unit level (``strong alignment'').  
	\item \textbf{The Zippering Theorems:}  We show that two such solutions that are weakly (that is, linearly) equivalent at \emph{only} the terminal layer ``zipper up'', becoming weakly equivalent at \emph{all} upstream layers -- and thus (due to the first theorem) are also axis-aligned upstream as well.  
\end{enumerate}
\noindent 
Each result is formulated in exact and soft forms. 
The exact versions, which express a clean mathematical ideal, assume perfect weak equivalence and recover exact axis alignment. The soft versions, which apply better to real-world situations, measure deviations from exact equivalence and relate those errors to quantitative lower bounds on approximate axis alignment.

The first of these two results says, essentially, that {\bfseries\itshape\uline{privileged axes emerge because of the constraints of the task}}. The second says, essentially, that there are not that many minimal solutions to hard tasks, implying that {\bfseries\itshape\uline{the empirically-observed convergent evolution results of goal-driven NeuroAI are probably mathematically inevitable}}. Taken together, they also suggest that in the domain of ``hard tasks'', {\bfseries\itshape\uline{it does not really matter what metric you pick}} -- because the apparently loosest metric under normal consideration (mere performance on tasks) ends up forcing the accession of the apparently stricter one (privileged axes at the unit level).

\begin{center}
***
\end{center}
\begin{figure}[t]
    \centering
    \includegraphics[width=.85\linewidth]{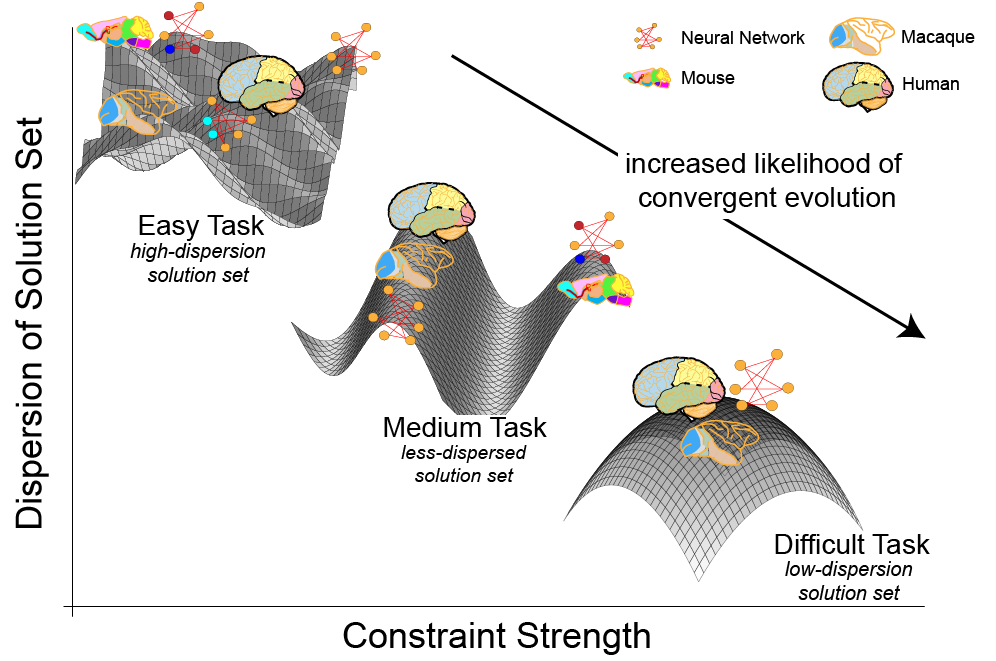}
    \caption{\textbf{The contravariance principle.} The work of \cite{cao2024contravariance} notes that, given a fixed architecture size, the dispersion of a set of solutions to an easy task (upper left of figure) is larger than the dispersion of a set of solutions to an hard task (lower right).  In other words, the solution set dispersion is \emph{contravariant} in (inversely related to) the constraint strength.  This idea provides an informal understanding of why convergent evolution emerges between artificial neural networks and real brains.  This paper provides one version of how the contravariance program can be formalized, yielding nontrivial results for NeuroAI research going forward. (Figure modified from \cite{cao2024contravariance}.)}
    \label{fig:contravariance}
\end{figure}

These results are a mathematical formalization of the \emph{contravariance principle}, in the sense originally described by \citet{cao2024contravariance}. 
The contravariance concept described in that work notes that the harder the task that a network is constrained to do, the fewer the solutions there will be to it, and therefore the more likely it will be that any two solutions to it will be similar to each other in important respects (Fig. \ref{fig:contravariance}).   Contravariance thus seeks to explain the convergent evolution observed in NeuroAI by invoking the constraining power of hard tasks.  In \cite{cao2024contravariance}, contravariance is developed as an informal philosophical idea.  While the intuitions there seem strong and self-consistent, the notions of task difficulty and solution set size are not mathematically specific, so conclusions remain at a conceptual level. The present work formulates a version of contravariance precisely, and shows that it has the type of nontrivial consequences for our understanding of DNN models of the brain that the original philosophical version intended.  

While \citet[Corollary 5]{nayebi2026capable} recently showed that two minimal systems that achieve vanishing regret on the same tasks necessarily must have an invertible mapping (isomorphism) between them, the theorems here go much further, specifying the \emph{complexity} of this mapping as being linear in the cases of DNNs with ReLUs and softplus nonlinearities, thereby explaining the alignment we see in practice under such linear mappings.

\section{Weak--Strong Equivalence}
\label{sec:main-weak-strong}
\begin{definitionguide}{Intuition}
This section introduces two notions of alignment. \emph{Weak alignment} means that two representations agree up to a linear/affine map. \emph{Strong alignment} means that individual coordinate axes match up, possibly after permutation and rescaling. The main result of the section shows that, under a usedness assumption on the nonlinearities and a hardness assumption on the task, weak alignment across adjacent layers forces strong axis alignment.
\end{definitionguide}

We start by formalizing the setting of two neural networks with similar hierarchical architectures.  Let $N\in\{A,B\}$ denote one of two networks.  At layer $\ell$, write
\[
 z^N_\ell(x)=W^N_\ell h^N_{\ell-1}(x)+b^N_\ell\in\R^{d^N_\ell},
 \qquad
 h^N_\ell(x)=\varphi(z^N_\ell(x)),
\]
where the nonlinearity $\varphi$ is applied coordinatewise.  For the results of this work, we will mainly consider ReLu ($\rho$) and softplus ($\sigma$) nonlinearities:
\[
 \rho(t)=\max\{t,0\}\qquad\text{and}\qquad \sigma(t)=\log(1+e^t).
\]
The functions $\rho$ and $\sigma$ represent two of the main classes of nonlinearity used in real-world neural networks, and require somewhat different treatment technically (the details of which are in the Appendix, along with a discussion of how the common argument structure applies to other nonlinearity types).
We will denote the full-layer map of affine-followed-by-nonlinear operations, from layer-$\ell$ preactivations to layer-$(\ell+1)$ preactivations, by $\Phi_\ell$:
\[
 \Phi^N_{\ell,\varphi}(z):=W^N_{\ell+1}\varphi(z)+b^N_{\ell+1}.
\]
Let $\Omega$ be the set of inputs to the network (the ``data''), and let $M^N_\ell=z^N_\ell(\Omega)\subseteq \R^{d^N_\ell}$ be the activations (sometimes called the ``representation'') at layer $\ell$.

The notion of weak alignment formalizes the idea of using affine maps to compare neural networks (see Fig. \ref{fig:alignment_defs}A).
\begin{definition}[Weak alignment]
A layer-$\ell$ comparison map from $A$ to $B$ is an injective affine map
\[
E_\ell(z)=T_\ell z+a_\ell,
\qquad
T_\ell:\R^{d^A_\ell}\to\R^{d^B_\ell}.
\]
The networks are weakly equivalent at layer $\ell$ on the task set if
\[
 z^B_\ell(x)=E_\ell z^A_\ell(x)
 \qquad \forall x\in\Omega.
\]
\end{definition}

In contrast to weak alignment, strong alignment asks whether a $B$-coordinate is just the corresponding $A$-coordinate, possibly after a rescaling (see Fig. \ref{fig:alignment_defs}B).
\begin{definition}[Strong alignment]
Let $J^N_\ell$ denote the set (counted \emph{without} multiplicity) of coordinate axes in network $N$. Networks $A$ and $B$ are \emph{exactly strongly equivalent} at layer $\ell$ if there is a comparison map $E_{\ell}$ and a bijective map $\pi:J_\ell^A\to J_\ell^B(E_\ell)$ and for which, for all $j \in J^A_\ell$, 
\[
 (E_\ell z)_{i = \pi(j)}=\alpha z_j \qquad \forall z\in \Omega,
\]
where $\alpha\neq0$ for the case of ReLU and $\alpha=1$ for Softplus.  Here $z$ denotes a layer-$\ell$ preactivation vector of network $A$, e.g. the activity of A-unit $j$, while $(E_\ell z)_i$ is the corresponding mapped activity in B-coordinate $i$. 
For non-strongly aligned network pairs, it is useful to quantify the amount of strong alignment present via the metric
\[
\AxisAlign_{\ell}(A,B)
=
\frac{1}{d_\ell^A}
\max_{\pi}
\#\Bigl\{
j\in J_\ell^A:
(E_\ell z)_{\pi(j)}=\alpha_j z_j \text{ for } \alpha_j \neq 0
\Bigr\},
\]
where \(\pi\) ranges over injective maps
$\pi:J_\ell^A\to J_\ell^B(E_\ell)$ (and $\alpha = 1$ for softplus).
In words, \(\AxisAlign_{\ell}\) is simply the fraction of \(A\)-side layer
\(\ell\) axes that are visibly and exactly matched under \(E_\ell\), normalized
by the full \(A\)-side layer width \(d_\ell^A\). By definition, \(\AxisAlign_{\ell} = 1\) is equivalent to strong alignment.
\end{definition}

\begin{figure}[t]
\centering
\begin{tikzpicture}[>=Latex, font=\small, scale=0.92]
\begin{scope}
  \node[align=center] at (3.3,5.15) {\textbf{A. Weak alignment at adjacent layers}};
  \draw[gray!60] (-0.55,0.05) rectangle (7.15,4.75);
  \node[draw, rounded corners, fill=blue!8, minimum width=2.0cm, minimum height=0.75cm] (Aell) at (1.3,3.75) {$z_\ell^A\in C$};
  \node[draw, rounded corners, fill=blue!8, minimum width=2.4cm, minimum height=0.75cm] (Aellp) at (5.55,3.75) {$z_{\ell+1}^A=\Phi_\ell^A(z_\ell^A)$};
  \node[draw, rounded corners, fill=green!8, minimum width=2.2cm, minimum height=0.75cm] (Bell) at (1.3,1.05) {$E_\ell z_\ell^A\approx z_\ell^B$};
  \node[draw, rounded corners, fill=green!8, minimum width=2.35cm, minimum height=0.75cm] (Bellp) at (5.55,1.05) {$\Phi_\ell^B(E_\ell z_\ell^A)$};
  \draw[->, thick] (Aell) -- node[above] {$\Phi_\ell^A$} (Aellp);
  \draw[->, thick] (Aell) -- node[left] {$E_\ell$} (Bell);
  \draw[->, thick] (Bell) -- node[below] {$\Phi_\ell^B$} (Bellp);
  \draw[->, thick] (Aellp) -- node[right] {$E_{\ell+1}$} (Bellp);
  \node[align=center, text width=7.0cm] at (3.3,-0.65) {\footnotesize Weak alignment asks whether two networks are similar up to affine transform.};
\end{scope}
\begin{scope}[xshift=8.5cm]
  \node[align=center] at (3.8,5.15) {\textbf{B. Strong axis alignment}};
  \draw[gray!60] (-0.55,0.05) rectangle (8.05,4.75);
  \draw[->] (0.70,0.75) -- (0.70,4.15);
  \draw[->] (0.70,0.75) -- (3.10,0.75);
  \node[rotate=90] at (0.26,2.45) {\footnotesize samples $z\in C$};
  \node[below] at (2.20,0.75) {$z_j$};
  \draw[blue!70!black, very thick] (1.10,1.05) -- (2.55,3.65);
  \node[blue!70!black] at (1.75,4.00) {\footnotesize $A$-axis $j$};
  \draw[->] (5.35,0.75) -- (5.35,4.15);
  \draw[->] (5.35,0.75) -- (7.60,0.75);
  \node[rotate=90] at (4.90,2.45) {\footnotesize samples $z\in C$};
  \node[below] at (7.00,0.75) {$(E_\ell z)_i$};
  \draw[green!60!black, very thick] (5.75,1.35) -- (7.35,2.65);
  \node[green!50!black, align=center] at (6.55,3.95) {\footnotesize $B$-axis $i$\\[-1mm]\footnotesize after $E_\ell$};
  \draw[->, thick] (3.15,2.30) -- (4.70,2.30);
  \node[align=center] at (3.65,2.88) {\footnotesize match up\\[-1mm]\footnotesize to scale $\alpha$};
  \node[align=center, text width=7.8cm] at (3.75,-0.78) {\footnotesize Strong axis alignment asks whether coordinates are matched up to rescaling.};
\end{scope}
\end{tikzpicture}
\caption{Two notions of similarity. \textbf{A:} Weak alignment. \textbf{B:} Strong axis alignment.}
\label{fig:alignment_defs}
\end{figure}

\noindent It is natural to suppose that weak alignment is indeed generally weaker than strong alignment: e.g. that two networks might often be weakly equivalent without also being strongly equivalent.  However, our first main result is to show that in an important case -- namely, where two adjacent layers are weakly equivalent and where the task the network solves is suitably hard -- the distinction between weak and strong alignment collapses. This result is based on the intuition that if two networks are weakly aligned at layer $\ell$, but have misaligned privileged axes, then after the nonlinearity at layer $\ell+1$, they will be different enough that they will not even be weakly aligned there. In other words, if the two task representations are weakly related before and after a nonlinearity, the weak map must preserve the nonlinear signatures of individual coordinates: ReLU kink traces in the nonsmooth case, and Softplus affine-ridge curvature in the smooth case. This is the essential mathematical reason that native axes can become privileged rather than arbitrary.

This intuition relies on the nonlinearity having a nontrivial effect -- that is, it must be ``used'':
\begin{definition}[Used axes, informal]
A layer-$\ell$ axis is \emph{used} if the task and the next layer expose a genuine nonlinear effect of that axis. For ReLU, this means that the task ``crosses'' the zero trace set $\{z_j=0\}$ and the outgoing column $W_{\ell+1}[:,j]$ is nonzero. For Softplus, this means that the outgoing column is nonzero, the coordinate genuinely varies on the task set, and the counted coordinates are not redundant through exact duplicates, opposites, or affine sign-cancellations.  (See Appendix Definitions~\ref{def:relu-used} and~\ref{def:sp-used} for the precise version of these definitions.)
\end{definition}

\begin{figure}[t]
\centering
\begin{tikzpicture}[>=Latex, scale=1.0]
\begin{scope}
  \node[align=center,text width=5.8cm] at (2.8,4.78) {\small\textbf{A. Used vs. degenerate gates}};
  \draw[gray!60] (-0.6,-0.25) rectangle (6.2,4.35);
  \draw[->] (0,0) -- (5.6,0) node[right] {$z_1$};
  \draw[->] (0,0) -- (0,3.9) node[above] {$z_2$};
  \filldraw[fill=blue!15, draw=blue!70!black, thick] plot[smooth cycle,tension=0.9] coordinates {(1.0,0.8) (1.6,2.8) (3.3,3.2) (4.7,2.2) (4.2,0.9) (2.5,0.5)};
  \node[blue!70!black] at (3.2,2.0) {$C$};
  \draw[red!80!black, thick, dashed] (2.8,0.2) -- (2.8,3.55) node[above] {\footnotesize used trace};
  \draw[gray!80!black, thick, dashed] (5.05,0.2) -- (5.05,3.55) node[above] {\footnotesize not seen};
  \fill (2.8,1.35) circle (1.4pt) node[below right] {$p$};
  \fill[black] (2.35,1.05) circle (1.3pt) node[below left] {$z^-$};
  \fill[black] (3.25,1.65) circle (1.3pt) node[above right] {$z^+$};
\end{scope}
\begin{scope}[xshift=7.55cm]
  \node[align=center,text width=5.8cm] at (2.8,4.78) {\small\textbf{B. Kink preservation}};
  \draw[gray!60] (-0.6,-0.25) rectangle (6.2,4.35);
  \draw[->] (0,0) -- (2.35,0) node[right] {$z_1$};
  \draw[->] (0,0) -- (0,3.5) node[above] {$z_2$};
  \filldraw[fill=blue!12, draw=blue!70!black, thick] plot[smooth cycle,tension=0.9] coordinates {(0.45,0.65) (0.8,2.45) (1.65,2.65) (2.15,1.75) (1.9,0.75) (1.1,0.45)};
  \node[blue!70!black] at (1.55,2.0) {$C$};
  \draw[red!80!black, thick, dashed] (1.25,0.15) -- (1.25,3.15) node[above] {$\{z_j=0\}$};
  \fill (1.25,1.16) circle (1.3pt);
  \draw[->, thick] (2.55,1.8) -- (3.35,1.8) node[midway,above] {\footnotesize $\Phi_\ell^A$, $E_{\ell+1}$};
  \draw[green!55!black, very thick] (3.65,0.95) -- (4.55,1.75) -- (5.70,3.15);
  \fill (4.55,1.75) circle (1.3pt) node[below right] {\footnotesize kink};
  \draw[->, green!55!black, thick] (4.18,1.42) -- (3.88,1.15);
  \draw[->, green!55!black, thick] (4.92,2.18) -- (5.28,2.68);
  \node[green!45!black] at (3.95,0.78) {\footnotesize left slope};
  \node[green!45!black] at (5.18,3.38) {\footnotesize right slope};
\end{scope}
\end{tikzpicture}
\caption{Used gates and kink preservation. \textbf{A:} Only a crossed, downstream-used zero set contributes a task-visible nonlinear signature. \textbf{B:} A used ReLU boundary creates a slope change, and an injective affine map cannot erase that change.}
\label{fig:relu-intuition}
\end{figure}
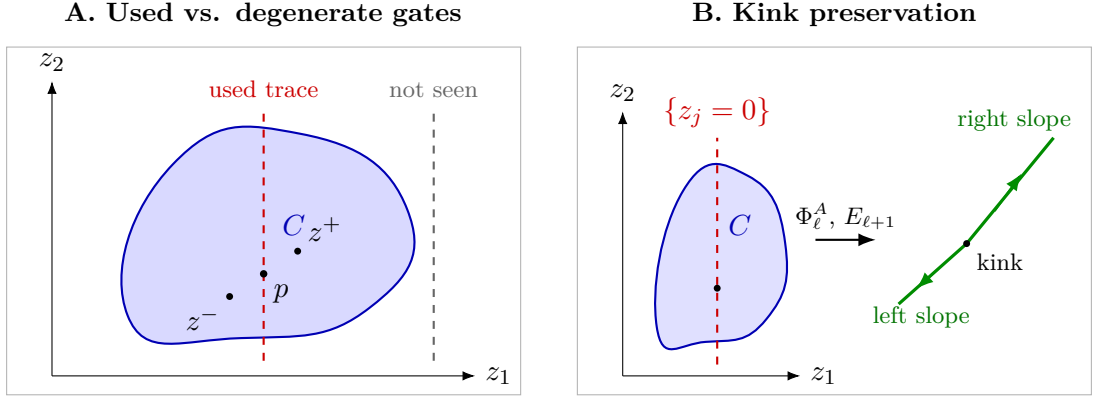

Non-used ReLU axes can be removed from the nonlinear gate requirements of the network: an unused gate contributes nothing downstream, while a gate whose trace is not crossed is affine or zero on the task patch and can be absorbed into the adjacent affine part.  This ``shrinkability'' principle is formalized in Appendix Lemma~\ref{lem:relu-shrinkability}, and motivates the following definition of used-axis budget:

\begin{definition}[Required used-axis budget]
Let $L$ be the loss function for the task.  Let $\mathcal F_{\ell,\le k}$ be the networks of the fixed macroarchitecture that solve the task using at most $k$ used layer-$\ell$ axes. Define
\begin{equation}
 m_\ell(\varepsilon)
 =
 \min\{k:\exists f\in\mathcal F_{\ell,\le k}\text{ with }L(f)\le\varepsilon\},
 \label{eq:main-m-budget}
\end{equation}
with $m_\ell(\varepsilon)=+\infty$ if no such $k$ exists. The quantity $m_\ell(\varepsilon)$ is a \textbf{measure of task difficulty}: it counts how many nonlinear axes the task forces this layer to use.
\end{definition}
With these ideas in mind, we can state the first main theorem:
\begin{theorem}[Exact weak--strong equivalence]
\label{thm:main-exact-weak-strong}
Fix adjacent layers $\ell$ and $\ell+1$. Suppose $A$ and $B$ are exactly weakly equivalent at both layers, and suppose the layer-$\ell$ axes used by the task satisfy the appropriate adjacent usedness assumptions. Then every used $A$-side axis is exactly matched to a distinct pulled-back $B$-side axis. Consequently, if $A$ solves the task to loss at most $\varepsilon$, then
\[
 \AxisAlign_\ell(A,B)
 \ge
 \frac{m_\ell(\varepsilon)}{d_\ell^A}.
\]
In particular, if all layer-$\ell$ axes are task-used, weak alignment across the adjacent pair of layers implies strong alignment. 
\end{theorem}
\noindent The formal statement of this theorem, as well as the proofs, are given in appendix \S\ref{sec:exact-ws}. The basic proof method is a little bit different for the ReLU and softplus cases, but essentially shows that if a gate is used, it causes an effect on outputs at the next layer, such that when the axes are unaligned before the nonlinearity, the outcomes after the nonlinearity are not affinely related.  A network optimized to perform a task that strongly constrains it will have many or all of its gates used. Thus, the theorem says that two layers of successive weak alignment, for a sufficiently hard task, force strong privileged axis-alignment at the earlier layer.  

\textbf{Key implication for NeuroAI:} We can interpret Theorem~\ref{thm:main-exact-weak-strong} as saying, essentially: {\bfseries\itshape\uline{hard tasks cause privileged axes, and under the condition of a hard task, weak and strong alignment are equivalent.}}

\section{Asymptotic Weak--Strong Equivalence}
\label{sec:main-asymptotic}
\begin{definitionguide}{Intuition}
Section~\ref{sec:main-weak-strong} showed that exact weak alignment across adjacent layers forces exact axis alignment. This section makes that statement robust: small weak-alignment errors force many axes to be approximately aligned, and for sequences of pairs of networks that are converging in the weak-alignment sense, strong alignment eventually appears at the task-used axes---that is, units whose activity actually matters for solving the task, rather than being redundant or ignored downstream.
\end{definitionguide}

The formulation in Theorem~\ref{thm:main-exact-weak-strong} has intentionally strong assumptions. Empirically, two representations are almost never exactly weakly equivalent. It is therefore natural to seek a version of the theorem that describes how many gates must align given a certain defect in weak alignment. What we really want to know, in fact, is that if a sequence of networks $A_n,B_n$, such as might arise during optimization, converge toward each other in terms of weak alignment, the same will occur with strong alignment as well (see Fig. \ref{fig:asymptotic-weak-strong-main}). Formalizing this concept requires making soft versions of the two alignment concepts. 

\begin{definition}[Soft weak alignment]
To measure failures of weak alignment, let
\[
 \omega_\ell(E_\ell;A,B)
 =
 \left(\E_\Omega\|z^B_\ell(x)-E_\ell z^A_\ell(x)\|_2^2\right)^{1/2}.
\]
When $\omega_\ell=0$, exact weak alignment holds at layer $\ell$.
\end{definition}

\begin{definition}[Soft strong alignment, informal]
For $0<\theta\le1$, define $\AxisAlign_{\ell}^\theta(A,B)$ as the fraction of used $A$-side axes that can be injectively matched to $B$-side axes with normalized accuracy at least $\theta$, normalized to the variability of the axes themselves. By definition, $\AxisAlign^1_\ell=1$ is exact strong alignment in layer $\ell$. (See \S\ref{sec:asymptotic-ws} for a completely formal definition.)
\end{definition}

\begin{figure}[t]
\centering
\resizebox{\textwidth}{!}{%
\begin{tikzpicture}[>=Latex,font=\small,x=0.88cm,y=1.05cm]

\def\TopY{4.25}
\def\AxisBottom{0.70}
\def\AxisTop{3.88}
\def\BottomTextY{-0.20}
\def\XAxisLabelY{0.36}

\begin{scope}
\node[font=\bfseries] at (3.55,\TopY) {Converging in Representation};

\draw[->,thick] (0,\AxisBottom) -- (0,\AxisTop);
\draw[->,thick] (0,\AxisBottom) -- (7.10,\AxisBottom);
\node[font=\small] at (3.55,\XAxisLabelY) {optimization time};
\node[rotate=90,font=\small] at (-0.60,2.29) {linear misalignment};

\draw[blue!70!black,very thick,smooth]
  plot coordinates {(0.35,3.58) (1.20,3.27) (2.55,2.62) (4.35,1.70) (6.45,0.98)};
\draw[green!55!black,very thick,smooth]
  plot coordinates {(0.35,2.42) (1.20,2.48) (2.55,2.10) (4.35,1.48) (6.45,0.95)};

\foreach \x/\ya/\yb/\labA/\labB/\dxA/\dyA/\dxB/\dyB in {
  1.05/3.33/2.46/{A_{n-2}}/{B_{n-2}}/-0.34/0.34/0.08/-0.34,
  2.30/2.74/2.18/{A_{n-1}}/{B_{n-1}}/-0.42/0.32/0.08/-0.34,
  4.05/1.86/1.58/{A_n}/{B_n}/-0.28/0.34/0.12/-0.34,
  5.95/1.11/0.98/{A_{n+1}}/{B_{n+1}}/-0.54/0.31/0.14/-0.30
}{
  \fill[blue!70!black] (\x,\ya) circle (2.1pt);
  \fill[green!55!black] (\x,\yb) circle (2.1pt);
  \draw[densely dashed,gray] (\x,\ya) -- (\x,\yb);
  \node[blue!70!black,font=\scriptsize,fill=white,inner sep=1pt]
    at ($(\x,\ya)+(\dxA,\dyA)$) {$\labA$};
  \node[green!55!black,font=\scriptsize,fill=white,inner sep=1pt]
    at ($(\x,\yb)+(\dxB,\dyB)$) {$\labB$};
}

\node[align=center,font=\scriptsize] at (3.55,\BottomTextY)
  {Weak errors $\omega_\ell,\omega_{\ell+1}$ shrink along the sequence};
\end{scope}

\draw[->,very thick] (7.72,2.32) -- (9.02,2.32);
\node[font=\scriptsize,fill=white,inner sep=1pt] at (8.37,2.62)
  {asymptotic};
\node[font=\scriptsize,fill=white,inner sep=1pt] at (8.37,2.02)
  {equivalence};

\begin{scope}[xshift=8.95cm]
\node[font=\bfseries] at (3.55,\TopY) {Convergence in Axes};

\draw[->,thick] (0,\AxisBottom) -- (0,\AxisTop);
\draw[->,thick] (0,\AxisBottom) -- (7.15,\AxisBottom);

\node[font=\scriptsize] at (0.95,\XAxisLabelY) {$n-2$};
\node[font=\scriptsize] at (2.70,\XAxisLabelY) {$n-1$};
\node[font=\scriptsize] at (4.45,\XAxisLabelY) {$n$};
\node[font=\scriptsize] at (6.20,\XAxisLabelY) {$n+1$};

\draw[blue!70!black,very thick]  (0.62,0.88)--(1.30,3.55);
\draw[green!55!black,very thick] (0.92,0.88)--(1.62,1.92);
\draw[<->,gray,densely dashed] (1.76,1.92)--(1.44,3.55);

\draw[blue!70!black,very thick]  (2.37,0.88)--(3.05,3.43);
\draw[green!55!black,very thick] (2.60,0.88)--(3.24,2.36);
\draw[<->,gray,densely dashed] (3.40,2.36)--(3.18,3.43);

\draw[blue!70!black,very thick]  (4.12,0.88)--(4.80,3.30);
\draw[green!55!black,very thick] (4.27,0.88)--(4.91,2.73);
\draw[<->,gray,densely dashed] (5.06,2.73)--(4.95,3.30);

\draw[blue!70!black,very thick]  (5.87,0.88)--(6.55,3.22);
\draw[green!55!black,very thick] (5.92,0.88)--(6.57,3.13);
\draw[<->,gray,densely dashed] (6.72,3.13)--(6.70,3.22);

\draw[->,thin,gray] (1.82,3.65) -- (2.20,3.65);
\draw[->,thin,gray] (3.57,3.65) -- (3.95,3.65);
\draw[->,thin,gray] (5.32,3.65) -- (5.70,3.65);

\draw[blue!70!black,very thick] (0.62,3.95)--(1.08,3.95);
\node[font=\scriptsize,anchor=west] at (1.18,3.95) {$A$ axis};
\draw[green!55!black,very thick] (4.42,3.95)--(4.88,3.95);
\node[font=\scriptsize,anchor=west] at (4.98,3.95) {$B$ axis};

\node[align=center,font=\scriptsize] at (3.55,\BottomTextY)
  {All used axes eventually align};
\end{scope}

\end{tikzpicture}%
}
\caption{Asymptotic weak--strong alignment along optimization. In a common use case, two independently initialized networks are trained on the same task. As their adjacent weak-alignment errors shrink, the asymptotic weak--strong theorem converts this into strong-axis convergence: every task-required used axis is eventually $\theta$-aligned.}
\label{fig:asymptotic-weak-strong-main}
\end{figure}
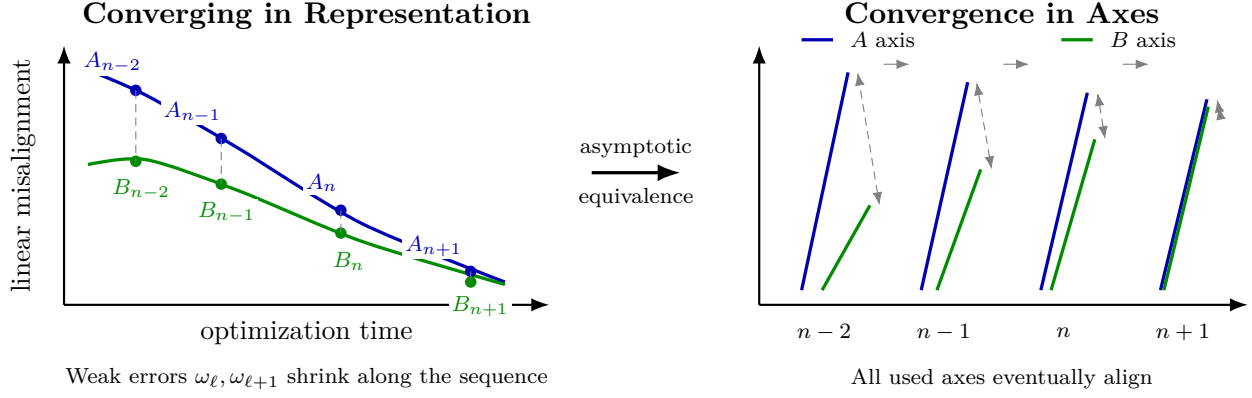

\noindent We are now in a position to state an asymptotic version of weak--strong equivalence. Here, our goal is not to compare a single pair of networks with a single pair of comparison maps, but rather, to understand the behavior of convergent sequences within a  \emph{set} of networks and such comparison maps.
\begin{theorem}[Asymptotic weak--strong equivalence, informal]
\label{thm:main-asymp-weak-strong}
Let $\mathcal K$ be a compact set of tuples $p=(A,B,E_\ell,E_{\ell+1})$. Then for each $\theta\in(0,1)$ there is a constant $\kappa_{\mathcal K}(\theta)>0$ such that, whenever $A$ and $B$ solve the task to loss at most $\varepsilon$,
\begin{equation}
\AxisAlign_{\ell}^{\theta}(A,B)
\ge
\frac{m_\ell(\varepsilon)}{d_\ell^A}
-
\frac{\bigl(\|W_{\ell+1}^B\|\omega_\ell(p)+\omega_{\ell+1}(p)\bigr)^2}
{\kappa_{\mathcal K}(\theta)^2 d_\ell^A}.
\label{eq:main-asymptotic-bound}
\end{equation}
Moreover, if $p_n\in\mathcal K$ is a sequence of network pairs which are converging to weak alignment in two adjacent layers $\ell$ and $\ell + 1$ -- i.e. that $\omega_\ell(p_n) \to 0$ and $\omega_{\ell+1}(p_n) \to 0$) -- then there is an $N$ such that for $n \geq N$, all task-required used axes at layer $\ell$ are eventually $\theta$-aligned:
\[
\AxisAlign_{\ell}^{\theta}(A_n,B_n)
\ge
\frac{m_\ell(\varepsilon)}{d_\ell^A}.
\]
\end{theorem}
\noindent The formal theorem and proof are Appendix Theorem~\ref{thm:common-asymptotic}. The intuition behind the proof is that, if misaligned privileged axes persisted while the networks got closer in the weak equivalence sense, a convergent subsequence would limit to an exactly weakly aligned pair with misaligned axes -- contradicting Theorem \ref{thm:main-exact-weak-strong}.  

Theorem \ref{thm:main-asymp-weak-strong} shows that if the error is already small at both sides of a ReLU step, then many task-required gates cannot be badly misaligned, and the defect from axis alignment is an error term controlled by the two adjacent weak equivalence deficits. In fact, the result is stronger: full axis alignment must manifest not just in an asymptotic limit, but in an actual specific network.  The result assumes that the set of networks is compact: it is very natural to arrange compactness of a set of comparison problems by requiring that all network and comparison map parameters have norm bounded by some constant $M$ (see Appendix for further details). 

\textbf{Key implication for NeuroAI:}  Asymptotic weak-strong equivalence emerges: {\bfseries\itshape\uline{if two sequences of networks tend toward having similar representations during optimization toward a common goal, their axes eventually become aligned}}.

\section{Zippering Up to Full Contravariance}
\label{sec:main-zippering}

\begin{definitionguide}{Intuition}
The previous section showed that weak alignment on both sides of a layer can force strong axis alignment there. Here we show that, for \emph{minimal networks}, weak alignment propagates upstream: terminal weak alignment forces strong axis alignment one layer earlier, and iterating this step ``zippers'' strong alignment upstream through the compared layers.
\end{definitionguide}

\begin{figure}[t]
\centering

\begin{minipage}[t]{0.49\textwidth}
\vspace{0pt}
\centering
\textbf{(A) Empirical evidence of zippering.}\\[0.3ex]

\includegraphics[width=\linewidth]{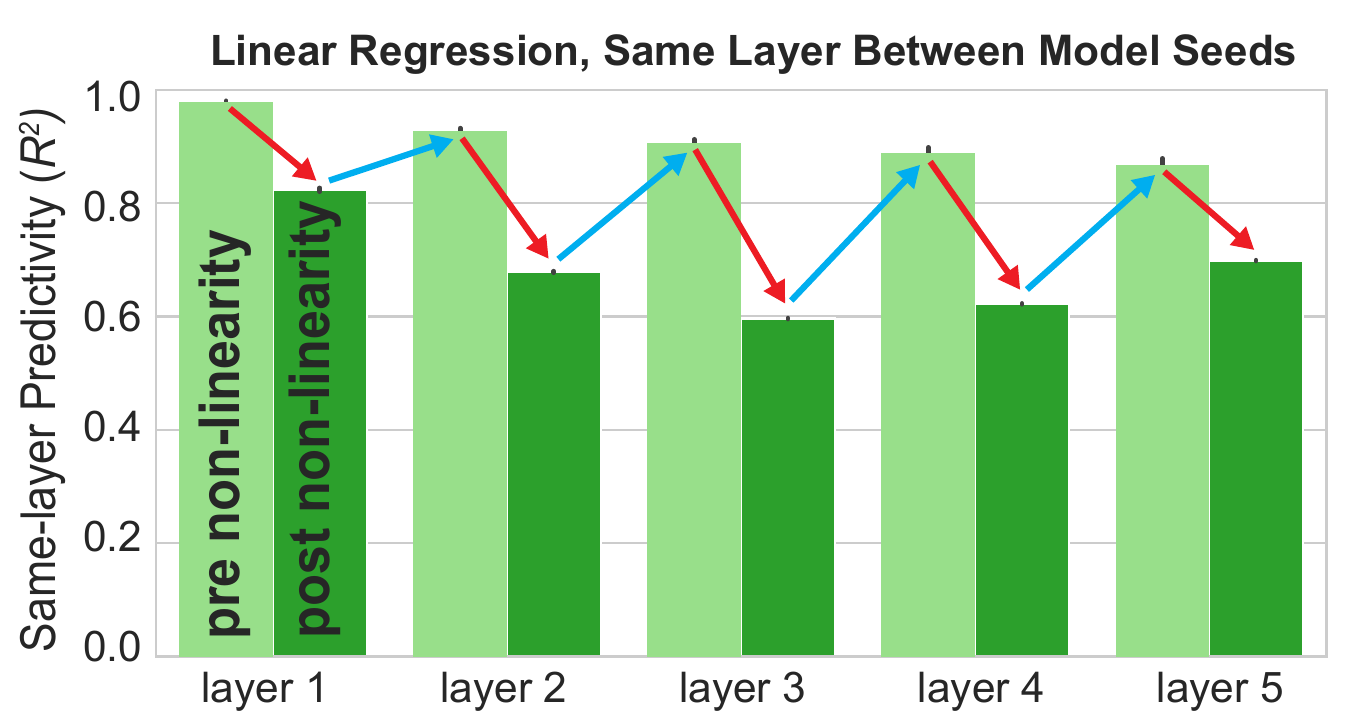}

\end{minipage}
\begin{minipage}[t]{0.50\textwidth}
\vspace{0pt}
\centering

\textbf{(B) The zippering mechanism.}\\[0.3ex]

\resizebox{\linewidth}{!}{%
\begin{tikzpicture}[>=Latex,x=1.55cm,y=1cm,font=\scriptsize]

  \tikzset{
    box/.style={
      draw,
      rounded corners,
      minimum width=.95cm,
      minimum height=.50cm,
      inner sep=1pt,
      align=center,
      fill=blue!6
    },
    dotbox/.style={
      box,
      minimum width=.52cm
    },
    map/.style={->,thick},
    eqmap/.style={->,very thick},
    zip/.style={->,very thick,dashed,blue!70!black},
    back/.style={->,ultra thick,orange!85!black}
  }

  \node[box]    (Aq)    at (0,2.0) {$z^A_q$};
  \node[box]    (Aq1)   at (1,2.0) {$z^A_{q+1}$};
  \node[dotbox] (Adots) at (2,2.0) {$\cdots$};
  \node[box]    (As1)   at (3,2.0) {$z^A_{s-1}$};
  \node[box]    (As)    at (4,2.0) {$z^A_s$};

  \node[box]    (Bq)    at (0,0) {$z^B_q$};
  \node[box]    (Bq1)   at (1,0) {$z^B_{q+1}$};
  \node[dotbox] (Bdots) at (2,0) {$\cdots$};
  \node[box]    (Bs1)   at (3,0) {$z^B_{s-1}$};
  \node[box]    (Bs)    at (4,0) {$z^B_s$};

  \draw[map] (Aq)    -- node[midway,above=3pt,font=\tiny,inner sep=1pt] {$W\sigma$} (Aq1);
  \draw[map] (Aq1)   -- node[midway,above=3pt,font=\tiny,inner sep=1pt] {$W\sigma$} (Adots);
  \draw[map] (Adots) -- node[midway,above=3pt,font=\tiny,inner sep=1pt] {$W\sigma$} (As1);
  \draw[map] (As1)   -- node[midway,above=3pt,font=\tiny,inner sep=1pt] {$W\sigma$} (As);

  \draw[map] (Bq)    -- node[midway,below=3pt,font=\tiny,inner sep=1pt] {$W\sigma$} (Bq1);
  \draw[map] (Bq1)   -- node[midway,below=3pt,font=\tiny,inner sep=1pt] {$W\sigma$} (Bdots);
  \draw[map] (Bdots) -- node[midway,below=3pt,font=\tiny,inner sep=1pt] {$W\sigma$} (Bs1);
  \draw[map] (Bs1)   -- node[midway,below=3pt,font=\tiny,inner sep=1pt] {$W\sigma$} (Bs);

  \draw[eqmap] (As)  -- node[right=-14pt,font=\tiny,inner sep=1pt] {$E_s$ $\ $ \textbf{given}} (Bs);
  \draw[zip]   (As1) -- node[left=4pt,font=\tiny,inner sep=1pt] {$E_{s-1}$} (Bs1);
  \draw[zip]   (Aq1) -- node[right=2pt,font=\tiny,inner sep=1pt] {$E_{q+1}$} (Bq1);
  \draw[zip]   (Aq)  -- node[right=2pt,font=\tiny,inner sep=1pt] {$E_q$} (Bq);

  \draw[back] (3.85,2.72) -- (0.15,2.72);

  \node[
    orange!85!black,
    fill=white,
    inner sep=1.5pt,
    align=center,
    font=\scriptsize
  ] at (2.0,3.05)
    {upstream zipper:
     $E_s \Rightarrow E_{s-1}\Rightarrow\cdots\Rightarrow E_q$};

\end{tikzpicture}%
}

\end{minipage}
\hfill

\caption{
\textbf{(A)} Same-layer predictivity between different initialization seeds of an AlexNet-like model seeds. Post the nonlinearity at each layer (dark green bars), the gap to weak equivalence is quite large; but at each pre-nonlinearity layer (light green bars), the gap to weak equivalence is much less. In other words, two networks' representations are forced apart by the nonlinearity at each layer, but end up re-converging by action of linear network stage at the next layer, creating the saw-tooth-like ``zipper'' pattern seen in the figure.  Reproduced from \cite{thobani2025iatc}.  \textbf{(B)} Zippering across layers. Starting from terminal equivalence at layer $s$, generic identifiability of the softplus expansion through layer $s-1$ produces $E_{s-1}$, and the same argument travels upstream through $s-2,\ldots,q$.  At each step, the aligned target at layer $r+1$ has two generic softplus expansions through layer $r$; identifiability matches the arguments and produces $E_r$.
}
\label{fig:zippering}

\end{figure}
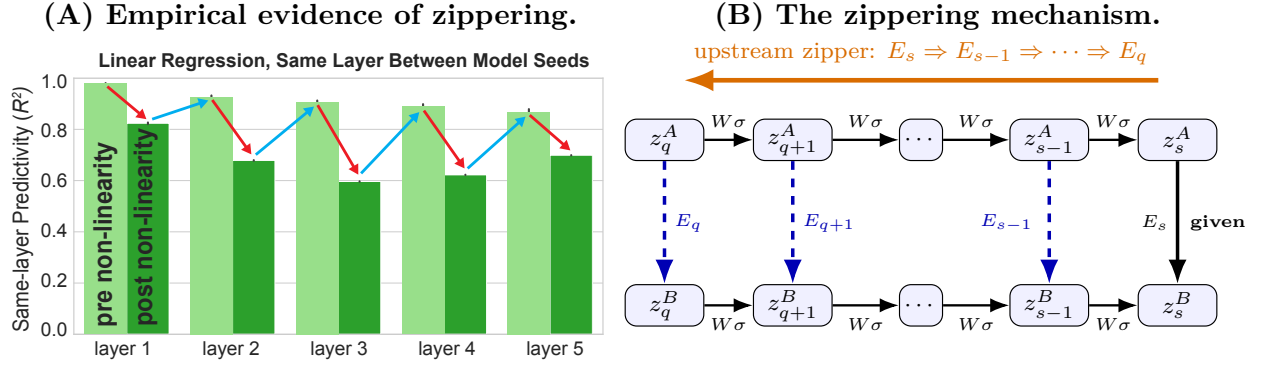

The weak-strong equivalence results assume adjacent layers of weak alignment as a starting point. Weak-strong equivalence thus explains why if you have two networks that are aligned up to linear mapping at all layers (or nearly so), you end up with axis alignment as well. The results so far, however, do not explain why two networks optimized for the same downstream task might be weakly aligned across the hierarchy in the first place. The empirical observations described in the introduction and manifest in Fig.~\ref{fig:motivation} have not been fully explained yet.
A coarse overlapping-task account can explain why some linearly decodable
shared directions might exist: a flexible linear map can recover whatever
task-relevant variables two systems happen to share. But this does not explain
why the shared component should be large, hierarchically ordered,
performance-dependent, or accompanied by privileged-axis structure. The
contravariance claim is stronger: sufficiently hard tasks restrict the space of
acceptable solutions, so that independently learned systems are forced toward similar representations; under the minimality conditions below, this weak similarity can become strong and layerwise-organized similarity.

To understand the situation in greater detail, it is useful to consider a recently-observed phenomenon of DNNs called ``zippering'' (Fig.~\ref{fig:zippering}A). As first described in \citet{thobani2025iatc}, the core zippering finding is that for multiple AlexNet-like networks optimized from different random seeds for ImageNet categorization, the resulting networks are quite close to being linearly aligned at each pre-nonlinearity layer (light green bars in Fig.~\ref{fig:zippering}A). But then at each post-nonlinearity layer, the networks come apart (dark green bars in Fig.~\ref{fig:zippering}A), exhibiting much lower linear similarity. This leads to a characteristic ``zippering'' pattern where the networks repeatedly come back together and diverge in representation.  

Using the weak-strong equivalence theorems, we can now interpret the post-nonlinearity divergence phenomenon as the nonlinear gates being used and minimal for the task (for if the post-nonlinear representations were perfectly linearly aligned, the gates would have had to be irrelevant). Then, the next layer of affine operations in each network seed suppress the seed-specific variability and cause re-convergence.  Of course, since full-rank affine operations cannot change the linear span of a representation, the fact of re-convergence means that the affine operations are not fully well-conditioned, meaning that task-irrelevant dimensions have been pushed into seed-specific dimensions crossed by the used gates at the previous layer, and then are ``squeezed out'' by the re-convergent affine operations at the next layer. Zippering is the evidence of networks putting seed-independent common task-useful information into shared privileged axes at each layer.  But what leads to the repeated re-convergence in the first place?  The weak-strong equivalence mechanism does not answer this.

It turns out that we can prove a much stronger result extending Weak-Strong equivalence upstream. Essentially, we can show that, under the proper conditions, strong equivalence at the terminal end of a hierarchical network causes weak equivalence at the upstream layer, which in turn triggers strong equivalence, and this process repeats iteratively upstream. The intuition behind this zippering mechanism is a generalization of the intuition behind the weak-strong equivalence mechanism: if two networks fail to be linearly equivalent at an upstream stage, and their nonlinearities are necessary for the task, the disagreement will be amplified by the next nonlinear stage, and then by the next, eventually preventing the two networks from solving the same terminal task.  

Proving this stronger result requires much more powerful technical tools (making up the bulk of the appendix section), as well as a few additional important conditions beyond those needed for weak-strong equivalence. First, the usedness definition from above is not quite enough for zippering. At a backward zippering step, the hidden scalar arguments $z^N_{r,j}(x)$ are functions of the original input, not merely affine coordinates in the layer-$r$ representation.  To state the zippering results, we upgrade from individual gate-wise usedness condition to a within-layer \emph{minimality} condition:  

\begin{definition}[Layer-wise minimality, informal.]
A network layer is \emph{minimal} if the functions in its distinct units are identifiably separatable from each other. (This informal definition is made rigorous in Appendix Definitions~\ref{def:sp-minimal} and~\ref{def:relu-eff}.)
\end{definition}

This notion of minimality and the used-axis notion of task hardness are chosen, rather than global measures such as Kolmogorov complexity (which is uncomputable) or representation entropy, because they directly constrain the layerwise solution space relevant to contravariance: minimality removes redundant gates, while hardness forces many of the remaining gates to be used, leaving little room for distinct solutions, consistent with the original intuition of \citet{cao2024contravariance}.

Another important issue for zippering is the concept of \emph{null networks} (see \cite{vlavcic2021affine} and \citet{fefferman1993recovering}) -- that is, networks that have multiple nontrivial layers, each of which is not itself the identity function, but whose total end-to-end effect is to compute the identity.  The existence of null networks would seem to impede the zippering theorem.  For one might imagine inserting nontrivial identity blocks at different depths within a network, such that the end layer was the same for both but which as a pair clearly are not axis-aligned at each layer.  The technical obstruction is a \emph{one-step null pair}: two one-step expansions of the form
\[
 b+\sum_jv_j\varphi(g_j)
 =
 \widetilde b+\sum_i\widetilde v_i\varphi(\widetilde g_i)
\]
that represent the same next-layer function but use different hidden arguments (where $g_i, g_j$ are the concomitant functions of the previous layers).  If such null pairs were common, terminal equivalence would not determine the previous layer. Results in Appendix \S\ref{sec:genericity} show that in minimal networks null pairs of most types in are ruled out entirely, and in the remaining types, are extremely rare: specifically, they lie on a low-dimensional ``exceptional set'' of parameters that is measure zero in the space of all network parameters.\footnote{Appendix \S\ref{sec:null_atlas} gives a catalog of possible counterexample concepts for zippering -- essentially, a catalog of null network situations -- and shows how they are either ruled out by minimality or lie on a low-dimensional space.}
That is, we are able to prove that minimal networks are ``generically layer-wise null-trivial'' (the set of minimal nontrivial null networks are measure-zero in the space of possible parameters at each layer). By Appendix Lemma~\ref{lem:collision-dimension}, Theorem~\ref{thm:generic-null-common}, and Corollaries~\ref{cor:generic-null-specializations}--\ref{cor:generic-exact-zip-null}, any nontrivial failure mode that remains is confined to a lower-dimensional exceptional set rather than an open family of solutions; Appendix Theorems~\ref{thm:common-asymp-zip}--\ref{thm:common-soft-zip} give the corresponding quantitative/soft version.

With these ideas in mind, we can state the zippering theorem:

\begin{theorem}[Exact zippering, informal]
\label{thm:main-exact-zip}
Let $A$ and $B$ be same-depth networks using either Softplus or ReLU. Suppose they are weakly equivalent at a terminal layer $s$. Suppose further that at each backward step from $s-1$ down to $q$, the relevant one-step expansions are minimal. Then, generically (that is, except on an exceptional set of network parameters of zero measure), alignment zippers upstream: for every $r=q,\ldots,s-1$, the layer-$r$ representations are strongly aligned.
\end{theorem}
\noindent The formal statement and proof of the zippering theorem is given in Appendix Theorem~\ref{thm:abstract-exact-zip}.  In essence, its proof is a backward induction: expand the common target at layer $r+1$ in two ways, use layer-wise null-triviality to match the arguments at layer $r$, and continue.

We note that the same proof also applies when the two networks have \emph{different depths or
widths}, as long as one compares corresponding one-layer steps and the counted
axes at each step can be matched; see Remark~\ref{rem:unequal-depths-widths} following Appendix Theorem~\ref{thm:abstract-exact-zip}.
Appendix~\S\ref{sec:depth-warped-block-weak-strong} also proves a stronger depth-warped version:
a layer's preferred axes need not match one fixed layer in the other network, but may be distributed
across a corresponding window of nearby layers.

As with Theorem \ref{thm:main-exact-weak-strong}, the conditions of Theorem \ref{thm:main-exact-zip} are idealized.  In real-world networks, terminal equivalence is only approximate. We thus also prove a soft version of the zippering theorem, which allows there to be some error in the terminal-layer equivalence, and then computes a bound on the upstream propagation of this error.  Establishing this bound requires making one more assumption on the network: that is, the derivative matrices (the Jacobians) of the network operations are \emph{regular}, e.g. that the singular values of its transition matrices are nonzero. This is a well-trodden condition in the literature~\citep{pennington2017resurrecting}, and we prove that it is generic (that is, it only fails on a measure-zero set of exceptional network parameters).  With Jacobian regularity, we are able to establish upper-bound constants $L_\ell$ on how badly error at a downstream layer $\ell+1$ can propagate to error at layer $\ell$, e.g. 
$$\|z_\ell^B - E_\ell z_\ell^A\| \leq L_\ell \|z_{\ell+1}^B - E_{\ell+1} z_{\ell+1}^A\|. $$
\noindent We can then prove:
\begin{theorem}[Soft zippering, informal]
\label{thm:main-soft-zip}
For minimal networks, under the Jacobian regularity assumption (which is generically true), if the terminal layer-$s$ mismatch error $e_s$ is small enough, then
\begin{equation}
\AxisAlign_{r}^{\theta}(A,B)
\ge
1-
\frac{e_s^2\prod_{\ell=r}^{s-1}L_\ell^2}
{(1-\theta)^2 d_r}.
\label{eq:main-soft-zip}
\end{equation}
Appendix Theorems~\ref{thm:common-asymp-zip} and~\ref{thm:common-soft-zip} give the exact statements of this result.  
\end{theorem}
\noindent In words, small terminal error forces a large fraction of upstream axes to align. As terminal weak error goes to zero, the fraction of badly aligned upstream axes goes to zero. Of course, the failure of axis alignment is worse when the original terminal-layer error $e_s$ is larger, or when the demanded amount of qualitative alignment $\theta$ is stricter (close to 1).

While Theorems \ref{thm:main-exact-zip} and \ref{thm:main-soft-zip} are powerful results, to apply to real-world cases, they rely on some important assumptions: first, that the networks (artificial and biological) involved be minimal, and secondly that the dynamics of learning do not disrupt genericity (e.g. that learning does not cause the parameters to concentrate on the ``exceptional set'' where zippering fails). Evaluating the correctness of these assumptions is a key point for future work.

\textbf{Key implications for NeuroAI:} Modulo the biological correctness of minimality and learning dynamics assumptions, the major NeuroAI results of the past decade and a half are not accidental.  Rather, {\bfseries\itshape\uline{convergent evolution between artificial networks and real brains is likely to be, for minimal solutions to sufficiently hard tasks, mathematically inevitable.}}

\section{Weak-Strong Equivalence for RSA and CKA}

Representational Similarity Analysis (RSA) is one of the most common metrics used to compare neural networks to brain data~\citep{kriegeskorte2008}. RSA computes stimulus-by-stimulus correlation matrices (known as Representational Similarity Matrices, or RSMs) for each feature representation, and then compares these matrices, typically using Spearman or Pearson correlation.  RSA is in some sense a stricter metric than linear similarity (see \cite{thobani2025iatc} for a thorough discussion). Given, however, that we have shown that apparently weaker linear similarity implies apparently much stronger privileged axis alignment under suitable task difficulty conditions, what can we say about RSA under the same circumstances?  

The key obstruction to weak-strong equivalence for RSA is the existence of \emph{task-irrelevant feature symmetries}: operations on feature sets that leave the privileged axes fixed but change RSMs.  An example of such a symmetry is \emph{duplication}. Our results on privileged axes imply that the set of privileged axes can be mapped to each other, but do not guarantee in the most general case that the number of copies of these axes present in the representation are the same. This presents no problem for linear equivalence, which can discount multiplicity by reweighting. But RSA is sensitive to multiplicity, so that two representations with the same set of privileged axes, but with different multiplicities of these axes, can have arbitrarily different RDMs.  Thus, high axis alignment (strong alignment), or high linear similarity (weak alignment) do not guarantee of RSA similarity, even under the hard-task high-$\gbudget$ condition.  

What is needed to bring RSA ``into the fold,'' so to speak, is some control over the task-irrelevant symmetries that affect RSA -- for example, a non-redundancy criterion of some kind.  A version of this idea has recently been explored in the recent work of \cite{theiss2026parameter}.  In Appendix \S\ref{sec:rsa-weak-strong}, we show how to synthesize the Theiss results with the weak-strong equivalence framework.  We find that: 
\begin{itemize}
      \item RSMs can be decomposed into a sum of two terms: a unique \emph{core} RSM, associated with a core task-relevant geometry, and a residual ``symmetry-generated'' term that is task-irrelevant.
      \item The core RSM geometry satisfies Weak-Strong Equivalence, while the task-irrelevant portion can completely dominate the raw RSA comparison, leading to the appearance of arbitrarily different RSMs between functionally equivalent networks. 
      \item It is possible remove the task-irrelevant component via several different reasonable ``canonicalization'' procedures, and there is a simple positive relationship between canonicalized RSA to Ridge regression results.
     \item The kind of biased spatial biased sampling often seen in neuroscience experiments mimics the effect of adding task-irrelevant symmetries, and can make the Ridge-RSA relationship appear noisier than it really is.
  \end{itemize}

Another popular metric for model-brain comparison is the linear version of Centered Kernel Analysis (CKA)~\citep{Kornblith2019CKA}. CKA is similar to RSA, except that instead of raw feature comparison, it compares centered Gram geometries. Like raw RSA, raw CKA is sensitive to unequal axis scaling, duplication, redundant features, and biased measurement. Appendix~\S\ref{sec:cka-weak-strong} proves analogous results for linear CKA as obtained for RSA. Appendix~\S\ref{sec:rsa-cka-relationship} further shows that near-perfect linear CKA guarantees a correspondingly good scaled Procrustes alignment, thereby converting high CKA into the weak affine alignment used by our main theorems---consistent with the similar empirical behavior of CKA and Procrustes observed by \citet{bo2024functional}. It also explicates the RSA--CKA relationship in a general context, illustrating divergences that are small on the unique core geometry but exacerbated by task-irrelevant symmetries.

These results contextualize some recent empirical results. \citet{soni2024metricchoice} show that while Ridge, raw RSA, and raw CKA are reasonably correlated across a model zoo ($r$-values of 0.73 for Ridge-RSA, 0.67 for Ridge-CKA, and 0.64 for RSA-CKA), these relationships are not perfect.  Our results show these divergences could be caused either by actual task-irrelevant symmetry-driven feature differences, or by the confounding effects of non-uniform neural sampling in the target brain data. These are, of course, very conceptually different possibilities, and resolving this issue is an interesting question for future work.

Our results also have an implication for which metrics should be correlated with behavior under different model-constraint regimes.  Ridge is weak-strong equivalent, while raw RSA/CKA are not invariant to task-irrelevant symmetries. This suggests that, in the task-constrained regime, Ridge should be \emph{more} correlated with downstream tasks and behavioral patterns than RSA/CKA. 
Verifying this empirically is an interesting question for future work. 
It is instructive to contrast this prediction with the findings of \citet{bo2024functional}, where it is shown that RSA/CKA correlate better with behavioral metrics than ridge. 
In that work, the model zoo contains many networks with random filters and low performance levels, which lie outside the high-task-performance regime needed to trigger contravariance. 
RSA/CKA is thus more sensitive than ridge to feature differences created within a single architecture by filter randomization in the low-performance regime, a sensitivity that results in variability under task-orthogonal feature changes in the performance-constrained regime.

\textbf{Key implications for NeuroAI:} RSA and CKA are sensitive to task-irrelevant factors, but can be corrected to reveal that {\bfseries\itshape\uline{there is a core task-relevant representational geometry, which obeys the same kind of hard-task weak-strong equivalence as privileged axes}}.

\section{Contravariance for Transformers}
Transformers are the dominant microcircuit architecture of AI, due to their robust learning and generalization capacities.  How they might relate to neuroscience is an open question — but regardless of the resolution of that thorny topic, any theory of the Contravariance Principle in DNNs would be incomplete without a treatment of the transformer circuit.  But are there even ``strong objects'' in transformers in the first place? That is to say, are there components of the computation for which weak alignment of the surrounding computations forces that component itself to match? Some recent work \cite{Kapoor2025Bridging} has suggested perhaps not. 
  
In Appendix \S\ref{sec:transformer-weak-strong}, we do a deep dive into contravariance theorems for transformer circuits. 
To understand our findings, it is useful to distinguish three “streams” (sometimes called “branches”) within the transformer circuit (see Fig. \ref{fig:transformer-architecture-ws}):
\begin{enumerate}
\item the \textcolor{red}{attention stream} (\textcolor{red}{red} in Fig. \ref{fig:transformer-architecture-ws}) is the portion of the transformer circuit in which inputs are mixed using a learnable pairwise multiplicative operation,
\item the \textcolor{green}{MLP stream} (\textcolor{green}{green}) is the portion in which inputs are mixed using a short linear-nonlinear cascade,
\item the \textcolor{blue}{residual stream} (\textcolor{blue}{blue}) of the transformer combines the MLP and attention stream additively and passes forward to the next layer.
\end{enumerate}

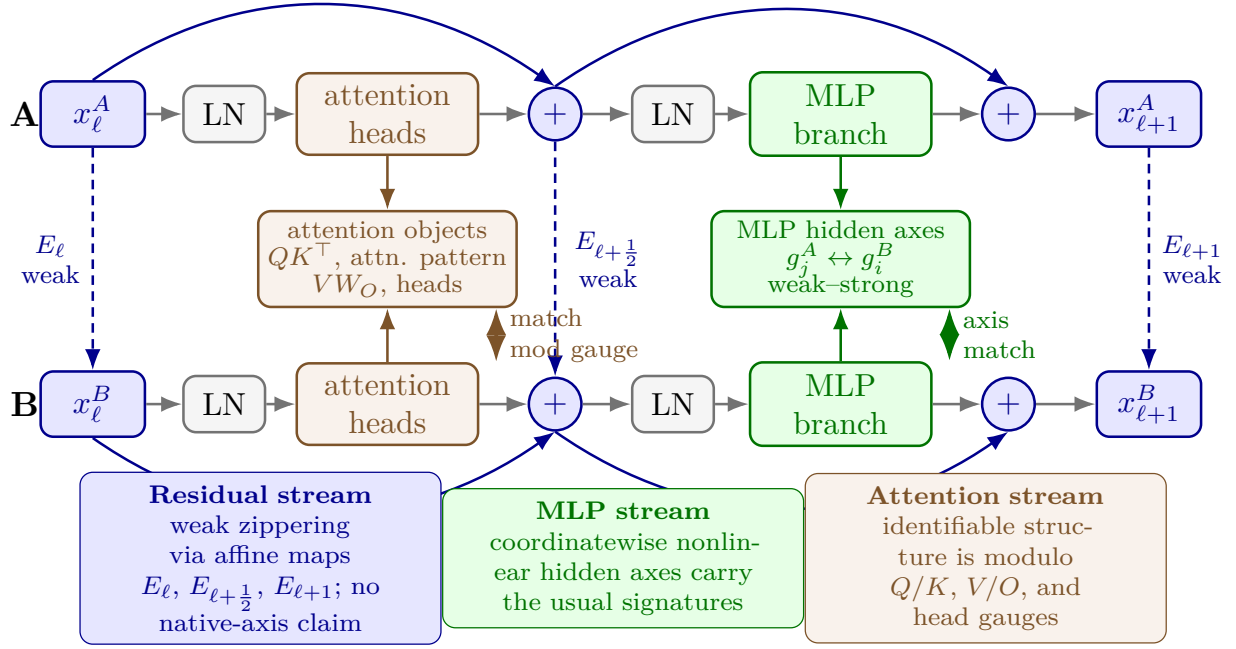
\begin{figure}[t]
\centering
\resizebox{\textwidth}{!}{%
\begin{tikzpicture}[
  x=1cm,y=1cm,
  >=Latex,
  font=\small,
  residual/.style={draw=blue!55!black, fill=blue!10, text=blue!55!black},
  attncol/.style={draw=brown!65!black, fill=brown!10, text=brown!65!black},
  mlpcol/.style={draw=green!45!black, fill=green!10, text=green!45!black},
  neutral/.style={draw=black!55, fill=gray!8, text=black},
  res/.style={
    residual, rounded corners, thick,
    minimum width=1.15cm, minimum height=.72cm,
    align=center
  },
  op/.style={
    neutral, rounded corners, thick,
    minimum width=.90cm, minimum height=.62cm,
    align=center
  },
  plus/.style={
    residual, circle, thick,
    minimum size=.58cm,
    inner sep=0pt,
    align=center
  },
  attnbranch/.style={
    attncol, rounded corners, thick,
    minimum width=2.00cm, minimum height=.86cm,
    align=center
  },
  mlpbranch/.style={
    mlpcol, rounded corners, thick,
    minimum width=2.00cm, minimum height=.86cm,
    align=center
  },
  attnsig/.style={
    attncol, rounded corners, thick,
    minimum width=2.85cm, minimum height=.90cm,
    align=center, font=\scriptsize,
    inner sep=3pt
  },
  mlpsig/.style={
    mlpcol, rounded corners, thick,
    minimum width=2.85cm, minimum height=.90cm,
    align=center, font=\scriptsize,
    inner sep=3pt
  },
  resnote/.style={
    residual, rounded corners,
    text width=3.70cm,
    align=center, font=\scriptsize,
    inner sep=4pt
  },
  mlpnote/.style={
    mlpcol, rounded corners,
    text width=3.70cm,
    align=center, font=\scriptsize,
    inner sep=4pt
  },
  attnnote/.style={
    attncol, rounded corners,
    text width=3.70cm,
    align=center, font=\scriptsize,
    inner sep=4pt
  },
  flow/.style={->, thick, draw=black!55},
  resflow/.style={->, thick, draw=blue!55!black},
  attnflow/.style={->, thick, draw=brown!65!black},
  mlpflow/.style={->, thick, draw=green!45!black},
  weak/.style={->, thick, densely dashed, draw=blue!55!black},
  attnstrong/.style={<->, very thick, draw=brown!65!black},
  mlpstrong/.style={<->, very thick, draw=green!45!black}
]

\def\yA{3.55}
\def\yB{0.35}

\node[font=\bfseries] at (-0.75,\yA) {A};
\node[res]        (Ax0)   at (0.00,\yA) {$x^A_\ell$};
\node[op]         (Aln1)  at (1.45,\yA) {LN};
\node[attnbranch] (Aattn) at (3.25,\yA) {attention\\heads};
\node[plus]       (Aadd1) at (5.10,\yA) {$+$};
\node[op]         (Aln2)  at (6.40,\yA) {LN};
\node[mlpbranch]  (Amlp)  at (8.25,\yA) {MLP\\branch};
\node[plus]       (Aadd2) at (10.10,\yA) {$+$};
\node[res]        (Ax1)   at (11.65,\yA) {$x^A_{\ell+1}$};

\draw[flow] (Ax0) -- (Aln1);
\draw[flow] (Aln1) -- (Aattn);
\draw[flow] (Aattn) -- (Aadd1);
\draw[flow] (Aadd1) -- (Aln2);
\draw[flow] (Aln2) -- (Amlp);
\draw[flow] (Amlp) -- (Aadd2);
\draw[flow] (Aadd2) -- (Ax1);
\draw[resflow] (Ax0.north) to[out=35,in=145] (Aadd1.north);
\draw[resflow] (Aadd1.north) to[out=35,in=145] (Aadd2.north);

\node[font=\bfseries] at (-0.75,\yB) {B};
\node[res]        (Bx0)   at (0.00,\yB) {$x^B_\ell$};
\node[op]         (Bln1)  at (1.45,\yB) {LN};
\node[attnbranch] (Battn) at (3.25,\yB) {attention\\heads};
\node[plus]       (Badd1) at (5.10,\yB) {$+$};
\node[op]         (Bln2)  at (6.40,\yB) {LN};
\node[mlpbranch]  (Bmlp)  at (8.25,\yB) {MLP\\branch};
\node[plus]       (Badd2) at (10.10,\yB) {$+$};
\node[res]        (Bx1)   at (11.65,\yB) {$x^B_{\ell+1}$};

\draw[flow] (Bx0) -- (Bln1);
\draw[flow] (Bln1) -- (Battn);
\draw[flow] (Battn) -- (Badd1);
\draw[flow] (Badd1) -- (Bln2);
\draw[flow] (Bln2) -- (Bmlp);
\draw[flow] (Bmlp) -- (Badd2);
\draw[flow] (Badd2) -- (Bx1);
\draw[resflow] (Bx0.south) to[out=-35,in=-145] (Badd1.south);
\draw[resflow] (Badd1.south) to[out=-35,in=-145] (Badd2.south);

\draw[weak] (Ax0.south) -- node[left,font=\scriptsize,align=center,text=blue!55!black] {$E_\ell$\\weak} (Bx0.north);
\draw[weak] (Aadd1.south) -- node[right,font=\scriptsize,align=center,xshift=2pt,text=blue!55!black] {$E_{\ell+\frac12}$\\weak} (Badd1.north);
\draw[weak] (Ax1.south) -- node[right,font=\scriptsize,align=center,text=blue!55!black] {$E_{\ell+1}$\\weak} (Bx1.north);

\node[attnsig] (attnsig) at (3.25,1.95)
  {attention objects\\[-1pt]
   $QK^\top$, attn. pattern\\[-1pt]
   $VW_O$, heads};

\node[mlpsig] (mlpsig) at (8.25,1.95)
  {MLP hidden axes\\[-1pt]
   $g^A_j \leftrightarrow g^B_i$\\[-1pt]
   weak--strong};

\draw[attnflow] (Aattn.south) -- (attnsig.north);
\draw[attnflow] (Battn.north) -- (attnsig.south);

\draw[mlpflow] (Amlp.south) -- (mlpsig.north);
\draw[mlpflow] (Bmlp.north) -- (mlpsig.south);

\draw[attnstrong]
  ([xshift=1.20cm]attnsig.south)
  --
  node[right,font=\scriptsize,align=left,text=brown!65!black] {match\\mod gauge}
  ([xshift=1.20cm]Battn.north);

\draw[mlpstrong]
  ([xshift=1.20cm]mlpsig.south)
  --
  node[right,font=\scriptsize,align=left,text=green!45!black] {axis\\match}
  ([xshift=1.20cm]Bmlp.north);

\node[resnote] at (1.85,-1.35)
  {\textbf{Residual stream}\\
   weak zippering via affine maps\\
   $E_\ell$, $E_{\ell+\frac12}$, $E_{\ell+1}$; no native-axis claim};

\node[mlpnote] at (5.85,-1.35)
  {\textbf{MLP stream}\\
   coordinatewise nonlinear hidden axes carry\\
   the usual signatures};

\node[attnnote] at (9.85,-1.35)
  {\textbf{Attention stream}\\
   identifiable structure is modulo\\
   $Q/K$, $V/O$, and head gauges};

\end{tikzpicture}
}
\caption{Transformer weak--strong structure in the updated residual-peeling formulation.  Residual states zipper weakly across depth via affine maps, but this does not privilege the native residual coordinates.  Inside each block, the MLP branch has ordinary coordinatewise-nonlinear hidden axes, while attention has task-visible heads and head signatures modulo its exact internal gauges.  Peeling the MLP and attention sublayers recovers the intermediate weak maps.}
\label{fig:transformer-architecture-ws}
\end{figure}

We find that weak-strong equivalence obtains for the MLP and attention branches within a transformer, but \emph{not} for the residual stream.  Weak alignment zippering does, however, obtain for the residual stream, and thus cascades through the whole network in a hierarchical transformer. Hence, there are privileged representational axes throughout transformer hierarchies, but they must be identified outside the residual stream. 
These results reframe the recent negative finding for privileged-axis results in ViTs described by \citet{Kapoor2025Bridging}: permutation or soft-matching tests on residual or CLS-token coordinates will miss the privileged structure when they look in the residual stream, but will find them in other circuit components.  

\textbf{Key implication for NeuroAI:} The different components of the transformer architecture behave somewhat differently to each other, but {\bfseries\itshape\uline{privileged axes and the core contravariance mechanisms apply to transformers in nearly the same way they do in simpler affine nonlinear networks}}. 

\section{Empirical Results from the ML Literature}

The theoretical results of this paper have many relationships for the extensive empirical work done on neural network feature representations in the machine learning community -- without reference, in specific, to neural data. Here we review these connections briefly, in vision and language models.  

\textbf{Vision.} Empirical work in vision broadly supports rough weak and strong alignment across networks trained on similar hard visual tasks, with an ordered hierarchy in which early layers tend to match early layers, middle layers match middle layers, and later layers match later layers\footnote{We note that apparent exceptions among top Brain-Score models~\citep{schrimpf2020integrative}---such as later layers being assigned to V2 than V4---are weak evidence against hierarchical correspondence, because the V1/V2 layer choices are selected on the small, specialized Freeman-Ziemba texture benchmark~\citep{freeman2013functional} and are highly sensitive to metric directionality, unit sampling, and averaging across heterogeneous model sets. This is an empirical problem of reliably identifying corresponding stages, orthogonal to our conditional theory once those stages are fixed, and motivates the emphasis in the NeuroAI Turing Test \citep{feather2025neuroai}, as well as our results herein, on target brain datasets with enough \emph{diverse} stimuli to support reliable noise ceilings.}.  Early work on convergent learning found that independently trained networks can learn individually similar features as well as shared low-dimensional subspaces with non-identical bases \citep{Li2015ConvergentLearning}. SVCCA and CKA made this more systematic: SVCCA showed bottom-up convergence during training and provided an affine-invariant way to compare layers \citep{Raghu2017SVCCA}, while CKA reliably identifies layer correspondences across initializations and different architectures \citep{Kornblith2019CKA}. Comparisons between CNNs and vision transformers complicate the picture but preserve the broad hierarchical story: ViTs and CNNs differ in how they propagate spatial information, yet still show structured layerwise relationships \citep{Raghu2021DoVisionTransformers}. Recent large-scale investigations using linear regression, Procrustes, permutation, and soft-matching metrics find that hierarchical correspondence persists even under stricter alignment notions, and that permutation/soft-matching scores are meaningfully above chance, consistent with both our asymptotic weak-strong equivalence and soft zippering results~\citep{KhoslaWilliams2023SoftMatching, Kapoor2025Bridging}.
In a complementary RNN setting, \citet{huang2025measuring} find that harder RNN tasks make the learned dynamics more similar across seeds even as the weights become more different. This is what our theory predicts: task difficulty constrains the task-relevant computation, but many different weight configurations can still implement that computation through permutations, rescalings, and redundant or unused units.

\textbf{LLMs.} The language-model literature gives a similar but somewhat messier picture: representations across different models often align by relative depth, but architecture family, pretraining objective, and fine-tuning task matter substantially. For contextual word representation models, Wu et al. found that different architectures can have similar full representations even when individual neurons differ, that same-family models are more similar than cross-family models, and that lower layers are generally more similar across architectures than higher layers \citep{Wu2020SimilarityContextual}. BERTology surveys likewise emphasize that transformer layers encode different kinds of information across depth rather than being interchangeable \citep{Rogers2020BERTology}. Fine-tuning studies show that adaptation is often conservative and primarily affects top layers, with variation across tasks \citep{Merchant2020BERTFineTuning}, while CKA analyses of fine-tuned RoBERTa and ALBERT find block-diagonal structure, with strong similarity within early-layer and late-layer clusters \citep{Phang2021FineTunedClusters}. More recent cross-architecture LLM work, with better architectures solving harder versions of the prediction tasks, is especially supportive of the hierarchical weak-alignment picture: across several dozen open-weight LLMs, layers at corresponding relative depths tend to induce similar activation geometries, with the whole progression ``stretched and squeezed'' across architectures with different numbers of layers \citep{WolframSchein2025SimilarDepths}. \textbf{Thus, for LLMs, weak alignment zippering appears common in the empirical literature.}

\section{Discussion}

The results of this paper suggest several interesting additional considerations. 
 
\textbf{Extension to other architectures.}  The main weak--strong and zippering theorems are stated for transparent affine-ReLU or affine-softplus blocks, because that is where the mechanism is easiest to illustrate. We also address (\S \ref{sec:transformer-weak-strong}) residual networks and transformers with softplus and ReLU MLP layers. However, the core ideas already transfer, or are probably likely to do so, to a wide variety of architectural motifs, including e.g. CNNs, RNNs, and other nonlinearity types, including eLU and GeLU nonlinearities.  See Appendix \ref{sec:extensions} for a detailed discussion of this issue. 

\textbf{Extension to other metrics.}
Appendices~\S\ref{sec:rsa-weak-strong}--\S\ref{sec:cka-weak-strong} extend the weak--strong framework to RSA and linear CKA, and relate high linear CKA to scaled Procrustes and ridge alignment. Future work should develop corresponding results for soft- and semi-matching, nonlinear-kernel comparison measures, and other metrics used in model--brain comparisons. An equally important empirical direction is to test whether canonicalization and controlled neural sampling reduce the metric disagreements observed in current benchmarks, as our theory indicates.

\textbf{Relationship to the Platonic Representation Hypothesis.} The Platonic Representation Hypothesis (PRH) of~\cite{huh2024platonic} is an intriguing recent conceptual proposal in the NeuroAI-adjacent literature.  The PRH is the idea that networks learned from quite different tasks in the same underlying ``world'' have surprisingly convergent internal representations.  For instance, it is shown empirically that networks trained on visual tasks and networks trained on text next-token prediction have nontrivial representational alignment that is not obviously explained by data, architecture, or loss objective similarity.  It is believed that the reason for this convergence is that both vision and text are just two different modes referring to the same underlying ``platonic'' reality; and that the structure of that reality is reflected in common network representations. The PRH is related to, but different from contravariance.  Contravariance says (essentially) that two different solutions to the same task have very highly aligned internal representations, whereas the PRH illustrates that there is a surprisingly high lower bound on the alignment between solutions to two \emph{different} task distributions generated on top of a common underlying generative model. Unsurprisingly, the PRH lower bound is much lower than the contravariance lower bound.  The PRH remains mostly empirically justified for now; but the success of formal methods for contravariance illustrated in this work suggest it may soon be possible to attack the PRH with a similar level of rigor. 

\textbf{What is the likely status of minimality in modern neural networks?} 
The used-axis budget \(m_\ell(\epsilon)\) counts layer-\(\ell\) nonlinear axes forced by the task, so it should be compared to the layer width \(d_\ell\), not to the total number of trainable parameters. Ordinary overparameterization therefore by no means implies that \(m_\ell\ll d_\ell\): a model can have many redundant parameterizations while still using most axes in a given representation layer. A useful empirical prior is that \(m_\ell/d_\ell\) is often not tiny, and may be moderate or high, in real task-trained networks. For example, ResNet’s channel and fully connected widths are modest relative to the breadth of ImageNet categorization, ViT-B-scale models use widths such as \(768\)-dimensional embeddings and \(3072\)-dimensional MLP layers for dense reusable visual features, and LLMs face a broad next-token prediction task that may require many residual-stream or MLP directions. Continued scaling gains in large language models point in the same direction. Scaling-law and Chinchilla-style analyses show that, over broad regimes, increasing model capacity together with data and compute continues to reduce held-out loss.  This suggests that additional capacity is still being converted into task-useful computation, rather than merely adding irrelevant axes everywhere.  Thus, the $m_\ell/d_\ell$ ratio is likely better understood as a proxy for test performance rather than model parameter count.  Of course, truly estimating these quantities is an endeavor that will require empirical probes such as layerwise bottlenecking, pruning or ablation curves, or width-restricted retraining.

\textbf{Future Work.} Internal to the logic of this work, a number of future directions beckon. Two of the more interesting are:

\begin{enumerate}
\item \textit{Extensive empirical comparisons.} Several constants appear in our results that are a function of architecture class only. We do not explicitly calculate them here, but it should be possible to do so, at least approximately, for standard architecture types. The results of these calculations could then be plugged into the soft alignment and zippering bounds (eqs. \ref{eq:main-asymptotic-bound} and \ref{eq:main-soft-zip}). Once such calculations have been done, it will then be natural to compare these theoretical estimates to empirical results from mapping DNNs to each other (as in \cite{thobani2025iatc}) or to brain data (as in \cite{khosla2024privileged}).

\item \textit{Understanding how learning dynamics affect genericity.}
The genericity results of this work say that the ``exceptional set'' of conditions where zippering fails occurs with probability 0, assuming that the distribution of optimized network parameters is uniform in the space of all parameters. But the distribution of parameters actually created by network optimization is far from uniform.
Do learning dynamics happen for some reason to be biased toward the exceptional set -- places where null-networks live or where the Jacobian is ill-conditioned? We see no immediate  reason to think this happens, but understanding the interaction between contravariance and learning dynamics more generally and deeply is an important domain of future work.
\end{enumerate}

\vspace{.05in}
\begin{definitionguide}{Overall Implications for NeuroAI Research.}
Our results tentatively suggest that, more broadly, the subfield of ``NeuroAI core methods'' could productively focus some effort on the following two broad issues:
\begin{itemize}
\item Measuring the {\bfseries\itshape\uline{hardness}} of real-world tasks. It would be of great interest to have clear, standardized, scales of task constraint difficulty that could be validated and estimated across multiple domains and tasks -- perhaps the gate budget concept we introduce is a place to start. 
\item Evaluating the empirical status of {\bfseries\itshape\uline{minimality}}, or departures therefrom, both in ``real-world'' artificial neural networks, and \emph{also in actual brains}. Estimating the extent to which such systems satisfy the kind of nondegeneracy conditions identified in this work is a crucial ingredient for creating a solid theoretical foundation for NeuroAI.  We speculate that minimality in real brains is an energetic efficiency that is evolutionarily advantageous. 
\end{itemize}

And, perhaps somewhat more controversially: 
\begin{itemize}
    \item We suggest that the field move away from the question of finding ``the right metric of brain-model comparison''; our results suggest this issue is a bit moot.
    \item \textbf{Take the experimental-design consequences of contravariance seriously!} By giving the ideas of \cite{cao2024contravariance} clear mathematical foundations, the present results justify the contravariance principle explanation for apparently convergent evolution between DNN models and real brains.  But the contravariance principle had a larger remit than just explaining existing results in NeuroAI: it gave a normative prescription for designing effective neuroscience experiments.  Specifically, contravariance suggests that instead of designing highly reduced experiments with simple versions of tasks that are easy to solve with hand-designed network architectures, instead {\bfseries\itshape\uline{neuroscientists should work with fairly complex, computationally hard, versions of tasks, since those are the ones that are likely to avoid spurious too-simple solutions and force convergent evolution between models and the real brain}}. Picking tasks involving natural stimuli that are within the ethological range of the organism, thereby requiring little training for the organism, is one heuristic for picking computationally difficult tasks.   The present results strongly reinforce that prescription.
\end{itemize}
\end{definitionguide}

\section*{Acknowledgements}
The authors have for some years known the statement and proof outline for the exact weak--strong theorem. They have also suspected that a zippering theorem had to be correct, both from intuition and empirical results.  However, it was only when recently exploring these ideas with ChatGPT-5.5-Pro that the key steps for the  zippering theorems became clear. It helped us quickly explore the situation for softplus networks, which we knew had similar empirical zippering behavior to ReLU, allowing us to build intuition to get to the full ReLU soft zippering theorem. It was here that ChatGPT truly shined: it identified the equivalent concept of minimality for softplus networks, as embodied in the Softplus minimality definition, an idea that might have eluded the (human) authors of this paper for some time. It also suggested the proof of the Softplus one-step rigidity lemma.  ChatGPT was again similarly helpful in formulating the idea that Jacobian regularity would be the right condition to support quantitative identifiability (the local quantitative one-step identifiability lemma). These ideas seem (to us humans) to be remarkably creative and incisive.  All final proofs were verified by the authors, who take full responsibility for the results herein. We thank Reece Keller, Ishan Kalburge, Marvin Theiss, Rosa Cao, Judy Fan, Meenakshi Khosla, Aude Maier, Leo Kozachkov, Dan Kunin, and Tony Zador for helpful feedback on an initial draft of the manuscript. A.N. thanks: the Burroughs Wellcome Fund (CASI award), Foresight Institute, and Protocol Labs for funding. D.L.K.Y. thanks: Simons Foundation grant 543061, National Science Foundation CAREER grant 1844724, Office of Naval Research grant S5122, ONR MURI00010802, ONR MURI S5847, and ONR MURI 1141386- 493027. 

\begin{center}
***

{\fontencoding{T1}\fontfamily{calligra}\selectfont 
S.D.G.}

***
\end{center}

\bibliographystyle{unsrtnat}
\bibliography{refs_integrated}

\appendix
\clearpage

\appendix
\section{Appendix}
\label{sec:appendix}

\noindent This appendix contains the formal definitions, statements, and complete proofs of the exact and asymptotic weak--strong equivalences (\S\S\ref{sec:common-setup}-\ref{sec:asymptotic-ws}), exact zippering theorem (\S\ref{sec:exact-zippering}), the asymptotic and soft zippering theorems (\S\ref{sec:soft-zippering}) and the genericity analysis (\S\ref{sec:genericity}).  We then show how the weak-strong equivalence ideas extend to networks of differnt depth through depth warping (\S\ref{sec:depth-warped-block-weak-strong}), and how weak-strong equivalence and zippering apply in transformer networks (\S\ref{sec:transformer-weak-strong}). Though not directly required for the main proofs, we also give conceptually useful analyses of how regularity assumptions lead network depth-complexity to be well-defined (\S\ref{sec:depth_faith}), an atlas of null-pair solutions and how the assumptions avoid them (\S\ref{sec:null_atlas}), and how our results may generalize to other architectures and activations (\S\ref{sec:extensions}).  We extend our results on weak-strong equivalence to the RSA metric in \S\ref{sec:rsa-weak-strong}.

\etocsettocstyle{\subsection*{Appendix table of contents}}{}
\localtableofcontents

\setcounter{theorem}{0}
\setcounter{lemma}{0}
\setcounter{proposition}{0}
\setcounter{corollary}{0}
\setcounter{definition}{0}
\setcounter{assumption}{0}
\setcounter{remark}{0}

\subsection{Common setup and alignment notions}
\label{sec:common-setup}

In this work primarily consider ReLU and Softplus activation functions
\[
 \varphi\in\{\sig,\rhoR\},
 \qquad
 \sig(t)=\log(1+e^t),
 \qquad
 \rhoR(t)=\max\{t,0\}.
\]
We define affine-nonlinear networks $N\in\{A,B\}$ via the usual formulas alternating and affine step (the hidden layer) and a pointwise nonlinear output step:
\[
 z_\ell^N(x)=W_\ell^Nh_{\ell-1}^N(x)+b_\ell^N,
 \qquad
 h_\ell^N(x)=\varphi(z_\ell^N(x)),
\]
with $\varphi$ applied coordinatewise.  The full-layer operation is denoted by
\[
 \Phi_{\ell,\varphi}^N(z)=W_{\ell+1}^N\varphi(z)+b_{\ell+1}^N.
\]
The input set (the ``data'') is denoted by $M_\ell^N=z_\ell^N(\Omega)$.  It will often by convenient to work on a \emph{task patch} $C$, that is,  =is any subset of $C \subseteqq M_\ell^A$ with nonempty relative interior in its affine hull.

Capturing the intuitive definitions of weak alignment as affine mappability, we define:
\begin{definition}[Weak alignment]
A layer-$\ell$ comparison map is an injective affine map
\[
 E_\ell(z)=T_\ell z+a_\ell,
 \qquad
 T_\ell:\R^{d_\ell^A}\to\R^{d_\ell^B}.
\]
The networks are \emph{weakly equivalent} at layer $\ell$ on task patch $C$ if
\[
 z_\ell^B(x)=E_\ell z_\ell^A(x) \quad \text{for all } x \ in C.
\]
\end{definition}

The follow basic commutation lemma on weak alignment is useful below. 
\begin{lemma}[Adjacent commutation]
\label{lem:adjacent-commutation-common}
If exact weak alignment holds at layers $\ell$ and $\ell+1$, then Fig. \ref{fig:alignment_defs}A commutes, e.g.
\[
 \Phi_{\ell,\varphi}^B(E_\ell z)=E_{\ell+1}\Phi_{\ell,\varphi}^A(z)
 \qquad\forall z\in C.
\]
\end{lemma}

\begin{proof}
For any $z\in C\subseteq M^A_\ell$, choose $x\in\Omega$ such that $z=z^A_\ell(x)$. Then
$$\begin{aligned}
 \Phi^B_\ell(E_\ell z)
 &=\Phi^B_\ell(E_\ell z^A_\ell(x)) \\
 &=\Phi^B_\ell(z^B_\ell(x)) \\
 &=z^B_{\ell+1}(x) \\
 &=E_{\ell+1}z^A_{\ell+1}(x) \\
 &=E_{\ell+1}\Phi^A_\ell(z^A_\ell(x)) \\
 &=E_{\ell+1}\Phi^A_\ell(z).
\end{aligned}$$
\end{proof}

It will sometimes be necessary to account for scalar symmetries when comparing axes.  The allowed scalar symmetry depends on the activation: Softplus has no nontrivial real scaling symmetry, while ReLU is homogeneous and therefore permits nonzero real rescalings in adjacent weak--strong comparisons.  This motivates the following scalar symmetry sets:
\[ \mathcal H_{\sig}=\{1\},
 \qquad
 \mathcal H_{\rhoR}=\R\setminus\{0\}.
\]

We can now define strong alignment precisely: 

\begin{definition}[Strong alignment]
Given $\phi$, $E_\ell$, an $A$-axis $j$ and pulled-back $B$-axis $i$ are \emph{strongly aligned} on $C$ if
\[
 (E_\ell z)_i=\alpha z_j
 \qquad\forall z\in C
\]
for some $\alpha\in\mathcal H_\varphi$.  Let $J_\ell^{N,\varphi}(C)$denote the set of axes of network $N$ on task patch $C$, and define
\[
 \AxisAlign_{\ell,\varphi}(A,B;E_\ell,C)
 =\frac1{d_\ell^A}\max_\pi
 \#\left\{j\in J_\ell^{A,\varphi}(C):
 (E_\ell z)_{\pi(j)}=\alpha_jz_j,
 \ \alpha_j\in\mathcal H_\varphi\right\},
\]
where $\pi$ ranges over injective matchings into the counted $B$-axes.
\end{definition}

It is useful to introduce the concept of a one-layer softplus expansion: the input-output formula for any given softplus layer $r$, viewed as a function of a set of scalar-valued functions $g_j$ realizing the network operations in layers below $r$. This notation
separates the lower part of the network, which produces scalar hidden axes, from the next layer, which applies softplus to those axes and linearly recombines them:

\begin{definition}[One-step expansion]
\label{def:one-step-expansion}
A one-step expansion is
\[
 G=\bigl(b,\{(v_j,g_j)\}_{j=1}^k\bigr),
 \qquad
 Y_G^\varphi(x)=b+\sum_{j=1}^kv_j\varphi(g_j(x)).
\]
In an actual network layer, this notation is simply the normal layer update rewritten with the ''dictionary''
\[
 g_j=z_{r,j}^N,
 \qquad
 v_j=W_{r+1}^N[:,j],
 \qquad
 b=b_{r+1}^N.
\]
If $E(y)=Ty+a$ is affine, define
\[
 EG=\bigl(Tb+a,\{(Tv_j,g_j)\}_{j=1}^k\bigr),
 \qquad
 Y_{EG}^\varphi=E(Y_G^\varphi).
\]
When $T$ is injective, many properties introduced below (such as usedness and affine sign-noncancellation) are preserved.
\end{definition}

\subsection{Exact Weak--strong equivalence}
\label{sec:exact-ws}

The common strategy of the exact weak-strong equivalence is to define ``usedness'' for a nonlinear element of the network, and then prove that an affine transformation cannot undo a used nonlinearity.
But, the concrete definition of usedness is a little bit different depending on the activation function.  For ReLU, we observe ``kinks'' where the nonlinearity acts; for softplus, we use conditions on the non-redundancy of the expansion coefficients. 

\subsubsection{ReLU: usedness as visible kink traces}

The key object for ReLU usedness is the trace associated with an axis (sometimes called a ``gate''):
\begin{definition}[Used ReLU gates.]
\label{def:relu-used}
For a ReLU axis $j$, define its trace on a task patch $C$ by
\[
 \Gamma_j^N(C)=\{z\in C:z_j=0\}.
\]
Given a comparison map $E$, the pulled-back $B$-trace is
\[
 \Gamma_i^B(E;C)=\{z\in C:(E z)_i=0\}.
\]
Let $V$ be the affine hull of $C$.  A trace $\Gamma(C)=C\cap\{a(z)=0\}$ is \emph{regular at point $p\in\Gamma(C)$} if $a|_V$ is not the zero affine functional, so near $p$ inside $V$ the trace is a codimension-one affine slice.  The patch $C$ \emph{crosses} this trace at $p$ if, for every $\eta>0$, there are $z^+,z^-\in C\cap B_\eta(p)$ with $a(z^+)>0$ and $a(z^-)<0$.  A set of traces is \emph{regular and separated on $C$} if every trace in the set contains a relatively open patch of regular crossed points $p$, and no two distinct counted traces agree on a relatively open subset of such a patch (in other terms, the trace has a ``private'' regular crossed piece).  A ReLU gate $j$ is \emph{used on $C$} if (i) its trace $\Gamma_j^N(C)$ contains a regular crossed patch, (ii) and the outgoing column is nonzero: $W_{\ell+1}^N[:,j]\neq0$.
\end{definition}

The word ``used'' is deliberately broader than the technical downstream condition $W_{\ell+1}[:,j]\neq0$.  It says both that the task crosses the nonlinear boundary and that the next affine map reads out the resulting kink.  This is the notion that belongs in the nonlinear gate budget.

The next lemma explains why used gates are the right budgeted object for measuring task difficulty: a gate that is not used may still be physically present in the network, but it is not needed as a nonlinear boundary on the task patch.

\begin{lemma}[Shrinkability of non-used ReLU gates]
\label{lem:relu-shrinkability}
For the adjacent block
\[
 \Phi_{\ell,\rhoR}(z)=W_{\ell+1}\rhoR(z)+b_{\ell+1}
\]
on a task patch $C$, a gate that is not used contributes no indispensable nonlinear boundary on $C$.  More precisely: if $W_{\ell+1}[:,j]=0$, the gate can be deleted; if $C$ does not cross $\{z_j=0\}$, then $\rhoR(z_j)$ is either $0$ or affine on each component of $C$ and its contribution can be absorbed into the affine part of the block; if two counted affine traces coincide on a regular patch, their ReLU contributions can be merged, up to an affine residual coming from $\rhoR(t)-\rhoR(-t)=t$.
\end{lemma}

\begin{proof}
The first case is immediate.  In the second case, if $z_j<0$ on the patch then $\rhoR(z_j)=0$, while if $z_j>0$ then $\rhoR(z_j)=z_j$ and the contribution $W_{\ell+1}[:,j]z_j$ is affine in the layer coordinate.  If two affine traces coincide on $V=\operatorname{aff}(C)$, the defining affine functions are proportional: $h=\alpha g$ or $h=-\alpha g$ with $\alpha>0$.  In the positive case, $v\rhoR(g)+\widetilde v\rhoR(\alpha g)=(v+\alpha\widetilde v)\rhoR(g)$.  In the negative case, use $\rhoR(g)-\rhoR(-g)=g$ to rewrite the pair as one ReLU term plus an affine function.  Hence at most one nonlinear boundary remains.
\end{proof}
\noindent Accordingly, the gate budget below counts task-visible nonlinear slope changes, not physically present but redundant coordinates.

\begin{definition}[Required used ReLU-gate budget]
\label{def:relu-used-budget}
Let $\mathcal F_{\ell,\le k}^{\rhoR}$ be the fixed-macroarchitecture networks solving the task with at most $k$ used layer-$\ell$ ReLU gates.  Define
\[
 m_\ell^{\rhoR}(\eps)=
 \min\{k:\exists f\in\mathcal F_{\ell,\le k}^{\rhoR},\ L(f)\le\eps\}.
\]
\end{definition}
\noindent We are now in a position to state and prove the exact ReLU version of weak-strong alignment equivalence. 
\begin{theorem}[Exact ReLU weak--strong equivalence]
\label{thm:relu-adjacent}\label{thm:relu-exact-adjacent}
Fix a task patch $C\subseteq z_\ell^A(\Omega)$ and let $\Omega_C=\{x:z_\ell^A(x)\in C\}$.  Suppose exact weak equivalence holds at layers $\ell$ and $\ell+1$ on $\Omega_C$.  Assume the counted $A$-side and pulled-back $B$-side ReLU trace families are used on $C$ in the sense of Definition~\ref{def:relu-used}.  Then there is an injective map
\[
 \pi:J_\ell^{A,\rhoR}(C)\to J_\ell^{B,\rhoR}(E_\ell;C)
\]
with
\[
 (E_\ell z)_{\pi(j)}=\alpha_jz_j,
 \qquad \alpha_j\neq0,
 \qquad z\in C.
\]
If $A$ reaches loss at most $\eps$, then
\[
 \AxisAlign_{\ell,\rhoR}(A,B;E_\ell,C)
 \ge\frac{m_\ell^{\rhoR}(\eps)}{d_\ell^A}.
\]
\end{theorem}

\begin{proof}
Write $V=\operatorname{aff}(C)$.  On each relatively open region of $C$ on which the signs of all coordinates are fixed, $\rhoR$ is affine and therefore both $\Phi_{\ell,\rhoR}^A$ and $\Phi_{\ell,\rhoR}^B\circ E_\ell$ are affine.  Fix $j\in J_\ell^{A,\rhoR}(C)$ and choose a private regular point
\[
 p\in\Gamma_j^A(C)
\]
which lies on no other counted trace.  In a neighborhood of $p$, every term except the $j$th is affine through the trace.  If $u\in V-V$ is a crossing direction with $u_j\neq0$, then the one-sided directional derivatives of the $j$th contribution differ by
\[
 W_{\ell+1}^A[:,j]u_j.
\]
Equivalently, the full derivative matrices on the two sides differ by
\[
 W_{\ell+1}^A[:,j]e_j^\top\big|_{V-V}.
\]
This jump is nonzero because the task patch crosses the trace and $W_{\ell+1}^A[:,j]\neq0$.  Hence a relatively open part of every used trace belongs to the nondifferentiability set of the restricted block.  Conversely, away from the union of the coordinate traces the sign pattern is locally constant, so the block is affine.  Thus, up to lower-dimensional intersections of traces,
\[
 \Sing(\Phi_{\ell,\rhoR}^A|_C)
 =\bigcup_{j\in J_\ell^{A,\rhoR}(C)}\Gamma_j^A(C).
\]
The analogous statement holds for the pulled-back $B$-block.  Its used trace family is
\[
 \Gamma_i^B(E_\ell;C)=\{z\in C:(E_\ell z)_i=0\},
\]
and, again up to lower-dimensional trace intersections,
\[
 \Sing((\Phi_{\ell,\rhoR}^B\circ E_\ell)|_C)
 =\bigcup_{i\in J_\ell^{B,\rhoR}(E_\ell;C)}\Gamma_i^B(E_\ell;C).
\]
\noindent By Lemma~\ref{lem:adjacent-commutation-common},
\[
 \Phi_{\ell,\rhoR}^B\circ E_\ell
 =E_{\ell+1}\circ\Phi_{\ell,\rhoR}^A
 \qquad\text{on }C.
\]
Writing $E_{\ell+1}(y)=T_{\ell+1}y+a_{\ell+1}$, a derivative jump $J$ of the $A$-block is sent to $T_{\ell+1}J$.  Since $T_{\ell+1}$ is injective, $J\neq0$ implies $T_{\ell+1}J\neq0$.  Thus postcomposition cannot erase an $A$-side kink.  The equality above also prevents the left side from having a kink where the right side is differentiable.  Therefore the two displayed singular sets coincide on the task patch.

Now fix a connected private open patch $U\subset\Gamma_j^A(C)$.  It is contained in the finite union of the pulled-back $B$-traces.  If no one of those traces contained a relatively open subset of $U$, then each intersection with $U$ would have codimension at least one inside $U$ and finitely many such intersections could not cover $U$.  Hence some $B$-trace agrees with $\Gamma_j^A(C)$ on a relatively open patch.  The trace-separation part of usedness makes the match unique, and two distinct $A$-traces cannot be assigned to the same separated $B$-trace.  We obtain an injective matching $i=\pi(j)$.

For such a matched pair, the nonzero affine functions
\[
 z\longmapsto z_j,
 \qquad
 z\longmapsto(E_\ell z)_i
\]
have the same regular codimension-one zero set on an open subset of $V$.  Two nonzero affine functionals on $V$ with the same zero hyperplane are proportional, so there is $\alpha_j\neq0$ with
\[
 (E_\ell z)_{\pi(j)}=\alpha_jz_j
 \qquad\forall z\in C.
\]
This proves the axis-matching statement.

Finally, a task solution of loss at most $\eps$ must contain at least $m_\ell^{\rhoR}(\eps)$ used layer-$\ell$ gates.  Every such gate is matched by the construction above, so at least that many $A$-axes contribute to the numerator of $\AxisAlign_{\ell,\rhoR}$.  Dividing by $d_\ell^A$ yields the stated lower bound.
\end{proof}

\subsubsection{Softplus: usedness via affine ridge curvature}

\begin{figure}[H]
\centering
\begin{tikzpicture}[x=1cm,y=1cm]
\tikzset{
  panelbox/.style={rounded corners,draw=black!55,fill=gray!6,line width=.45pt},
  panelnote/.style={font=\scriptsize,align=center,text width=4.35cm},
  smalllabel/.style={font=\scriptsize,fill=white,inner sep=1.4pt,align=center},
  paneltitle/.style={font=\bfseries\scriptsize,align=center,text width=4.55cm}
}

\begin{scope}
\node[paneltitle] at (0,1.82) {A. Used curvature ridge};
\draw[panelbox] (-2.30,-1.35) rectangle (2.30,1.35);
\node[font=\scriptsize,anchor=north west] at (-2.10,1.18) {$C$};

\draw[blue!35,very thick]
  plot[smooth] coordinates {(-2.00,-0.82) (-1.35,-0.28) (-0.63,0.27) (0.22,0.82) (1.30,1.18)};
\draw[blue!70!black,very thick]
  plot[smooth] coordinates {(-2.08,-1.00) (-1.40,-0.46) (-0.68,0.07) (0.18,0.61) (1.22,1.02)};
\draw[blue!35,very thick]
  plot[smooth] coordinates {(-1.76,-1.17) (-1.08,-0.64) (-0.36,-0.10) (0.50,0.43) (1.56,0.81)};

\node[smalllabel,anchor=west] at (-1.88,0.48) {$g_j$ varies on $C$};
\node[smalllabel,anchor=west] at (0.36,-0.12) {$\sigma''(g_j)>0$};
\draw[->,thick] (0.62,-0.88) -- (1.14,-0.88);
\node[smalllabel,anchor=west] at (1.20,-0.88) {$v_j\neq 0$};

\node[panelnote] at (0,-1.95)
{nonconstant argument\\and nonzero outgoing coefficient};
\end{scope}

\begin{scope}[xshift=5.05cm]
\node[paneltitle] at (0,1.82) {B. Reducedness};
\draw[panelbox] (-2.30,-1.35) rectangle (2.30,1.35);
\node[font=\scriptsize,anchor=north west] at (-2.10,1.18) {$C$};

\draw[blue!70!black,very thick]
  plot[smooth] coordinates {(-1.98,-0.95) (-1.24,-0.38) (-0.50,0.20) (0.36,0.78) (1.40,1.14)};
\draw[orange!85!black,dashed,very thick]
  plot[smooth] coordinates {(-1.98,-0.95) (-1.24,-0.38) (-0.50,0.20) (0.36,0.78) (1.40,1.14)};

\draw[purple!75!black,very thick]
  plot[smooth] coordinates {(0.92,-1.08) (1.18,-0.51) (1.45,0.10) (1.78,0.98)};

\node[smalllabel,anchor=west] at (-1.96,0.58) {$g_j$};
\node[smalllabel] at (-0.38,-0.26) {$g_i=g_j$ or $g_i=-g_j$};
\node[smalllabel,anchor=west] at (0.72,-0.95) {$h$ distinct};

\node[panelnote] at (0,-1.95)
{same ridge signature\\is counted only once};
\end{scope}

\begin{scope}[xshift=10.10cm]
\node[paneltitle] at (0,1.82) {C. Sign-noncancellation};
\draw[panelbox] (-2.30,-1.35) rectangle (2.30,1.35);
\node[font=\scriptsize,anchor=north west] at (-2.10,1.18) {$C$};

\draw[blue!70!black,very thick]
  plot[smooth] coordinates {(-1.95,-0.95) (-1.22,-0.36) (-0.50,0.23) (0.34,0.80) (1.38,1.12)};
\draw[orange!85!black,dashed,very thick]
  plot[smooth] coordinates {(-1.95,-0.95) (-1.22,-0.36) (-0.50,0.23) (0.34,0.80) (1.38,1.12)};

\draw[green!50!black,very thick] (-1.68,-1.02) -- (1.52,-0.18);

\node[smalllabel] at (0.02,0.78) {$\sigma(g)-\sigma(-g)=g$};
\node[smalllabel,anchor=west] at (-1.98,0.16) {opposite pair};
\node[smalllabel,anchor=west] at (0.22,-0.46) {affine residual};

\node[panelnote] at (0,-1.95)
{opposite terms may collapse\\to something affine};
\end{scope}

\end{tikzpicture}
\caption{Softplus usedness as a curvature-based analogue of ReLU usedness.  Panel A shows the basic requirement: a Softplus term must have a nonconstant argument and a nonzero outgoing coefficient, so that its smooth curvature is visible on the task patch.  Panel B illustrates reducedness: duplicate or opposite arguments generate the same curvature signature and should not be counted as distinct nonlinear axes.  Panel C illustrates sign-noncancellation: because $\sigma(g)-\sigma(-g)=g$, a collection of opposite-sign terms can leave only an affine residual unless this possibility is explicitly ruled out.}
\label{fig:sp-used-curvature}
\end{figure}
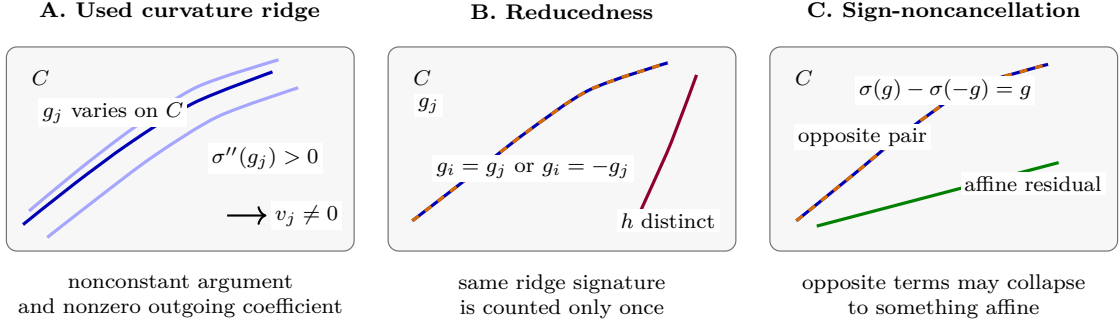

In the Softplus case, there is no exact kink.  Instead, every nonconstant Softplus argument contributes smooth curvature through $\sigma''(g)$.  Figure~\ref{fig:sp-used-curvature} illustrates why this is the analogue of ReLU usedness.  A ReLU gate is used when the task crosses a kink trace and the next affine map reads out the resulting slope jump.  A Softplus term is used when the argument genuinely varies, its outgoing coefficient is nonzero, and the resulting curvature ridge is not being counted redundantly through duplicates, opposites, or affine sign-cancellations.  Intuitively, in proving the weak--strong equivalence in this setting, our job will be to show that adjacent affine mappability forces all relevant slopes within an expansion to be matched.

We posit that the conditions analogous to ReLU usedness are:
\begin{definition}[Used Softplus expansion]
\label{def:sp-used}\label{def:sp-used-adjacent}
After deleting unused terms and absorbing constant arguments into the bias, a Softplus expansion is \emph{used} if:
\begin{enumerate}[label=(\roman*)]
\item every $v_j\neq0$ and every $g_j$ is nonconstant;
\item it is reduced: $g_i=\pm g_j$ implies $i=j$;
\item it is sign-noncancelling: for every nonempty $S$,
\[
 \inf_{a\in\R^m}\left\|\sum_{j\in S}v_jg_j-a\right\|_{L^2}>0.
\]
\end{enumerate}
\end{definition}
\noindent To understand what motivated these conditions as replacement for the ReLU usedness condition, we show that these conditions enable a softplus version of shrinkability. For ReLU, a gate can be shrinkable because the patch may not cross its zero trace; then $\rho(z_j)$ is either zero or affine on that patch. For Softplus, there is no ``inactive side'': if $g$ is nonconstant and the outgoing coefficient is nonzero, then $\sigma(g)$ contributes genuine smooth curvature on any open patch, because $\sigma''(t)>0$ everywhere. The Softplus shrinkability result therefore needs to be about algebraic redundancies: zero coefficients, constant arguments, duplicate arguments, opposite arguments, and the sign ambiguity coming from the softplus identity $\sigma(t)-\sigma(-t) = t$.  Specifically:

\begin{lemma}[Shrinkability of Softplus degeneracies]
Let \(C\) be a task patch with nonempty relative interior in \(V=\operatorname{aff}(C)\).  Consider a vector-valued Softplus expansion on \(C\)
\[
F(z)=a(z)+\sum_{j=1}^k v_j\sigma(g_j(z)),
\qquad z\in C,
\]
where \(a:V\to\mathbb R^m\) is affine, each \(v_j\in\mathbb R^m\), and
each \(g_j:V\to\mathbb R\) is affine.  Then the following operations preserve $F$ and do not increase the number of nonlinear Softplus
terms:
\begin{enumerate}
\item If \(v_j=0\), the \(j\)-th term may be deleted.

\item If \(g_j\equiv \gamma\) on \(V\), the \(j\)-th term may be deleted
and the affine part replaced by
\[
a(z)\longmapsto a(z)+v_j\sigma(\gamma).
\]

\item If \(g_i=g_j\) on \(V\), then the two terms may be replaced by one
term:
\[
v_i\sigma(g_i)+v_j\sigma(g_j)
=
(v_i+v_j)\sigma(g_j).
\]
If \(v_i+v_j=0\), both nonlinear terms disappear.

\item If \(g_i=-g_j\) on \(V\), then the two terms may be replaced by one
Softplus term plus an affine residual:
\[
v_i\sigma(g_i)+v_j\sigma(g_j)
=
(v_i+v_j)\sigma(g_j)-v_i g_j.
\]
Since \(g_j\) is affine, the residual \(-v_i g_j\) can be absorbed into
the affine part \(a\).  If \(v_i+v_j=0\), no nonlinear Softplus term
remains from this pair.

\item More generally, if for a nonempty set \(S\)
\[
\sum_{j\in S} v_j g_j
\]
is constant on \(C\), then simultaneously replacing \(g_j\) by \(-g_j\)
for all \(j\in S\) changes the expansion only by a bias term.
\end{enumerate}
\noindent Consequently, modulo affine residuals, every finite Softplus expansion can be reduced to one in which all remaining nonlinear terms have nonzero coefficient, nonconstant argument, and no two remaining arguments are equal up to sign.  The additional affine sign-noncancellation
condition rules out the remaining orientation ambiguity caused by \(\sigma(t)-\sigma(-t)=t\).
\end{lemma}

\begin{proof}
The zero-coefficient and constant-argument cases are immediate.  If
\(v_j=0\), the term contributes nothing.  If \(g_j\equiv\gamma\), then
\(v_j\sigma(g_j)=v_j\sigma(\gamma)\), which is constant and can be folded
into the affine part.

If \(g_i=g_j\), then
\[
v_i\sigma(g_i)+v_j\sigma(g_j)
=
(v_i+v_j)\sigma(g_j),
\]
so duplicate arguments merge.

If \(g_i=-g_j\), use the Softplus identity
\[
\sigma(t)-\sigma(-t)=t,
\qquad\text{equivalently}\qquad
\sigma(-t)=\sigma(t)-t.
\]
Writing \(g_i=-g_j\), we get
\[
v_i\sigma(g_i)+v_j\sigma(g_j)
=
v_i\sigma(-g_j)+v_j\sigma(g_j)
=
v_i(\sigma(g_j)-g_j)+v_j\sigma(g_j),
\]
hence
\[
v_i\sigma(g_i)+v_j\sigma(g_j)
=
(v_i+v_j)\sigma(g_j)-v_i g_j.
\]
The last term is affine in \(z\), because \(g_j\) is affine on \(V\), so
it can be absorbed into \(a\).

For the final claim, suppose
\[
\sum_{j\in S}v_jg_j\equiv c
\]
is constant on \(C\).  Then
\[
\sum_{j\in S}v_j\sigma(-g_j)
=
\sum_{j\in S}v_j(\sigma(g_j)-g_j)
=
\sum_{j\in S}v_j\sigma(g_j)-c.
\]
Thus flipping all signs in \(S\) changes the represented function only
by the constant \(-c\), which can be absorbed into the bias.  This is
exactly the orientation ambiguity excluded by affine
sign-noncancellation.
\end{proof}
\noindent Thus, we can let $m_\ell^{\sig}(\eps)$ denote the budget count for such softplus used axes, exactly in analogy the ReLU case.\footnote{The intricacy of the above argument is why, for the weak-strong equivalence-type of result, the ReLU case is easier to think about that the softplus case. Interestingly, the opposite will be true for the zippering-type results, where monodromy makes our lives easier. Ideally we could find a unification of all the arguments that reduce all this complexity.}  We can now also formulate and prove the softplus version of weak-strong equivalence:

\begin{theorem}[Exact Softplus weak--strong equivalence]
\label{thm:sp-adjacent}
Fix $C\subseteq z_\ell^A(\Omega)$ and suppose exact weak equivalence holds at layers $\ell$ and $\ell+1$.  Assume the transformed $A$-side and pulled-back $B$-side adjacent expansions are used in the sense of Definition~\ref{def:sp-used}.  Then there is an injective matching
\[
 \pi:J_\ell^{A,\sig}(C)\to J_\ell^{B,\sig}(E_\ell;C)
\]
with
\[
 (E_\ell z)_{\pi(j)}=z_j
 \qquad\forall z\in C.
\]
If $A$ reaches loss at most $\eps$, then
\[
 \AxisAlign_{\ell,\sig}(A,B;E_\ell,C)
 \ge\frac{m_\ell^{\sig}(\eps)}{d_\ell^A}.
\]
\end{theorem}

\begin{proof}
As above, write $E_r(y)=T_ry+a_r$.  As with the ReLU case, by Lemma~\ref{lem:adjacent-commutation-common}, exact weak equivalence at layers
$\ell$ and $\ell+1$ gives
\[
 \Phi^B_{\ell,\sig}(E_\ell z)
 =
 E_{\ell+1}\Phi^A_{\ell,\sig}(z),
 \qquad z\in C.
\]
After deleting unused terms and absorbing constant arguments into the bias, this equality can be
written as
\begin{equation}
 \widehat b+\sum_{j=1}^{k}\widehat v_j\sig(g_j(z))
 =
 \widetilde b+\sum_{i=1}^{\widetilde k}\widetilde v_i\sig(\widetilde g_i(z)),
 \qquad z\in C,
 \label{eq:sp-ws-expansion-identity}
\end{equation}
where the notations are described above, e.g. $g_j(z)=z_j$, $\widetilde g_i(z)=(E_\ell z)_i$, $\widehat b=T_{\ell+1}b^A_{\ell+1}+a_{\ell+1}$, $\widehat v_j=T_{\ell+1}W^A_{\ell+1}[:,j]$, 
$\widetilde b=b^B_{\ell+1}$, and $\widetilde v_i=W^B_{\ell+1}[:,i]$. Since $T_{\ell+1}$ is injective, it preserves nonzero outgoing coefficients: $W^A_{\ell+1}[:,j]\neq0$ implies $\widehat v_j\neq0$. It also preserves sign-noncancellation. Indeed, if for some nonempty set $S$ the transformed sum
\[
 \sum_{j\in S}\widehat v_jg_j
 =
 T_{\ell+1}\left(\sum_{j\in S}W^A_{\ell+1}[:,j]g_j\right)
\]
were constant on $C$, then taking differences between any two points of $C$ and using injectivity of $T_{\ell+1}$ would imply that $\sum_{j\in S}W^A_{\ell+1}[:,j]g_j$ was already constant on $C$, contrary to the sign-noncancellation assumption. Thus the transformed $A$-side expansion remains used, reduced, nonconstant, and sign-noncancelling; the pulled-back $B$-side expansion has these properties by hypothesis. 

Now, let $V$ be the affine hull of the task patch $C$ and let $L$ be its tangent space.  We choose $z_0$ in the relative interior of $C$ and choose $\xi \in L$ ``generically'': we exclude all pairs $(z_0,\xi)$ for which some nonconstant argument has zero slope on the line $\xi$. We also exclude accidental sign-coincidences on the chosen line, i.e. for every pair $q,q'$ and every $\delta\in\{+1,-1\}$ such that $q\neq \delta q'$ as affine functions on $V$, we exclude the set
\[
 \bigl\{(z_0,\xi):
 q(z_0)-\delta q'(z_0)=0,\ 
 (\operatorname{lin}(q)-\delta\operatorname{lin}(q'))(\xi)=0
 \bigr\}
\] where $\operatorname{lin}(q)$ is the linear part of the operation $q$. This is a finite union of proper affine-algebraic subsets of $\relint(C)\times L$, and so has measure 0.  Thus we can generically choose $(z_0,\xi)$ outside this ``bad'' exceptional set.  

Then every restricted argument has nonzero slope, and two restricted affine functions agree up to sign only if the original affine functions already agree up to that sign on all of $V$. For $|t|$ small, the line $z(t)=z_0+t\xi$ lies in $C$.  Along this line write
\[
 g_j(z_0+t\xi)=s_jt+r_j,
 \qquad
 \widetilde g_i(z_0+t\xi)=\widetilde s_it+\widetilde r_i,
\]
with $s_j\neq0$ and $\widetilde s_i\neq0$ for every $i,j$.  Restricting \eqref{eq:sp-ws-expansion-identity} to the line gives equality of
real-analytic functions on a nonempty interval, hence on all of $\R$:
\begin{equation}
 \widehat b+\sum_{j=1}^{k}\widehat v_j\sig(s_jt+r_j)
 =
 \widetilde b+\sum_{i=1}^{\widetilde k}\widetilde v_i\sig(\widetilde s_it+\widetilde r_i).
 \label{eq:sp-ws-line-identity}
\end{equation}
\noindent Now put
\[
 \eta(t)=\sig''(t)=\frac{1}{4\cosh^2(t/2)}.
\]
Differentiating \eqref{eq:sp-ws-line-identity} twice yields:
\begin{equation}
 \sum_{j=1}^{k}\widehat v_js_j^2\eta(s_jt+r_j)
 =
 \sum_{i=1}^{\widetilde k}\widetilde v_i\widetilde s_i^{\,2}\eta(\widetilde s_it+\widetilde r_i).
 \label{eq:sp-ws-ridge-identity}
\end{equation}
This is a so-called ``ridge identity'': an equality between finite sums of translated and rescaled curvature bumps obtained by slicing the original affine ridges along the generic line.  The curvature term $z\mapsto \eta(q(z))$, is a ``ridge function'': it is constant along the affine slices $q(z)=\mathrm{constant}$ and varies only in the normal direction.  Restricting to a generic line converts this geometric ridge into a translated and rescaled one-dimensional bump $t\mapsto \eta(st+r)$, whose slope and offset can then be identified by Fourier decay and phase.
This line-restriction method is standard in the ridge-function literature; see
\citet{Pinkus2015RidgeFunctions}.

Following this same line of literature, we now take Fourier transforms to separate the one-D bumps for analysis. For Softplus curvature,
\[
 \widehat\eta(\omega)=\frac{\pi\omega}{\sinh(\pi\omega)}.
\]
For every $s\neq0$,
\[
 \mathcal F\{s^2\eta(st+r)\}(\omega)
 =
 |s|e^{i\omega r/s}\widehat\eta(\omega/s).
\]
Taking Fourier transforms in \eqref{eq:sp-ws-ridge-identity} gives
\begin{equation}
 \sum_j
 \widehat v_j|s_j|e^{i\omega r_j/s_j}\widehat\eta(\omega/s_j)
 =
 \sum_i
 \widetilde v_i|\widetilde s_i|e^{i\omega\widetilde r_i/\widetilde s_i}
 \widehat\eta(\omega/\widetilde s_i).
 \label{eq:sp-ws-fourier-identity}
\end{equation}
As $\omega\to+\infty$,
\begin{equation}
 |s|\,\widehat\eta(\omega/s)
 =
 2\pi\omega e^{-\pi\omega/|s|}(1+o(1)).
 \label{eq:sp-ws-fourier-asymp}
\end{equation}
Thus the terms with largest absolute slope have the slowest exponential decay and can be isolated
from all smaller-slope terms.  To that end, let
\[
 S=
 \max\Bigl(
 \{|s_j|:1\le j\le k\}
 \cup
 \{|\widetilde s_i|:1\le i\le\widetilde k\}
 \Bigr).
\]
Multiply \eqref{eq:sp-ws-fourier-identity} by $e^{\pi\omega/S}/(2\pi\omega)$ and use
\eqref{eq:sp-ws-fourier-asymp}.  Every term with absolute slope strictly smaller than $S$ tends
to zero.  The maximal-slope terms satisfy
\begin{equation}
 P_S(\omega):=
 \sum_{|s_j|=S}\widehat v_je^{i\omega r_j/s_j}
 -
 \sum_{|\widetilde s_i|=S}\widetilde v_ie^{i\omega\widetilde r_i/\widetilde s_i}
 =
 o(1).
 \label{eq:sp-ws-max-slope-poly}
\end{equation}
\noindent For a maximal-slope term, define its phase frequency by $\tau=r/s$. 
Among affine functions with the same absolute slope $S$, equality of $\tau$ is exactly equality
up to sign as one-dimensional affine functions.  Indeed, if $s'=\pm s$ and $r'/s'=r/s$, then
$r'=\pm r$ with the same sign, so $s't+r'=\pm(st+r)$.

Group the terms in \eqref{eq:sp-ws-max-slope-poly} by this common phase frequency.  For each
frequency $\tau$, define the grouped vector coefficient
\begin{equation}
 c_\tau
 =
 \sum_{\substack{j:\ |s_j|=S\\ r_j/s_j=\tau}}\widehat v_j
 -
 \sum_{\substack{i:\ |\widetilde s_i|=S\\
 \widetilde r_i/\widetilde s_i=\tau}}\widetilde v_i.
 \label{eq:sp-ws-grouped-coeff}
\end{equation}
\noindent With this notation, $P_S(\omega)=\sum_{\tau}c_\tau e^{i\omega\tau}.$
We claim that all of these coefficients $c_\tau$ are zero.  Since $P_S(\omega)=o(1)$, we have
\[
\lim_{T\to\infty}\frac1T\int_0^T\|P_S(\omega)\|^2\,d\omega=0.
\]
On the other hand,
\begin{align}
 \frac1T\int_0^T\|P_S(\omega)\|^2\,d\omega
 &=
 \frac1T\int_0^T
 \left\|
 \sum_{\tau}c_\tau e^{i\omega\tau}
 \right\|^2d\omega
 \notag\\
 &=
 \sum_{\tau,\tau'}
 \langle c_\tau,c_{\tau'}\rangle
 \left(
 \frac1T\int_0^T e^{i\omega(\tau-\tau')}\,d\omega
 \right).
 \label{eq:sp-ws-averaged-square}
\end{align}
If $\tau=\tau'$, the parenthesized factor is $1$.  If $\tau\neq\tau'$, then
\[
 \frac1T\int_0^T e^{i\omega(\tau-\tau')}\,d\omega
 =
 \frac{e^{iT(\tau-\tau')}-1}{iT(\tau-\tau')}
 \longrightarrow0.
\]
Letting $T\to\infty$ in \eqref{eq:sp-ws-averaged-square} gives
\[
 0=
 \sum_\tau \|c_\tau\|^2.
\]
Therefore $c_\tau=0$ for every maximal-slope frequency $\tau$. By reducedness and by the generic choice of the line, each frequency group contains at most one $A$-side term and at most one $B$-side term.  Since all coefficient vectors are nonzero, $c_\tau = 0$ implies that every maximal-slope $A$-term is paired with a unique maximal-slope $B$-term, and conversely.  For such a pair $(j,i)$,
\[
 \widetilde g_i(z_0+t\xi)
 =
 \varepsilon_j g_j(z_0+t\xi)
 \qquad\text{for all }t,
\]
for some $\varepsilon_j\in\{+1,-1\}$, and
$\widetilde v_i=\widehat v_j$. Because $\eta$ is an even function and the slopes differ only by sign, the paired terms contribute identically
to the ridge identity \eqref{eq:sp-ws-ridge-identity}.  Subtract all such matched maximal-slope
pairs from both sides of \eqref{eq:sp-ws-ridge-identity}.  If any terms remain, repeat the same
largest-slope Fourier isolation with the next-largest absolute slope.  Since the number of terms is
finite, this iteration terminates.

Thus all restricted affine arguments have been matched.  We obtain a bijection
\[
 \pi:\{1,\ldots,k\}\to\{1,\ldots,\widetilde k\}
\]
and signs $\varepsilon_j\in\{+1,-1\}$ such that
\begin{equation}
 \widetilde g_{\pi(j)}(z_0+t\xi)
 =
 \varepsilon_jg_j(z_0+t\xi)
 \qquad\text{for all }t,
 \label{eq:sp-ws-linewise-sign-match}
\end{equation}
and
\begin{equation}
 \widetilde v_{\pi(j)}=\widehat v_j.
 \label{eq:sp-ws-coeff-match}
\end{equation}
The generic-line choice promotes this linewise equality to equality on all of $V$.  Indeed, if
$\widetilde g_{\pi(j)}-\varepsilon_jg_j$ were not identically zero on $V$, then
\eqref{eq:sp-ws-linewise-sign-match} would mean that the chosen pair $(z_0,\xi)$ lies in one of
the excluded accidental sign-coincidence sets.  Hence
\begin{equation}
 \widetilde g_{\pi(j)}(z)
 =
 \varepsilon_jg_j(z)
 \qquad
 \forall z\in V.
 \label{eq:sp-ws-global-sign-match}
\end{equation}
\noindent It remains to remove the negative orientations.  Substitute
\eqref{eq:sp-ws-global-sign-match} and \eqref{eq:sp-ws-coeff-match} into the original expansion
identity \eqref{eq:sp-ws-expansion-identity}.  After reindexing by $\pi$,
\[
 \widehat b+\sum_j\widehat v_j\sig(g_j)
 =
 \widetilde b+\sum_j\widehat v_j\sig(\varepsilon_jg_j).
\]
Let $S_-=\{j:\varepsilon_j=-1\}$. Cancelling the positively oriented matched terms gives
\[
 \widehat b-\widetilde b+
 \sum_{j\in S_-}\widehat v_j\bigl(\sig(g_j)-\sig(-g_j)\bigr)=0.
\]
Using the Softplus identity $\sig(t)-\sig(-t)=t,$
we obtain
\begin{equation}
 \widehat b-\widetilde b+
 \sum_{j\in S_-}\widehat v_jg_j=0.
 \label{eq:sp-ws-negative-orientation}
\end{equation}
If $S_-$ were nonempty, \eqref{eq:sp-ws-negative-orientation} would say that a nonempty
downstream-weighted sum of arguments is constant on $C$, contradicting sign-noncancellation.
Therefore $S_-=\varnothing$.  All signs are positive, and so 
$$\widetilde g_{\pi(j)}(z)=g_j(z) \quad \forall z\in V.$$ 
Since $g_j(z)=z_j$ and $\widetilde g_i(z)=(E_\ell z)_i$, this is exactly
\[
 (E_\ell z)_{\pi(j)}=z_j
 \qquad
 \forall z\in C.
\]
This proves the strong Softplus axis-matching statement.

Finally, if $A$ reaches loss at most $\eps$, then by definition of the Softplus used-axis budget it
uses at least $m^\sig_\ell(\eps)$ layer-$\ell$ Softplus axes.  The argument above matches every
such used $A$-side axis to a distinct pulled-back $B$-side axis.  Therefore at least
$m^\sig_\ell(\eps)$ axes contribute to the numerator of $\AxisAlign_{\ell,\sig}$, and
\[
 \AxisAlign_{\ell,\sig}(A,B;E_\ell,C)
 \ge
 \frac{m^\sig_\ell(\eps)}{d^A_\ell}.
\]
\end{proof}

\subsection{Asymptotic weak--strong equivalence}
\label{sec:asymptotic-ws}

We now develop a soft versions of weak-strong equivalence, by measuring quantitative defects in alignment and then relating them.  The structure of the definitions and theorem proof is identical for both ReLU and softplus activations.  

\begin{definition}[Weak alignment defect]
Define
\[
 \omega_\ell(E_\ell;A,B)
 =\left(\E_\Omega\|z_\ell^B(x)-E_\ell z_\ell^A(x)\|_2^2\right)^{1/2}.
\]  By definition, $\omega_ell = 0 $ implies exact weak alignment.
\end{definition}

\begin{definition}[Soft axis score]\label{def:soft-axis-score}
For $\varphi\in\{\sig,\rhoR\}$, define
\[
 s_{ij}^\varphi(E_\ell;C)
 =\sup_{\alpha\in\mathcal H_\varphi}
 \left[
 1-\frac{\|(E_\ell z)_i-\alpha z_j\|_{L^2(C)}}
 {\|(E_\ell z)_i\|_{L^2(C)}+|\alpha|\|z_j\|_{L^2(C)}}
 \right].
\]
Let $\AxisAlign_{\ell,\varphi}^\theta$ be the maximal fraction of counted $A$-axes that can be injectively matched with score at least $\theta$.  
\end{definition}

\begin{theorem}[Asymptotic weak--strong equivalence]
\label{thm:common-asymp-ws}\label{thm:common-asymptotic}
Fix $\varphi\in\{\sig,\rhoR\}$.  Let $\calK$ be a compact family of adjacent comparison tuples $p = (A, B, E_\ell, E_{\ell + 1})$.  For every $\theta\in(0,1)$ there is $\kappa_{\calK,\varphi}(\theta)>0$ such that
\[
 \AxisAlign_{\ell,\varphi}^\theta(A,B;E_\ell,C)
 \ge
 \frac{m_\ell^\varphi(\eps)}{d_\ell^A}
 -\frac{\bigl(\|W_{\ell+1}^B\|\,\omega_\ell+\omega_{\ell+1}\bigr)^2}
 {\kappa_{\calK,\varphi}(\theta)^2d_\ell^A}.
\]
Moreover, if $\{p_n = (A_n, B_n, E_{\ell, n}, E_{\ell+1, n})\}$ is a sequence of comparison tuples in $\calK$ for which the defect sequences $\omega_{\ell, n} := \ell(E_{\ell,n}; A_n, B_n)$ and $\omega_{\ell+1, n} := \ell(E_{\ell+1,n}; A_n, B_n)$ both converge to zero, then there exists $N$ such that, for all $n\ge N$,
\[
 \AxisAlign_{\ell,\varphi}^\theta(A_n,B_n;E_{\ell,n},C_n)
 \ge \frac{m_\ell^\varphi(\eps)}{d_{\ell,n}^A}.
\]
In words, every axis forced by the task budget is eventually $\theta$-aligned.
\end{theorem}

\begin{proof}
For a comparison tuple $p=(A,B,E_\ell,E_{\ell+1})$, let $B^\theta(p)$ be the minimum number of counted task-required $A$-axes that fail threshold $\theta$ under the best injective matching into the counted $B$-axes.  Define the \emph{commutation defect}
\[
 \Delta_{\ell,\varphi}(p)
 =\|\Phi_{\ell,\varphi}^B(E_\ell z_\ell^A)-E_{\ell+1}\Phi_{\ell,\varphi}^A(z_\ell^A)\|_{L^2(\Omega)}.
\]
Now let 
\[
 \kappa_{\calK,\varphi}(\theta)
 =\inf_{\substack{p\in\calK\\ B^\theta(p)>0}}
 \frac{\Delta_{\ell,\varphi}(p)}{\sqrt{B^\theta(p)}}.
\]
If the indexing set is empty, every tuple in $\calK$ already has all counted axes $\theta$-aligned and the theorem is immediate.  Otherwise we first show that the infimum is positive.

Assume for contradiction that it is zero.  Then there is a sequence $p_n\in\calK$ with $B^\theta(p_n)>0$ and
\[
 \frac{\Delta_{\ell,\varphi}(p_n)^2}{B^\theta(p_n)}\longrightarrow0.
\]
The number of counted axes is uniformly bounded by $d_\ell^A$, so $\Delta_{\ell,\varphi}(p_n)\to0$.  Compactness gives a convergent subsequence $p_n\to p_*\in\calK$.  The network operations, affine comparison maps, and $L^2$ norms vary continuously on the bounded family, hence
\[
 \Delta_{\ell,\varphi}(p_*)=0.
\]
The zero-defect member satisfies the hypotheses of the appropriate exact weak--strong theorem.  It therefore admits an injective matching $\pi_*$ for which every counted task-required axis has score exactly one.  There are only finitely many matched pairs, and each pairwise score is continuous in the tuple parameters on the quantitatively nondegenerate family.  Since $\theta<1$, the same fixed matching has every score greater than $\theta$ for all sufficiently large $n$.  This implies $B^\theta(p_n)=0$, contradicting the choice of the sequence.  Thus
\[
 \kappa_{\calK,\varphi}(\theta)>0.
\]
By the definition of the infimum,
\[
 B^\theta(p)
 \le\frac{\Delta_{\ell,\varphi}(p)^2}
 {\kappa_{\calK,\varphi}(\theta)^2}.
\]
\noindent It remains to relate the commutation defect to the two weak-alignment errors.  Put $Z=z_\ell^A(x)$.  Since $\Phi_{\ell,\varphi}^B(z_\ell^B(x))=z_{\ell+1}^B(x)$,
\begin{align*}
 &\Phi_{\ell,\varphi}^B(E_\ell Z)
 -E_{\ell+1}\Phi_{\ell,\varphi}^A(Z)\notag\\
 &\quad=
 \bigl[\Phi_{\ell,\varphi}^B(E_\ell z_\ell^A(x))
 -\Phi_{\ell,\varphi}^B(z_\ell^B(x))\bigr]
 +\bigl[z_{\ell+1}^B(x)-E_{\ell+1}z_{\ell+1}^A(x)\bigr].
\end{align*}
Both ReLU and Softplus are one-Lipschitz, so
\[
 \|\Phi_{\ell,\varphi}^B(u)-\Phi_{\ell,\varphi}^B(v)\|_2
 \le\|W_{\ell+1}^B\|\,\|u-v\|_2.
\]
Taking $L^2(\Omega)$ norms in this identity and applying the triangle inequality yields
\[
 \Delta_{\ell,\varphi}(p)
 \le\|W_{\ell+1}^B\|\,\omega_\ell(p)+\omega_{\ell+1}(p).
\]
At least $m_\ell^\varphi(\eps)$ used axes are task-required, and at most $B^\theta(p)$ of those fail threshold $\theta$.  Therefore
\[
 \AxisAlign_{\ell,\varphi}^\theta
 \ge\frac{m_\ell^\varphi(\eps)-B^\theta(p)}{d_\ell^A}.
\]
Combining the bad-axis bound with the defect bound proves the displayed inequality.

Finally, suppose $p_n\in\calK$ and both weak errors tend to zero.  Compactness uniformly bounds $\|W_{\ell+1}^{B_n}\|$, so the defect bound gives $\Delta_{\ell,\varphi}(p_n)\to0$.  The bad-axis bound then gives $B^\theta(p_n)\to0$.  Since $B^\theta(p_n)$ is a nonnegative integer, it equals zero for all sufficiently large $n$.  Thus every task-required counted axis is eventually $\theta$-aligned.
\end{proof}

\subsection{Minimality, regularity, and exact zippering}
\label{sec:exact-zippering}

Weak--strong alignment views the nonlinear block as a function of the layer coordinate $z$.  The arguments are affine: $z_j$ on one side and $(E_\ell z)_i$ on the other. Zippering starts with only the next representation aligned and asks one to recover the unknown deep scalar arguments $z_{r,j}^N(x)$ as functions of the original input.   The basic challenge in achieving this goal is to overcome is a non-zero probability of \emph{collisions}:  networks that have the same deep input-output relationship but which have different internal structures. 

Our high-level strategy for doing this is to:
\begin{enumerate} 
	\item First strengthen the notion of usedness from single axes/gates to \emph{minimality} of entire layer-wise expansions (\S\ref{sec:minimality-crossmatch}).  Minimality is an intrinsic condition on each expansion: every counted unit leaves a ``private'' task-visible nonlinear signature.  From equality of two one-step expansions, those private signatures one side of the quality must be reproduced on the other side, giving a unique cross-match. 

	\item We then seek to understand what happens when two minimal layers collide.  We show that the property of \emph{identifiability} -- that is, the recovery of equal internal constituent operations -- can be ensured by a set of very simple conditions on \emph{regularity} (\S\ref{sec:softplus-regularity} and \S\ref{sec:relu-regularity}).

	\item When then show how identifiability, based on minimality and regularity, leads to exact zippering (\S\ref{sec:exact-zippering}) -- assuming perfect terminal equivalence of last-layer representations.

\end{enumerate}
\noindent Looking ahead, in the next sections that follow, we show that (when one more additional condition on the Jacobian regularity is added) the same kind of strategy applies to a softer case where perfect terminal equivalence is not known (\S\ref{sec:soft-zippering}), and then that the regularity conditions themselves are generic (\S\ref{sec:genericity}).

\subsubsection{Null pairs and identifiability }
\label{sec:null_pairs}
We start with some basic definitions:

\begin{definition}[Realization map, null pairs, collisions, and identifiability]
\label{def:null-pair-collision}
Fix $\varphi\in\{\sig,\rhoR\}$.  Let $\Theta_r^\varphi$ be the manifold of layer-$r$ one-step expansions. Define the \emph{realization map} by
\[
 F_r^\varphi:\Theta_r^\varphi\longrightarrow L^2(\Omega;\R^{d_{r+1}}),
 \qquad F_r^\varphi: \vartheta \longmapsto Y_{G_\vartheta}^\varphi,
\]
so called because $F_r^\varphi$ is the function that sends one-step parameters to the function they realize.
Two parameter values \(\vartheta,\widetilde\vartheta\) form a \emph{null pair},
or \emph{collision}, if
\[
F_r^\varphi(\vartheta)=F_r^\varphi(\widetilde\vartheta)
\] e.g. they represent the same function on $\Omega$.
The corresponding \emph{collision set} is
\[
 \mathcal C_r^\varphi
 =
 \{(\vartheta,\widetilde\vartheta):
 F_r^\varphi(\vartheta)=F_r^\varphi(\widetilde\vartheta)\}.
\]
Each activation function has some unavoidable hidden-unit
symmetries--permutation for Softplus, and permutation plus positive rescaling for
ReLU--the ``obvious self-collisions'' form the \emph{diagonal}
\[
 \Delta_r^\varphi
 =
 \{(\vartheta,\vartheta):\vartheta\in\Theta_r^\varphi\}.
\]
A collision is \emph{nontrivial} if it lies away from this diagonal after the
allowed hidden-unit matching has been fixed.
When \(\mathcal C_r^\varphi\) is decomposed into smooth pieces, a
\emph{collision stratum} means one such smooth piece.  An \emph{off-diagonal
collision stratum} is a smooth piece contained in
\(\mathcal C_r^\varphi\setminus\Delta_r^\varphi\).

The opposite of collision is \emph{identifiability}. In the Softplus case, identifiability means that, after reindexing,
\[
 \widetilde g_j=g_j,
 \qquad
 \widetilde v_j=v_j,
 \qquad
 \widetilde b=b.
\]
In the ReLU case, because  $\rhoR(\alpha_jg_j)=\alpha_j\rhoR(g_j)$, identifiability means that, after reindexing, there are $\alpha_j>0$ such that
\[
 \widetilde g_j=\alpha_jg_j,
 \qquad
 \alpha_j\widetilde v_j=v_j,
 \qquad
 \widetilde b=b.
\]
\end{definition}

The exact zippering proof is a common backward induction once this one-step identifiability target has been verified.  We now verify it by two different signature theories.

\subsubsection{Intrinsic minimality and private cross-matching}
\label{sec:minimality-crossmatch}

\begin{figure}[t]
\centering
\resizebox{\linewidth}{!}{%
\begin{tikzpicture}[x=1cm,y=1cm,>=stealth,font=\small]

\definecolor{panelgray}{RGB}{247,247,247}
\definecolor{softblue}{RGB}{10,20,170}
\definecolor{leftfill}{RGB}{224,224,248}
\definecolor{rightfill}{RGB}{241,235,223}

\begin{scope}[shift={(0,0)}]

  \draw[rounded corners=0.12cm, line width=0.5pt, fill=panelgray]
    (0,0) rectangle (6.4,4.45);

  \node[font=\bfseries\Large] at (3.2,4.05) {Softplus};

  \draw[->, line width=0.45pt] (0.95,0.80) -- (0.95,3.45);
  \draw[->, line width=0.45pt] (0.75,0.85) -- (5.75,0.85);
  \node[font=\large] at (0.72,3.70) {$\sigma(g)$};
  \node[font=\large] at (5.98,0.83) {$g$};

  \draw[softblue, line width=1.0pt, smooth]
    plot coordinates {
      (1.15,1.05) (1.65,1.06) (2.10,1.12) (2.50,1.28)
      (2.90,1.62) (3.35,2.12) (3.85,2.68) (4.45,3.08)
      (5.05,3.38) (5.40,3.54)
    };

  \node[
    softblue,
    align=center,
    font=\normalsize,
    fill=panelgray,
    inner sep=1.6pt,
    text width=2.15cm
  ] at (5.20,1.72) {smooth curvature\\signature};

  \node[
    align=center,
    font=\normalsize,
    fill=panelgray,
    inner sep=1pt
  ] at (3.2,0.28) {pole-separation keeps ridges distinct};

\end{scope}

\begin{scope}[shift={(8.35,0)}]

  \draw[rounded corners=0.12cm, line width=0.5pt, fill=panelgray]
    (0,0) rectangle (6.4,4.45);

  \node[font=\bfseries\Large] at (3.2,4.03) {ReLU};
  \node[font=\normalsize] at (3.2,3.58) {exposed fresh facet};

  \coordinate (SW) at (0.75,0.95);
  \coordinate (NW) at (0.75,3.45);
  \coordinate (NE) at (5.65,3.45);
  \coordinate (SE) at (5.65,0.95);

  \coordinate (Fa) at (0.75,2.92);
  \coordinate (Fb) at (5.65,1.38);

  \fill[rightfill] (Fa) -- (NW) -- (NE) -- (Fb) -- cycle;

  \fill[leftfill] (SW) -- (Fa) -- (Fb) -- (SE) -- cycle;

  \draw[line width=0.4pt] (SW) rectangle (NE);

  \draw[softblue, line width=1.0pt] (Fa) -- (Fb);

  \node[
    softblue,
    font=\Large,
    fill=panelgray,
    inner sep=1pt
  ] at (3.00,2.18) {$F$};

  \node[
    font=\normalsize,
    fill=rightfill,
    inner sep=1pt
  ] at (4.55,2.58) {$g>0$};

  \node[
    font=\normalsize,
    fill=leftfill,
    inner sep=1pt
  ] at (1.70,1.78) {$g<0$};

  \node[
    font=\small,
    inner sep=1pt
  ] at (2.65,1.16) {on $C$: $g(x)=a_C^\top x+c_C$};

  \node[
    softblue,
    font=\footnotesize,
    inner sep=1pt
  ] at (3.2,0.24) {$F=C\cap\{g=0\}=C\cap\{a_C^\top x+c_C=0\}$};

\end{scope}

\draw[->, line width=0.7pt] (6.82,2.18) -- (7.92,2.18);

\node[
  align=center,
  font=\small,
  fill=panelgray,
  inner sep=1.2pt,
  text width=1.55cm
] at (7.35,2.92) {same role:\\identify axes};

\end{tikzpicture}%
}
\caption{Softplus and ReLU use different nonlinear signatures to identify hidden axes. For softplus,
the identifying signature is a smooth curvature pattern, with pole-separation ensuring that distinct
ridges are not confused. For ReLU, the analogous signature is an exposed fresh facet: inside a cell
\(C\) where the inherited lower computation is affine, \(g(x)=a_C^\top x+c_C\), and the zero set
\(F=C\cap\{g=0\}=C\cap\{a_C^\top x+c_C=0\}\) separates the regions \(g<0\) and \(g>0\).}
\label{fig:minimality}
\end{figure}
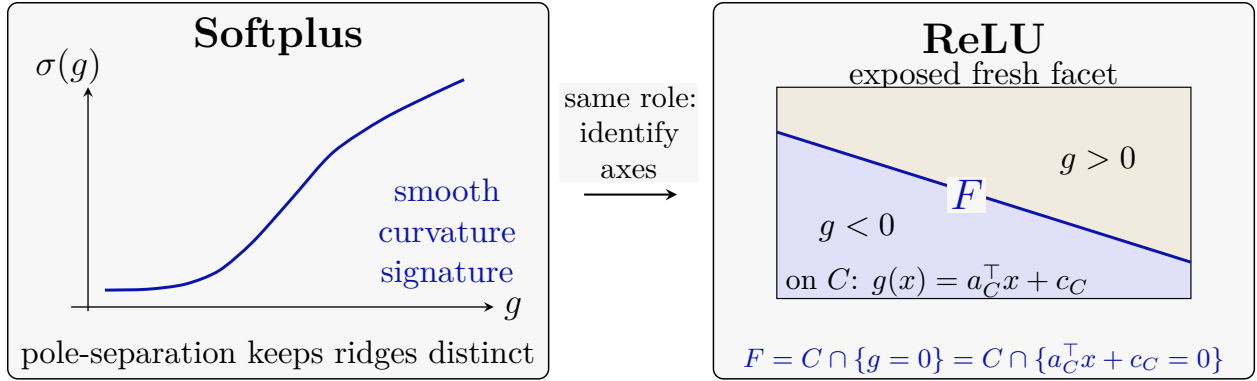

Intrinsic minimality is the part of the argument that does not yet try to reconstruct an entire hidden scalar function.  Its job is simpler: it ensures that each counted unit has a private nonlinear signature, and therefore that equality of two one-step expansions forces a unique cross-side match for that signature. For Softplus the private signature is an ``outer pole patch'' (Fig. \ref{fig:minimality}, left). For ReLU it is an ``exposed fresh facet'' (Fig. \ref{fig:minimality}, right).

\paragraph{Softplus minimality and private pole cross-matching.}

The adjacent Softplus theorem used only affine-ridge usedness.  In a deep one-step expansion the arguments $g_j=z_{r,j}(x)$ are no longer affine functions of the compared layer coordinate; they inherit all lower-layer nonlinear structure.  Softplus minimality adds the internal pole-separation condition saying that each sign-class leaves a private pole signature that cannot be confused with another sign-class in the same expansion:

\begin{figure}[t]
\centering
\begin{tikzpicture}[x=1cm,y=1cm]
\tikzset{
  paneltitle/.style={font=\bfseries\small,align=center,text width=4.8cm},
  smalllabel/.style={font=\scriptsize,fill=white,inner sep=1.4pt,align=center},
  panelnote/.style={font=\scriptsize,align=center,text width=4.55cm}
}

\begin{scope}
\node[paneltitle] at (0,2.72) {A. Outer Softplus poles};

\draw[->] (-2.05,0)--(2.05,0) node[right] {$\Re t$};
\draw[->] (0,-1.75)--(0,1.88) node[above] {$\Im t$};

\foreach \y/\lab in {-1.38/$-3\pi i$,-0.46/$-\pi i$,0.46/$\pi i$,1.38/$3\pi i$}{
  \fill[red!70!black] (0,\y) circle (2.1pt);
  \node[font=\scriptsize,anchor=west] at (0.22,\y) {\lab};
}

\node[panelnote] at (0,-2.55)
{Poles occur where $1+e^t=0$.  Pulling these sets back by $g^\C$ gives pole patches of $Q_g=1+e^{g^\C}$.};
\end{scope}

\begin{scope}[xshift=5.95cm]
\node[paneltitle] at (0,2.72) {B. Pole separation};

\draw[->] (-2.15,0)--(2.15,0) node[right] {$\Re u$};
\draw[->] (0,-1.75)--(0,1.88) node[above] {$\Im u$};

\draw[red!75!black,very thick]
  plot[smooth] coordinates {(-1.35,-1.48) (-0.92,-0.86) (-0.50,-0.28) (-0.10,0.36) (0.35,1.42)};
\draw[blue!70!black,very thick]
  plot[smooth] coordinates {(0.82,-1.55) (1.08,-0.90) (1.35,-0.18) (1.62,0.60) (1.92,1.52)};

\draw[line width=4.2pt,orange!85!black] (-0.63,-0.56)--(-0.28,-0.05);

\node[smalllabel,anchor=west] at (-1.82,1.28) {$Q_{g_j}=0$};
\node[smalllabel,anchor=west] at (1.12,-1.42) {$Q_h=0$};
\node[smalllabel] at (-1.15,-1.18) {private patch\\$\Sigma_j$};

\node[panelnote] at (0,-2.55)
{Pole separation means that each sign-class has an open pole patch not shared by any other same-side sign-class.};
\end{scope}

\begin{scope}[xshift=11.90cm]
\node[paneltitle] at (0,2.72) {C. Monodromy};

\draw[->] (-2.15,0)--(2.15,0) node[right] {$\Re u$};
\draw[->] (0,-1.75)--(0,1.88) node[above] {$\Im u$};

\draw[red!75!black,very thick]
  plot[smooth] coordinates {(-1.35,-1.48) (-0.92,-0.86) (-0.50,-0.28) (-0.10,0.36) (0.35,1.42)};
\draw[line width=4.2pt,orange!85!black] (-0.63,-0.56)--(-0.28,-0.05);

\draw[black,thick,->] (-1.15,0.28) arc[start angle=165,end angle=-145,radius=.58];
\node[font=\scriptsize] at (-1.42,0.92) {$\gamma$};

\node[smalllabel,align=center,text width=2.95cm] at (1.18,1.18)
{$\log Q_{g_j}$ changes\\by $2\pi i$ around $\gamma$};
\node[smalllabel] at (-1.08,-1.10) {$\Sigma_j$};

\node[panelnote] at (0,-2.55)
{A loop around a private patch detects the matched term.  MLR then uses the Laurent description to identify the scalar argument up to sign.};
\end{scope}

\end{tikzpicture}
\caption{Softplus pole signatures used in the zippering proof.  Panel A shows the complex pole locations that define the outer factors $Q_g=1+e^{g^\C}$.  Panel B illustrates pole separation: a sign-class must have a private open pole patch not shared by other sign-classes in the same expansion.  Panel C illustrates the monodromy test used for cross-matching and regularity.  A small loop around a private patch changes $\log Q_{g_j}=\sigma(g_j^\C)$ by $2\pi i$; equality of two one-step expansions forces the same monodromy on the other side, and MLR is the condition that turns the resulting matched pole data into identification of the underlying scalar argument up to sign.}
\label{fig:sp-poles}
\end{figure}

\begin{definition}[Softplus pole terminology and minimality]
\label{def:sp-minimal}
Let \(g\) be a real-valued argument, and let \(g^\C\) denote a chosen analytic continuation on a common complex continuation domain.  Write
\[
 Q_g=1+e^{g^\C}.
\]
The zero set $\{Q_g=0\}$ is the \emph{outer Softplus pole divisor} of \(g\). (The word ``outer'' is used to signify that this divisor comes from the current Softplus factor $\sig(g^\C)=\log(1+e^{g^\C})$, not from singularities already inherited inside the lower-depth argument \(g^\C\). The term \emph{divisor} is used in the analytic geometry sense of being the zero set of a holomorphic factor $Q_g$.) An \emph{outer Softplus pole patch} of \(g\) is a connected smooth codimension-one patch $\Sigma\subset\{Q_g=0\}$ on which \(dQ_g\neq0\), or, equivalently, a small loop around \(\Sigma\) changes the branch of $\log Q_g=\sig(g^\C)$ by \(2\pi i\). Since
\[
 Q_{-g}=1+e^{-g^\C}=e^{-g^\C}Q_g,
\]
the arguments \(g\) and \(-g\) have the same outer Softplus pole divisor.  Thus pole patches naturally belong to the sign-class $[g]=\{g,-g\}$. A pole patch for \([g]\) is \emph{private} to an expansion if no different sign-class in the same expansion shares a nonempty open piece of that same smooth pole surface.  The expansion is \emph{pole-separated modulo sign} if every sign-class has such a private outer Softplus pole patch. For two sign-classes \([g]\) and \([h]\), we say that they are \emph{divisor-matched} if an outer Softplus pole patch of \([g]\) and an outer Softplus pole patch of \([h]\) share a nonempty open piece of the same smooth pole surface, or, equivalently, \([g]\) and \([h]\) have the same local Softplus pole signature on that patch.  When the two sign-classes lie on opposite sides of a null-pair equality, we also call this a \emph{cross-match}. A Softplus one-step expansion is \emph{minimal} if it is used in the sense of Definition~\ref{def:sp-used} and is pole-separated modulo sign.  Thus Softplus minimality is adjacent usedness plus the internal pole-separation condition needed to make each sign-class visible.
\end{definition}
Figure~\ref{fig:sp-poles} sketches the pole-separation and monodromy ideas implicit in this definition.  The key result of the softplus minimality definition is the following result:

\begin{lemma}[A private Softplus pole patch has a unique cross match]
\label{lem:sp-crossmatch}
Suppose
\[
 b+\sum_{j=1}^kv_j\sig(g_j)
 =\widetilde b+\sum_{i=1}^{\widetilde k}\widetilde v_i\sig(\widetilde g_i)
\]
on $\Omega$.  If both expansions are used and individually pole-separated modulo sign, every private pole patch on one side is shared by a unique sign-class on the other side.
\end{lemma}

\begin{proof}
We prove the statement for a private pole patch on the left; the argument with the two sides interchanged is identical.
Choose a private pole patch $\Sigma$ for the sign-class $[g_j]$.  Since
\[
 Q_{-g_j}=1+e^{-g_j^\C}=e^{-g_j^\C}(1+e^{g_j^\C})=e^{-g_j^\C}Q_{g_j},
\]
$g_j$ and $-g_j$ have the same outer pole divisor.  Thus, after choosing the representative if necessary, we may write
\[
 \Sigma\subset\{Q_j=0\},
 \qquad
 Q_j:=Q_{g_j}=1+e^{g_j^\C},
 \qquad
 dQ_j\neq0\text{ on }\Sigma.
\]

Pick a point $p\in\Sigma$ generically: avoid singular points of all pole divisors and avoid all lower-dimensional intersections of $\Sigma$ with other pole divisors.  More explicitly, if a same-side sign-class $[g_a]\neq[g_j]$ does not share a relatively open subset of $\Sigma$, then choose $p$ outside
\[
 \Sigma\cap \{Q_{g_a}=0\}.
\]
Do the same for the right-side pole divisors that do not contain a relatively open subset of $\Sigma$.  Since there are only finitely many divisors, this removes only a finite union of lower-dimensional analytic subsets of $\Sigma$.

Let $U$ be a small complex neighborhood of $p$ such that $U\cap\Sigma$ is smooth and connected.  Because $dQ_j(p)\neq0$, $Q_j$ is a local defining equation for $\Sigma$.  Choose a small loop $\gamma$ in $U\setminus\Sigma$ linking $\Sigma$ once positively.  Equivalently, in local coordinates one may take
\[
 Q_j(\gamma(\theta))=\varepsilon e^{i\theta},
 \qquad 0\le \theta\le 2\pi,
\]
with the remaining coordinates fixed.  For any locally nonzero holomorphic factor $Q$, define its monodromy around $\gamma$ by
\[
 \Delta_\gamma\log Q
 :=\text{analytic continuation of }\log Q\text{ along }\gamma-\log Q
 =\int_\gamma \frac{dQ}{Q}.
\]
Then
\[
 \Delta_\gamma\log Q_j=2\pi i.
\]
By the choice of $p$ and $U$, every other left-side sign-class has no pole in $U$; hence, for $a\neq j$,
\[
 \Delta_\gamma\log Q_{g_a}=0.
\]

Suppose, for contradiction, that no right-side sign-class shares a relatively open subset of $\Sigma$.  Then our generic choice of $p$ and smallness of $U$ ensure that every right-side factor $Q_{\widetilde g_i}$ is also nonzero on $U$, so
\[
 \Delta_\gamma\log Q_{\widetilde g_i}=0
 \qquad\text{for all }i.
\]
Analytically continue the null-pair identity
\[
 b+\sum_a v_a\log Q_{g_a}
 =
 \widetilde b+\sum_i\widetilde v_i\log Q_{\widetilde g_i}
\]
around $\gamma$ and subtract the original branch.  The constants have zero monodromy, so the left side changes by
\[
 \sum_a v_a\Delta_\gamma\log Q_{g_a}=2\pi i\,v_j,
\]
while the right side changes by
\[
 \sum_i\widetilde v_i\Delta_\gamma\log Q_{\widetilde g_i}=0.
\]
Thus equality would force $2\pi i\,v_j=0$, contradicting usedness, which gives $v_j\neq0$.

Therefore some right-side sign-class $[\widetilde g_i]$ must share a relatively open subset of $\Sigma$; this is the desired cross-match.  It is unique: if two different right-side sign-classes shared a relatively open subset of the same patch, then they would share an open pole patch with each other, contradicting right-side pole separation modulo sign.  Reducedness of the used expansion then prevents two distinct terms from representing the same sign-class.  Hence the private left patch has a unique right-side match.  Repeating the same argument with left and right interchanged gives the converse statement for private right-side patches.
\end{proof}

\paragraph{ReLU minimality and private facet cross-matching.}
\begin{figure}[t]
\centering
\begin{tikzpicture}[x=1cm,y=1cm]
\tikzset{
  panelbox/.style={rounded corners,draw=black!55,fill=gray!6,line width=.45pt},
  paneltitle/.style={font=\bfseries\small,align=center,text width=4.45cm},
  smalllabel/.style={font=\scriptsize,fill=white,inner sep=1.4pt,align=center},
  panelnote/.style={font=\scriptsize,align=center,text width=4.35cm}
}

\begin{scope}
\node[paneltitle] at (0,2.45) {A. Task-visible fresh facet};

\draw[panelbox] (-2.15,-1.35) rectangle (2.15,1.35);
\node[font=\scriptsize,anchor=north west] at (-2.00,1.19) {cell $C$};

\draw[blue!70!black,very thick]
  (-1.72,-0.95) -- (1.62,0.92);
\draw[line width=4.2pt,orange!85!black]
  (-0.42,-0.22) -- (0.42,0.25);

\node[smalllabel] at (-1.22,0.63) {$g_j<0$};
\node[smalllabel] at (1.15,-0.68) {$g_j>0$};
\node[smalllabel] at (0.02,0.50) {$F_j\subset\{g_j=0\}$};

\draw[->,thick] (-0.52,-0.73) -- (-0.12,-0.36);
\node[smalllabel] at (-0.85,-0.88) {crossed};
\draw[->,thick] (0.92,-1.04) -- (1.42,-1.04);
\node[smalllabel,anchor=east] at (2.00,-1.04) {$v_j\neq0$};

\node[panelnote] at (0,-2.38)
{A counted ReLU unit has a fresh zero facet inside an inherited affine cell, the task crosses it, and the next layer reads out the kink.};
\end{scope}

\begin{scope}[xshift=5.55cm]
\node[paneltitle] at (0,2.45) {B. Internal nonredundancy};

\draw[panelbox] (-2.15,-1.35) rectangle (2.15,1.35);
\node[font=\scriptsize,anchor=north west] at (-2.00,1.19) {cell $C$};

\draw[blue!70!black,very thick]
  (-1.72,-0.84) -- (1.58,0.78);
\draw[orange!85!black,dashed,very thick]
  (-1.72,-0.84) -- (1.58,0.78);

\draw[green!50!black,very thick]
  (-1.42,-1.05) -- (1.42,-0.55);

\node[smalllabel] at (-1.02,0.72) {$g$};
\node[smalllabel] at (0.43,0.02) {$h\sim_\pm g$};
\node[smalllabel] at (0.58,-0.95) {affine residual};
\node[smalllabel] at (0.20,1.03) {$\rhoR(g)-\rhoR(-g)=g$};

\node[panelnote] at (0,-2.38)
{Sign-scale copies are not distinct axes.  Opposite ReLU pairs can also collapse to an affine term, so affine sign-cancellations are excluded.};
\end{scope}

\begin{scope}[xshift=11.10cm]
\node[paneltitle] at (0,2.45) {C. Internal exposure};

\draw[panelbox] (-2.15,-1.35) rectangle (2.15,1.35);
\node[font=\scriptsize,anchor=north west] at (-2.00,1.19) {cell $C$};

\draw[blue!70!black,very thick]
  (-1.72,-0.92) -- (1.60,0.88);
\draw[line width=4.2pt,orange!85!black]
  (-0.42,-0.22) -- (0.42,0.23);

\draw[purple!75!black,very thick]
  (-1.42,0.90) -- (1.48,-0.72);

\draw[red!75!black,thick] (0.88,0.48) -- (1.08,0.68);
\draw[red!75!black,thick] (1.08,0.48) -- (0.88,0.68);

\node[smalllabel] at (-0.88,-0.78) {private patch\\$F_j$};
\node[smalllabel] at (1.10,-0.96) {other trace};
\node[smalllabel,align=center,text width=2.55cm] at (0.80,1.05)
{no different same-side unit shares an open piece};

\node[panelnote] at (0,-2.38)
{Each marked fresh facet must contain an open patch not shared by any genuinely different same-expansion argument.};
\end{scope}

\end{tikzpicture}
\caption{EFF-minimality for ReLU one-step expansions.  Panel A illustrates task-visible nonlinear use: inside an inherited affine cell, the argument has a fresh zero facet, the task crosses it, and the outgoing coefficient is nonzero.  Panel B illustrates internal nonredundancy: sign-scale duplicates are not counted as separate axes, and opposite pairs can leave only an affine residual through $\rhoR(g)-\rhoR(-g)=g$.  Panel C illustrates internal exposure: a marked fresh facet must have a private open patch not shared by another genuinely different same-side argument.}
\label{fig:relu-eff-minimality}
\end{figure}
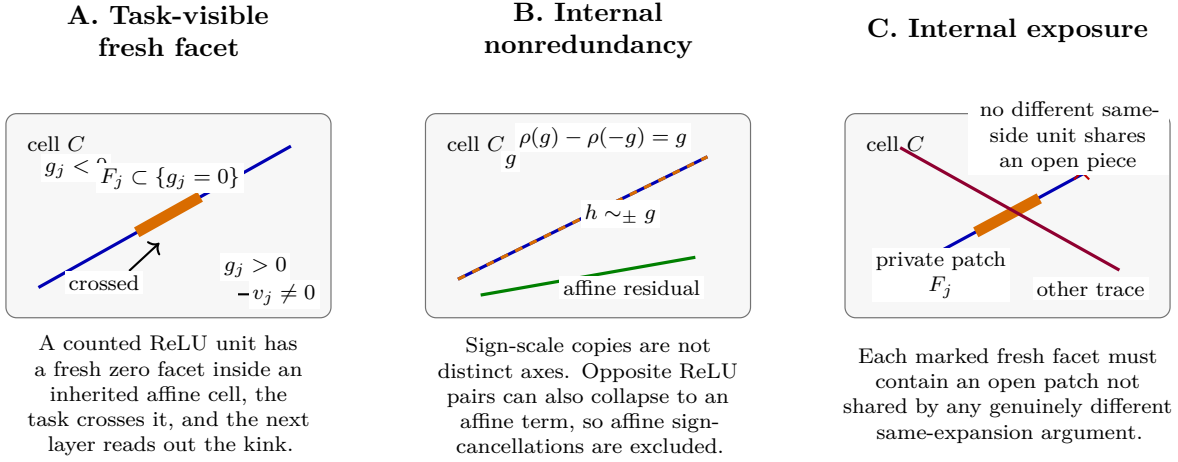

We now develop a ReLU analog to the minimality concept and results of the previous subsection. In a deep one-step expansion the scalar arguments are continuous and piecewise affine functions $g_j=z_{r,j}(x)$, so the crossed trace must be strengthened by a crossed \emph{fresh facet} inside a lower-layer affine cell (Fig \ref{fig:relu-eff-minimality}).  The exposed-facet part plays the same role as Softplus pole separation: it makes the nonlinear signature of one unit private rather than confused with another unit's signature.

\begin{definition}[ReLU facet terminology]
\label{def:relu-facet-terminology}
A scalar function is \emph{CPWA} if it is continuous and piecewise affine.  For a finite family of scalar CPWA functions, a \emph{common inherited affine complex} is a finite polyhedral subdivision $\calP$ of $\Omega$ on whose full-dimensional cells every function is affine.  This complex records the kink structure inherited from lower layers.  For scalar functions $g,h$, write $h\sim_+g$ if $h=\alpha g$ for some $\alpha>0$, and $h\sim_\pm g$ if $h=\eps\alpha g$ for $\eps\in\{\pm1\}$ and $\alpha>0$.  If $g|_C(x)=a_C^\top x+c_C$ on a full-dimensional inherited cell $C$, a \emph{fresh zero facet} is a nonempty relatively open codimension-one patch
\[
 F\subset\relint(C)\cap\{g=0\}
\]
with $a_C\neq0$ and both signs of $g$ present in $C$.  The facet is fresh because the kink of $\rhoR(g)$ across $F$ is created by the current outer ReLU, not inherited from $\partial C$.  A finite list of such facets for $g_j$ is a \emph{marked fresh-facet atlas}. In analogy with the softplus private pole, a marked facet is \emph{private}, or \emph{internally exposed}, if no other same-expansion argument vanishes on an open subset of it unless that argument is globally sign-scale equivalent to $g_j$.
\end{definition}

We can now define ReLU minimality, again by adding to the usedness condition a signature separation condition, here a played by fresh facets:
\begin{definition}[Minimal ReLU one-step expansion]
\label{def:relu-eff}
A ReLU one-step expansion is \emph{Exposed-Fresh-Face (EFF)-minimal}, if it satisfies the following three groups of conditions.
\begin{enumerate}[label=(\Alph*)]
\item \textbf{Task-visible nonlinear use.}  Every counted term has $v_j\neq0$, and every argument $g_j$ has at least one marked fresh zero facet.  This is the deep version of adjacent ReLU usedness: the task crosses a newly created ReLU boundary and the next layer reads out the resulting kink.
\item \textbf{Internal nonredundancy.}  Distinct same-side arguments are not sign-scale equivalent, and the expansion is affine sign-noncancelling:
\[
 \inf_{a\in\R^m}\left\|\sum_{j\in S}v_jg_j-a\right\|_{L^2}>0
\]
for every nonempty $S$.  This excludes duplicate axes and the affine identity $\rhoR(g)-\rhoR(-g)=g$.
\item \textbf{Internal signature separation.}  The marked fresh-facet atlases are internally exposed: no genuinely different same-expansion argument shares an open piece of a marked fresh facet.  This is the ReLU analogue of Softplus pole separation.
\end{enumerate}
\end{definition}
\noindent As with the softplus case, the purpose of EFF-minimality is expressed by the following cross-match lemma. 

\begin{lemma}[A private marked facet has a unique cross match]
\label{lem:relu-crossmatch}
Suppose
\[
 b+\sum_jv_j\rhoR(g_j)
 =\widetilde b+\sum_i\widetilde v_i\rhoR(\widetilde g_i).
\]
If both expansions are EFF-minimal, every private marked facet has a unique cross-side gate carrying the same open fresh zero patch.
\end{lemma}

\begin{proof}
We prove the claim for a private marked facet on the left; the converse follows by exchanging the two expansions.
Refine the inherited affine complex, if necessary, so that all left- and right-side arguments are affine on each full-dimensional cell of the refinement.  Let $F$ be a private marked facet of the left-side argument $g_j$, and let $C$ be a full-dimensional inherited cell containing $F$ in its relative interior.  On $C$ write
\[
 g_j(x)=a_j^\top x+c_j,
 \qquad a_j\neq0,
 \qquad
 F\subset \relint(C)\cap\{a_j^\top x+c_j=0\}.
\]
Because $F$ is fresh, both signs of $g_j$ occur in $C$.  Choose a normal vector $n$ with
\[
 a_j^\top n>0.
\]
For $p\in F$ generic and $t>0$ small, the points $p+tn$ and $p-tn$ lie in the two sides of $F$ inside $C$.

Choose $p$ away from lower-dimensional intersections with all other marked facets and with any right-side zero facet that does not contain a relatively open subset of $F$.  Near such a $p$, all left-side terms except the $j$th are affine through $F$.  Define the one-sided derivative matrices of the represented left function
\[
 Y_L(x):=b+\sum_a v_a\rhoR(g_a(x))
\]
by
\[
 D^+Y_L(p):=\lim_{t\downarrow0}D Y_L(p+tn),
 \qquad
 D^-Y_L(p):=\lim_{t\downarrow0}D Y_L(p-tn).
\]
For the $j$th term,
\[
 D^+\bigl(v_j\rhoR(g_j)\bigr)(p)-D^-\bigl(v_j\rhoR(g_j)\bigr)(p)
 =v_j a_j^\top,
\]
up to the harmless sign determined by the choice of $n$.  All other left terms have the same derivative on the two sides of $F$, so
\[
 D^+Y_L(p)-D^-Y_L(p)=v_j a_j^\top.
\]
This rank-one matrix is nonzero because $v_j\neq0$ by task-visible nonlinear use and $a_j\neq0$ by freshness.

The two one-step expansions represent the same function, so for
\[
 Y_R(x):=\widetilde b+\sum_i\widetilde v_i\rhoR(\widetilde g_i(x))
\]
we must also have
\[
 D^+Y_R(p)-D^-Y_R(p)=v_j a_j^\top\neq0.
\]
Suppose, for contradiction, that no right-side argument vanishes on a relatively open subset of $F$.  By the generic choice of $p$, no right-side zero set passes through $p$, and after the common refinement all right-side arguments are affine on the cell containing the two small one-sided neighborhoods.  Hence each sign $\operatorname{sgn}(\widetilde g_i)$ is locally constant across $F$ near $p$, so each term $\widetilde v_i\rhoR(\widetilde g_i)$ is affine with the same derivative on the two sides.  Therefore
\[
 D^+Y_R(p)-D^-Y_R(p)=0,
\]
contradicting the nonzero jump above.

Thus at least one right-side argument $\widetilde g_i$ has a zero facet containing a relatively open subset of $F$.  This gives a cross-side gate carrying the same open fresh zero patch.  The match is unique: if two different right-side arguments vanished on an open subset of that patch, internal exposure on the right would force them to be globally sign-scale equivalent, and the internal nonredundancy condition rules out two distinct same-side terms in the same sign-scale class.  Reversing the roles of the two expansions proves the same statement for private right-side facets.
\end{proof}

\subsubsection{Softplus regularity and identifiability}
\label{sec:softplus-regularity}

Softplus minimality gives unique cross-matches of private pole patches.  The remaining question is whether a shared pole patch determines the underlying scalar argument -- that is identiability.  Ensuring this happens is where regularity enters.  Specifically, we introduce a relationship between monodromy and an algebraic object called the Laurent ring.  The point of this formalism is to build a bridge from complex-analytic information like monodromy to elementary algebra in a Laurent polynomial ring.  The basic intuition for doing this is that the bad objects we need to rule out (nontrivial collisions) can be characterized by algebraic conditions on the ring of polynomials containing (analytic continuations of) softplus expansions. Essentially, we find that collisions are non-primitive polynomials in the Laurent polynomial ring, while good objects (non-collisions) are the primitive monomials in this ring.  Analytically, a collision is a pair of different expansions whose pole data is not enough to distinguish them.  Algebraically, the same phenomenon appears as torsion, composite exponents, or unexpected common factors among Laurent binomials.  The condition that such things not happen is regularity -- in this case, for softplus, \emph{Monodromy-Laurent Regularity} (MLR).   Having introduced the MLR concept, we show that it leads to identifiability as desired.  For background on the idea of a complex-analytic/algebraic geometry interface and on Laurent monomial coordinates, see \citep{GriffithsHarris1978Principles, CoxLittleSchenck2011Toric}.

\paragraph{From monodromy to Laurent coordinates.}

A Softplus layer contains outer factors
\[
 Q_g=1+e^{g^\C}.
\]
When we analytically continue around a lower-layer pole loop, logarithmic pieces inside $g^\C$ may change by multiples of $2\pi i$.  Consequently, $e^{g^\C}$ may come back multiplied by a nonzero complex scalar.  We record these scalar multipliers by the concept of a \emph{character} in a \emph{monodromy coordinate chart}. 

\begin{definition}[Monodromy coordinates and characters]
\label{def:monodromy-character}
Let $\widehat\calU$ be a domain on which the exponentials $e^{g_\alpha}$ are single-valued and meromorphic.  Number the poles of softplus arguments $1, \ldots, p$.  Then let $\tau_1, \ldots, \tau_p$ denote the analytic continuations associated with each pole. For $n=(n_1,\ldots,n_p)\in\Z^p$, write the expression
\[
\tau^n:=\tau_1^{n_1}\cdots\tau_p^{n_p}.
\]
This means a path going around the first selected loop $n_1$ times, the second $n_2$ times, and so on. We will call this fixed choice of continuation domain and selected loops a \emph{monodromy coordinate chart}.  It is just the local ``coordinate system'' in which we compare the monodromy of all the relevant arguments. With the monodromy coordinate chart fixed, an argument $g_\alpha$ has \emph{monodromy character} $chi_\alpha:\Z^p\to\C^*$ if $\tau^n(e^{g_\alpha})=\chi_\alpha(n)e^{g_\alpha}$ for all $n\in\Z^p.$ Equivalently, if $\tau_a(e^{g_\alpha})=\lambda_{\alpha,a}e^{g_\alpha}$ for $\lambda_{\alpha,a}\in\C^*$ then
\[
\chi_\alpha(n)=\prod_{a=1}^p\lambda_{\alpha,a}^{\,n_a}.
\]
\end{definition}
\noindent Characters are multiplicative:
\[
 \chi_\alpha(n+n')=\chi_\alpha(n)\chi_\alpha(n').
\]
For an affine-Softplus layer, we can compute the characters explicitly (the so-called ``character rule''). Suppose
\[
 g=\beta+\sum_aw_aH_a,
 \qquad H_a=\log Q_a.
\]
If continuation by $\tau^n$ sends
\[
 H_a\longmapsto H_a+2\pi i\,\nu_a(n),
\]
then
\[
 e^g
 \longmapsto
 \exp\!\left(2\pi i\sum_aw_a\nu_a(n)\right)e^g.
\]
Thus
\[
 \chi_g(n)=
 \exp\!\left(2\pi i\sum_aw_a\nu_a(n)\right).
\]
Affine weights are therefore monodromy exponents. Exact integer or rational relations among these exponents are the source of the ``bad'' character resonances causing collision that we exclude below with the regularity concept. To get there, first we set up some basic properties of characters. 
\begin{lemma}[Independence of distinct characters]
\label{lem:sp-artin}
Let $f_1,\ldots,f_m$ be nonzero meromorphic functions in a fixed monodromy setup.  Suppose each $f_a$ has character $\chi_a$, meaning
\[
 \tau^n f_a=\chi_a(n)f_a
 \qquad\forall n\in\Z^p.
\]
If the characters $\chi_1,\ldots,\chi_m$ are pairwise distinct, then $f_1,\ldots,f_m$ are linearly independent over $K$.
\end{lemma}

\begin{proof}
Suppose
\[
 \sum_{a=1}^m c_af_a=0,
 \qquad c_a\in K,
\]
is a nontrivial relation with the smallest possible number of nonzero terms.  After dividing by one nonzero coefficient, assume $c_1=1$.  Since the characters are distinct, choose $n\in\Z^p$ with
\[
 \chi_1(n)\neq\chi_2(n).
\]
Applying $\tau^n$ to the relation gives
\[
 \sum_a c_a\chi_a(n)f_a=0,
\]
because every coefficient $c_a\in K$ is unchanged by continuation.  Subtracting $\chi_1(n)$ times the original relation gives
\[
 \sum_a c_a\bigl(\chi_a(n)-\chi_1(n)\bigr)f_a=0.
\]
The $a=1$ term is eliminated, while the $a=2$ coefficient remains nonzero.  This is a shorter nontrivial relation, contradicting minimality.  Hence no such relation exists.
\end{proof}

The previous lemma says that different characters behave like independent monomials.  We now make that statement algebraic.  Let $\Gamma=\langle \chi_\alpha\rangle$ be the finitely generated group of characters produced by the arguments under comparison.  When $\Gamma$ has no torsion, choose basic characters $\gamma_1,\ldots,\gamma_s$ freely generating $\Gamma$, and choose meromorphic functions $T_1,\ldots,T_s$ whose characters are $\gamma_1,\ldots,\gamma_s$.  These $T_a$'s may be taken locally as products and quotients of the $e^{g_\alpha}$'s corresponding to a chosen character basis.  Then every $e^{g_\alpha}$ can be written uniquely as
\[
 e^{g_\alpha}=c_\alpha T^{m_\alpha},
 \qquad
 c_\alpha\in K^*,
 \qquad
 m_\alpha\in\Z^s,
\]
where $T^{m_\alpha}=T_1^{m_{\alpha,1}}\cdots T_s^{m_{\alpha,s}}$.

Thus, the exponentials $e^{g_\alpha}$ become invariant coefficients times Laurent monomials.   Let $K$ be the set of meromorphic functions that do not change under continuation loops:
\[
 K=
 \{f\in\calM(\widehat\calU):\tau^n f=f\ \text{for every }n\in\Z^p\}.
\]
Then the algebra generated by these exponentials is identified the \emph{Laurent polynomial ring}
\[
 K[T_1^{\pm1},\ldots,T_s^{\pm1}],
\]
whose elements are finite $K$-linear combinations of Laurent monomials.  The point of these coordinates is that each Softplus outer factor $1+e^g$ becomes a Laurent with two terms (a \emph{binomial}):

\begin{definition}[Laurent binomial and primitive exponent]
A \emph{Laurent binomial} is an expression
\[
 B_{c,m}(T)=1+cT^m,
 \qquad c\in K^*,\quad m\in\Z^s\setminus\{0\}.
\]
In our setting this is exactly the outer Softplus factor $1+e^g$ after writing $e^g=cT^m$.

The exponent $m$ is \emph{primitive} if it is not a nontrivial integer multiple of another lattice vector:
\[
 m\neq q m_0
 \qquad\text{for every integer }|q|>1.
\]
Equivalently,
\[
 \gcd(|m_1|,\ldots,|m_s|)=1.
\]
Nonprimitive exponents are what allow composite resonances such as
\[
 1+Z^3=(1+Z)(1-Z+Z^2).
\]
The primitive-exponent condition rules out this kind of hidden factorization.
\end{definition}

\begin{lemma}[Primitive Laurent binomials are prime]
\label{lem:sp-prime}
If $m$ is primitive, then $1+cT^m$ is prime in $K[T_1^{\pm1},\ldots,T_s^{\pm1}]$.
\end{lemma}

\begin{proof}
A primitive vector $m\in\Z^s$ can be completed to an integral basis
\[
 e'_1=m,
 \qquad
 e'_2,\ldots,e'_s
\]
of $\Z^s$.  Equivalently, there is a determinant $\pm 1$ (``unimodular'') matrix in $GL_s(\Z)$ whose first basis vector is $m$.  The associated monomial change of variables is a Laurent-ring isomorphism obtained by setting
\[
 S_1=T^{e'_1}=T^m,
 \qquad
 S_\nu=T^{e'_\nu}\quad(\nu=2,\ldots,s).
\]
Because $m$ is unimodular, the Laurent monomial $T^u$ is uniquely expressible in the $S_\nu$ coordinates, and conversely.  Hence
\[
 K[T_1^{\pm1},\ldots,T_s^{\pm1}]
 \cong
 K[S_1^{\pm1},\ldots,S_s^{\pm1}],
\]
and $1+cT^m$ becomes $1+cS_1$.

The quotient by $1+cS_1$ is obtained by setting $S_1=-c^{-1}$, which is a unit because $c\in K^*$.  Therefore
\[
 K[S_1^{\pm1},\ldots,S_s^{\pm1}]/(1+cS_1)
 \cong
 K[S_2^{\pm1},\ldots,S_s^{\pm1}],
\]
an integral domain.  Hence the ideal generated by $1+cT^m$ is prime, and the binomial is a prime element of the unique-factorization Laurent ring.
\end{proof}

\begin{lemma}[Associates of two-term Laurent binomials]
\label{lem:sp-associates}
Two Laurent polynomials are \emph{associates} if they differ by multiplication by a Laurent unit $aT^k$.  If
\[
 1+dT^n=aT^k(1+cT^m),
\]
then either
\[
 n=m,
 \qquad d=c,
 \qquad k=0,
 \qquad a=1,
\]
or
\[
 n=-m,
 \qquad d=c^{-1},
 \qquad k=n,
 \qquad a=d.
\]
For real-valued arguments this means $h=g$ or $h=-g$.
\end{lemma}

\begin{proof}
Expanding the right side gives
\[
 aT^k+acT^{k+m}.
\]
Laurent monomials are linearly independent over $K$, so the two-element support sets must agree:
\[
 \{0,n\}=\{k,k+m\}.
\]
If $k=0$ and $k+m=n$, then $m=n$ and coefficient comparison gives $a=1$ and $c=d$.  If $k=n$ and $k+m=0$, then $n=-m$ and coefficient comparison gives $a=d$ and $dc=1$.  These are the only two matchings of the two support sets.  Returning to $e^g=cT^m$ and $e^h=dT^n$ gives $e^h=e^g$ in the first case and $e^h=e^{-g}$ in the second; real-valuedness on $\Omega$ gives $h=g$ or $h=-g$.
\end{proof}

The character and primitivity conditions rule out the algebraic resonances inside the Laurent ring itself: torsion characters, commensurate characters, and composite binomials such as $1+Z^3=(1+Z)(1-Z+Z^2)$.  There remains one further way a divisor match could be misleading (e.g. be a nontrivial collision): two distinct Laurent polynomials might appear to share a pole patch only after being pulled back through a degenerate lower-feature map.  Laurent divisor-faithfulness rules out this pullback accident: a shared analytic pole patch must come from a genuine common Laurent factor before pullback.

\begin{definition}[Laurent divisor-faithfulness] Suppose $e^g=cT^m$ and $e^h=dT^n$ are two coordinates.  Let $\Sigma$ be a shared outer pole patch, meaning a connected patch on which both $1+e^g$ and $1+e^h$ vanish. The pair is \emph{Laurent divisor-faithful} at $\Sigma$ if the two Laurent binomials $1+cT^m$ and $1+dT^n$ have a common nonunit factor in the Laurent ring
\[
 K[T_1^{\pm1},\ldots,T_s^{\pm1}].
\]
In words, the shared pole patch must come from a real shared Laurent factor, not merely from a collapse caused by the lower-feature map.
\end{definition}

The genericity section \S\ref{sec:genericity}  gives a rank criterion for verifying the faithfulness condition is a real-world reasonable assumption (Lemma~\ref{lem:sp-rank}).  For now, the exact zippering theorem below will simply assume Laurent divisor-faithfulness as part of regularity:

\begin{definition}[Monodromy--Laurent regularity]
\label{def:MLR-unified}
A collision pair $(g,h)$, together with a shared outer pole patch $\Sigma$, is \emph{monodromy--Laurent regular} (MLR) if:
\begin{enumerate}[label=(R\arabic*)]
\item the character group generated by the compared arguments has no torsion;
\item the Laurent exponent vectors are primitive;
\item the pair is Laurent divisor-faithful at $\Sigma$.
\end{enumerate}
Thus MLR says that the monodromy data can be represented by true Laurent monomials, the relevant Softplus factors do not secretly factor through a composite exponent, and any shared pole patch reflects a genuine common Laurent divisor.  The ``shared pole patch'' in (R3) is exactly the shared outer pole patch $\Sigma$ on which $[g]$ and $[h]$ collide.
\end{definition}

All of this setup culminates in the key identifiability result:
\begin{theorem}[Softplus one-step identifiability under MLR]
\label{thm:sp-onestep}
Suppose
\[
 b+\sum_{j=1}^kv_j\sig(g_j)
 =\widetilde b+\sum_{i=1}^{\widetilde k}\widetilde v_i\sig(\widetilde g_i)
\]
on $\Omega$.  Assume both one-step Softplus expansions are minimal, and assume that every cross-matched pair supplied by Lemma~\ref{lem:sp-crossmatch} is MLR on its selected shared pole patch.  Then $k=\widetilde k$ and, after a permutation,
\[
 \widetilde g_j=g_j,
 \qquad
 \widetilde v_j=v_j,
 \qquad
 \widetilde b=b.
\]
\end{theorem}

\begin{proof}
Lemma~\ref{lem:sp-crossmatch} gives a unique cross-match for every private sign-class on either side.  Since every sign-class in a minimal expansion has a private pole patch, these cross-matches define a bijection between the left and right sign-classes.  After reindexing, the matched pairs are $[g_j]$ and $[\widetilde g_j]$.

Fix one such matched pair and let $\Sigma_j$ be the selected shared outer pole patch witnessing the match.  By MLR, the monodromy setup is torsion-free, so after choosing Laurent coordinates we may write
\[
 e^{g_j}=c_jT^{m_j},
 \qquad
 e^{\widetilde g_j}=d_jT^{n_j},
\]
with $m_j,n_j\in\Z^s$ primitive.  The associated outer Softplus factors are the Laurent binomials
\[
 P_j(T)=1+c_jT^{m_j},
 \qquad
 R_j(T)=1+d_jT^{n_j}.
\]
The shared pole patch means that the pullbacks of $P_j$ and $R_j$ vanish on the same regular analytic hypersurface patch.  Laurent divisor-faithfulness says that this shared pulled-back divisor is not merely an accident of the lower-feature map: $P_j$ and $R_j$ have a common nonunit factor in the Laurent ring.

By primitivity, Lemma~\ref{lem:sp-prime} says that both $P_j$ and $R_j$ are prime.  Two prime Laurent binomials with a common nonunit factor must be associates, so
\[
 1+d_jT^{n_j}=aT^k(1+c_jT^{m_j})
\]
for some Laurent unit $aT^k$.  Lemma~\ref{lem:sp-associates} leaves only two possibilities: either $n_j=m_j$ and $d_j=c_j$, or $n_j=-m_j$ and $d_j=c_j^{-1}$.  In the first case $e^{\widetilde g_j}=e^{g_j}$; in the second $e^{\widetilde g_j}=e^{-g_j}$.

On the real task domain $g_j$ and $\widetilde g_j$ are real-valued.  Hence $e^{\widetilde g_j}=e^{g_j}$ implies $\widetilde g_j-g_j\in2\pi i\Z$ and therefore $\widetilde g_j=g_j$, while $e^{\widetilde g_j}=e^{-g_j}$ implies $\widetilde g_j+g_j\in2\pi i\Z$ and therefore $\widetilde g_j=-g_j$.  Thus
\[
 \widetilde g_j=\eps_jg_j,
 \qquad \eps_j\in\{+1,-1\}.
\]

The coefficient vectors also match.  Write the two equal represented functions near the private pole patch as
\[
F_L=b+\sum_k v_k \log(1+e^{g_k}),
\qquad
F_R=\tilde b+\sum_i \tilde v_i \log(1+e^{\tilde g_i}).
\]
Choose a small loop $\gamma$ around a point where $1+e^{g_j}=0$, and where no other left-side term has such a zero. Then every left-side term except $j$ is single-valued around $\gamma$, while
\[
\log(1+e^{g_j})
\longmapsto
\log(1+e^{g_j})+2\pi i.
\]
Thus, after analytic continuation once around $\gamma$,
\[
F_L \longmapsto F_L+2\pi i\,v_j.
\]
Because this private pole patch is matched on the right side, the same loop winds around exactly one right-side term, namely $\tilde g_{\pi(j)}$, so
\[
F_R \longmapsto F_R+2\pi i\,\tilde v_{\pi(j)}.
\]
Since $F_L=F_R$ as analytically continued functions, their changes around the same loop must be equal:
\[
2\pi i\,v_j=2\pi i\,\tilde v_{\pi(j)}.
\]
Hence $\tilde v_{\pi(j)}=v_j$.

It remains only to eliminate negative orientations.  Let $S=\{j:\eps_j=-1\}$.  After cancelling every positively oriented matched term and using the coefficient equality, the null relation becomes
\[
 b-\widetilde b+
 \sum_{j\in S}v_j\bigl[\sig(g_j)-\sig(-g_j)\bigr]=0.
\]
The Softplus identity $\sig(t)-\sig(-t)=t$ turns this into
\[
 b-\widetilde b+
 \sum_{j\in S}v_jg_j=0.
\]
If $S$ were nonempty, this would say that a nonempty downstream-weighted sum of arguments is constant, contradicting sign-noncancellation.  Hence $S=\varnothing$.  All arguments and coefficients agree after the permutation, and the remaining equality of constants gives $b=\widetilde b$.
\end{proof}

\subsubsection{ReLU regularity and identifiability}
\label{sec:relu-regularity}

Here we introduce the ReLU analog of the regularity concept and prove the corresponding identifiability result.  The reader will probably feel that the math in this section is easier (or at any rate, more elemantary) than in the previous section. The price we pay for that is that the assumptions are stronger. Because of the lack of analytic structure for ReLU, there is no way we know of to translate ReLU expansions into characters and apply any kind of algebraic reasoning.  As a result, the arguments are in a sense simpler but require more blunt force (as expressed by needing \emph{multi-facet signatures} and an explicit assumption of \emph{coherence}, defined below). 

Because of the nature of ReLU, the regularity arguments below (and the associated genericity arguments in \S\ref{sec:genericity}) necessarily end up having to study sets that are not smooth everywhere (such as the ReLU collision sets).  A \emph{stratum} is one smooth piece of such a set. A \emph{stratification} is a decomposition into smooth pieces, usually locally finite so that only finitely many pieces meet any small neighborhood.  We use strata so that tangent spaces and dimension counts are well-defined piece by piece. Thus an ``$X$-stratum'' means a smooth piece of the set $X$.

We can now define Multi-Facet Regularity, the ReLU analog of MLR: 

\begin{definition}[Multi-facet signature and MFR]
\label{def:relu-signature}
\label{def:MFR-unified}
Fix a ReLU one-step comparison and a common inherited affine complex
\(\calP\).  An \emph{inherited cell} means a full-dimensional cell
\(C\in\calP\), so that every argument is affine on \(C\). Choose marked fresh-facet slots
\[
 F_\nu\subset C_\nu,
 \qquad \nu=1,\ldots,L.
\]
For a scalar argument \(g\), write its affine formula on the inherited cell
\(C_\nu\) as
\[
 g|_{C_\nu}(x)=a_\nu^\top x+c_\nu,
 \qquad
 q_\nu(g):=(a_\nu,c_\nu).
\]
The zero facet is unchanged if \(q_\nu(g)\) is multiplied by a nonzero scalar, so
we record only the scalar-multiple class $[q_\nu(g)]$. Then the \emph{multi-facet signature} of \(g\) is the finite list
\[
 \Sigma_r(g)
 =
 \bigl(C_\nu,\nu,[q_\nu(g)]\bigr)_{\nu=1}^L.
\]
The entries remember which inherited cell and which marked-facet slot they came
from.  The signature \emph{identifies} \(g\) if, for any other scalar argument \(h\),
\[
 \Sigma_r(h)=\Sigma_r(g)
 \quad\Longrightarrow\quad
 h=\eps\alpha g
 \quad\text{for some }\eps\in\{\pm1\},\ \alpha>0.
\]
In words, the marked facets determine the whole argument, up to the unavoidable ReLU sign and positive-scale ambiguity.
In a null-pair equality, lemma \ref{lem:relu-crossmatch} matches individual fresh facets across the two sides.  The matching is \emph{coherent} for a pair \((g,h)\) if all marked facets of \(g\) are matched to marked facets of the same single competitor \(h\), and conversely.  A coherently matched pair \((g,h)\) is \emph{multi-facet regular} (MFR), if the coherently matched facet records have the same multi-facet signature and that signature identifies the argument:
\[
 \Sigma_r(h)=\Sigma_r(g)
 \quad\Longrightarrow\quad
 h=\eps\alpha g,
 \qquad
 \eps\in\{\pm1\},\quad \alpha>0.
\]
\end{definition}
\noindent With this definition in mind, we can get the needed identifiability result:

\begin{theorem}[ReLU identifiability under MFR]
\label{thm:relu-onestep}
Suppose
\[
 b+\sum_{j=1}^kv_j\rhoR(g_j)
 =\widetilde b+\sum_{i=1}^{\widetilde k}\widetilde v_i\rhoR(\widetilde g_i)
\]
on $\Omega$.  Assume both expansions are EFF-minimal, the induced cross-facet matching is coherent, and every coherently matched scalar pair is MFR.  Then $k=\widetilde k$ and, after a permutation, there are $\alpha_j>0$ such that
\[
 \widetilde g_j=\alpha_jg_j,
 \qquad
 \alpha_j\widetilde v_j=v_j,
 \qquad
 \widetilde b=b.
\]
\end{theorem}

\begin{proof}
By Lemma~\ref{lem:relu-crossmatch}, private facets match across the equality.  Coherence says that all marked facets of a given $g_j$ match one competitor $\widetilde g_{\pi(j)}$.  MFR then reconstructs the whole scalar argument from its labelled signature, giving
\[
 \widetilde g_{\pi(j)}=\eps_j\alpha_jg_j,
 \qquad \eps_j\in\{\pm1\},\quad \alpha_j>0.
\]
Reducedness makes this matching bijective.

Now compare the derivative jumps across a marked facet of $g_j$.  If $\eps_j=+1$, then $\rhoR(\widetilde g_{\pi(j)})=\alpha_j\rhoR(g_j)$ near the facet.  If $\eps_j=-1$, then the active side is reversed, but the jump magnitude is still $\alpha_j\widetilde v_{\pi(j)}a_j^\top$.  Equality of the represented functions therefore gives
\[
 v_ja_j^\top=\alpha_j\widetilde v_{\pi(j)}a_j^\top,
\]
and since $a_j\neq0$, $v_j=\alpha_j\widetilde v_{\pi(j)}$.

Let $S=\{j:\eps_j=-1\}$.  After substituting the matched coefficients and cancelling the positively oriented terms, the equality becomes
\[
 b-\widetilde b+\sum_{j\in S}v_j\bigl[\rhoR(g_j)-\rhoR(-g_j)\bigr]=0.
\]
Using $\rhoR(t)-\rhoR(-t)=t$, this reduces to
\[
 b-\widetilde b+\sum_{j\in S}v_jg_j=0.
\]
Affine sign-noncancellation forces $S=\varnothing$, and then the remaining constant terms give $b=\widetilde b$.
\end{proof}

\begin{remark}
To reiterate, while this math is shorter and more elementary than in the previous softplus section, the MFR definition is basically clobbering things with its multi-face and coherence conditions, whereas in the MLR case, we can prove many things without having to assume them.  The saving grace is that coherent MFR can be shown to be generic (as seen in \S\ref{sec:genericity}) and supports the same genericity reasoning as the softplus condition, leading to the powerful activation-independent null-net theorem \ref{thm:generic-null-common}. 
\end{remark}

\subsubsection{Minimal null pairs and exact zippering}
\label{sec:exact-zip}

We now return to the organizing strategy stated at the beginning of the section.  The Softplus and ReLU arguments have supplied the activation-specific inputs: minimality gives unique private cross-matches, and regularity makes those matches identify the scalar arguments.  The remaining step is activation-independent.  At a backward zippering step, two one-step expansions represent the same next-layer function; that is, they form a null pair.  The following corollary records exactly what the two activation-specific identifiability theorems have proved: such a null pair determines the previous-layer arguments after the only unavoidable hidden-unit matching, namely permutation for Softplus and permutation plus positive rescaling for ReLU.

\begin{corollary}[Activation-specific one-step identifiability]
\label{thm:pointwise-null-common}
Let a minimal null pair satisfy the relevant regularity hypothesis: MLR on every Softplus cross-match, or coherent MFR on every ReLU cross-match.  Then the two expansions are the same after the activation's unavoidable hidden-unit matching.  Concretely, after a permutation,
\[
 \widetilde g_j=g_j,
 \qquad
 \widetilde v_j=v_j,
 \qquad
 \widetilde b=b
\]
in the Softplus case, while in the ReLU case there are \(\alpha_j>0\) such that
\[
 \widetilde g_j=\alpha_jg_j,
 \qquad
 \alpha_j\widetilde v_j=v_j,
 \qquad
 \widetilde b=b.
\]
\end{corollary}

\begin{proof}
This is just the combination of the two activation-specific identifiability theorems: Theorem~\ref{thm:sp-onestep} for Softplus and Theorem~\ref{thm:relu-onestep} for ReLU.
\end{proof}

\begin{theorem}[Abstract exact zippering]
\label{thm:abstract-exact-zip}
Let $A,B$ be same-depth, same-width affine-$\varphi$ networks, and let $q<s$.  Suppose terminal weak equivalence holds:
\[
 z_s^B=E_sz_s^A
\]
for an injective affine map $E_s$.  Assume that at each backward step $r=s-1,\ldots,q$, the two one-step expansions compared at layer $r$ satisfy the appropriate hypotheses of Corollary~\ref{thm:pointwise-null-common}: Softplus minimality plus MLR, or ReLU EFF-minimality plus coherent MFR.  Then alignment zippers upstream.  For every $r=q,\ldots,s-1$, the layer-$r$ arguments are exactly matched: in the Softplus case, after a permutation,
\[
 z_{r,\pi_r(j)}^B=z_{r,j}^A,
\]
and in the ReLU case, after a permutation and positive scales $\alpha_{r,j}>0$,
\[
 z_{r,\pi_r(j)}^B=\alpha_{r,j}z_{r,j}^A.
\]
\end{theorem}

\begin{proof}
At the top step, expand $z_s^B$ and $E_sz_s^A$ through layer $s-1$:
\begin{align*}
 z_s^B&=b_s^B+\sum_iW_s^B[:,i]\varphi(z_{s-1,i}^B),\\
 E_sz_s^A&=T_sb_s^A+a_s+\sum_jT_sW_s^A[:,j]\varphi(z_{s-1,j}^A).
\end{align*}
These are two equal one-step expansions.  Corollary~\ref{thm:pointwise-null-common} applies, so the arguments at layer $s-1$ match: by a permutation for Softplus, and by a permutation together with positive coordinatewise scales for ReLU.  Let $E_{s-1}$ denote the corresponding concrete coordinate map; then
\[
 z_{s-1}^B=E_{s-1}z_{s-1}^A.
\]

Now assume inductively that layer $r+1$ has already been matched, so
\[
 z_{r+1}^B=E_{r+1}z_{r+1}^A
\]
for the coordinate map obtained in the previous step.  Expanding both sides through layer $r$ gives two equal one-step expansions.  By the recursive hypothesis, they satisfy the relevant Softplus or ReLU assumptions, so Corollary~\ref{thm:pointwise-null-common} again matches their scalar arguments.
Thus, layer $r$ is aligned in the same activation-specific sense.
Repeating this argument from $s-1$ down to $q$ proves the theorem.
\end{proof}

\begin{remark}[Unequal depths and counted widths.]\label{rem:unequal-depths-widths}
We note that the same-depth, same-width hypotheses in
Theorem~\ref{thm:abstract-exact-zip}, and the common layer indexing in
Theorems~\ref{thm:common-asymp-zip}--\ref{thm:common-soft-zip}, are mainly
notational to keep proofs straightforward, but they hold without loss of generality for unequal depths and widths. Namely, more generally, the proofs can use only a paired sequence of one-step comparison
problems
\[
a_q<a_{q+1}<\cdots<a_s,
\qquad
b_q<b_{q+1}<\cdots<b_s,
\]
with the same number of backward zippering steps. Replacing layer $r$ by the
paired layers $(a_r,b_r)$ gives the same induction, with constants $L_r$
attached to the paired step $(a_r,b_r)\to(a_{r+1},b_{r+1})$.

For unequal raw widths, the statement should be read in terms of the counted
minimal one-step expansions. If extra units are unused, constant, duplicate,
or otherwise outside the counted minimal expansion, they may be deleted or
absorbed as in the shrinkability reductions. On the remaining counted axes,
the exact theorem matches the two minimal expansions after the unavoidable
activation symmetry. Equivalently, in a directional A-to-B formulation, one may
write an injective matching
\[
\pi_r:J^{A,\varphi}_{a_r}\hookrightarrow J^{B,\varphi}_{b_r},
\]
provided the unmatched B-side coordinates are not additional counted minimal
terms. Otherwise those terms are genuine extra nonlinear signatures and must be
matched, controlled, or the comparison direction/axis set must be changed.

With this notation, the soft bound has the same form, except that $d_r$ should
be read as the number of counted A-side axes at the paired layer,
\[
d^A_{a_r}:=|J^{A,\varphi}_{a_r}|,
\]
and the product of inverse constants is taken over the paired zippering steps:
\[
\AxisAlign^{\theta}_{a_r,b_r,\varphi}(A,B)
\ge
1-
\frac{
e_s^2\prod_{t=r}^{s-1}L_t^2
}{
(1-\theta)^2c^2 d^A_{a_r}
}.
\]
This reduces to our above theorems when the paired layers have a common indexing and common counted width.
This same notational flexibility is distinct from, and weaker than, the
depth-warped result of Appendix~\S\ref{sec:depth-warped-block-weak-strong},
where a preferred axis may correspond to a window of nearby layers rather than
to one fixed layer.
\end{remark}

\begin{corollary}[Exact MLR Softplus zippering]
If every backward comparison is individually minimal and every cross-matched scalar pair is MLR, then terminal weak equivalence zippers upstream into exact permutation alignment.
\end{corollary}

\begin{proof}
Theorem~\ref{thm:sp-onestep} supplies Softplus one-step identifiability at each backward step, so Theorem~\ref{thm:abstract-exact-zip} applies.
\end{proof}

\begin{corollary}[Exact MFR ReLU zippering]
If every backward comparison is EFF-minimal, coherently matched, and MFR, then terminal weak equivalence zippers upstream into alignment by permutation plus positive coordinatewise rescaling.
\end{corollary}

\begin{proof}
Theorem~\ref{thm:relu-onestep} supplies ReLU one-step identifiability at each step.
\end{proof}

\subsection{Quantitative minimality, Jacobian regularity, and soft zippering}
\label{sec:soft-zippering}

\subsubsection{Quantitative minimality}

The exact theory in \S\ref{sec:exact-zippering} uses qualitative minimality: each counted unit must leave a private nonlinear signature, and MLR or MFR must make the resulting cross-match determine the scalar argument.  For asymptotic and soft zippering we need the same conditions with positive margins.  A convergent sequence of comparison problems should not be able to lose a counted unit in the limit, for example by letting an outgoing coefficient vanish, making an argument constant, merging two arguments, creating an affine sign-cancellation, or losing the private pole or facet certificate.  The definitions below are exactly these positive-margin versions of the A.4 minimality conditions.  Their role is to produce compact comparison families whose exact null limits still satisfy the qualitative A.4 hypotheses.  The same margins also keep the denominators in the soft axis scores from degenerating.

\begin{definition}[Quantitative Softplus minimality]\label{def:quant-sp-min}
A $k$-term Softplus expansion is $(M,c)$-minimal if its finite-dimensional realization parameters are bounded by $M$ and:
\begin{enumerate}[label=(S\arabic*)]
\item $\|v_j\|\ge c$ and $\Var(g_j)\ge c^2$;
\item $\|g_i-g_j\|_{L^2},\|g_i+g_j\|_{L^2}\ge c$ for $i\neq j$;
\item every nonempty signed subset has affine-cancellation distance at least $c$;
\item every sign-class has a regular pole certificate separated from the other same-expansion pole sets by at least $c$.
\end{enumerate}
\end{definition}

\begin{definition}[Quantitative ReLU EFF-minimality]\label{def:quant-relu-eff}
A normalized $k$-term ReLU expansion is $(M,c)$-EFF-minimal if its parameters are bounded by $M$, every argument is normalized by $\|g_j\|_{L^2}=1$, and:
\begin{enumerate}[label=(R\arabic*)]
\item $\|v_j\|\ge c$ and both $\|\rhoR(g_j)\|_{L^2}$ and $\|\rhoR(-g_j)\|_{L^2}$ are at least $c$;
\item distinct arguments are separated modulo sign and positive scale by at least $c$;
\item every nonempty signed subset has affine-cancellation distance at least $c$;
\item each marked facet atlas has quantitative lower bounds on slope, patch size, two-sided crossing, and separation from other same-side fresh traces.
\end{enumerate}
\end{definition}

\begin{lemma}[Quantitative minimality gives compact comparison families]
\label{lem:quant-min-compact}
For a fixed finite-dimensional architecture chart, fixed term count, and fixed $M,c$, each quantitative family above is compact after fixing the finite hidden-unit order.  Moreover, every limit point remains qualitatively minimal in the sense used in \S\ref{sec:exact-zippering}.  In particular, the counted arguments have score denominators bounded away from zero uniformly in the family: in the Softplus case $\|g_j\|_{L^2}\ge c$, and in the normalized ReLU case $\|g_j\|_{L^2}=1$.
\end{lemma}

\begin{proof}
The parameter bound gives a compact ambient set.  Every quantitative condition is a closed inequality on the chosen analytic chart or fixed ReLU activation-pattern stratum, meaning a smooth piece where the relevant ReLU cell and facet pattern is fixed.  Imposing all of these inequalities selects a closed subset of that compact set.  Choosing one ordering of the finitely many hidden units preserves compactness.  For ReLU we also normalize each argument as in the preceding definition, so multiplying an argument by a positive constant is not treated as a distinct point.

The positive margins are precisely what keeps limits from falling out of the qualitative A.4 hypotheses.  If a coefficient vanished, an argument became constant, two arguments became duplicate or opposite, a signed affine cancellation appeared, or a private pole or facet certificate disappeared, then one of the displayed margins would have to go to zero.  Since all margins remain at least $c$, every limit point remains qualitatively minimal.  Finally, $\Var(g_j)\ge c^2$ implies $\|g_j\|_{L^2}\ge c$ for Softplus, while ReLU arguments are normalized to have $L^2$ norm one.
\end{proof}

\subsubsection{Minimality, regularity, and asymptotic zippering}

Throughout this subsection and the next, the phrase \emph{the appropriate activation-specific quantitative minimality and regularity hypotheses} means the following.  First, the comparison pair lies in one of the compact quantitative-minimal families supplied by Lemma~\ref{lem:quant-min-compact}.  Second, any exact null pair that appears as a limit inside that family satisfies the corresponding A.4 regularity condition: MLR on the exact cross-matched pole pairs for Softplus, and coherent MFR on the exact cross-matched facet pairs for ReLU.  This is not a new assumption; it is just the way the A.4 exact one-step identifiability results are made stable under compact limits.

We will measure hidden-argument error with one notation for both activations.

\begin{definition}[Matched-argument distance]
\label{def:darg}
Let $G$ and $\widetilde G$ be two one-step expansions, and write $g$ and $\widetilde g$ for their scalar argument vectors. The \emph{allowed matching matrices} are $E=P$ in the Softplus case, where $P$ is a permutation matrix, and $E=DP$ in the ReLU case, where $P$ is a permutation matrix and $D$ is a positive diagonal matrix. For such an $E$, define
\begin{equation}
 d_{\arg,\varphi}(G,\widetilde G;E)
 =
 \|\widetilde g-Eg\|_{L^2(\Omega;\ell_2)}.
 \label{eq:darg-definition}
\end{equation}
When the matching is not specified, set 
$$d_{\arg,\varphi}(G,\widetilde G) = \inf_{\text{allowed }E}d_{\arg,\varphi}(G,\widetilde G;E).$$ 
\end{definition}
For axis-count statements in what follows, use the soft-axis score from Definition~\ref{def:soft-axis-score} applied to the layer-$r$ zippering match.  The only change is that in the ReLU zippering score the scale parameter is restricted to $\alpha>0$, because exact ReLU zippering only allows positive rescaling.  

The next lemma says that quantitative minimality leads output function distance to be uniformly continuous in  configuration space distance:
\begin{lemma}[Uniform inverse error bound from minimality and regularity]
\label{lem:common-inverse-error}
Fix $\varphi\in\{\sig,\rhoR\}$ and let $\calK_r^\varphi$ be a compact family of one-step comparison pairs $(G,\widetilde G)$ satisfying the appropriate activation-specific quantitative minimality and regularity hypotheses.  Then there is a nondecreasing error-bound function $\omega_r^\varphi(t)\to0$ as $t\downarrow0$ such that, for every $(G,\widetilde G)\in\calK_r^\varphi$ with
\[
 \|Y_G^\varphi-Y_{\widetilde G}^\varphi\|_{L^2}\le t,
\]
there is a matching map of the allowed form $E$ satisfying
\[
 d_{\arg,\varphi}(G,\widetilde G;E)
 \le\omega_r^\varphi(t).
\]
\end{lemma}

\begin{proof}
We first show that hidden-argument error must go to zero whenever the represented-function error goes to zero.  Suppose not.  Then there are pairs $(G_n,\widetilde G_n)\in\calK_r^\varphi$ and a number $\eta>0$ such that
\[
 \|Y_{G_n}^\varphi-Y_{\widetilde G_n}^\varphi\|_{L^2}\longrightarrow0,
 \qquad
 d_{\arg,\varphi}(G_n,\widetilde G_n)\ge\eta.
\]
By compactness, after passing to a subsequence,
\[
 (G_n,\widetilde G_n)
 \longrightarrow
 (G_*,\widetilde G_*)\in\calK_r^\varphi.
\]
Continuity of the realization map gives $Y_{G_*}^\varphi=Y_{\widetilde G_*}^\varphi$, so the limit is an exact null pair.  The hypotheses of the family are exactly the activation-specific hypotheses needed to apply Corollary~\ref{thm:pointwise-null-common}: Softplus minimality plus MLR, or ReLU EFF-minimality plus coherent MFR.  Therefore the limiting null pair has the activation-specific one-step identifiability form.  Concretely, there is a matching map of the allowed form $E_*$ such that
$\widetilde g_*=E_*g_*$, and hence
\[
 d_{\arg,\varphi}(G_*,\widetilde G_*;E_*)=0.
\]
Using the same nearby matching for $n$ large, and in the ReLU case using the nearby positive rescalings supplied by the local coordinates, gives
\[
 d_{\arg,\varphi}(G_n,\widetilde G_n)\longrightarrow0,
\]
contradicting $d_{\arg,\varphi}(G_n,\widetilde G_n)\ge\eta$.
Now define
\[
 \omega_r^\varphi(t)
 =
 \sup\bigl\{d_{\arg,\varphi}(G,\widetilde G):
 (G,\widetilde G)\in\calK_r^\varphi,
 \|Y_G^\varphi-Y_{\widetilde G}^\varphi\|_{L^2}\le t
 \bigr\},
\]
with $\omega_r^\varphi(t)=0$ if the set is empty.  The preceding paragraph says exactly that $\omega_r^\varphi(t)\to0$ as $t\downarrow0$.  Replacing it by $\sup_{0\le s\le t}\omega_r^\varphi(s)$ if necessary makes it nondecreasing.  By the definition of $d_{\arg,\varphi}$, after increasing $\omega_r^\varphi(t)$ harmlessly on $t>0$ if needed, a matching map of the allowed form $E$ can be chosen with the displayed bound.
\end{proof}
\noindent The above result enables us to prove an asymptotic version of zippering:

\begin{theorem}[Asymptotic zippering]
\label{thm:common-asymp-zip}
Fix $\varphi\in\{\sig,\rhoR\}$ and $q<s$.  Let $(A_n,B_n,E_{s,n})$ be comparison problems with terminal error
\[
 e_{s,n}=\|z_{s,n}^B-E_{s,n}z_{s,n}^A\|_{L^2}\to0.
\]
Assume that, at each backward step $r=s-1,\ldots,q$, the one-step comparison pairs encountered in the zippering induction lie in a fixed compact family $\calK_r^\varphi$ satisfying the appropriate activation-specific quantitative minimality and regularity hypotheses.  Then, for every $r=q,\ldots,s-1$ and every fixed $\theta<1$,
\[
 \AxisAlign_{r,\varphi}^{\mathrm{zip},\theta}(A_n,B_n)\longrightarrow1,
\]
and in fact this quantity equals $1$ for all sufficiently large $n$.
\end{theorem}

\begin{proof}
We prove the statement by moving one layer at a time from $s$ back to $q$.

At the top step, compare the transformed A-side expansion of $E_{s,n}z_{s,n}^A$ with the B-side expansion of $z_{s,n}^B$.  These are the represented functions of
\[
 E_{s,n}G_{s-1,n}^A
 \qquad\text{and}\qquad
 G_{s-1,n}^B.
\]
Their represented-function distance is exactly the terminal mismatch:
\[
 \|Y_{E_{s,n}G_{s-1,n}^A}^\varphi
      -Y_{G_{s-1,n}^B}^\varphi\|_{L^2}
 =
 \|E_{s,n}z_{s,n}^A-z_{s,n}^B\|_{L^2}
 =e_{s,n}.
\]
Lemma~\ref{lem:common-inverse-error} applied at layer $s-1$ gives a matching map of the allowed form $E_{s-1,n}$ with
\[
 e_{s-1,n}
 :=d_{\arg,\varphi}
 \bigl(E_{s,n}G_{s-1,n}^A,G_{s-1,n}^B;E_{s-1,n}\bigr)
 \le
 \omega_{s-1}^\varphi(e_{s,n}).
\]
By the definition of $d_{\arg,\varphi}$, this is the same as
\[
 \|z_{s-1,n}^B-E_{s-1,n}z_{s-1,n}^A\|_{L^2}
 \le
 \omega_{s-1}^\varphi(e_{s,n}).
\]
Since $e_{s,n}\to0$ and $\omega_{s-1}^\varphi(t)\to0$, we get $e_{s-1,n}\to0$.

Now suppose the construction has reached layer $r+1$, so that a matching map of the allowed form $E_{r+1,n}$ has been found and
\[
 e_{r+1,n}
 :=d_{\arg,\varphi}
 \bigl(E_{r+2,n}G_{r+1,n}^A,G_{r+1,n}^B;E_{r+1,n}\bigr)
 \to0.
\]
Equivalently,
\[
 \|z_{r+1,n}^B-E_{r+1,n}z_{r+1,n}^A\|_{L^2}=e_{r+1,n}.
\]
The B-side one-step expansion through layer $r$ represents $z_{r+1,n}^B$, while the transformed A-side expansion $E_{r+1,n}G_{r,n}^A$ represents $E_{r+1,n}z_{r+1,n}^A$.  Hence
\[
 \|Y_{E_{r+1,n}G_{r,n}^A}^\varphi
      -Y_{G_{r,n}^B}^\varphi\|_{L^2}
 =e_{r+1,n}.
\]
Applying Lemma~\ref{lem:common-inverse-error} at layer $r$ gives a matching map of the allowed form $E_{r,n}$ with
\[
 e_{r,n}
 :=d_{\arg,\varphi}
 \bigl(E_{r+1,n}G_{r,n}^A,G_{r,n}^B;E_{r,n}\bigr)
 \le
 \omega_r^\varphi(e_{r+1,n}).
\]
Because $e_{r+1,n}\to0$, the right-hand side tends to zero.  Iterating this estimate gives the displayed composition of the error-bound functions.

It remains to translate matched-argument closeness into the terms of the axis alignment score function.  Fix a layer $r$.  The preceding paragraph gives
\[
 d_{\arg,\varphi}
 \bigl(E_{r+1,n}G_{r,n}^A,G_{r,n}^B;E_{r,n}\bigr)
 \longrightarrow0.
\]
Thus every matched coordinate error in \eqref{eq:darg-definition} tends to zero.  In the Softplus case, this is ordinary coordinatewise convergence after the identified permutation.  In the ReLU case, the same positive rescaling used in the matched-argument distance is allowed in the zippering score.  Therefore the score of every recovered matched pair tends to $1$.  Since there are only finitely many coordinates, every score is eventually above any fixed $\theta<1$.  Hence $\AxisAlign_{r,\varphi}^{\mathrm{zip},\theta}(A_n,B_n)=1$ for all sufficiently large $n$.
\end{proof}

\subsubsection{Jacobian regularity and soft zippering}

Ideally, we can go further than just having an asymptotic result and actually figure out how much error upstream is bounded by downstream error.  The missing ingredient in doing this is controlling the ill-conditioning of the Jacobian of the realization map from configurations to outputs.  Specifically:

\begin{definition}[Expansion map and derivative Gram matrix]
\label{def:jacobian-regularity}
Let $\Theta_r^\varphi$ be a local finite-dimensional chart of one-step expansions after fixing the hidden-unit order.  In the ReLU case, each argument is normalized in the chart by imposing $\|g_j\|_{L^2}=1$; this simply avoids listing the same ReLU unit many times after multiplying its argument by a positive constant and dividing its outgoing coefficient by the same constant.  Define the realization map as
$$F_r^\varphi(\vartheta)=Y_{G_\vartheta}^\varphi$$
so that its Jacobian is
$$J_r^\varphi(\vartheta)=DF_r^\varphi(\vartheta)^*DF_r^\varphi(\vartheta).$$  Then the expansion is \emph{Jacobian-regular} if
\[
 \lambda_{\min}(J_r^\varphi)>0.
\]
\end{definition}
\noindent For Softplus, we can explicitly calculate:
\[
 \delta Y=\delta b+\sum_j\delta v_j\sig(g_j)+\sum_jv_j\sig'(g_j)\delta g_j.
\]
For ReLU within an affine cell,
\[
 \delta Y=\delta b+\sum_j\delta v_j\rhoR(g_j)+\sum_jv_j\mathbf1_{\{g_j>0\}}\delta g_j.
\]
(For ReLU this derivative is taken after the normalization just described; otherwise the exact rescaling identity $(v,g)\mapsto(\alpha^{-1}v,\alpha g)$ would appear as a zero direction even though it does not change the unit.)

With this in mind we can actually bound the output-configuration relationship in an explicit way:

\begin{lemma}[Quantitative one-step error inversion]
\label{lem:common-localinverse}
Let $\calK_r^\varphi$ be a compact comparison family satisfying the appropriate activation-specific quantitative minimality and regularity hypotheses.  Let $G_r$ be a one-step expansion in this family, and suppose that $G_r$ is Jacobian-regular in a local chart.  Then there are $\eps_r>0$ and $L_r<\infty$ such that
\[
 \|Y_{G_r}^\varphi-Y_{\widetilde G_r}^\varphi\|_{L^2}\le\eps_r
\]
implies a matching map of the allowed form $E$ satisfying
\[
 d_{\arg,\varphi}(G_r,\widetilde G_r;E)
 \le
 L_r\|Y_{G_r}^\varphi-Y_{\widetilde G_r}^\varphi\|_{L^2}.
\]
\end{lemma}

\begin{proof}
The proof has two parts.  First, compactness and the A.4 exact theory show that a sufficiently close near-collision must use the same local matching as $G_r$.  Second, the Jacobian lower bound gives a linear estimate inside that local chart.

\emph{Step 1: localization in one local chart.}
Choose a local chart $U$ around $G_r$ small enough that its hidden-unit matching is fixed.  For ReLU, keep the arguments normalized throughout $U$ by dividing each argument by its $L^2$ norm and putting the scale into the outgoing coefficient.  Let $d_{\mathrm{full}}$ be any metric inducing the topology of the compact comparison family and measuring all expansion data, including coefficients and bias.  The compactness argument from Lemma~\ref{lem:common-inverse-error} gives
\[
 \|Y_G^\varphi-Y_{\widetilde G}^\varphi\|_{L^2}\to0
 \quad\Longrightarrow\quad
 d_{\mathrm{full}}(G,\widetilde G)\to0.
\]
Indeed, otherwise a convergent subsequence would limit to an exact null pair whose hidden arguments are not identified by the allowed activation-specific matching, contradicting Corollary~\ref{thm:pointwise-null-common}.  Therefore there is $\eps_r^{(0)}>0$ such that
\[
 \|Y_{G_r}^\varphi-Y_{\widetilde G_r}^\varphi\|_{L^2}
 \le\eps_r^{(0)}
\]
implies that $\widetilde G_r$ lies inside $U$, after the unique nearby permutation and, in the ReLU case, after the same simple normalization of scalar arguments.

\emph{Step 2: a lower singular-value estimate in coordinates.}
Let $\vartheta_0$ be the coordinate of $G_r$ and write the coordinate of $\widetilde G_r$ as $\vartheta_0+h$.  Set
\[
 \sigma_r
 =\sqrt{\lambda_{\min}(J_r^\varphi(\vartheta_0))}>0.
\]
Since $J_r^\varphi=DF_r^\varphi{}^*DF_r^\varphi$, this means
\[
 \|DF_r^\varphi(\vartheta_0)h\|_{L^2}
 \ge\sigma_r\|h\|
 \qquad\text{for every coordinate tangent }h.
\]
Continuity of the derivative lets us shrink $U$ so that
\[
 \|DF_r^\varphi(\vartheta)-DF_r^\varphi(\vartheta_0)\|_{\mathrm{op}}
 \le\frac{\sigma_r}{2}
 \qquad\forall\vartheta\in U.
\]
Take $\eps_r\le\eps_r^{(0)}$ small enough that the segment $\vartheta_0+th$, $0\le t\le1$, remains in $U$.  The integral form of Taylor's theorem gives
\begin{align*}
 F_r^\varphi(\vartheta_0+h)-F_r^\varphi(\vartheta_0)
 &=DF_r^\varphi(\vartheta_0)h\notag\\
 &\quad+\int_0^1
 \bigl[DF_r^\varphi(\vartheta_0+th)-DF_r^\varphi(\vartheta_0)\bigr]h\,dt.
\end{align*}
Using the lower singular-value estimate at $\vartheta_0$ and the derivative-variation bound,
\[
 \|F_r^\varphi(\vartheta_0+h)-F_r^\varphi(\vartheta_0)\|_{L^2}
 \ge \sigma_r\|h\|
 -\int_0^1\frac{\sigma_r}{2}\|h\|\,dt\notag\\
 \ge\frac{\sigma_r}{2}\|h\|.
\]
Thus
\[
 \|h\|
 \le\frac{2}{\sigma_r}
 \|Y_{G_r}^\varphi-Y_{\widetilde G_r}^\varphi\|_{L^2}.
\]
\noindent On the fixed local chart, the map from coordinates to the matched scalar arguments is locally Lipschitz.  For ReLU, the positive scale allowed in $E$ is just the scale recorded by the local coordinates before normalization.  Therefore, for some $C_r^{\arg}<\infty$,
\[
 d_{\arg,\varphi}(G_r,\widetilde G_r;E)
 \le C_r^{\arg}\|h\|
\]
for the fixed local matching.  Taking
\[
 L_r=\frac{2C_r^{\arg}}{\sigma_r}
\]
gives the displayed estimate.
\end{proof}

With that explicit error bounding at each single step, we can state a quantitative soft zippering theorem:

\begin{theorem}[Soft zippering]
\label{thm:common-soft-zip}
Fix $\varphi\in\{\sig,\rhoR\}$ and $q<s$.  Assume every backward comparison encountered in the zippering induction lies in compact families satisfying the appropriate activation-specific quantitative minimality and regularity hypotheses, and assume the realized B-side expansion at each step is Jacobian-regular, with constants $L_r,\eps_r$.  If the terminal error $e_s=\|z_s^B-E_sz_s^A\|_{L^2}$ satisfies
\[
 e_s\prod_{t=r+1}^{s-1}L_t\le\eps_r
 \qquad(r=q,\ldots,s-1)
\]
for the Lemma \ref{lem:common-localinverse} constants $L_r, \eps_r$ for every $r=q,\ldots,s-1$, then
\[
\AxisAlign_{r,\varphi}^{\theta}(A,B)
 \ge 1-\frac{e_s^2\prod_{t=r}^{s-1}L_t^2}{(1-\theta)^2c^2d_r}
\]
where for softplus $c$ is the quantitative bound from definition \ref{def:quant-sp-min}, and for ReLU $c = 1$.
\end{theorem}

\begin{proof}
We carry out the backward induction explicitly.  For a one-step expansion $G_r^N$, recall that
\[
 Y_{G_r^N}^\varphi=z_{r+1}^N.
\]
If $E(y)=Ty+a$ is affine, then $EG_r^A$ denotes the transformed expansion obtained by replacing the bias by $Tb+a$ and every outgoing column by $Tv_j$; its represented function is $E z_{r+1}^A$.  The scalar arguments of $EG_r^A$ are still the layer-$r$ arguments $z_{r,j}^A$.

\emph{Step 1: the terminal backward step.}
The two layer-$(s-1)$ expansions are
\begin{align*}
 E_sz_s^A
 &=T_sb_s^A+a_s+\sum_jT_sW_s^A[:,j]\varphi(z_{s-1,j}^A),\\
 z_s^B
 &=b_s^B+\sum_iW_s^B[:,i]\varphi(z_{s-1,i}^B).
\end{align*}
Their represented-function distance is exactly
\[
 \|Y_{E_sG_{s-1}^A}^\varphi-Y_{G_{s-1}^B}^\varphi\|_{L^2}
 =\|E_sz_s^A-z_s^B\|_{L^2}=e_s.
\]
The small-error hypothesis for $r=s-1$ says $e_s\le\eps_{s-1}$.  Lemma~\ref{lem:common-localinverse} therefore gives a matching map of the allowed form $E_{s-1}$ such that
\[
 e_{s-1}:=
 d_{\arg,\varphi}
 \bigl(E_sG_{s-1}^A,G_{s-1}^B;E_{s-1}\bigr)
 \le L_{s-1}e_s.
\]
By \eqref{eq:darg-definition}, this is the same as
\[
 \|z_{s-1}^B-E_{s-1}z_{s-1}^A\|_{L^2}
 \le L_{s-1}e_s.
\]
This proves the desired bound at the top zippering step.

\emph{Step 2: the inductive step.}
Suppose that for some $r\le s-2$ a matching map of the allowed form $E_{r+1}$ has already been constructed and
\[
 e_{r+1}:=
 d_{\arg,\varphi}
 \bigl(E_{r+2}G_{r+1}^A,G_{r+1}^B;E_{r+1}\bigr)
 \le e_s\prod_{t=r+1}^{s-1}L_t.
\]
Equivalently,
\[
 \|z_{r+1}^B-E_{r+1}z_{r+1}^A\|_{L^2}=e_{r+1}.
\]
The B-side expansion of $z_{r+1}^B$ is
\[
 z_{r+1}^B=b_{r+1}^B+
 \sum_iW_{r+1}^B[:,i]\varphi(z_{r,i}^B),
\]
and the transformed A-side expansion is
\[
 E_{r+1}z_{r+1}^A=E_{r+1}b_{r+1}^A+
 \sum_jE_{r+1}W_{r+1}^A[:,j]\varphi(z_{r,j}^A),
\]
where the affine translation part of $E_{r+1}$ is included in the transformed bias.  These are precisely the represented functions of $G_r^B$ and $E_{r+1}G_r^A$, and their distance is $e_{r+1}$.  The assumed smallness condition gives
\[
 e_{r+1}
 \le e_s\prod_{t=r+1}^{s-1}L_t
 \le\eps_r.
\]
Applying Lemma~\ref{lem:common-localinverse} at layer $r$ produces a matching map of the allowed form $E_r$ and
\[
 e_r:=
 d_{\arg,\varphi}
 \bigl(E_{r+1}G_r^A,G_r^B;E_r\bigr)
 \le L_re_{r+1}
 \le e_s\prod_{t=r}^{s-1}L_t.
\]
Backward induction proves the $d_{\arg,\varphi}$ bound for every $r=q,\ldots,s-1$.

\emph{Step 3: convert the matched argument error into an axis count.}
Fix a layer $r$ and let $B_r^\theta$ be the number of coordinates that fail the zippering threshold under the matching supplied by $E_r$.  The induction gives
\[
 d_{\arg,\varphi}
 \bigl(E_{r+1}G_r^A,G_r^B;E_r\bigr)^2
 \le e_s^2\prod_{t=r}^{s-1}L_t^2.
\]

For Softplus, $E_r$ is a permutation.  If $\pi_r$ is that permutation, the preceding display reads
\[
 \sum_{j=1}^{d_r}
 \|z_{r,\pi_r(j)}^B-z_{r,j}^A\|_{L^2}^2
 \le e_s^2\prod_{t=r}^{s-1}L_t^2.
\]
If coordinate $j$ has score below $\theta$, then
\begin{align*}
 \|z_{r,\pi_r(j)}^B-z_{r,j}^A\|_{L^2}
 &>(1-\theta)
 \bigl(\|z_{r,\pi_r(j)}^B\|_{L^2}+\|z_{r,j}^A\|_{L^2}\bigr)\\
 &\ge(1-\theta)c,
\end{align*}
because quantitative nonconstancy gives $\|z_{r,j}^A\|_{L^2}\ge c$.  Every bad coordinate therefore contributes more than $(1-\theta)^2c^2$ to the total matched squared argument error.  Hence
\[
 B_r^\theta(1-\theta)^2c^2
 \le e_s^2\prod_{t=r}^{s-1}L_t^2.
\]
Since $\AxisAlign_{r,\sig}^{\mathrm{zip},\theta}=1-B_r^\theta/d_r$, this gives the Softplus bound.

For ReLU, write the positive scale in the matched coordinate as $\alpha_j>0$, so that the coordinate error controlled by $d_{\arg,\rhoR}$ is
\[
 \|z_{r,\pi_r(j)}^B-\alpha_jz_{r,j}^A\|_{L^2}.
\]
If a matched pair has positive-scale score below $\theta$, then this particular positive scale also fails the score test, so
\[
 \|z_{r,\pi_r(j)}^B-\alpha_jz_{r,j}^A\|_{L^2}
 >(1-\theta)
 \bigl(\|z_{r,\pi_r(j)}^B\|_{L^2}+\alpha_j\|z_{r,j}^A\|_{L^2}\bigr).
\]
For normalized ReLU arguments, $\|z_{r,\pi_r(j)}^B\|_{L^2}=1$ and $\|z_{r,j}^A\|_{L^2}=1$, so the right-hand side is $(1-\theta)(1+\alpha_j)\ge(1-\theta)$.  Thus every bad ReLU coordinate contributes more than $(1-\theta)^2$ to the total matched squared argument error, and
\[
 (1-\theta)^2B_r^\theta
 \le e_s^2\prod_{t=r}^{s-1}L_t^2.
\]
Dividing by $d_r$ gives the ReLU bound.  The stronger factor $4(1-\theta)^2$ would follow if the controlled error were the normalized comparison at the fixed scale $\alpha=1$.  Here, however, $d_{\arg,\rhoR}$ controls the error at the recovered positive scale $\alpha_j$, which is the scale needed for the zippering map itself.
\end{proof}

\subsection{Genericity}
\label{sec:genericity}

\subsubsection{The overall genericity strategy}
This section asks why the minimality, regularity, and Jacobian assumptions used above should hold for typical parameterized networks.    Propositions~\ref{prop:generic-sp-min} and~\ref{prop:generic-relu-eff} show that qualitative minimality fails only on lower-dimensional bad sets; Proposition~\ref{prop:generic-quant-min-margins} then upgrades this to positive $c$-margins with high probability on compact parameter families.  Those are the compact quantitative-minimal families used in Lemma~\ref{lem:common-inverse-error}.

We will establish that:
\begin{enumerate}
\item Intrinsic minimality fails only on analytic or algebraic degeneracy sets, conditional on task-usedness and task exposure;
\item The regularity mechanisms fail only on resonance, incidence, or rank-defect sets, provided the required condition holds at one point on each relevant stratum;
\item Jacobian regularity, after fixing the standard hidden-unit order and ReLU normalization, fails on the zero set of an analytic Gram determinant, provided the full-rank condition holds at least at one point.
\end{enumerate}
The first two give generic identifiability.  The third controls the generic geometry of exact null-pair collisions.  The generic null-net theorem below combines them and converts pairwise genericity into generic uniqueness of the set of parameters realizing the same function as an individual network.  The final positive-margin proposition then explains how the qualitative full-measure statements feed back into the quantitative hypotheses of \S\ref{sec:soft-zippering}.

\subsubsection{Genericity of minimality}
In this section, we show that qualitative softplus and ReLU minimality are generic, and that the quantitative versions are high probability conditional on qualitative minimality holding. 

\begin{proposition}[Generic qualitative Softplus minimality]
\label{prop:generic-sp-min}
Fix a set of Softplus expansions for which at least one private regular pole witness exists for each counted sign-class. Then failure of qualitative Softplus minimality is contained in a finite union of proper real-analytic subsets and has empty interior and measure zero.
\end{proposition}

\begin{proof}
Work on one connected analytic parameter chart and fix the finitely many marked pole neighborhoods used by the definition.  The elementary failure modes can be represented by nonnegative analytic defect functions.  For example, set
\begin{align*}
 U_j(\vartheta)&=\|v_j(\vartheta)\|^2,\\
 N_j(\vartheta)&=\|g_j(\vartheta)-\E_\Omega g_j(\vartheta)\|_{L^2}^2,\\
 D_{ij}^{\pm}(\vartheta)&=\|g_i(\vartheta)\pm g_j(\vartheta)\|_{L^2}^2.
\end{align*}
For a nonempty subset $S$, write
\[
 f_{S,\vartheta}=\sum_{j\in S}v_j(\vartheta)g_j(\vartheta),
 \qquad
 C_S(\vartheta)=\|f_{S,\vartheta}-\E_\Omega f_{S,\vartheta}\|_{L^2}^2.
\]
Then $U_j=0$ is unusedness, $N_j=0$ is constancy, $D_{ij}^{\pm}=0$ is duplication or opposition, and $C_S=0$ is affine sign-cancellation.  Bounded analytic dependence on the parameters and compactness of $\Omega$ permit differentiation under the integral, so these functions are real analytic.  The existence of one minimal witness ensure that none of the relevant defects vanishes identically.  Hence each corresponding bad set is a proper analytic zero set.

For the pole clause, let $Q_j=1+e^{g_j^{\C}}$ on the marked continuation chart.  Failure of regularity of the selected pole patch is the vanishing of the appropriate differential minors of $Q_j$.  Accidental sharing of a relatively open pole component by two distinct sign-classes is, after fixing the local incidence type, an analytic component-incidence condition.  The private-pole witness makes each such incidence set proper.  There are finitely many counted indices, pairs, subsets, and marked patches.  Their union therefore has empty interior and measure zero.
\end{proof}

\begin{proposition}[Generic qualitative ReLU EFF-minimality]
\label{prop:generic-relu-eff}
Fix a finite-dimensional ReLU cell/facet chart of one-step expansions that contains at least one EFF-minimal expansion.  Then the set of parameters in the chart for which the expansion is not EFF-minimal is contained in a finite union
of proper real-algebraic subsets, and thus has empty interior and measure zero in the chart.
\end{proposition}

\begin{proof}
On this fixed cell/facet chart, the lower-layer ReLU pattern does not change.  Hence there is a
fixed finite list of inherited cells and marked fresh-facet slots, and on each cell every counted
argument has the affine form
\[
 g_j^\vartheta|_C(x)=a_{j,C}(\vartheta)^\top x+c_{j,C}(\vartheta).
\]
The coefficient vectors \(a_{j,C}(\vartheta)\) and constants \(c_{j,C}(\vartheta)\) are polynomial in
the raw parameters, and real-analytic after the normalization used to remove positive ReLU scale.
Therefore all algebraic relations among the cellwise affine pieces are described by finitely many
polynomial, or after normalization real-analytic, equations.

We now record the possible failures of EFF-minimality in these coordinates.  First, task-visible
fresh-facet use can fail only if one of the selected fresh-facet certificates degenerates: for example,
the relevant affine normal vanishes, the marked zero facet ceases to be fresh inside the inherited
cell, or the fixed two-sided crossing condition collapses.  In the chosen chart these are detected by
finitely many polynomial or analytic certificate functions.

Second, internal exposure can fail if another counted argument shares an open piece of a marked
fresh facet without being globally sign-scale equivalent.  On the cell containing that facet, sharing
an open piece of the zero set forces proportionality of the two augmented affine coefficient vectors,
\[
 (a_{j,C}(\vartheta),c_{j,C}(\vartheta))
 \parallel
 (a_{i,C}(\vartheta),c_{i,C}(\vartheta)).
\]
This proportionality is exactly the vanishing of the corresponding \(2\times2\) minors.  Global
sign-scale duplication is also described by finitely many such proportionality equations, imposed
compatibly across all inherited cells.

Third, affine sign-cancellation can fail.  For a fixed nonempty subset \(S\) of counted terms, the
cellwise expression
\[
 \sum_{j\in S} v_j(\vartheta)g_j^\vartheta
\]
is constant on the task set only if, on every full-dimensional inherited cell \(C\),
\[
 \sum_{j\in S} v_j(\vartheta)a_{j,C}(\vartheta)^\top=0,
\]
and the corresponding cellwise constants agree on cells lying in the same connected component.
These are again finitely many polynomial, or after normalization real-analytic, equations.

There are only finitely many counted gates, inherited cells, marked facets, pairs of arguments, and
nonempty subsets \(S\).  Thus the total EFF-minimality failure set is contained in a finite union of
algebraic or analytic zero sets.

It remains only to see that these zero sets are proper. By assumption, the chart contains at least
one EFF-minimal expansion.  At that parameter value none of the above EFF-failure conditions
holds.  Therefore none of the corresponding certificate functions vanishes identically on the chart. Each bad set is consequently a proper real-algebraic subset, or after normalization a proper
real-analytic subset.  A finite union of such proper subsets has empty interior and measure zero.
\end{proof}

Having established that the qualitative forms of minimality are generic, we now establish that the quantitative margins version is high probability:
\begin{proposition}[Positive margins for quantitative minimality are generic]
\label{prop:generic-quant-min-margins}
Fix a layer and work inside one compact subset \(K\) of a finite-dimensional parameter or solution
chart, with the relevant Softplus pole-patch data or ReLU cell/facet data fixed.  Let \(\mu\) be a
probability law on \(K\).  Suppose the qualitative minimality failure set,
\[
 \calB_{\mathrm{qual}}
 =
 \{\vartheta\in K:\text{the layer is not qualitatively minimal at }\vartheta\},
\]
has \(\mu\)-measure zero.  Then for every \(\delta>0\), there exist \(M<\infty\) and \(c>0\) such that, with \(\mu\)-probability at least \(1-\delta\), the layer satisfies the corresponding quantitative minimality conditions with parameter bound \(M\) and margin \(c\).
\end{proposition}

\begin{proof}
On the fixed chart, the qualitative minimality conditions are tested by finitely many continuous
nonnegative certificate functions $d_1,\ldots,d_N$. These include the relevant coefficient norms, variation or crossing quantities, duplicate/opposite separation distances, affine sign-cancellation distances, and private pole/facet margins.  Qualitative
minimality holds exactly when all certificates are positive.  Therefore, with
\[
 d(\vartheta)=\min_{1\le a\le N}d_a(\vartheta),
\]
we have $\calB_{\mathrm{qual}} = \{d = 0\}$. By assumption, \(\mu(\{d=0\})=0\).  Since the sets \(\{d<\eta\}\) decrease to \(\{d=0\}\) as
\(\eta\downarrow0\), continuity of probability from above gives $\mu(\{d<\eta\})\to0$. Choose \(c>0\) such that $\mu(\{d<c\})\le\delta$. Then, outside a set of probability at most \(\delta\), every certificate satisfies \(d_a\ge c\).

Finally, \(K\) is compact, so the parameter norm is bounded on \(K\); choose \(M\) larger than this
bound.  The constants \(M\) and \(c\) are exactly the boundedness and positive-margin inputs in the
corresponding quantitative minimality definition.  If finitely many layers are considered at once,
apply the same argument to each layer and use a finite union bound.
\end{proof}

\subsubsection{Genericity of Jacobian regularity}
We now show that Jacobian regularity is generi. 
\begin{proposition}[Generic Jacobian regularity]
\label{prop:generic-jacobian}
Fix a finite-dimensional chart for one-step expansions after the activation symmetries have been removed. If the Jacobian regularity condition holds at least at one point, then the Jacobian-singular set
\[
 \{\vartheta:\det J_r^\varphi(\vartheta)=0\}
\]
is a proper real-analytic subset of the chart and therefore has empty interior and measure zero.
\end{proposition}

\begin{proof}
Let $d=\dim\Theta_r^\varphi$.  In the chosen coordinates, $J_r^\varphi$ is the $d\times d$ Gram matrix
\[
 (J_r^\varphi)_{ab}(\vartheta)
 =\langle \partial_aF_r^\varphi(\vartheta),
          \partial_bF_r^\varphi(\vartheta)\rangle_{L^2}.
\]
The expansion is full-rank at $\vartheta$ exactly when this Gram matrix is nonsingular, or equivalently when $\det J_r^\varphi(\vartheta)\neq0$.

For Softplus, the first-variation features are analytic functions of the parameters and of the input.  On compact $\Omega$, the integrals defining the Gram entries may be differentiated under the integral sign.  Hence every entry of $J_r^\sig$ is real analytic in $\vartheta$.

For ReLU, the reason for fixing the activation pattern is simple: if the pattern does not change, then the same finite cell decomposition of $\Omega$ is used throughout the chart.  On each cell, each ReLU gate is either active or inactive, so the indicator $\mathbf 1_{\{g_j>0\}}$ appearing in the derivative formula is constant on that cell.  The derivative features are therefore affine in $x$ on each cell, with coefficients depending polynomially or analytically on the parameters.  If the cell boundaries move with the parameters, subdivide the cells into finitely many simplices and pull each integral back to a fixed reference simplex.  The simplex vertices and the integrand coefficients vary analytically, so each cell integral, and hence each Gram entry, is real analytic on the fixed-pattern chart.

Thus $\det J_r^\varphi$ is a real-analytic function.  The full-rank condition at one point says this determinant is nonzero at least at one point, so it is not identically zero.  Its zero set is therefore a proper real-analytic subset of the chart.  Such a subset has empty interior and measure zero.  This is the claimed generic Jacobian regularity.
\end{proof}

\subsubsection{Genericity of the Softplus MLR regularity}

The MLR hypotheses remove the concrete ways a Softplus pole match can fail to determine an argument.  There are two main exceptional mechanisms:
\begin{itemize}
    \item First, the monodromy characters can satisfy exact integer relations: this produces torsion, cross-character commensurability, or nonprimitive exponents such as $1+Z^3=(1+Z)(1-Z+Z^2)$.  These are composite resonances.
    \item Second, even when the Laurent binomials are irreducible, the lower representation map may collapse the relevant Laurent coordinates, so two distinct Laurent hypersurfaces appear to share a pulled-back divisor.  This is a failure of the rank condition for the lower-feature map.
\end{itemize}    
\noindent The next two results show that both mechanisms are described by affine integer-resonance equations or by vanishing analytic minors, and thus form lower-dimensional exceptional sets.

\begin{lemma}[Character and composite resonances are lower-dimensional]
After fixing a monodromy setup, relations of the form
\[
 \chi_i^q=\chi_j^{q'}
\]
for nonzero integers $q,q'$, torsion relations, and nonprimitive exponent relations are exact integer-affine constraints on the layer weights.  Unless forced by tying, their union is a countable union of proper affine subsets and has measure zero.
\end{lemma}

\begin{proof}
Choose a basis of the finitely generated character group in this setup.  The character of a current argument is determined by an affine-linear function of its incoming weights.  A relation $\chi_i^q=\chi_j^{q'}$, or a torsion relation, therefore has the form
\[
 q\,w_i-q'\,w_j-\nu=0,
 \qquad \nu\in\Z^p,
\]
with the analogous finite integer-linear expression for more characters.  For each fixed integer tuple this is an affine subspace of the untied weight parameter space.  If the relation is not forced, the subspace is proper.  Nonprimitivity of an exponent is likewise an exact divisibility relation $m=q m_0$, $|q|>1$, in the character lattice.  There are only countably many integer tuples and lattice elements, so the union of all such proper affine sets has Lebesgue measure zero and empty interior.
\end{proof}

The exact Softplus theorem in \S\ref{sec:exact-zippering} assumes Laurent divisor-faithfulness as part of MLR.  The next lemma is a ``verification tool'' for that clause: it says that a shared pulled-back divisor is genuine whenever the lower-feature map has large enough rank.

\begin{lemma}[Rank criterion for Laurent divisor-faithfulness]
\label{lem:sp-rank}
Suppose
\[
 e^g=cT^m,
 \qquad
 e^h=dT^n
\]
in a fixed monodromy chart, and let $\Sigma$ be a shared outer pole patch.  If near a generic point of $\Sigma$ there is a complex submanifold $L$ on which $c,d$ are constant and
\[
 \rank_\C D(T|_L)=s,
\]
then the pair is Laurent divisor-faithful at $\Sigma$.
\end{lemma}

\begin{proof}
Write
\[
 P(T)=1+cT^m,
 \qquad
 R(T)=1+dT^n.
\]
On $L$, the pulled-back functions are $P(T(x))=1+e^{g(x)}$ and $R(T(x))=1+e^{h(x)}$.  The shared patch $\Sigma\cap L$ is a regular piece of the zero set of $P(T(x))$, and by assumption $R(T(x))$ also vanishes on it.

Since $D(T|_L)$ has rank $s$ at a generic point of the patch, $T|_L$ is locally a submersion there.  Therefore the image under $T$ of a relatively open piece of $\Sigma\cap L$ contains a relatively open piece of the Laurent hypersurface $\{P=0\}$ in $(\C^*)^s$.  The second Laurent binomial $R$ vanishes on that image, because its pullback $R(T(x))$ vanishes on the shared patch.

Thus $R$ vanishes on a nonempty analytic open subset of an irreducible component of $\{P=0\}$.  If $m$ is primitive, $P$ is prime by Lemma~\ref{lem:sp-prime}; more generally, on the selected irreducible component this means the corresponding irreducible Laurent factor of $P$ divides $R$. Hence the common pulled-back divisor comes from a genuine common Laurent factor.  This is exactly Laurent divisor-faithfulness.
\end{proof}

Once a particular local Softplus pole match has been fixed, the next proposition shows that the rank condition in Lemma \ref{lem:sp-rank} fails only on a lower-dimensional exceptional subset; hence Laurent divisor-faithfulness is generic within that fixed pole-matching family.

\begin{proposition}[Generic Laurent divisor-faithfulness]
Let \(S\) be a family of Softplus comparison parameters and $g_\theta, h_\theta$ a pair of divisor-matched arguments in $S$. Suppose the rank condition of Lemma~\ref{lem:sp-rank} holds at one point of \(S\).  Then Laurent divisor-faithfulness holds for an open dense full-measure subset of \(S\).
\end{proposition}

\begin{proof}
Let \(S\) denote the connected divisor-incidence stratum in the proposition.  Thus, for every
parameter value in \(S\), the same pair of Softplus pole divisors is being required to share the
same kind of local pole patch.  We work in a local analytic coordinate chart on \(S\); a locally
finite collection of such charts gives the same conclusion globally.

For a parameter value \(\theta\in S\), write the two Laurent forms as
\[
 e^{g_\theta}=c_\theta T_\theta^m,
 \qquad
 e^{h_\theta}=d_\theta T_\theta^n,
\]
where
\[
 T_\theta=(T_{\theta,1},\ldots,T_{\theta,s})
\]
is the Laurent-coordinate map into \((\C^*)^s\).  In other words, the functions
\(T_{\theta,a}\) are the local Laurent coordinates in which the two Softplus factors become
\[
 P_\theta(T)=1+c_\theta T^m,
 \qquad
 R_\theta(T)=1+d_\theta T^n.
\]

The rank criterion in Lemma~\ref{lem:sp-rank} is checked on a local complex submanifold
\(L_\theta\) through the shared pole patch, chosen so that the invariant coefficients
\(c_\theta,d_\theta\) are constant on \(L_\theta\).  Choose a reference point on the shared pole patch.  The local
submanifolds \(L_\theta\), the reference points \(x_\theta\), and the restrictions
\[
 T_\theta|_{L_\theta}:L_\theta\to(\C^*)^s
\]
vary analytically with \(\theta\). Now let
\[
 J_\theta
 := D\bigl(T_\theta|_{L_\theta}\bigr)_{x_\theta}.
\]
The condition required by Lemma~\ref{lem:sp-rank} is $\rank_\C J_\theta=s$. This is equivalent to saying that at least one \(s\times s\) minor of \(J_\theta\) is nonzero.  Define
\[
 \Delta(\theta)
 =
 \sum_M \left|\det M(J_\theta)\right|^2,
\]
where the sum is over all relevant \(s\times s\) minors.  This is a real-analytic
nonnegative function on the coordinate chart in \(S\), and
\[
 \Delta(\theta)>0
 \quad\Longleftrightarrow\quad
 \rank_\C J_\theta=s.
\]

By hypothesis, the rank condition of Lemma~\ref{lem:sp-rank} holds at at least one point of
the connected stratum.  Hence \(\Delta\) is not identically zero.  Therefore its
zero set
\[
 Z
 =
 \{\theta\in S:\Delta(\theta)=0\}
\]
is a proper real-analytic subset of \(S\).  Such a subset has empty interior and measure zero. Consequently, $S\setminus Z$ is open, dense, and full-measure in \(S\). For every \(\theta\in S\setminus Z\), the restricted Laurent-coordinate map \(T_\theta|_{L_\theta}\) has full rank \(s\) at the chosen point. Thus Lemma~\ref{lem:sp-rank}
applies. It says that the shared pulled-back pole patch comes from a genuine common Laurent
factor of
\[
 P_\theta(T)=1+c_\theta T^m
 \qquad\text{and}\qquad
 R_\theta(T)=1+d_\theta T^n.
\]
That is exactly Laurent divisor-faithfulness at the selected shared pole patch.  Hence Laurent
divisor-faithfulness holds on an open dense full-measure subset of \(S\).
\end{proof}

\subsubsection{Genericity of ReLU MFR regularity}

The ReLU exceptional mechanisms parallel the Softplus ones, but with kink-facet signatures replacing pole divisors.  Once the task has exposed enough fresh facets, there are, as with the softplus case, two main ways MFR can fail:
\begin{itemize}
    \item First, two distinct normalized piecewise affine arguments may have the same labelled multi-facet signature; this is a failure of the signature map to be locally injective.  The first proposition below rules out persistent off-diagonal equal signatures when the secant derivative has the required rank at one point. 
    \item Second, the facets of one source argument may be split across several competitor arguments, so that no single competitor carries the whole signature.  The second proposition rules out persistent split matching when the analogous split-incidence derivative has the required rank at one point.
\end{itemize}
Together, these results justify the ReLU part of the general-position assumption on collision strata to be used in the generic null-net theorem below in Section~\ref{sec:genericity}.

Let $\Sigma_r:\Xi_r\to\R^{N_r}$ be the signature map from Definition~\ref{def:relu-signature}, with $p_r=\dim\Xi_r$.  An off-diagonal signature-collision stratum is a smooth piece of the set of distinct normalized arguments with the same labelled multi-facet signature.  The rank test below uses the derivative of the difference map
\[
 H(\xi,\widetilde\xi)=\Sigma_r(\xi)-\Sigma_r(\widetilde\xi).
\]
Following the standard secant-variety terminology, we call this derivative the \emph{secant derivative}~\citep{Landsberg2012Tensors}:
\[
 \mathcal T_{\xi,\widetilde\xi}(h,\widetilde h)
 =D\Sigma_r(\xi)h-D\Sigma_r(\widetilde\xi)\widetilde h.
\]

\begin{proposition}[Generic multi-facet recovery]
\label{prop:generic-mfr}
Suppose every connected off-diagonal signature-collision stratum contains one point with
\[
 \rank\mathcal T_{\xi,\widetilde\xi}\ge p_r+1.
\]
Then every off-diagonal signature-collision stratum has dimension at most $p_r-1$.  Hence ambiguous equal-signature values occur only on a lower-dimensional subset of the normalized argument chart.
\end{proposition}

\begin{proof}
Let $S$ be a connected smooth off-diagonal stratum of
\[
 \{(\xi,\widetilde\xi):\Sigma_r(\xi)=\Sigma_r(\widetilde\xi),\ \xi\neq\widetilde\xi\}.
\]
At the assumed point $p=(\xi,\widetilde\xi)$, the tangent space satisfies
\[
 T_pS\subseteq\ker\mathcal T_{\xi,\widetilde\xi},
\]
because differentiating the signature equality along $S$ gives zero.  The domain of $\mathcal T$ has dimension $2p_r$, and its rank at that point is at least $p_r+1$.  Therefore
\[
 \dim S=\dim T_pS
 \le 2p_r-(p_r+1)=p_r-1.
\]
The relevant rank condition is the nonvanishing of finitely many $(p_r+1)\times(p_r+1)$ minors.  Since one such minor is nonzero at that point, the rank inequality holds on an open dense, full-measure subset of the connected stratum.  Finally, projection to the first argument chart is locally Lipschitz and cannot increase Hausdorff dimension.  Thus normalized arguments with a distinct equal-signature partner form a set of dimension at most $p_r-1$.
\end{proof}

This proposition is the ReLU analogue of generic Laurent divisor-faithfulness for Softplus: it says that ambiguous same-signature pairs are lower-dimensional.  It handles the case where one source argument is matched to one competitor argument.  The remaining issue is split matching, where different facets of one source argument are explained by different competitors.

For split matching, write the signature componentwise as
\[
 \Sigma_r=(\Sigma_{r,1},\ldots,\Sigma_{r,L}).
\]
Fix $q\ge2$ and a surjective assignment $\tau:\{1,\ldots,L\}\to\{1,\ldots,q\}$, meaning that facet $\nu$ of a source argument is matched by competitor $\eta_{\tau(\nu)}$.  Define
\[
 H_\tau(\xi,\eta_1,\ldots,\eta_q)
 =\bigl(\Sigma_{r,\nu}(\xi)-\Sigma_{r,\nu}(\eta_{\tau(\nu)})\bigr)_{\nu=1}^L.
\]
A split-incidence stratum is a smooth piece of the set $H_\tau^{-1}(0)$, where this fixed split assignment is imposed.

\begin{proposition}[Split matching is generically lower-dimensional]
\label{prop:generic-split}
Suppose that, for every nontrivial split assignment $\tau$, every connected smooth stratum of the split-incidence set $H_\tau^{-1}(0)$ contains a point at which
\[
 \rank DH_\tau\ge q p_r+1.
\]
Then the projection of the split-matching set to the source argument chart has Hausdorff dimension at most $p_r-1$.
\end{proposition}

\begin{proof}
The domain of $H_\tau$ has dimension $(q+1)p_r$.  Let $S_\tau$ be a smooth incidence stratum, and choose a point of $S_\tau$ where the displayed rank condition holds.  At that point,
\[
 T S_\tau\subseteq\ker DH_\tau,
\]
so
\[
 \dim S_\tau
 \le(q+1)p_r-(q p_r+1)=p_r-1.
\]
As above, the rank inequality is the nonvanishing of analytic minors and therefore holds generically on the stratum once it holds at one point.  Projection to the source coordinate cannot increase Hausdorff dimension.  For a fixed finite facet atlas and finite width bound, only finitely many nontrivial assignments $\tau$ occur, so the union of their projected sets still has dimension at most $p_r-1$.
\end{proof}

Taken together, propositions~\ref{prop:generic-mfr} and~\ref{prop:generic-split} explain why the ReLU signature-recovery hypotheses hold generically on any full-dimensional collision stratum, after discarding lower-dimensional signature and split-incidence exceptions.

\subsubsection{The generic null-net theorem}

The preceding subsections gave activation-specific genericity inputs: Softplus MLR failures and
ReLU MFR failures occur only on lower-dimensional exceptional sets, provided the relevant rank
conditions hold somewhere on the stratum being considered.  The purpose of this subsection is to
combine those activation-specific inputs with the geometry of the realization map.  The conclusion
is not a complete classification of all null networks.  Rather, it says that nontrivial one-step null
partners cannot occur on an open, full-dimensional family.

Recall the definitions of realization map \(F_r^\varphi\), collision set \(\mathcal C_r^\varphi\), and off-diagonal collision strata from Definition~\ref{def:null-pair-collision}. Throughout this subsection we work inside a fixed normalized finite-dimensional realization chart \(\Theta_r^\varphi\), where ``normalized'' means that the hidden-unit order has been fixed, and in the ReLU case the positive-rescaling freedom has also been fixed.

\begin{lemma}[Dimension bound for collision strata]
\label{lem:collision-dimension}
Let \(S\) be a smooth off-diagonal collision stratum, and suppose that
\(F_r^\varphi\) is an immersion at both members of every collision pair in \(S\):
\[
 \rank DF_r^\varphi(\vartheta)
 =
 \rank DF_r^\varphi(\widetilde\vartheta)
 =
 d_r^\Theta,
 \qquad
 (\vartheta,\widetilde\vartheta)\in S.
\]
Then $\dim S\le d_r^\Theta$.
\end{lemma}

\begin{proof}
Fix \(p=(\vartheta,\widetilde\vartheta)\in S\).  Define the secant derivative
\[
 \calS_p(h,\widetilde h)
 =
 DF_r^\varphi(\vartheta)h
 -
 DF_r^\varphi(\widetilde\vartheta)\widetilde h .
\]
If \((h,\widetilde h)\in T_pS\), then differentiating the collision identity $F_r^\varphi(\vartheta(t)) = F_r^\varphi(\widetilde\vartheta(t))$
along a curve in \(S\) gives
\[
 \calS_p(h,\widetilde h)=0.
\]
Hence $T_pS\subseteq\ker\calS_p$.
The restriction
\[
 h\longmapsto \calS_p(h,0)=DF_r^\varphi(\vartheta)h
\]
is injective by the immersion hypothesis, and has rank \(d_r^\Theta\).  Therefore $\rank \calS_p\ge d_r^\Theta$. By rank-nullity,
\[
 \dim S
 =
 \dim T_pS
 \le
 \dim\ker\calS_p
 \le
 2d_r^\Theta-d_r^\Theta
 =
 d_r^\Theta.
\]
\end{proof}

\begin{remark}[Full-dimensional collision strata]
The dimension equality case in the lemma above has useful geometric meaning.  If $\dim S=d_r^\Theta$, then the first-coordinate projection
\[
 \pi_1:S\to\Theta_r^\varphi,
 \qquad
 \pi_1(\vartheta,\widetilde\vartheta)=\vartheta,
\]
is locally a diffeomorphism.  Indeed, if
\[
 (0,\widetilde h)\in\ker D\pi_1(p),
\]
then \((0,\widetilde h)\in T_pS\), so
\[
 DF_r^\varphi(\widetilde\vartheta)\widetilde h=0,
\]
and the immersion hypothesis at \(\widetilde\vartheta\) forces \(\widetilde h=0\).  Thus
\(D\pi_1(p)\) is injective, hence an isomorphism between spaces of dimension \(d_r^\Theta\).
Therefore, locally,
\[
 S=\{(\vartheta,\Psi(\vartheta))\},
 \qquad
 F_r^\varphi(\Psi(\vartheta))=F_r^\varphi(\vartheta).
\]
Similarly, the second-coordinate projection is locally a diffeomorphism.  Thus a full-dimensional
off-diagonal collision stratum is locally a nontrivial realization-preserving reparameterization.
This explains why full-dimensional collision strata are the dangerous case.
\end{remark}

\begin{definition}[General position on full-dimensional collision strata]
\label{def:csnsgp}
A normalized realization chart satisfies \emph{general position on full-dimensional collision strata}
if every \(d_r^\Theta\)-dimensional off-diagonal collision stratum contains at least one point where
the activation-specific one-step identifiability result applies.  Concretely, at that point the null pair
is intrinsically minimal and MLR in the Softplus case; in the ReLU case it is EFF-minimal,
coherently matched, and MFR.

The genericity results above explain how this condition is meant to be checked: the Softplus
rank criterion controls the lower-feature-rank failure in Laurent divisor-faithfulness, and the ReLU
signature and split-matching results control the same-signature and split-facet failures.  Thus the
condition says that a full-dimensional collision stratum is not trapped entirely inside the relevant
exceptional sets.
\end{definition}

\begin{theorem}[Generic activation-independent null-net theorem]
\label{thm:generic-null-common}
Work in the normalized finite-dimensional chart and collision stratification described above.  Assume:
\begin{enumerate}[label=(\roman*)]
\item \(F_r^\varphi\) is Jacobian-regular on the chart, i.e. \(DF_r^\varphi\) has full column rank
\(d_r^\Theta\) at every point of the chart;
\item the general-position condition on full-dimensional collision strata holds.
\end{enumerate}
Then every off-diagonal collision stratum has dimension at most \(d_r^\Theta-1\).  Consequently
the set of parameters with a nontrivial collision,
\[
 \calE_r^\varphi
 =
 \{\vartheta:\exists\widetilde\vartheta\neq\vartheta,\,
 F_r^\varphi(\widetilde\vartheta)=F_r^\varphi(\vartheta)\},
\]
has Hausdorff dimension at most \(d_r^\Theta-1\), hence empty interior and measure zero in the
chart.
\end{theorem}

\begin{proof}
Let \(S\) be any smooth off-diagonal collision stratum.  By Jacobian regularity,
\(F_r^\varphi\) is an immersion at every point of the chart.  Hence the hypothesis of
Lemma~\ref{lem:collision-dimension} applies at both coordinates of every collision pair in \(S\). Therefore $\dim S\le d_r^\Theta$.
We next rule out the full-dimensional case.  Suppose, for contradiction, that $\dim S=d_r^\Theta$. By general position, \(S\) contains a point
$p_0=(\vartheta_0,\widetilde\vartheta_0)$
where the activation-specific one-step identifiability theorem applies.  Since \(p_0\) is a collision,
the two represented one-step functions are equal:
\[
 F_r^\varphi(\vartheta_0)
 =
 F_r^\varphi(\widetilde\vartheta_0).
\]
The activation-specific identifiability result then says that the two one-step expansions differ only
by the unavoidable hidden-unit matching: permutation for Softplus, and permutation plus positive
coordinatewise scaling for ReLU.  But the chart is normalized precisely so that this hidden-unit
matching has already been fixed.  Therefore the two parameter values are equal in the chart: $\vartheta_0=\widetilde\vartheta_0$. Equivalently, $p_0\in \Delta_r^\varphi$,
where
\[
 \Delta_r^\varphi=\{(\vartheta,\vartheta):\vartheta\in\Theta_r^\varphi\}
\]
is the diagonal of trivial self-collisions.  This contradicts the fact that \(S\) is off diagonal, i.e.
\(S\subset \calC_r^\varphi\setminus\Delta_r^\varphi\).  Hence no off-diagonal collision stratum has
dimension \(d_r^\Theta\), and so every off-diagonal collision stratum satisfies
\[
 \dim S\le d_r^\Theta-1.
\]

It remains to pass from collision pairs to individual parameters.  The set
\(\calE_r^\varphi\) is not itself a set of pairs; it is the set of first coordinates of off-diagonal
collision pairs.  In other words,
\[
 \calE_r^\varphi
 =
 \pi_1\bigl(\calC_r^\varphi\setminus\Delta_r^\varphi\bigr),
 \qquad
 \pi_1(\vartheta,\widetilde\vartheta)=\vartheta.
\]
Thus the first-coordinate projection simply forgets the nontrivial partner and remembers the
parameter value that has such a partner.

On each local coordinate chart, \(\pi_1\) is Lipschitz, so Hausdorff dimension cannot increase:
\[
 \dim_H \pi_1(S)
 \le
 \dim_H S
 \le
 d_r^\Theta-1
\]
for every off-diagonal stratum \(S\).  Since the off-diagonal collision set is a countable locally finite
union of such smooth strata, and since
\[
 \calE_r^\varphi
 =
 \bigcup_{S}
 \pi_1(S),
\]
countable stability of Hausdorff dimension gives
\[
 \dim_H\calE_r^\varphi\le d_r^\Theta-1.
\]
A subset of a \(d_r^\Theta\)-dimensional smooth chart with Hausdorff dimension at most
\(d_r^\Theta-1\) has zero chart-Lebesgue measure and cannot contain a nonempty open set.  This
proves the measure-zero and empty-interior conclusions.
\end{proof}

\begin{corollary}[Generic Softplus and ReLU null-pair uniqueness]
\label{cor:generic-null-specializations}
In any normalized chart satisfying the ambient setup above, Jacobian regularity plus the appropriate
general-position condition implies generic null-pair uniqueness.  For Softplus, this means that
non-permutation null pairs form a measure-zero set.  For ReLU, this means that nontrivial null
pairs modulo permutation and positive coordinatewise scaling form a measure-zero set.
\end{corollary}

\begin{proof}
Apply Theorem~\ref{thm:generic-null-common} with \(\varphi=\sig\) and use the Softplus
specialization of the activation-specific one-step identifiability result; then apply it with
\(\varphi=\rhoR\) and use the ReLU specialization.  The diagonal of obvious self-collisions is
ordinary equality after choosing the Softplus permutation representative, and ordinary equality
after choosing the ReLU permutation and positive-scale normalization.
\end{proof}

\begin{corollary}[Generic exact zippering from generic null triviality]
\label{cor:generic-exact-zip-null}
Suppose that, at each of finitely many compared layers, the relevant normalized one-step chart
satisfies the ambient setup of this subsection and the two hypotheses of
Theorem~\ref{thm:generic-null-common}.  Then, outside a finite union of layerwise measure-zero
sets of parameters with nontrivial collision partners, exact terminal weak equivalence zippers
upstream.  The recovered maps are permutations for Softplus and permutations followed by positive
coordinatewise scalings for ReLU.
\end{corollary}

\begin{proof}
Begin at layer \(s\).  Exact terminal weak equivalence identifies the two layer-\((s-1)\) one-step
represented functions: the B-side expansion of \(z_s^B\) and the transformed A-side expansion of
\(E_sz_s^A\) form a null pair.  If the realized B-side expansion is outside the projected set of
parameters with a nontrivial collision partner at layer \(s-1\), Theorem~\ref{thm:generic-null-common}
shows that this null pair is matched by the allowed hidden-unit map.  Thus it produces the preceding
comparison map \(E_{s-1}\): a permutation for Softplus, or a permutation followed by positive
coordinatewise scaling for ReLU.

Assume inductively that \(E_{r+1}\) has been recovered and
\[
 z_{r+1}^B=E_{r+1}z_{r+1}^A.
\]
Expanding both sides through layer \(r\) again creates a one-step null pair.  Outside the layer-\(r\)
set of parameters with a nontrivial collision partner, generic null triviality produces \(E_r\).  Backward
induction reaches every layer \(r=q,\ldots,s-1\).  There are only finitely many compared layers, so
the union of the layerwise exceptional sets remains measure zero.
\end{proof}

The preceding MLR and MFR genericity results verify the activation-specific identifiability input in
the general-position condition, while Jacobian regularity supplies the geometric input used in
Lemma~\ref{lem:collision-dimension}.  Neither ingredient alone is enough.  Pairwise MLR/MFR
genericity cannot exclude a realization-preserving graph that is lower-dimensional in the product
but projects onto an open set of individual networks, while Jacobian regularity alone cannot exclude
two distinct immersed parameter families with the same realized functions.  General position on
full-dimensional collision strata combines the two.

\begin{remark}[Relation to complete null-network classifications]
\label{rem:vlacic-bolcskei}
The null-network program of Vla\v{c}i\'{c} and B\"olcskei seeks an exhaustive activation-specific description of all regular zero-realizing networks and thereby obtains global identifiability modulo exact activation symmetries.  The theorem above deliberately asks for less: it rules out full-dimensional families of nontrivial collisions without classifying every exceptional resonant null network. For Softplus, this avoids the hardest (and unsolved) missing step in a literal complete classification.  One need not prove a factorization-completeness theorem for every divisor-free unit or nested resonance.  Composite integer-weight, nested-identity constructions or lower-feature-rank-defect null networks may survive on lower-dimensional exceptional sets; they do not obstruct generic zippering unless an entire open family of realizations is trapped there.  The same logic avoids a complete ReLU classification of every construction based on changes on ReLU-inactive regions or split-facet matching, which would probably be even harder than the softplus case.  The generic null-net theorem is therefore weaker than universal parameter identifiability, but exactly strong enough for the contravariance conclusions pursued here.
\end{remark}

\begin{remark}[Optimizer nonconcentration and generic asymptotic/soft zippering]
The role of genericity is only to justify that in the space of parameters, the conditions needed for zippering are full-measure. However, they do not guarantee that the exceptional set is avoided by learning algorithms.  We are implicitly assuming, in interpreting the zippering theorems as applying to optimized networks, that the distribution produced by training, restricted to the relevant solution or comparison stratum, does not concentrate on the intrinsic-minimality, signature-recovery, nontrivial collision, or Jacobian-singular exceptional sets.   Proving that this assupmtion is correct is an important topic for future work. 
\end{remark}

\subsection{Depth-warped Weak-Strong Equivalence}

\label{sec:depth-warped-block-weak-strong}

The weak--strong theorem compares adjacent layers.  Empirically, however, networks of different
depths often appear to have corresponding single axes at nearby but not identical depths.  The
right comparison is then not layer \(a\) of \(A\) against one prescribed layer of \(B\), but layer
\(a\) of \(A\) against a window of \(B\)-layers.  We call this \emph{depth warping}. 

The results in the previous sections perform, in their essence, a kind of ``signature matching'' (where signatures are ReLU kink traces or Softplus poles): representations align because we can show that linear alignment forces alignment of signatures. In thinking about depth warping, we work with the same \emph{kinds} of signatures as above, but we can no longer assume that the signatures live in exactly corresponding layers.   The new issue is that, when \(B\) is deeper than \(A\),
the counterpart of an \(A\)-axis may appear anywhere in a short window of \(B\)-layers.  Thus a
whole layer of \(A\)-axes can be distributed across several nearby layers of \(B\).  This is the formal
version of the intuition that a deeper model, such as a large ResNet, may smear the preferred axes
of a shallower AlexNet-like or ResNet-18-like model across a longer block.

\begin{figure}[t]
\centering
\resizebox{\textwidth}{!}{%
\begin{tikzpicture}[
  x=1cm,y=1cm,
  >=Latex,
  font=\small,
  block/.style={
    draw, rounded corners, thick,
    align=center, minimum width=2.05cm, minimum height=.92cm,
    inner sep=4pt, fill=gray!5
  },
  smallblock/.style={
    draw, rounded corners, thick,
    align=center, minimum width=1.45cm, minimum height=.80cm,
    inner sep=3pt, fill=gray!5
  },
  axis/.style={
    draw, circle, thick,
    minimum size=.30cm,
    inner sep=0pt,
    fill=white
  },
  flow/.style={->, thick},
  weak/.style={->, thick, densely dashed},
  match/.style={->, very thick},
  split/.style={->, thick, dashed, gray!65},
  label/.style={
    align=center, font=\footnotesize,
    fill=white, inner sep=2pt
  },
  legend/.style={
    draw, rounded corners,
    align=center, font=\footnotesize,
    text width=2.95cm,
    fill=gray!4,
    inner sep=4pt
  }
]

\def\yA{4.20}
\def\yB{1.25}

\node[block] (Ain) at (0.00,\yA) {A input\\to block};
\node[block, minimum width=2.65cm] (Alay) at (4.25,\yA) {A layer \(a\)\\counted axes};
\node[block] (Aout) at (8.70,\yA) {A output\\of block};

\draw[flow] (Ain) -- (Alay);
\draw[flow] (Alay) -- (Aout);

\node[axis] (a1) at (3.55,3.55) {};
\node[axis] (a2) at (4.25,3.55) {};
\node[axis] (a3) at (4.95,3.55) {};

\node[block] (Bin) at (0.00,\yB) {B input\\to block};
\node[smallblock] (B0) at (2.55,\yB) {B layer\\\(b_0\)};
\node[smallblock] (B1) at (4.45,\yB) {B layer\\\(b_1\)};
\node[smallblock] (B2) at (6.35,\yB) {B layer\\\(b_2\)};
\node[smallblock] (B3) at (8.25,\yB) {B layer\\\(b_3\)};
\node[block] (Bout) at (11.10,\yB) {B output\\of block};

\draw[flow] (Bin) -- (B0);
\draw[flow] (B0) -- (B1);
\draw[flow] (B1) -- (B2);
\draw[flow] (B2) -- (B3);
\draw[flow] (B3) -- (Bout);

\draw[thick, dashed, rounded corners]
  (1.62,0.45) rectangle (9.18,2.18);

\node[label] at (5.40,0.12)
  {depth window \(I(a)=\{b_0,b_1,b_2,b_3\}\)};

\node[axis] (b0a) at (2.55,1.88) {};
\node[axis] (b1a) at (4.45,1.88) {};
\node[axis] (b3a) at (8.25,1.88) {};
\node[axis] (b3b) at (8.82,1.88) {};

\draw[weak] (Ain.south) -- node[left,font=\footnotesize,fill=white,inner sep=1pt] {\(E_{\rm in}\)} (Bin.north);

\draw[weak] (Aout.south east) to[out=-55,in=115]
  node[right,font=\footnotesize,fill=white,inner sep=1pt] {\(E_{\rm out}\)}
  (Bout.north west);

\draw[match] (a1) to[out=-115,in=105] (b0a);
\draw[match] (a2) to[out=-90,in=100] (b1a);
\draw[match] (a3) to[out=-30,in=100] (b3a);

\node[axis, draw=gray!70] (sA) at (10.50,3.80) {};
\node[axis, draw=gray!70] (sB1) at (11.20,3.48) {};
\node[axis, draw=gray!70] (sB2) at (11.20,3.02) {};

\draw[split] (sA) -- (sB1);
\draw[split] (sA) -- (sB2);

\node[label, text=gray!70] at (11.35,2.48)
  {excluded split};

\node[legend] at (1.55,-1.35)
  {\textbf{A-side signatures}\\
   counted A-axes have private\\
   minimal/regular signatures};

\node[legend] at (5.45,-1.35)
  {\textbf{Depth-warped matches}\\
   singleton B-axis matches may\\
   appear at different B-depths};

\node[legend] at (9.75,-1.35)
  {\textbf{No splitting}\\
   one A-axis is not represented\\
   only by several B-axes};

\end{tikzpicture}%
}
\caption{Depth-warped singleton alignment.  A layer \(a\) in a shallower network is compared to a
window \(I(a)\) of layers in a deeper network.  The private signatures are the same minimal/regular
signatures used earlier: ReLU facets or Softplus pole/monodromy patches.  The new feature is that
the matched B-side axes may be spread over several layers in the window.  The singleton theorem
rules out the split alternative, where one A-axis is represented only by a packet of several B-axes.}
\label{fig:depth-warp-singleton}
\end{figure}
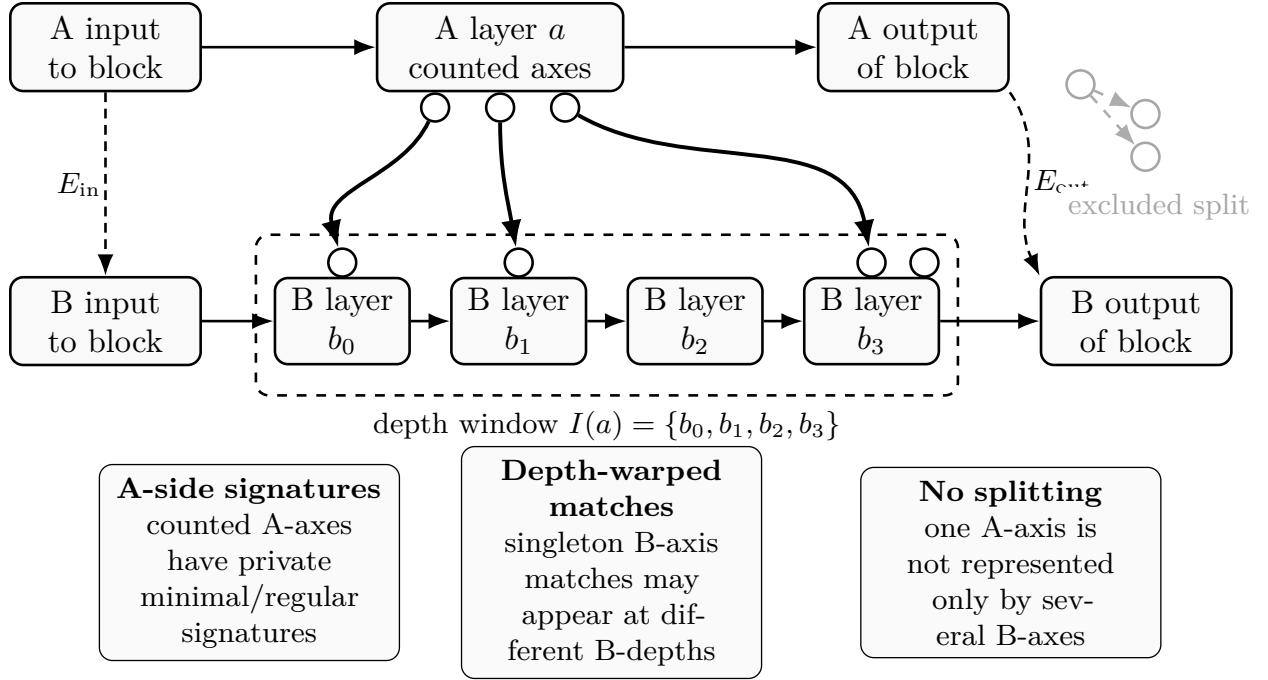

\begin{definition}[Depth warp and depth-warped axis alignment]
\label{def:depth-warp-axis-align}
Let \(A\) have a block of layers \(a_0,\ldots,a_1\), and let \(B\) have a block of layers
\(b_0,\ldots,b_1\).  A \emph{depth warp} assigns to each compared \(A\)-layer \(a\) a nonempty
window
\[
 I(a)\subseteq \{b_0,\ldots,b_1\}.
\]
For each such \(a\), let \(J_a^A\subseteq\{1,\ldots,d_a^A\}\) be the counted \(A\)-side axes, and let
\[
 J^B_{I(a)}=\{(b,i): b\in I(a),\;1\le i\le d_b^B\}
\]
be the \(B\)-side axes available in the corresponding depth window.  The \emph{exact
depth-warped axis-alignment score} is
\[
 \operatorname{DWAxisAlign}_{a,I(a),\varphi}(A,B)
 =
 \frac{1}{d_a^A}
 \max_{\pi}
 \#\left\{
 j\in J_a^A:
 z^B_{\pi(j)}=\alpha_j z^A_{a,j}\ \text{on }\Omega_0
 \text{ for some }\alpha_j\in H_\varphi
 \right\},
\]
where \(\pi\) ranges over injective maps \(J_a^A\hookrightarrow J^B_{I(a)}\), and if
\(\pi(j)=(b,i)\) we write \(z^B_{\pi(j)}=z^B_{b,i}\).
\end{definition}

A depth window is useful only if it is the right window.  We will call \(I(a)\)
\emph{signature-complete} for the counted \(A\)-axes if the \(B\)-side signatures corresponding to
those \(A\)-axes occur inside \(J^B_{I(a)}\).  If \(I(a)\) is the whole \(B\)-block this is automatic;
for a smaller window, it is the assertion that the proposed depth warp has selected the correct
local range of \(B\)-depths. 

\paragraph{The split obstruction.}
The only genuinely new obstruction in this subsection is \emph{splitting}.  The terminology is the
same as in the ReLU MFR genericity argument.  There, a split match occurs when different facets of
one source argument are explained by different competitor arguments, so no single competitor
recovers the whole multi-facet signature.  In the depth-warp setting, the same obstruction can occur
across depth: different pieces of one source axis's block signature may be explained by different
\(B\)-axes in the window \(I(a)\).  If that happens, the correct conclusion is not singleton axis
alignment but packet alignment.  The exact theorem below assumes this packet alternative is absent;
the genericity proposition then gives a rank/transversality condition under which the split
obstruction is lower-dimensional.

\begin{definition}[Depth-warped signature pieces and split assignments]
\label{def:depth-warp-signatures}
Fix a depth warp \(a\mapsto I(a)\).  For each counted \(A\)-axis \(j\in J_a^A\), choose a finite ``signature record''
\[
 \Sigma_j^A=(\Sigma_{j,1}^A,\ldots,\Sigma_{j,L_j}^A),
\]
i.e. just a list of the ReLU fresh-facet pieces, or Softplus pole/monodromy pieces, defining the $j$ axis. For a \(B\)-axis \(q=(b,i)\in J^B_{I(a)}\), say that \(q\) \emph{matches the piece} \(\Sigma_{j,\nu}^A\) if the corresponding \(B\)-side piece agrees with \(\Sigma_{j,\nu}^A\) in the fixed ReLU or Softplus signature coordinates.  A \emph{split assignment} for \(j\) is a map
\[
 \tau:\{1,\ldots,L_j\}\to J^B_{I(a)}
\]
whose image contains at least two \(B\)-axes.  The split assignment is \emph{realized} if, for every
\(\nu\), the \(B\)-axis \(\tau(\nu)\) matches the piece \(\Sigma_{j,\nu}^A\).
\end{definition}

This definition is deliberately parallel to the ReLU split-matching definition in
Proposition~\ref{prop:generic-split}.  In that proposition, the source signature is written as
\[
 \Sigma_r=(\Sigma_{r,1},\ldots,\Sigma_{r,L}),
\]
and a split assignment sends different components \(\Sigma_{r,\nu}\) to different competitors.  The
depth-warped definition is the same construction, except that the competitors may live at different
depths inside \(I(a)\).

Now we can state the depth-warped analog of the weak-strong alignment theorem.

\begin{theorem}[Exact depth-warped weak--strong alignment]
\label{thm:exact-depth-warped-weak-strong}
Let \(A\) and \(B\) be affine-\(\varphi\) networks, with \(\varphi\in\{\sigma,\rho\}\).  Let
\[
 \Psi_A=\Phi^A_{a_1-1}\circ\cdots\circ\Phi^A_{a_0},
 \qquad
 \Psi_B=\Phi^B_{b_1-1}\circ\cdots\circ\Phi^B_{b_0}
\]
be blocks of possibly different depths.  Suppose there are injective affine maps
\[
 E_{\rm in}:\mathbb R^{d^A_{a_0}}\to\mathbb R^{d^B_{b_0}},
 \qquad
 E_{\rm out}:\mathbb R^{d^A_{a_1}}\to\mathbb R^{d^B_{b_1}},
\]
such that the block endpoints are exactly weakly aligned:
\[
 \Psi_B(E_{\rm in}z)=E_{\rm out}\Psi_A(z)
\]
for all \(z\) in the input task patch of the \(A\)-block.  Assume (i) that the usual minimality-regularity conditions from earlier (non-depth-warped) identifiability results, and (ii) no split assignment from Definition~\ref{def:depth-warp-signatures} is realized.
Then there is an injective matching
$\pi:J_a^A\hookrightarrow J^B_{I(a)}$ 
such that, for every \(j\in J_a^A\), if \(\pi(j)=(b,i)\), then
$z^B_{b,i}=\alpha_j z^A_{a,j}$ for some $\alpha_j\in H_\varphi$. Consequently,
\[
\operatorname{DWAxisAlign}_{a,I(a),\varphi}(A,B)
 \ge
 \frac{|J_a^A|}{d_a^A}.
\]
\end{theorem}

\begin{proof}
Fix \(j\in J_a^A\).  By assumption (i), the \(A\)-axis \(j\) has a private visible nonlinear signature record:
a Softplus pole/monodromy record or a ReLU fresh-facet/multi-facet record.  Endpoint weak
equality says that the two block maps represent the same function after applying \(E_{\rm in}\) and
\(E_{\rm out}\).  Since \(E_{\rm out}\) is injective, it cannot erase a nonzero kink jump, monodromy
vector, or marked readout.  Hence every marked piece of the \(A\)-signature must be matched
somewhere on the \(B\)-side.  Since the chosen depth window is signature-complete, those matched
pieces lie inside \(J^B_{I(a)}\).

Choose, for each marked piece \(\Sigma_{j,\nu}^A\), one \(B\)-axis in the window that matches it.
This gives a map
\[
 \tau_j:\{1,\ldots,L_j\}\to J^B_{I(a)}.
\]
If the image of \(\tau_j\) contained two or more \(B\)-axes, then \(\tau_j\) would be a realized split
assignment.  This is excluded by assumption (ii).  Therefore all marked pieces of the private signature record
are matched by one and the same \(B\)-axis.  Call this axis
\[
 \pi(j)=(b,i).
\]

Now apply the activation-specific regularity assumption (i) to this singleton match.  In the Softplus
case, MLR identifies the matched scalar argument up to the Softplus sign ambiguity, with
sign-noncancellation selecting the admissible orientation.  In the ReLU case, MFR identifies the
matched CPWA argument up to the ReLU scalar symmetry.  Thus
\[
 z^B_{\pi(j)}=\alpha_j z^A_{a,j}
 \quad\text{on }\Omega_0,
 \qquad
 \alpha_j\in H_\varphi .
\]

It remains only to check that \(j\mapsto\pi(j)\) is injective; this is not a separate no-splitting
assumption, but a consequence of same-side separation.  Suppose, toward a contradiction, that two
distinct counted \(A\)-axes \(j_1\neq j_2\) are assigned to the same \(B\)-axis \(q\).  Then \(q\) has
the full private signature record of both \(j_1\) and \(j_2\).  By the singleton identifiability just used,
\[
 z^B_q=\alpha_1 z^A_{a,j_1}
 \quad\text{and}\quad
 z^B_q=\alpha_2 z^A_{a,j_2}
 \qquad\text{on }\Omega_0,
\]
with \(\alpha_1,\alpha_2\in H_\varphi\).  Hence \(z^A_{a,j_1}\) and \(z^A_{a,j_2}\) have the same
full signature record up to the allowed scalar symmetry, contradicting the same-side separation
part of assumption (i).  Therefore \(\pi\) is injective.

Every counted \(A\)-axis contributes to the numerator of
\(\operatorname{DWAxisAlign}_{a,I(a),\varphi}\), and the displayed lower bound follows after
dividing by \(d_a^A\).  
\end{proof}

\begin{proposition}[Split matching is generically lower-dimensional in a depth warp]
\label{prop:generic-depth-warp-split}
Suppose that every possible nontrivial split assignment has a split-rank witness: after allowing the \(B\)-side packet variables to vary, the equations expressing that split impose at least one independent
constraint on the source \(A\)-axis signature.  Then the set of source \(A\)-axis parameters that
admit a split match in \(I(a)\) is lower-dimensional and measure zero in the source chart.
\end{proposition}

\begin{proof}
We now spell out the rank condition in coordinates.  Let \(\Theta_A\) be the local chart for a source
\(A\)-axis signature, and let $p_A=\dim \Theta_A$. Fix a nontrivial split assignment
\[
 \tau:\{1,\ldots,L_j\}\to J^B_{I(a)}
\]
whose image contains \(q\ge2\) distinct \(B\)-axes.  Let \(\Theta_B^\tau\) be the local chart for those
\(q\) competitor \(B\)-axes, and put $p_B^\tau=\dim\Theta_B^\tau$. The split assignment imposes a finite system of incidence equations
\[
 H_\tau(\xi,\eta)=0,
 \qquad
 \xi\in\Theta_A,\quad \eta\in\Theta_B^\tau,
\]
saying that each marked piece of the \(A\)-signature is matched by the \(B\)-axis assigned to it by
\(\tau\).  The variables \(\eta\) are nuisance variables: they describe how the competitor
\(B\)-axes are allowed to move while trying to fit the pieces of the fixed \(A\)-signature.

The split-rank witness assumption means that every connected smooth stratum of
\(H_\tau^{-1}(0)\) contains a point where
\[
 \rank D H_\tau\ge p_B^\tau+1.
\]
The domain of \(H_\tau\) has dimension \(p_A+p_B^\tau\).  Let \(S_\tau\) be a connected smooth
stratum of \(H_\tau^{-1}(0)\).  At a rank-witness point,
\[
 T S_\tau\subseteq \ker D H_\tau,
\]
so rank-nullity gives
\[
 \dim S_\tau
 \le
 (p_A+p_B^\tau)-(p_B^\tau+1)
 =
 p_A-1.
\]
Because \(S_\tau\) is a smooth stratum, its dimension is constant; hence the whole stratum has
dimension at most \(p_A-1\).

Project the split-incidence stratum to the source \(A\)-axis chart:
\[
 \pi_A:S_\tau\to\Theta_A.
\]
This projection simply forgets the competitor \(B\)-axes and records which \(A\)-axis parameters
admit that split explanation.  Since \(\pi_A\) is locally Lipschitz, it cannot increase Hausdorff
dimension:
\[
 \dim_H\pi_A(S_\tau)\le \dim S_\tau\le p_A-1.
\]
There are only finitely many counted \(A\)-axes, marked signature pieces, \(B\)-axes in the finite
window, and split assignments \(\tau\).  Taking the finite union over all of them preserves the
dimension bound.  This proves that split matching is lower-dimensional and measure zero.

The inequality \(\rank D H_\tau\ge p_B^\tau+1\) has the same meaning as the split-incidence rank
condition in the ReLU MFR proof.  Rank \(p_B^\tau\) could be used merely to solve for the nuisance
\(B\)-competitors.  The extra \(+1\) says that the proposed split imposes at least one genuine
constraint on the source \(A\)-axis, so it cannot occur on a full-dimensional set of source axes.
\end{proof}

If \(I(a)\) consists of a single \(B\)-layer and the block has one nonlinear step, the depth-warped statement reduces to the ordinary adjacent weak--strong theorem.  The activation-specific minimality/regularity condition is exactly the old private-signature condition.  The new content is that, for unequal-depth blocks, the axis matches are allowed to land at different depths inside \(I(a)\).  The theorem says that the preferred axes of a shallower layer may be distributed across a window of a deeper network: different \(A\)-axes from the same layer can have their singleton matches in different \(B\)-layers. This is the sense in which a deeper model can smear a shallower
model's axis set across depth.

This is different from a ``packet match''.  A packet match would mean that one single \(A\)-axis is
represented only by a collection of several \(B\)-axes, with no single \(B\)-axis having the whole
private signature.  The exact singleton theorem assumes that this split packet alternative is absent, and Proposition~\ref{prop:generic-depth-warp-split} gives a concrete rank-witness condition under which the split alternative is lower-dimensional.  If that rank condition fails, the right conclusion would be a block-packet alignment theorem rather than a singleton axis-alignment theorem.

A useful way to understand the split-rank witness assumption required above is that it says packet explanations are not free parameters of the architecture. A single axis is a ``coherent marked object'': its local signature pieces all come from one scalar argument and share the same downstream mark. A split packet would have to imitate that coherence using several independently parameterized B-side axes. Critically, in an ordinary MLP, CNN, or ResNet block, different channels have separate filters and separate outgoing readout columns; convolution shares one channel across spatial positions, but it does not force several distinct channels to coordinate as one marked unit. Residual depth creates more locations where a feature can appear, but it also does not impose equations requiring several units at different depths to share the same mark and jointly reproduce one source signature. Thus a split-packet explanation requires extra coincidences: matching the right marked pieces, matching compatible readout directions, and doing so while the packet variables are allowed to move. The rank-witness assumption is the formal expression of this non-forcing claim. It says that, after the B-packet variables have been adjusted as freely as possible, the proposed split still imposes at least one independent constraint on the A-side source signature. That is why the split set becomes lower-dimensional rather than a generic alternative. The assumption would actually be somewhat suspect in architectures that \emph{do} structurally force packet behavior, such as explicitly tied multi-branch modules, grouped factorizations, hard bottlenecks, or other mechanisms that make several units share one effective mark by design. So analysis of such networks would have to done on a case-by-case basis.  But at least in ``ordinary'' independently parameterized architectures, however, packet coherence is an extra incidence constraint, not a built-in symmetry.

\subsection{Contravariance for Residual and Transformer Circuits}
\label{sec:transformer-weak-strong}
Here we present a version of the Weak--Strong Equivalence and Zippering theorems for transformer architectures. Because transformers are residual circuits, along the way we will end up also resolving the weak--strong equivalence relationship for residual blocks more generally. 

This section proceeds in four steps.  First, we identify the ``strong objects'' in the MLP stream (the usual privileged axes) and attention stream (``gauge-quotiented heads'').  We then explain why the native coordinates of the residual stream do not obey an analogous weak--strong theorem.  Next, we introduce a technique for handling general residual blocks (``peeling'') and apply it twice within each Transformer block to obtain weak-alignment zippering of the residual stream.  Finally, we show that the additional failures specific to residual peeling occur only on lower-dimensional exceptional sets under the same kind of witness assumptions used earlier.

\subsubsection{Strong objects in Transformer branches}
\label{sec:transformer-branch-strong}

\paragraph{Transformer block notation.}
For a Transformer network \(N\in\{A,B\}\), write \(x^N_\ell(\omega)\in\R^{d^N_\ell}\) for the
residual stream at block \(\ell\) on a task input \(\omega\in\Omega\).  In a pre-normalization block,
write
\[
 u_{\ell,{\rm attn}}^N=\operatorname{LN}_{\ell,{\rm attn}}^N(x_\ell^N),
 \qquad
 y_\ell^N=x_\ell^N+\mathcal A_\ell^N(u_{\ell,{\rm attn}}^N),
\]
and
\[
 u_{\ell,{\rm mlp}}^N=\operatorname{LN}_{\ell,{\rm mlp}}^N(y_\ell^N),
 \qquad
 x_{\ell+1}^N=y_\ell^N+\mathcal M_\ell^N(u_{\ell,{\rm mlp}}^N).
\]
The MLP branch has the one-step form
\[
 \mathcal M_\ell^N(u)
 =
 b_{\ell,{\rm out}}^N+
 \sum_j v_{\ell,j}^N\psi(g_{\ell,j}^N(u)),
 \qquad
 g_{\ell,j}^N(u)=\langle w_{\ell,j}^N,u\rangle+b_{\ell,j}^N,
\]
where \(\psi\in\{\rho,\sigma\}\) is applied coordinatewise.

A single attention head maps a token matrix \(X\) through
\[
 Q_h=XW_{Q,h},\qquad K_h=XW_{K,h},\qquad V_h=XW_{V,h},
\]
and then produces
\[
 L_h(X)=\frac{Q_hK_h^\top}{\sqrt{d_h}},
 \qquad
 P_h(X)=\operatorname{softmax}(L_h(X)),
 \qquad
 O_h(X)=P_h(X)V_hW_{O,h}.
\]
We summarize the behavior of head \(h\) on an input \(X\) by the pair
\[
 \Sigma_h(X):=(P_h(X),O_h(X)),
\]
consisting of its attention pattern and its output.  One may also keep the logit matrix \(L_h\), with two
logit matrices treated as equivalent when they differ by adding a constant to every entry of a row,
since such rowwise shifts do not change softmax.  For the statements below, however, the attention
pattern and output are the quantities that matter.  Biases or affine operations on the branch input can
be included by appending a constant coordinate to each token vector, so no generality is lost by using
the simpler linear notation.

\paragraph{Attention gauges.}
Attention has exact internal symmetries, which prevent its raw internal coordinates from being
strong objects.  We therefore quotient these symmetries out and compare heads only up to the
function-preserving changes of coordinates below.   Let \(R,S\in GL_{d_h}(\R)\), and define
\[
 W_Q'=W_QR,
 \qquad
 W_K'=W_KR^{-\top},
 \qquad
 W_V'=W_VS,
 \qquad
 W_O'=S^{-1}W_O.
\]

\begin{proposition}[Attention coordinate gauges]
\label{prop:attention-gauges}
The primed and unprimed parameter choices compute exactly the same attention head for every input \(X\).
\end{proposition}

\begin{proof}
The transformed query and key matrices are
\[
 Q'=XW_QR=QR,
 \qquad
 K'=XW_KR^{-\top}=KR^{-\top}.
\]
Thus
\[
 Q'(K')^\top=QRR^{-1}K^\top=QK^\top,
\]
so the logits and attention weights are unchanged.  Also
\[
 V'W_O'=XW_VSS^{-1}W_O=VW_O.
\]
Hence the head output is unchanged for every \(X\).
\end{proof}

Two attention heads are regarded as the same head, for purposes of strong comparison, if they are
related by these \(Q/K\) and \(V/O\) gauges, together with the usual permutation of head labels.
Thus the raw coordinates of \(Q,K,V\) are not identifiable.  What can be identified is the head's
attention pattern and output, modulo gauge.

\begin{corollary}[No raw-coordinate weak--strong theorem for attention]
\label{cor:no-raw-attention-axis}
Weak alignment before and after an attention head cannot, by itself, identify raw query, key, or value
coordinate axes.  At most it can identify quantities unchanged by the gauges: the attention pattern, the
logit matrix up to rowwise softmax shifts, the value--output map on the task inputs, and head identity
up to permutation.
\end{corollary}

\begin{proof}
Proposition~\ref{prop:attention-gauges} gives a continuum of different raw \(Q,K,V\) coordinate
systems with exactly the same represented attention function.  A generic dense choice of \(R\) or
\(S\) mixes the raw coordinates, so the transformed coordinates are not permutations or scalar
rescalings of the original coordinates.  Thus raw attention-coordinate alignment is not identifiable
from the represented function.  Any strong theorem for attention must be stated modulo the gauges.
\end{proof}

\begin{definition}[Usedness, budgets, and alignment scores]
\label{def:transformer-branch-objects}
The branch-specific components that can be strongly aligned are:
\begin{enumerate}[label=(\roman*)]
\item \emph{MLP hidden axes} \(g_{\ell,j}\).  \emph{Usedness} here is exactly in the sense of the earlier adjacent ReLU/Softplus weak--strong theorem: ReLU-type axes have crossed kink traces that are visible in the branch output and a nonzero readout,
while Softplus-type axes have nonconstant, nonredundant curvature ridges and a nonzero readout.  Let \(m_{{\rm mlp},\ell}^{\psi}(\epsilon)\) be the minimum number of used MLP hidden axes at branch \(\ell\) needed, within the fixed Transformer macroarchitecture, to solve the task to loss at most \(\epsilon\).  The corresponding denominator is the A-side MLP hidden width
\(d_{{\rm mlp},\ell}^A\).

\item \emph{Attention heads} \(h\).  A head is \emph{used} if the pair \(\Sigma_h=(P_h,O_h)\) varies on the task set and its output
contribution is not identically zero after the output readout.  A used head is minimal if there is some
task region on which its attention pattern and output are not shared by any other counted head, and if
the head cannot be deleted, combined with the branch's linear map and bias, or merged with another
counted head without changing the branch function on the task set.  Let \(m_{{\rm attn},\ell}(\epsilon)\) be the minimum number of used attention heads at block
\(\ell\) needed, within the fixed Transformer macroarchitecture, to solve the task to loss at most
\(\epsilon\).  The corresponding denominator is the A-side number of heads \(H_\ell^A\).
\end{enumerate}
\noindent For MLP branches, \(\AxisAlign_{{\rm mlp},\ell,\psi}(A,B)\) is the usual axis-alignment score
applied at the hidden MLP preactivation axes.  For attention branches, define
\[
 \HeadAlign_{{\rm attn},\ell}(A,B)
 =
 \frac{1}{H_\ell^A}
 \max_\pi
 \#\{h:\Sigma_h^A \text{ matches } \Sigma_{\pi(h)}^B
          \text{ on } \Omega\text{, modulo gauge}\},
\]
where \(\pi\) ranges over injective maps from A-side heads to B-side heads.  For attention, set
\[
 \AxisAlign_{{\rm attn},\ell}(A,B):=\HeadAlign_{{\rm attn},\ell}(A,B).
\]
\end{definition}

A term that is affine in the branch input \(X\) can be combined with the branch's existing linear map
and bias, so it need not be counted as a separate nonlinear component.  Whether that same term can
also be combined with the residual skip is a different question, taken up in
\S\ref{sec:residual-peeling}.

With these definitions in mind, we can state the usual shrinkability result needed to justify the
usedness definitions:

\begin{lemma}[Shrinkability of non-used Transformer axes and heads]
\label{lem:transformer-shrinkability}
The following reductions do not increase the number of counted MLP axes or attention heads.
\begin{enumerate}[label=(\roman*)]
\item Non-used MLP hidden axes may be deleted, merged, or reduced by the earlier ReLU/Softplus
shrinkability lemmas.  Any affine remainder produced by those reductions can be combined with the branch's existing
linear map and bias.
\item An attention head with \(O_h\equiv0\) on the task set may be deleted.
\item If \(P_h\) is constant on the task set, then the head output is linear in the branch input.  It may therefore be combined with the branch's existing linear map and bias.
\item Two heads with the same attention pattern may be merged when the sum of their value-output
products can be represented by one allowed head.
\end{enumerate}
Thus the branch budgets count MLP axes and attention heads whose nonlinear behavior is observable
on the task inputs, rather than merely counting physically present coordinates or heads.  
\end{lemma}

\begin{proof}
For MLP axes, the first statement is exactly the earlier shrinkability calculation.  Zero coefficients
delete.  Constant arguments contribute constants.  Duplicate arguments merge.  For either
\(\psi=\rho\) or \(\psi=\sigma\), an opposite pair obeys
\[
 v_+\psi(g)+v_-\psi(-g)
 =(v_++v_-)\psi(g)-v_-g,
\]
because \(\psi(g)-\psi(-g)=g\).  In the adjacent branch theorem, \(g\) is affine in the common branch input, so the last term is
part of the branch's linear map and bias.

If \(O_h\equiv0\), the head contributes nothing and can be deleted.  If \(P_h=P\) is constant, then
\[
 O_h(X)=PXW_{V,h}W_{O,h},
\]
which is linear in \(X\).  Finally, if two heads have the same pattern \(P\), their
combined contribution is
\[
 PXW_{V,1}W_{O,1}+PXW_{V,2}W_{O,2}
 =PX\bigl(W_{V,1}W_{O,1}+W_{V,2}W_{O,2}\bigr).
\]
This can be represented by one head with the same pattern whenever the matrix in parentheses has
rank no larger than the head's allowed internal dimension.  These operations prove all four claims.
\end{proof}

\begin{definition}[Branch-specific weak--strong hypotheses]
\label{def:transformer-branch-ws-hypotheses}
For an MLP branch, the weak--strong hypotheses are exactly the activation-specific adjacent usedness
hypotheses from the earlier ReLU/Softplus weak--strong theorem.  For an attention branch, the weak--strong hypotheses are:
\begin{enumerate}[label=(A\arabic*)]
\item \emph{head usedness/minimality:} every counted A-side head, and every B-side head after its input
has been composed with the weak input map, is used and minimal; in particular, each has some
attention-pattern-and-output behavior on the task set that is not shared by another counted head on
the same side;
\item \emph{one-head identification:} once one A-side head and one B-side head have been separated from their
respective multihead sums and shown to have the same behavior on the task inputs, they are the same
head modulo the gauges of Proposition~\ref{prop:attention-gauges} and head permutation;
\item \emph{no splitting:} the distinguishing behavior of each source head is reproduced by one head on
the other side, not only by several heads acting together.
\end{enumerate}
The no-splitting condition is needed because a head has two linked pieces of observable behavior---its
attention pattern and its output---and several heads could in principle reproduce those pieces jointly.
This is not a new assumption for MLP branches: the earlier adjacent theorem already handles the
corresponding issue through usedness, reducedness, and the identifying nonlinear behavior of each
axis.
\end{definition}

We next give a concrete and easily checked sufficient condition for (A2).  Define the two product
matrices of a head by
\[
 \Gamma_h:=\frac{W_{Q,h}W_{K,h}^\top}{\sqrt{d_h}},
 \qquad
 C_h:=W_{V,h}W_{O,h}.
\]
Then
\[
 P_h(X)=\operatorname{softmax}(X\Gamma_hX^\top),
 \qquad
 O_h(X)=P_h(X)XC_h.
\]
Choose finitely many task inputs \(\mathcal X=\{X^1,\ldots,X^R\}\), and write \(x_i^r\) for row
\(i\) of \(X^r\).  For a matrix \(D\), set
\[
 \mathcal Q_{\mathcal X}(D)
 :=\bigl((x_i^r)^\top D(x_j^r-x_k^r)\bigr)_{r,i,j,k},
\]
and, for fixed \(\Gamma\),
\[
 \mathcal V_{\Gamma,\mathcal X}(D)
 :=\bigl(P_\Gamma(X^r)X^rD\bigr)_{r=1}^R.
\]
The first map records every logit difference that rowwise softmax can detect; the second records how
changes in the value--output product affect the head output.  A component of \(\Gamma\) or \(C\) is
retained here exactly when changing it can change the attention pattern or output on at least one of the
selected task inputs.  We say that the selected task inputs contain enough variation to determine the
head when the two maps have full column rank on these retained components.

\begin{lemma}[A concrete sufficient condition for identifying one head]
\label{lem:attention-head-regularity}
Suppose two heads have already been separated from their respective multihead sums using (A1) and
(A3), and suppose their attention patterns and outputs agree on \(\mathcal X\).  Then
\[
 \Gamma-\widetilde\Gamma\in\ker\mathcal Q_{\mathcal X},
 \qquad
 C-\widetilde C\in\ker\mathcal V_{\Gamma,\mathcal X}.
\]
Consequently, if the selected task inputs contain enough variation in the sense just defined, the two
heads have the same components of \(\Gamma\) and \(C\) that affect their behavior on those inputs.
If, in addition, no unused internal head coordinates remain---equivalently, after restriction to these
task-relevant products, the query--key and value--output factorizations use latent dimensions equal to
the ranks of the corresponding products---then the individual matrices
\(W_Q,W_K,W_V,W_O\) are related exactly by the \(Q/K\) and \(V/O\) gauges of
Proposition~\ref{prop:attention-gauges}.
\end{lemma}

\begin{proof}
We prove the query--key claim, the value--output claim, and then the claim about the individual matrices.

For two row vectors \(a,b\), equality \(\operatorname{softmax}(a)=\operatorname{softmax}(b)\)
implies
\[
 e^{a_j-a_k}
 =\frac{e^{a_j}}{e^{a_k}}
 =\frac{e^{b_j}}{e^{b_k}}
 =e^{b_j-b_k}
\]
for every \(j,k\).  Hence \(a_j-a_k=b_j-b_k\), so \(a-b\) is constant along the row.  Apply this
row by row to
\[
 P_\Gamma(X^r)=P_{\widetilde\Gamma}(X^r).
\]
Writing \(D=\Gamma-\widetilde\Gamma\), we obtain
\[
 (x_i^r)^\top D(x_j^r-x_k^r)=0
\]
for all \(r,i,j,k\).  This is exactly
\(D\in\ker\mathcal Q_{\mathcal X}\).

The attention patterns now agree on every selected input.  Equality of the two head outputs therefore
gives
\[
 P_\Gamma(X^r)X^r(C-\widetilde C)=0
 \qquad(r=1,\ldots,R),
\]
which is precisely
\(C-\widetilde C\in\ker\mathcal V_{\Gamma,\mathcal X}\).  Full column rank on the retained components forces both differences to vanish there.

It remains to identify the individual matrices once unused internal coordinates have been removed.  Suppose
\[
 W_QW_K^\top=\widetilde W_Q\widetilde W_K^\top
\]
and both product representations use a latent dimension equal to the rank of this common product.  Then
\(W_Q\) and \(\widetilde W_Q\) have full column rank and the same column space, namely the column
space of the product.  Hence \(\widetilde W_Q=W_QR\) for some invertible \(R\).  Let \(L_Q\) be a
left inverse of \(W_Q\).  Substitution gives
\[
 W_QW_K^\top=W_QR\widetilde W_K^\top,
 \qquad
 W_K^\top=R\widetilde W_K^\top,
\]
and therefore \(\widetilde W_K=W_KR^{-\top}\).  The same argument applied to
\(W_VW_O=\widetilde W_V\widetilde W_O\) gives
\[
 \widetilde W_V=W_VS,
 \qquad
 \widetilde W_O=S^{-1}W_O
\]
for an invertible \(S\).  These are exactly the two attention gauges.
\end{proof}

These results now allow us to state the core branch-level Weak--Strong Equivalence result for transformers: 
\begin{theorem}[Transformer branch weak--strong equivalence]
\label{thm:transformer-branch-weak--strong}
Consider either an MLP branch or an attention branch of a Transformer block.  Suppose exact weak
alignment holds at the branch input and branch output.  For an MLP branch, assume the MLP one-step
expansions satisfy the activation-specific adjacent usedness hypotheses.  For an attention branch,
assume (A1)--(A3) of Definition~\ref{def:transformer-branch-ws-hypotheses}.  Then every used A-side MLP hidden axis or attention head has a distinct B-side match: MLP hidden
axes match as scalar axes, and attention heads match modulo attention gauge.  Consequently,
\[
 \AxisAlign_{{\rm mlp},\ell,\psi}(A,B)
 \ge
 \frac{m_{{\rm mlp},\ell}^{\psi}(\epsilon)}{d_{{\rm mlp},\ell}^A}
\]
for MLP branches, and
\[
 \AxisAlign_{{\rm attn},\ell}(A,B)
 =
 \HeadAlign_{{\rm attn},\ell}(A,B)
 \ge
 \frac{m_{{\rm attn},\ell}(\epsilon)}{H_\ell^A}
\]
for attention branches.
\end{theorem}

\begin{proof}
For an MLP branch, write the adjacent weak-alignment equations as
\[
 u^B=E_{\rm in}u^A,
 \qquad
 \mathcal M^B(u^B)=E_{\rm out}\mathcal M^A(u^A).
\]
Substitute the first equality into the second.  The B-side branch becomes
\[
 b^B_{\rm out}+\sum_i v_i^B\psi(g_i^B(E_{\rm in}u^A)),
\]
while the transformed A-side branch is
\[
 E_{\rm out}b^A_{\rm out}+
 \sum_j(E_{\rm out}v_j^A)\psi(g_j^A(u^A)).
\]
These are two one-step affine--coordinatewise-nonlinear expansions of the same common input
variable \(u^A\).  The earlier adjacent weak--strong theorem therefore gives an injective match of
every used A-side MLP hidden axis to a pulled-back B-side axis, with the activation's allowed scalar
symmetry.  Counting the task-required used axes proves the MLP bound.

For attention, first compose the B-side branch input with the weak input map so that both branches
are functions of the same task variable, and apply the weak output map to the A-side branch output so
that the two outputs are written in the same coordinates.  We then have equality of two multihead
branch functions.  By (A1), each counted A-side head has some attention-pattern-and-output behavior
not shared by another counted A-side head, so that behavior must also appear on the B side.  By (A3),
one B-side head reproduces the complete behavior, rather than several heads reproducing different
parts jointly.  Condition (A2) then identifies that head modulo the exact attention gauges and head
permutation.  Lemma~\ref{lem:attention-head-regularity} gives a concrete sufficient rank test for this
step.

The assignment is injective.  If two distinct A-side heads were assigned to the same B-side head,
condition (A2) would identify both A-side heads with that same B-side head.  Their attention patterns
and outputs would therefore agree, contrary to the distinguishing behavior required by (A1).  Thus
every task-required A-side head has a distinct B-side match.  Counting those heads gives the displayed
attention bound.
\end{proof}

As elsewhere in the appendix, we next ask whether (A1)--(A3) are generic.  
The corresponding MLP conditions are already generic by
\S\ref{sec:genericity}.  For attention, we use the same witness principle used there. Here a \emph{witness} is one parameter choice at which the relevant
nonvanishing or rank condition holds.  For (A3), the split-rank witness
additionally says that, even after the competitor heads are allowed to vary,
a proposed split still imposes at least one independent condition on the
source head.

\begin{proposition}[Genericity of the attention hypotheses]
\label{prop:generic-transformer-conditions}
Fix a finite-dimensional real-analytic family of attention-branch
comparisons, together with the finite
task measurements used to test (A1)--(A3).  If the corresponding witnesses
exist, including the split-rank witness of
Proposition~\ref{prop:generic-depth-warp-split} for every proposed split,
then (A1)--(A3) hold generically: the set on which any of them fails is
lower-dimensional, and hence has empty interior and measure zero.
\end{proposition}

\begin{proof}
For (A1), fix the finite task regions or evaluations used in the comparison.  A zero head output, loss of variation, or equality between two heads that were supposed to differ is again expressed by finitely many analytic equations.  The assumed witness shows that these equations do not vanish identically.

For (A2), failure of either map \(\mathcal Q_{\mathcal X}\) or \(\mathcal V_{\Gamma,\mathcal X}\) to have full rank is detected by the
vanishing of all of its largest relevant minors, meaning the square subdeterminants that certify the desired rank.  Failure of the reduced query--key or value--output factors to have full rank is likewise
detected by vanishing minors.  Whenever the required rank occurs at one point, these determinant equations define a
proper lower-dimensional subset.

It remains to handle splitting (A3).  Fix one proposed way, denoted by \(\tau\), for several B-side heads to
reproduce one source head.  Let \(\xi\) be local coordinates for the source head, of dimension \(p\),
and let \(\eta\) collect the parameters of the B-side heads used in the proposed split, of dimension
\(n_\tau\).  The requirement that those heads reproduce the selected attention-pattern-and-output
behavior of the source head can be written as finitely many equations
\[
 H_\tau(\xi,\eta)=0.
\]
Consider any connected smooth piece \(S_\tau\) of the solution set of these equations.  By the
split-rank witness, that piece contains a point where
\[
 \rank DH_\tau\ge n_\tau+1.
\]
At such a point its tangent space lies in \(\ker DH_\tau\), so rank--nullity gives
\[
 \dim S_\tau
 \le(p+n_\tau)-(n_\tau+1)=p-1.
\]
Forgetting the B-side parameters and retaining only \(\xi\) cannot increase Hausdorff dimension.  At
fixed width there are only finitely many proposed splits, and the solution sets can be decomposed into
smooth pieces with only finitely many meeting any small neighborhood.  Their union therefore remains
lower-dimensional and measure zero.  Combining the three kinds of failure proves the proposition.
\end{proof}

This genericity result does not settle the separate question of where learning dynamics place trained networks, but it shows that the branch conditions fail only for specially constrained parameter choices. 

\subsubsection{Why residual coordinates need not be privileged}
\label{sec:transformer-residual-no-axis}
The branch theorem identifies where strong objects live. It does not imply that the native coordinates of the residual stream are likewise privileged. None of those arguments identifies the native coordinates of the residual stream. Actually, this is not merely a limitation of our proof technique above: an exact change of residual basis can preserve the represented network while mixing all native residual coordinates.

\begin{proposition}[Residual weak alignment does not imply residual-coordinate alignment]
\label{prop:residual-no-native-axis}
Consider an idealized class of residual-stream networks closed under invertible residual-basis changes.
Thus, for every invertible \(R\in GL_d(\R)\) and every network \(A\) with residual states
\(x_\ell^A\), there is an equivalent network \(B\) with residual states
\[
 x_\ell^B=Rx_\ell^A
\]
at every layer, obtained by conjugating all residual-stream linear maps by \(R\).  Suppose the
coordinate functions of \(x_\ell^A\) are linearly independent on the task set.  For a generic dense
\(R\), the two networks are exactly weakly aligned at layer \(\ell\), but have no native
residual-coordinate strong alignment at that layer.
\end{proposition}

\begin{proof}
The equation \(x_\ell^B=Rx_\ell^A\) is exact weak alignment with comparison map \(E_\ell=R\).  If
a B-side residual coordinate were strongly aligned with an A-side residual coordinate, then for some
row \(r_i^\top\) of \(R\), some coordinate \(j\), and some nonzero scalar \(\alpha\),
\[
 r_i^\top x_\ell^A(\omega)=\alpha x_{\ell,j}^A(\omega)
 \qquad(\omega\in\Omega).
\]
Linear independence of the residual-coordinate functions forces
\[
 r_i^\top=\alpha e_j^\top.
\]
A generic dense invertible matrix has no row proportional to a coordinate vector.  Hence exact weak
alignment can hold with no native residual-coordinate strong alignment.
\end{proof}

Thus the appropriate residual-stream conclusion is weak affine alignment, not native-coordinate alignment.  We now explain how such a weak map can nevertheless be recovered backward through a residual block.

\subsubsection{Residual blocks and the peeling principle}
\label{sec:residual-peeling}
Turning our attention to zippering, we now ask how weak alignment can be propagated one step
backward through a residual sublayer.  The earlier branch results identify the nonlinear units or heads
inside a branch.  A residual sublayer adds an affine skip path, and that skip provides the route for
recovering the sublayer input once the branch contribution has been removed.  We call this procedure
\emph{residual peeling}.  This subsection proves the result for one abstract residual sublayer; the next
subsection applies it first to the MLP sublayer and then to the attention sublayer of each Transformer
block.

For one network $N\in\{A,B\}$, write a residual sublayer as
\begin{equation}
 \mathcal F^N(z)=S^Nz+c^N+F^N(z),
 \label{eq:residual-block-general-a8}
\end{equation}
where $S^Nz+c^N$ is the affine skip and $F^N$ is the nonlinear branch.  Suppose an injective affine
map $E_+(r)=T_+r+a_+$ already aligns the two sublayer outputs:
\begin{equation}
 \mathcal F^B(z^B(\omega))
 =E_+\mathcal F^A(z^A(\omega)).
 \label{eq:residual-output-comparison}
\end{equation}
The goal of one backward zippering step is to construct an injective affine map $E_-$ satisfying
$z^B=E_-z^A$ on the task set.  The skip suggests how to do this: identify which branch terms match,
subtract those terms from the output equality, and then solve the remaining skip equation for $z^B$.

The point requiring care is the subtraction step.  The branch-identification results may match an MLP
argument with the same or the opposite orientation.  Same-orientation terms cancel directly.  An
opposite-orientation match need not: both ReLU and Softplus satisfy
$\psi(t)-\psi(-t)=t$, so subtracting the two matched nonlinear terms can leave a term proportional to
the branch argument.  In an ordinary residual block, where the argument is affine in the residual input,
$g(z)=a^\top z+b$, this causes no new nonlinear difficulty:
\begin{equation}
 S z+c+v\psi(g(z))
 =\bigl(S+va^\top\bigr)z+(c+vb)+v\psi(-g(z)).
 \label{eq:ordinary-residual-affine-transfer}
\end{equation}
The remaining term $vg(z)$ is affine in $z$, so it merely changes the affine skip.  After that change,
one can still solve an affine equation for the input, provided the resulting map is injective.

A pre-LayerNorm branch is different.  Its scalar argument has the form
$g(\operatorname{LN}(z))$, which is generally not affine in $z$.  An opposite-orientation match can
therefore leave a function of $\operatorname{LN}(z)$ that cannot be included in an affine map of the
residual state.  The following example shows that this can genuinely prevent the aligned outputs from
determining affinely related inputs.

\medskip
\noindent\textbf{Example (LayerNorm can prevent direct recovery of the residual input).}
Let $\psi\in\{\rho,\sigma\}$, so that $\psi(t)-\psi(-t)=t$.  Use two residual features and standard
LayerNorm with $\epsilon>0$, unit gain, and zero learned bias.  On the task interval $-1<t<1$, let
\[
 y(t)=(t,-t),
 \qquad
 \operatorname{LN}(y(t))=\frac{(t,-t)}{\sqrt{t^2+\epsilon}}.
\]
Define
\[
 g(y):=e_1^\top\operatorname{LN}(y),
 \qquad
 f(t):=g(y(t))=\frac{t}{\sqrt{t^2+\epsilon}},
\]
and choose a nonzero readout vector in the all-ones direction, which LayerNorm does not detect:
\[
 v=c(1,1),\qquad c\ne0.
\]
The A-side residual sublayer is
\[
 F_A(y)=y+v\psi(g(y)).
\]
Now change the B-side unnormalized residual input by the input-dependent shift
\[
 \widetilde y(t):=y(t)+vg(y(t))=y(t)+vf(t).
\]
Because LayerNorm is unchanged by adding a scalar multiple of the all-ones vector,
$\operatorname{LN}(\widetilde y(t))=\operatorname{LN}(y(t))$.  Give the B-side unit the opposite
argument
\[
 \widetilde g(z):=-e_1^\top\operatorname{LN}(z).
\]
Then $\widetilde g(\widetilde y(t))=-g(y(t))$, and
\[
 \begin{aligned}
 F_B(\widetilde y(t))
 &=\widetilde y(t)+v\psi(\widetilde g(\widetilde y(t)))\\
 &=y(t)+vg(y(t))+v\psi(-g(y(t)))\\
 &=y(t)+v\psi(g(y(t)))\\
 &=F_A(y(t)).
 \end{aligned}
\]
Nevertheless, no affine map $E$ satisfies $\widetilde y(t)=E(y(t))$ on the interval.  Every affine
function of $y(t)=t(1,-1)$ is affine in $t$, whereas the all-ones component of $\widetilde y(t)$ is
$cf(t)$ and
\[
 f''(t)=-\frac{3\epsilon t}{(t^2+\epsilon)^{5/2}}
\]
is not identically zero.  Thus the two residual-sublayer outputs agree exactly even though their
unnormalized inputs are not affinely related.
\phantomsection\label{prop:layernorm-skip-transfer}
\medskip

The example shows exactly what must be checked in addition to branch identification.  The earlier
results can correctly pair the two nonlinear units, but after their contributions are compared, the term
$v(g\circ\operatorname{LN})$ remains.  A change in the residual input that LayerNorm does not detect
can cancel this term through the skip.  Consequently, output alignment determines an affine relation
between the inputs only when every term remaining after the branch comparison is itself affine in the
unnormalized residual state.

We use the following terminology for that requirement.  Relative to a residual-state family
$z:\Omega\to\mathbb R^d$, a term $r$ is \emph{affine in the unnormalized residual state} if
\begin{equation}
 r(\omega)=Az(\omega)+c
 \qquad(\omega\in\Omega)
 \label{eq:raw-residual-affine}
\end{equation}
for a fixed matrix $A$ and vector $c$.
\phantomsection\label{def:raw-residual-affine}
Notice that being affine in the branch input $X=\operatorname{LN}(z)$ is not enough: the weak map we
seek must be affine in $z$ itself.

For the peeling calculation, we now write the same residual sublayer in a form that keeps this point
explicit.  Apply the shrinkability reductions of Lemma~\ref{lem:transformer-shrinkability} to delete or
merge non-used branch terms.  If a term produced by those reductions has the form
\eqref{eq:raw-residual-affine}, combine it with the skip.  If it is affine only in
$\operatorname{LN}(z)$, leave it in the branch.  This is merely a way of organizing the same represented
function; it ensures that the affine part contains only terms that really are affine in the residual state.
Thus write
\begin{equation}
 \mathcal F^N(z^N)
 =S^Nz^N+c^N+\mathcal B^N(z^N),
 \label{eq:reduced-residual-sublayer}
\end{equation}
where $S^Nz^N+c^N$ contains the skip together with any branch terms already shown to be affine in
$z^N$, and $\mathcal B^N$ contains all remaining branch contributions.  For an untouched identity
skip, $S^N=I$.

We can now state precisely what remains after the branch-identification step.  Let the relevant MLP or
attention result provide a matching between the A-side and B-side branch objects.  Reindex the B-side
objects according to that matching.  For attention, Proposition~\ref{prop:attention-gauges} also allows
us to replace a matched head by a gauge-equivalent parameterization without changing the function it
computes.  These notational choices let us compare each matched pair directly.  Subtract the matched
branch contributions from \eqref{eq:residual-output-comparison}, placing any difference that does not
cancel into a single term $R(\omega)$.  The remaining equality is
\begin{equation}
 S^Bz^B+c^B
 =T_+S^Az^A+T_+c^A+a_+ +R(\omega).
 \label{eq:post-branch-cancellation}
\end{equation}
If $R$ is affine in $z^A$, this is an affine equation for $z^B$.  If $R$ is not affine in $z^A$, the
LayerNorm example shows that no affine input comparison need follow.

There is one further requirement.  Solving the affine equation must produce an \emph{injective} map,
because weak alignment is defined using an injective affine comparison.  This gives the following two
conditions.

\begin{definition}[Conditions for peeling one residual sublayer]
\label{def:transformer-peelable}
Assume the current branch in \eqref{eq:residual-output-comparison} has been identified by the relevant
earlier MLP or attention result, and let $R$ be the term in
\eqref{eq:post-branch-cancellation}.  The residual sublayer can be peeled when:
\begin{enumerate}[label=(\roman*)]
\item $R$ is affine in the unnormalized A-side residual state,
\begin{equation}
 R(\omega)=Lz^A(\omega)+d
 \label{eq:raw-affine-transfer}
\end{equation}
for fixed $L,d$; and
\item $S^B$ is invertible and $T_+S^A+L$ is injective.
\end{enumerate}
The first condition ensures that branch subtraction leaves an affine equation in the residual states.
The second ensures that solving this equation gives a valid injective weak comparison map.
\end{definition}

\begin{lemma}[Peeling one residual sublayer]
\label{lem:transformer-residual-peel}
\label{lem:transformer-mlp-transfer-remainder}
\label{lem:relu-layernorm-trace-recovery}
Suppose the current branch satisfies the relevant earlier identification hypotheses---Softplus
minimality plus MLR, the corresponding ReLU private-zero-set conditions, or attention conditions
(A1)--(A3)---and suppose the two conditions of
Definition~\ref{def:transformer-peelable} hold.  Then there is an injective affine map $E_-$ such that
\[
 z^B=E_-z^A
\]
on the task set.  Every counted MLP axis or attention head identified in the branch comparison is
matched as well.
\end{lemma}

\begin{proof}
We first identify the branch terms inside the full residual-sublayer equality.  For Softplus, the
private-pole and MLR argument applies to the composed scalar functions.  For attention, (A1)--(A3)
separate each head from the multihead sum and identify it modulo gauge.  For ReLU, let $g_j$ be one
counted scalar function after LayerNorm and choose a point $p$ on a distinctive smooth piece of
$\{g_j=0\}$, away from the other counted zero sets, with $dg_j(p)\ne0$.  Across this piece,
\[
 D_+\bigl(v_j\rho(g_j)\bigr)-D_-\bigl(v_j\rho(g_j)\bigr)
 =v_j\,dg_j(p)\ne0.
\]
The skip and the other current terms are smooth there, so an equal realization on the other side must
have a matching zero-set piece.  The same coherence and regularity conditions used earlier collect the
matched pieces onto one scalar function and determine it up to sign and positive scale; comparing the
derivative jumps determines the corresponding readout coefficient.  Applying an injective affine map
to the output cannot erase the nonzero jump.

Let $\pi$ denote the resulting matching, and reindex the B-side terms by $\pi$.  In the attention case,
choose gauge-equivalent representatives for the matched heads; Proposition~\ref{prop:attention-gauges}
shows that this does not change their outputs.  Matched attention-head outputs then cancel directly.

For an MLP branch, positively oriented matches also cancel directly.  To compute what remains from
negative orientations, write the output comparison as
\begin{equation}
 S^Bz^B+c^B+\sum_i v_i^B\psi(h_i)
 =T_+S^Az^A+T_+c^A+a_+
  +\sum_jT_+v_j^A\psi(g_j).
 \label{eq:residual-mlp-comparison}
\end{equation}
In the Softplus case,
$h_{\pi(j)}=\varepsilon_jg_j$ and $v_{\pi(j)}^B=T_+v_j^A$; in the ReLU case,
$h_{\pi(j)}=\varepsilon_j\alpha_jg_j$ and
$\alpha_jv_{\pi(j)}^B=T_+v_j^A$, with $\alpha_j>0$.  Let
$S_-:=\{j:\varepsilon_j=-1\}$.  For $j\in S_-$, the identity
$\psi(-g)-\psi(g)=-g$ gives
\[
 v_{\pi(j)}^B\psi(h_{\pi(j)})-T_+v_j^A\psi(g_j)
 =-T_+v_j^Ag_j,
\]
where ReLU positive homogeneity absorbs $\alpha_j$.  Summing over the matched pairs gives
\begin{equation}
 S^Bz^B+c^B
 =T_+S^Az^A+T_+c^A+a_+ +R_{S_-},
 \qquad
 R_{S_-}:=\sum_{j\in S_-}T_+v_j^Ag_j.
 \label{eq:mlp-transfer-remainder}
\end{equation}
Thus in the MLP case the term $R$ from \eqref{eq:post-branch-cancellation} is explicitly
$R_{S_-}$; in the attention case it is whatever remains after the matched head outputs have been
subtracted.

By condition (i), this remaining term has the form $R=Lz^A+d$.  Substitution into
\eqref{eq:post-branch-cancellation} gives
\[
 S^Bz^B+c^B
 =T_+S^Az^A+T_+c^A+a_++Lz^A+d.
\]
Since $S^B$ is invertible, solving for $z^B$ yields
\begin{equation}
 E_-(z)
 =(S^B)^{-1}(T_+S^A+L)z
 +(S^B)^{-1}(T_+c^A+a_++d-c^B).
 \label{eq:peeled-map}
\end{equation}
Condition (ii) makes the linear part of $E_-$ injective, so $E_-$ is a valid weak comparison map.
The branch-object matching was obtained before solving this affine equation and therefore remains in
force.
\end{proof}

\subsubsection{Peeling and Zippering for Transformers}
\label{sec:transformer-residual-zippering}
If we are comparing corresponding blocks in two transformers $A$ and $B$, every compared block has the two-step form
\begin{equation}
\begin{aligned}
 y_\ell^N
 &=x_\ell^N+
   \mathcal A_\ell^N\!\left(
     \operatorname{LN}_{\ell,{\rm attn}}^N(x_\ell^N)
   \right),\\
 x_{\ell+1}^N
 &=y_\ell^N+
   \mathcal M_\ell^N\!\left(
     \operatorname{LN}_{\ell,{\rm mlp}}^N(y_\ell^N)
   \right),
 \qquad N\in\{A,B\}.
\end{aligned}
\label{eq:pre-normalization-transformer-block}
\end{equation}
The forward computation applies attention first and the MLP second: thus, the two-step backward peeling argument we need to do proceeds in the reverse order.  Starting from the weak map $E_{\ell+1}$ at the right side of
Fig.~\ref{fig:transformer-architecture-ws}, we first peel the MLP sublayer and obtain the intermediate
map $E_{\ell+\frac12}$ at $y_\ell$; we then peel the attention sublayer and obtain $E_\ell$ at
$x_\ell$.  These are the three dashed residual-stream maps shown in the figure.  Formally:

\begin{lemma}[Two-step backward peeling of a Transformer block]
\label{lem:transformer-block-identifiability}
Suppose
\[
 x_{\ell+1}^B=E_{\ell+1}x_{\ell+1}^A
\]
on the task set.  Assume that the MLP sublayer satisfies its activation-specific branch-identification
hypotheses and the two peeling conditions of Definition~\ref{def:transformer-peelable}, and that the
attention sublayer satisfies (A1)--(A3) and the same two peeling conditions.  Then there are injective
affine maps $E_{\ell+\frac12}$ and $E_\ell$ such that
\[
 y_\ell^B=E_{\ell+\frac12}y_\ell^A,
 \qquad
 x_\ell^B=E_\ell x_\ell^A.
\]
The first step also matches all counted MLP hidden axes, and the second matches all counted attention
heads modulo gauge.
\end{lemma}

\begin{proof}
The proof consists of the two backward peeling steps displayed in
Fig.~\ref{fig:transformer-architecture-ws}.

\emph{First step: peel the MLP sublayer.}
The second line of \eqref{eq:pre-normalization-transformer-block} gives
\[
 x_{\ell+1}^N
 =y_\ell^N+
   \mathcal M_\ell^N\!\left(
     \operatorname{LN}_{\ell,{\rm mlp}}^N(y_\ell^N)
   \right).
\]
Substituting this formula for $N=A,B$ into the assumed comparison at $x_{\ell+1}$ yields
\[
 y_\ell^B+
 \mathcal M_\ell^B\!\left(
   \operatorname{LN}_{\ell,{\rm mlp}}^B(y_\ell^B)
 \right)
 =E_{\ell+1}\!\left[
 y_\ell^A+
 \mathcal M_\ell^A\!\left(
   \operatorname{LN}_{\ell,{\rm mlp}}^A(y_\ell^A)
 \right)
 \right].
\]
This is exactly a comparison of two residual sublayers with downstream map $E_{\ell+1}$ and raw
residual inputs $y_\ell^A,y_\ell^B$.  Lemma~\ref{lem:transformer-residual-peel} first identifies the
counted MLP axes inside this equality, then subtracts their matched contributions, and finally solves
the remaining affine skip equation.  It therefore produces an injective affine map
$E_{\ell+\frac12}$ satisfying
\[
 y_\ell^B=E_{\ell+\frac12}y_\ell^A.
\]
The half-index emphasizes that this map compares the intermediate residual states after attention and
before the MLP: it is the middle dashed arrow in Fig.~\ref{fig:transformer-architecture-ws}.

\emph{Second step: peel the attention sublayer.}
The first line of \eqref{eq:pre-normalization-transformer-block} gives
\[
 y_\ell^N
 =x_\ell^N+
   \mathcal A_\ell^N\!\left(
     \operatorname{LN}_{\ell,{\rm attn}}^N(x_\ell^N)
   \right).
\]
Substituting this formula into the newly obtained comparison at $y_\ell$ gives
\[
 x_\ell^B+
 \mathcal A_\ell^B\!\left(
   \operatorname{LN}_{\ell,{\rm attn}}^B(x_\ell^B)
 \right)
 =E_{\ell+\frac12}\!\left[
 x_\ell^A+
 \mathcal A_\ell^A\!\left(
   \operatorname{LN}_{\ell,{\rm attn}}^A(x_\ell^A)
 \right)
 \right].
\]
This is a second residual-sublayer comparison, now with downstream map $E_{\ell+\frac12}$ and raw
inputs $x_\ell^A,x_\ell^B$.  Applying Lemma~\ref{lem:transformer-residual-peel} again identifies
and subtracts the attention branch, matches the counted heads modulo gauge, and produces an
injective affine map $E_\ell$ satisfying
\[
 x_\ell^B=E_\ell x_\ell^A.
\]
This is the left dashed arrow in Fig.~\ref{fig:transformer-architecture-ws}.  Lemma
\ref{lem:attention-head-regularity} gives the more concrete statement that the parts of each matched
head product that affect the task inputs agree and, after unused internal coordinates have been removed,
the individual head matrices agree modulo gauge.

Notice that neither step tries to commute an affine comparison map through LayerNorm.  The branch
objects are identified inside the full residual-sublayer equality, and the unnormalized residual input
is recovered only after that branch has been removed.
\end{proof}

We can now just apply this two-step backward peeling argument iteratively to achieve weak zippering: 

\begin{theorem}[Weak zippering for Transformer residual streams]
\label{thm:transformer-weak-zippering}
Let $A$ and $B$ be Transformer networks whose compared blocks have the pre-normalization form
\eqref{eq:pre-normalization-transformer-block}; thus LayerNorms may occur throughout the networks,
but each one is applied before its local attention or MLP branch.  Suppose exact weak alignment holds
at a terminal residual layer $s$:
\[
 x_s^B=E_sx_s^A
\]
for an injective affine map $E_s$.  Assume that, at every MLP residual sublayer from block $s-1$ down to $q$, the existing activation-specific hypotheses identify the current MLP axes; at
every attention residual sublayer, (A1)--(A3) hold; and each residual sublayer satisfies the two
peeling conditions of Definition~\ref{def:transformer-peelable}.  Then residual weak alignment
zippers upstream.  For every $\ell=q,\ldots,s-1$, there are injective affine maps
$E_{\ell+\frac12}$ and $E_\ell$ such that
\[
 y_\ell^B=E_{\ell+\frac12}y_\ell^A,
 \qquad
 x_\ell^B=E_\ell x_\ell^A
 \qquad\text{on }\Omega.
\]
\end{theorem}

\begin{proof}
We apply Lemma~\ref{lem:transformer-block-identifiability} repeatedly, moving from the terminal
residual state toward the input.  The induction invariant is that, before processing block $\ell$, an
injective affine map $E_{\ell+1}$ has already been obtained at the block's output:
\[
 x_{\ell+1}^B=E_{\ell+1}x_{\ell+1}^A.
\]

For the first step, take $\ell=s-1$.  The required output comparison is the assumed terminal relation
$x_s^B=E_sx_s^A$.  Apply the block lemma to block $s-1$.  Its first peeling step removes the MLP
branch and gives
\[
 y_{s-1}^B=E_{s-\frac12}y_{s-1}^A;
\]
its second peeling step removes the attention branch and gives
\[
 x_{s-1}^B=E_{s-1}x_{s-1}^A.
\]
Thus one backward block step follows the three dashed maps in
Fig.~\ref{fig:transformer-architecture-ws} from right to left:
\[
 E_s\ \Longrightarrow\ E_{s-\frac12}\ \Longrightarrow\ E_{s-1}.
\]

Now suppose that the induction invariant holds for some block $\ell$, with
$q\leq\ell\leq s-2$.  The hypotheses of the theorem ensure that both residual sublayers of block
$\ell$ satisfy the assumptions of the block lemma.  Applying it first to the MLP sublayer gives the
intermediate comparison
\[
 y_\ell^B=E_{\ell+\frac12}y_\ell^A,
\]
and applying it next to the attention sublayer gives
\[
 x_\ell^B=E_\ell x_\ell^A.
\]
The latter relation is exactly the output comparison needed to begin the same argument for block
$\ell-1$.  Each newly produced map is injective because the one-sublayer peeling lemma constructs it
from the two conditions in Definition~\ref{def:transformer-peelable}.

Repeating this right-to-middle-to-left step for
$\ell=s-1,s-2,\ldots,q$ produces every map $E_{\ell+\frac12}$ and $E_\ell$ in the statement.  At
the same time, the MLP peel at each block matches the counted MLP axes, and the attention peel matches
the counted attention heads modulo gauge.  This completes the backward zippering induction.
\end{proof}

Combining all our key results about transformers together, we get:

\begin{corollary}[Strong objects throughout Transformer branches]
\label{cor:transformer-strong-objects-throughout}
Under Theorem~\ref{thm:transformer-weak-zippering}, every peeling step also matches the task-used
strong objects in the branch being removed.  Hence, at every compared block,
\[
 \AxisAlign_{{\rm mlp},\ell,\psi}(A,B)
 \ge
 \frac{m_{{\rm mlp},\ell}^{\psi}(\epsilon)}{d_{{\rm mlp},\ell}^A},
 \qquad
 \HeadAlign_{{\rm attn},\ell}(A,B)
 \ge
 \frac{m_{{\rm attn},\ell}(\epsilon)}{H_\ell^A}.
\]
Thus, whenever the exact hypotheses of the zippering theorem hold, Transformers have strong objects
throughout their non-residual branches: every task-required MLP hidden axis matches as a scalar
function on the task set, and every task-required attention head matches through its attention pattern
and output modulo gauge.  If every physical MLP axis or attention head is task-required, the
corresponding score is one.
\end{corollary}

\begin{proof}
At each block, the first step in the block lemma identifies every counted MLP axis before subtracting
the MLP branch, while the second step identifies every counted attention head before subtracting the
attention branch.  By definition of the two task budgets, any solution of loss at most $\epsilon$
contains at least the displayed number of task-used axes or heads.  Dividing those matched counts by
the A-side physical widths gives the two inequalities.  If the budgets equal the physical counts, every
counted axis or head is matched.
\end{proof}

\begin{remark}[How this explains recent findings from Kapoor--Srivastava--Khosla]
Theorem~\ref{thm:transformer-weak-zippering},
Proposition~\ref{prop:residual-no-native-axis}, and
Corollary~\ref{cor:transformer-strong-objects-throughout} allow us to make sense of some recent results about transformer models.  Kapoor, Srivastava, and Khosla compare vision models using linear regression, orthogonal Procrustes, permutation, and soft-matching metrics~\citep{Kapoor2025Bridging}.  For ViTs,
they find strong hierarchical representational alignment under linear and Procrustes-style metrics, but weak evidence for privileged native axes under permutation-style tests.  Linear and Procrustes metrics test weak alignment of the residual or CLS-token representation, which the zippering theorem predicts under its stated conditions.  Permutation and soft-matching tests on residual coordinates instead ask whether the observed residual basis itself is privileged; Proposition~\ref{prop:residual-no-native-axis}
shows that this need not be true.  The positive strong-alignment prediction is the branch-specific one in Corollary~\ref{cor:transformer-strong-objects-throughout}: task-used MLP hidden axes should align, and task-used attention heads should align modulo gauge.  Thus a negative result in the residual basis
is not evidence against Transformer weak--strong equivalence; it indicates that the test was applied to the stream in which the theory makes no native-axis claim.  
\end{remark}

\subsubsection{Why the peeling conditions are generic}
\label{sec:transformer-peeling-genericity}
The exact zippering theorem assumes the two peeling conditions in
Definition~\ref{def:transformer-peelable}.  We now ask how restrictive those conditions are.  They can fail in two different ways.  First, after the nonlinear branch has been removed, the affine map obtained from the skip equation might fail to be injective.  This is an ordinary matrix-rank failure: it means that the candidate upstream map collapses some residual-stream direction and therefore is not a valid weak comparison map.  Second, the remainder after branch subtraction might still be nonlinear in the unnormalized residual state, as in the LayerNorm example.  This is the genuinely new residual-block failure.  The next lemma treats the first issue; the proposition that follows treats the second.

\begin{lemma}[The affine injectivity condition]
\label{lem:transformer-affine-rank}
Let
\[
 M_-:=T_+S^A+L.
\]
The injectivity requirement in Definition~\ref{def:transformer-peelable} has the following properties.
\begin{enumerate}[label=(\alph*)]
\item If $S^A$ and $S^B$ are invertible and $L=0$, then it follows from injectivity of $T_+$.  In
particular, it is automatic for untouched identity skips.
\item If $S^B$ is invertible, $T_+S^A$ is injective, and
\[
 \|L\|_{\mathrm{op}}<\sigma_{\min}(T_+S^A),
\]
then $M_-$ is injective.  Thus injectivity is preserved under sufficiently small affine changes.
\item In a finite-dimensional real-analytic local coordinate system, or in a fixed ReLU region with an
unchanged activation pattern, failure of injectivity is lower-dimensional whenever injectivity holds at
one point of each connected coordinate region.
\end{enumerate}
\end{lemma}

\begin{proof}
For (a), an invertible $S^A$ followed by an injective $T_+$ is injective, and multiplication by the
invertible matrix $(S^B)^{-1}$ does not change injectivity.  For (b),
\[
 \sigma_{\min}(T_+S^A+L)
 \ge \sigma_{\min}(T_+S^A)-\|L\|_{\mathrm{op}}>0.
\]
For (c), $S^B$ is invertible exactly when $\det S^B\ne0$, and $M_-$ is injective exactly when
$\det(M_-^\top M_-)\ne0$.  These determinants are analytic functions of the local coordinates (and
polynomial on a fixed ReLU region).  A point where the rank condition holds is a witness that the
relevant determinant is not identically zero, so its zero set is a proper lower-dimensional subset.
\end{proof}

Thus the affine injectivity condition is automatic in the basic identity-skip case with no affine
remainder, stable under small changes, and generic whenever it holds somewhere in the parameter
family.  It is not another nonlinear-identification assumption; it only ensures that the affine relation
left after peeling qualifies as weak alignment.

The remaining question is whether a nonlinear LayerNorm cancellation can persist over an open family
of residual sublayers.  A \emph{witness} below means one parameter choice at which the stated rank is
actually attained, showing that the corresponding equations are not identically dependent.  We say
that the architecture \emph{forces} a cancellation only when an exact symmetry or tied mechanism
produces it for every nearby parameter choice, rather than merely at specially tuned values.

\begin{proposition}[Genericity of residual peeling]
\label{prop:transformer-transfer-rank}
\label{prop:generic-transformer-residual-peeling}
Fix finite-dimensional analytic local coordinate systems for a source residual sublayer, a B-side
sublayer, and the affine output comparison map; for ReLU, restrict when needed to a fixed smooth
region with an unchanged zero-trace pattern.  Assume the usual first-order regularity after the exact
attention gauges, hidden-unit order, and positive ReLU scales have been fixed.  Also assume that every
proposed nonlinear cancellation has a rank witness showing that, even after all B-side and comparison
parameters are allowed to vary, the cancellation still imposes at least one independent condition on
the source sublayer, and that no additional cancellation is forced by the architecture over an open
parameter family.  Then failure of residual peeling is lower-dimensional, and hence has empty interior
and measure zero, in the source local coordinate system.
\end{proposition}

\begin{proof}
First fix one proposed cancellation mechanism $\tau$.  Let the source coordinates be
$\xi\in\Theta_A$, with $\dim\Theta_A=p$, and let
$\eta\in\mathbb R^{n_\tau}$ collect all B-side, residual-input, and comparison parameters allowed to
vary.  Write the finitely many task-evaluation, derivative, pole-or-zero-trace, and attention-output
equations required by the cancellation as
\[
 H_\tau(\xi,\eta)=0.
\]
At a rank witness,
\[
 \rank DH_\tau\ge n_\tau+1.
\]
Hence every connected smooth piece $S_\tau$ of the solution set satisfies
\[
 \dim S_\tau\le(p+n_\tau)-(n_\tau+1)=p-1.
\]
Forgetting $\eta$ and retaining only $\xi$ cannot increase Hausdorff dimension, so the source
parameters admitting this cancellation have dimension at most $p-1$.  A finite union over proposed mechanisms has the same bound.  The extra $+1$ is important: rank
$n_\tau$ could merely determine the B-side and comparison parameters, while rank
$n_\tau+1$ says that the cancellation also constrains the source itself.

For a possible cancellation not listed in advance, consider a smooth family $S$ of bad exact pairs on
the region where the branch-identification and affine-rank conditions already hold.  First-order
regularity says that, after exact symmetries are fixed, varying only the B-side and comparison parameters
changes the represented output unless those parameters do not change.  Therefore the map from $S$ to
the source coordinates has injective derivative, so $\dim S\le p$.  If equality held, the inverse
function theorem would make the B-side and comparison parameters smooth functions of the source over
an open set.  By the assumption that the architecture does not force the cancellation, this family
contains a point where the leftover after branch subtraction is affine in the unnormalized source
residual state.  Lemma~\ref{lem:transformer-residual-peel} then produces an injective affine input map,
contradicting that $S$ consists of peeling failures.  Thus these unlisted failures also have dimension at
most $p-1$ after projection to the source coordinates.

Finally, failure of branch identification is lower-dimensional by
Proposition~\ref{prop:generic-transformer-conditions} and the earlier MLP genericity results, while the
affine rank failure is lower-dimensional by Lemma~\ref{lem:transformer-affine-rank}.  Adding these
exceptional sets preserves the conclusion.
\end{proof}

This result of course doesn't say that the LayerNorm example is impossible (since it obviously exists); rather, it says that such a
cancellation cannot persist throughout an open parameter family unless the architecture itself supplies an exact symmetry or tied mechanism that produces it for every nearby parameter choice.  As elsewhere in the paper, applying the generic conclusion to learned networks also requires the optimizer not to concentrate on the exceptional set, which we have not shown.

\subsection{Regularity and depth faithfulness}
\label{sec:depth_faith}
From here on out, the formal proof of our main results is done.  However, the regularity assumptions also have useful interpretive consequences, which are covered in this and the following sections.

The purpose of this section is to explain why regularity makes network depth a meaningful measure of complexity.  For ordinary affine--ReLU and affine--Softplus layers, once a private current nonlinear signature has been matched, regularity says that the unit cannot be recreated by functions already available at lower depth, except through the activation's unavoidable symmetries.  The new layer therefore contributes genuinely new nonlinear structure rather than merely rewriting a shallower computation.  This is a local version of the usual depth-separation idea in neural-network theory~\citep{EldanShamir2016Depth,Telgarsky2016Benefits}.

The residual-block and Transformer argument in \S\ref{sec:transformer-weak-strong} adds one distinction.  A term may be affine in the normalized branch input \(\operatorname{LN}(x)\) while remaining nonlinear in the raw residual state \(x\).  Such a term cannot simply be moved into the residual skip; only a term of the form \(Lx+c\) can.  Proposition~\ref{prop:layernorm-skip-transfer} gives an exact special-case counterexample in which LayerNorm hides a nonlinear branch term.  Residual peeling handles it in the same way as the earlier null-pair examples: first recover the current branch, then show that the remaining nonaffine branch--skip cancellations occur only on a lower-dimensional exceptional set.

\subsubsection{Softplus units cannot be pushed to lower layers}

At a backward zippering step $r$, the scalar arguments $g_j=z_{r,j}$ have already been produced by the lower part of the network.  The current layer applies $\sig$ to these arguments and then linearly recombines them.  To say that a Softplus unit at this stage is genuinely new, we compare it to a fixed space of functions that were already available before applying this current outer Softplus.

\begin{definition}[Softplus lower-depth space]
Fix a backward step $r$ and a chosen lower-depth model class, such as the layers below $r$ with specified widths, depth, and parameter bounds.  A \emph{Softplus lower-depth space} is a closed linear subspace
$\calL_r^\sig\subset L^2(\Omega)$
representing scalar functions already available before the current outer Softplus stage.  By construction, we assume that it contains constants, the current preactivations,
\[
 1,\qquad z_{r,1},\ldots,z_{r,d_r},
\]
and any other scalar functions representable by the chosen lower-depth model class.  Thus ``lower-depth'' means: available from the input and earlier layers, without using another copy of the current Softplus stage.
\end{definition}
Now suppose two matched Softplus arguments $g$ and $h$ share the same local outer pole patch.  Writing
\[
 Q_g=1+e^{g^\C},\qquad Q_h=1+e^{h^\C},
\]
this means locally that the common pole divisor can be factored out:
\[
 Q_h=uQ_g,
\]
where $u$ is holomorphic and nonzero along the shared patch.  On the real task domain, with a compatible logarithm branch, logarithms are well defined and we have:
\[
 \log u=\sig(h)-\sig(g).
\]
So the difference $\sig(h)-\sig(g)$ is the leftover part after the shared Softplus pole signature has been removed.

\begin{definition}[Softplus depth-faithfulness]
Let $\calL_r^\sig$ be a Softplus lower-depth space.  A matched pair $(g,h)$ is \emph{Softplus depth-faithful over $\calL_r^\sig$} if a trivial remainder can only happen if there is equality up to sign:
\[
 \sig(h)-\sig(g)\in\calL_r^\sig
 \quad\Longrightarrow\quad
 h=g\quad\text{or}\quad h=-g.
\]
\end{definition}
\noindent In the above definition, (he negative case is unavoidable for a single pair because $$\sig(-g)-\sig(g)=-g$$ and by definition $g$ itself is already a lower-depth preactivation.  The depth-faithfulness definition basically says: depth is a meaningful notion of computational complexity.   

Now, the key point is that MLR is sharp condition for forcing depth faithfullness,  which explains why regularity is the ``right'' notion for ruling out null pairs in the first place:   

\begin{proposition}[MLR implies matched-pair faithfulness]
If the matched pair $(g,h)$ is MLR, then
$h=\pm g$. In a full minimal expansion, the actual match $h=g$ obtains.
\end{proposition}

\begin{proof}
The matched-pair argument inside Theorem~\ref{thm:sp-onestep} shows that an MLR divisor match has only two possibilities:
\[
 h=g\quad\text{or}\quad h=-g.
\]
In the first case there is no leftover term:
\[
 \sig(h)-\sig(g)=0\in\calL_r^\sig.
\]
In the second case,
\[
 \sig(h)-\sig(g)=\sig(-g)-\sig(g)=-g\in\calL_r^\sig,
\]
because $g$ is one of the lower-depth preactivations.  Thus MLR gives exactly the Softplus depth-faithfulness dichotomy for this pair.  Finally, if a negatively oriented match occurred inside an equality of complete expansions, the terms with negative orientation would contribute a lower-depth affine combination
\[
 \sum_{j\in S}v_jg_j,
\]
which is forbidden by sign-noncancellation unless $S$ is empty.  Hence the complete minimal expansion keeps only the positive matches.
\end{proof}

\subsubsection{ReLU exposed facets cannot be pushed to lower layers}

We now illustrate that the exat analogous depth faithfulness arises for ReLU from the MFR condition (this is actually in a way how MFR was inspired in the first place).  For ReLU, the lower part of the network already determines a polyhedral cell decomposition $\calP$.  On each cell of $\calP$, all lower-depth scalar functions are affine.  A current ReLU unit is genuinely new only if its fresh facet creates a kink that cannot be represented by a function affine on those inherited cells.

\begin{definition}[ReLU lower-depth space and depth-faithfulness test]
Let $\calP$ be the common inherited affine complex for the ReLU scalar arguments being compared.  The \emph{ReLU lower-depth space} associated with $\calP$ is
\[
 \calL(\calP)
 =
 \{f\in L^2(\Omega): f|_C \text{ is affine on every full-dimensional cell } C\in\calP\}.
\]
Thus $\calL(\calP)$ is the space of scalar functions that do not introduce a new ReLU kink across the marked fresh facets. For scalar arguments $g,h$ and a scale $\alpha>0$, define the matched residual
\[
 R_\alpha(h,g)=\rhoR(h)-\alpha\rhoR(g).
\]
The pair $(g,h)$ is \emph{lower-depth} if $R_\alpha(h,g)\in\calL(\calP)$ for some $\alpha>0$, and \emph{ReLU depth-faithful over $\calL(\calP)$} if this can happen only in the two sign-scale cases
\[
 \rhoR(h)-\alpha\rhoR(g)\in\calL(\calP)
 \quad\Longrightarrow\quad
 h=\alpha g\quad\text{or}\quad h=-\alpha g.
\]
\end{definition}

\begin{proposition}[MFR implies ReLU matched-pair depth faithfulness]
If $(g,h)$ is coherently MFR, then there are $\alpha>0$ and $\eps\in\{+1,-1\}$ such that
$h=\eps\alpha g$. For this $\alpha$,
\[
 \rhoR(h)-\alpha\rhoR(g)
 =\begin{cases}0,&\eps=+1,\\-\alpha g,&\eps=-1,\end{cases}
 \in\calL(\calP).
\]
Thus MFR supplies the ReLU depth-faithfulness dichotomy for the matched pair.  In a full minimal expansion, only the positive scale remains valid.
\end{proposition}

\begin{proof}
By Definition~\ref{def:MFR-unified}, coherent MFR gives
\[
 h=\eps\alpha g
\]
for some $\alpha>0$ and $\eps\in\{+1,-1\}$.  If $\eps=+1$, positive homogeneity gives
\[
 \rhoR(h)=\rhoR(\alpha g)=\alpha\rhoR(g),
\]
so the residual is zero.  If $\eps=-1$, then
\[
 \rhoR(h)-\alpha\rhoR(g)
 =\alpha\bigl[\rhoR(-g)-\rhoR(g)\bigr]
 =-\alpha g.
\]
Because $g$ is affine on every inherited cell, $-\alpha g\in\calL(\calP)$.  Thus the only lower-depth residuals allowed by MFR are the two sign-scale possibilities.  At the expansion level, the negative case would create exactly the affine cancellation excluded by EFF-minimality, while the positive rescaling is the standard ReLU symmetry.
\end{proof}

Taking the last two propositions together, we see that MLR and MFR do more than just identify matched units, but also explain \emph{why those units belong at their current depth}.  After the private Softplus pole or ReLU fresh facet has been matched, any remaining difference is lower-depth only in the trivial sign or rescaling cases.  This is the sense in which the regularity assumptions make depth a faithful measure of the network's nonlinear complexity.

\subsubsection{Residual blocks: what may be moved into the skip}
\label{subsec:residual-depth-faithfulness}
The previous two subsections asked whether a current Softplus or ReLU unit can be replaced by
functions that were already available at lower depth.  In a residual block, there is one additional
possibility: part of the nonlinear branch might instead be moved into the skip connection.

For a residual representation \(x:\Omega\to\R^d\), the skip can carry functions of the form
\[
 \mathcal A_x
 =
 \bigl\{
   \omega\mapsto Lx(\omega)+c
   :
   L\in\R^{m\times d},\ c\in\R^m
 \bigr\}.
\]
These are exactly the functions that are affine in the representation already carried by the skip.
A function can be linear in \(\operatorname{LN}(x)\) without being affine in \(x\), so such a
function cannot in general be moved into the skip.

\begin{proposition}[Only affine branch remainders can be moved into the skip]
\label{prop:transformer-residual-depth-faithfulness}
Suppose the current MLP axes or attention heads can be identified using the branch results of
\S\ref{sec:transformer-weak-strong}.  Outside the lower-dimensional exceptional set of
Proposition~\ref{prop:generic-transformer-residual-peeling}, a counted branch object cannot be
removed by compensating for it with a nonlinear change in the residual representation.  Any
remainder that can be moved into the skip must belong to \(\mathcal A_x\).
\end{proposition}

\begin{proof}
The results of \S\ref{sec:transformer-weak-strong} first identify the counted objects in the two branches.

For an MLP branch, matched units with the same orientation cancel exactly.  A match with the
opposite orientation can leave a term of the form
\[
 R_{S_-}
 =
 \sum_{j\in S_-} T_+v_jg_j.
\]
If this remainder is affine in \(x\), then it belongs to \(\mathcal A_x\) and can simply be
combined with the affine skip.  If it is not affine in \(x\), removing it would require an
input-dependent change in the residual representation that LayerNorm does not detect.  This is exactly the exceptional cancellation analyzed in \S\ref{sec:transformer-weak-strong}; outside the
lower-dimensional exceptional set, it cannot occur.

For attention, conditions (A1)--(A3) identify each counted head modulo attention gauge.  After
the matched head outputs have been subtracted, a counted head could disappear only if its remaining
effect were canceled by the same kind of nonlinear change in the residual representation.  The generic peeling result excludes this possibility as well. Thus, for either kind of branch, only a remainder of the form \(Lx+c\) can be moved into the
skip.
\end{proof}

This is the residual-block version of depth faithfulness.  A counted MLP axis or attention head
contributes genuinely new nonlinear structure unless its remaining effect is already affine in the
representation carried by the skip.  The affine-rank condition in
Lemma~\ref{lem:transformer-affine-rank} has a separate role: after the nonlinear branch has been
removed, it checks only that the resulting affine relation is an injective weak comparison map.

\subsection{A null-network atlas: silly solutions to trivial problems, and how to avoid them}
\label{sec:null_atlas}

Null networks are silly complex ways of solving trivial input-output relationships. There are three essential ways these can be ruled out on the way to contravariance:
\begin{enumerate}
\item First, intrinsic minimality removes representations that are already redundant inside one expansion.  
\item Second, MLR or MFR rules out a nontrivial equal-function pair at a point where the relevant nonlinear signatures are
identifiable.  
\item Third, by not actually ruling them out but showing that they are rare.
The generic null-net theorem does \emph{not} claim that every exceptional identity is impossible:
it allows resonant nontrivial collisions on lower-dimensional subsets, but prevents them from filling a full-dimensional family of regular minimal networks. 
\end{enumerate}
As an expository utility in this section, we show examples of each of these things. In what follows a \emph{resonance set} means a lower-dimensional set of parameters where an exact algebraic relation, such as an integer-weight relation or a fixed ReLU activation-pattern coincidence, makes one of these exceptional identities possible.  Quantitative assumptions add positive distance from these bad sets and are used for stable, soft statements.  Table \ref{tab:counterexample-atlas} summarizes the types of silly networks that are then explained in greater detail in each of the following subsections. 

\begin{table}[t]
\centering
\small
\begin{tabular}{@{}L{.19\textwidth}L{.38\textwidth}L{.37\textwidth}@{}}
\toprule
Mechanism & Softplus diagnosis & ReLU diagnosis\\
\midrule
Unused or constant feature & The term is absorbed into the bias or removed; usedness and nonconstancy fail. & The branch produces no visible task kink; usedness or two-sided crossing fails.\\
Duplicate or opposite features & Coefficients can be redistributed, or the sign identity leaves an affine residual; reducedness, pole separation, and sign-noncancellation intervene. & Positive-scale duplicates are redundant under positive rescaling; opposite pairs can collapse to an affine term; sign-scale reduction and sign-noncancellation intervene.\\
Continuous feature symmetry & A postactivation shift can move in a downstream nullspace; the realization Jacobian has a kernel. & Positive scaling is removed by normalization; any additional continuous output-preserving motion violates Jacobian regularity after fixing the standard matching convention.\\
Shifted or nested identity & An affine output cannot retain private curvature divisors; deeper nested identities require exact factorization or integer-weight resonance. & An affine output has no kink jumps; deeper variants require coincident facets, invisibility on ReLU-inactive regions, or split matching.\\
LayerNorm branch--skip transfer & An opposite MLP match can leave \(v(g\circ\operatorname{LN})\).  A skip can hide this term only on a special parameter set; only terms affine in the raw residual state may be moved into the skip. & The same sign identity gives the ReLU version.  Private smooth traces recover the unit, and the rank test makes the remaining nonaffine cancellations lower-dimensional.\\
Discrete cross-family nontrivial collision & Possible on special non-MLR resonance sets, including exact integer weights and linked biases. & Possible on special fixed activation-pattern and incidence cases with proportional cellwise hyperplanes or incomplete task exposure.\\
\bottomrule
\end{tabular}
\caption{A guide to the assumptions.  Pointwise signature recovery excludes nontrivial collisions on the recovery set.  The generic theorem permits exceptional nontrivial collisions, but makes their projection into individual-network space lower-dimensional.}
\label{tab:counterexample-atlas}
\end{table}

\subsubsection{Unused, constant, duplicate, and opposite features}

The easiest null relations are internal redundancies.  If $v_j=0$, the term $v_j\varphi(g_j)$ is absent from
the represented function.  If $g_j$ is constant, then $v_j\varphi(g_j)$ can be folded into the bias.  If two
arguments coincide, then
\[
 v_1\varphi(g)+v_2\varphi(g)=(v_1+v_2)\varphi(g),
\]
so the division of coefficient mass between the two units is unidentifiable.

The sign-pair identities are especially important:
\begin{equation}
 \sig(t)-\sig(-t)=t,
 \qquad
 \rhoR(t)-\rhoR(-t)=t.
 \label{eq:two-sign-identities}
\end{equation}
For Softplus, $g$ and $-g$ have the same curvature-pole set.  For ReLU, they have the same zero facet with
opposite active sides.  Thus a pair of opposite units can erase its nonlinear signature and leave only an affine
function.  Same-side reducedness prevents duplicate or opposite sign-classes from being counted as separate
minimal units, while affine sign-noncancellation rules out a collection of such residual affine terms being
absorbed into the bias.

The quantitative margins serve a different purpose.  Bounds such as $\|v_j\|\ge c$, lower variance or crossing
margins, and positive separation from duplicates prevent a convergent sequence of counted networks from
approaching one of these degeneracies.  Parameter boundedness by $M$ does not by itself rule out a null relation;
it supplies compactness, so that a sequence of near-counterexamples has a limiting counterexample to which the
exact theory can be applied.

\subsubsection{Shifted identity blocks}

A tempting construction inserts an exact identity block before a module $G$ in one network and after $G$ in
another.  The terminal functions agree, while the internal axes appear at different depths.

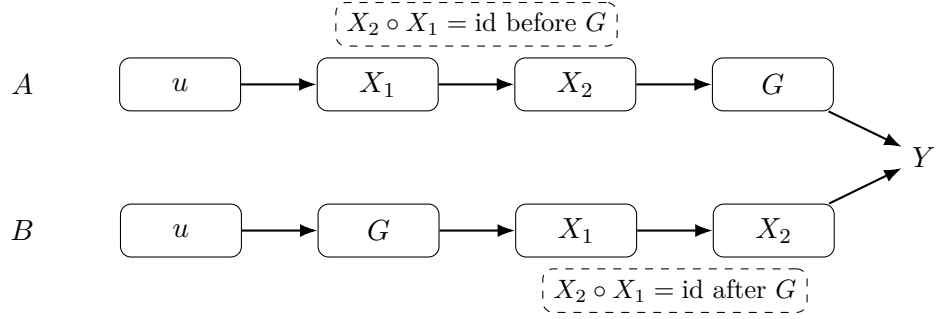
\begin{figure}[H]
\centering
\begin{tikzpicture}[node distance=8mm and 10mm]
\node (alab) {$A$};
\node[block,right=of alab] (au) {$u$};
\node[block,right=of au] (ax1) {$X_1$};
\node[block,right=of ax1] (ax2) {$X_2$};
\node[block,right=of ax2] (ag) {$G$};
\node[below=14mm of alab] (blab) {$B$};
\node[block,right=of blab] (bu) {$u$};
\node[block,right=of bu] (bg) {$G$};
\node[block,right=of bg] (bx1) {$X_1$};
\node[block,right=of bx1] (bx2) {$X_2$};
\node[right=17mm of $(ag)!0.5!(bx2)$] (y) {$Y$};
\draw[flow] (au)--(ax1);\draw[flow] (ax1)--(ax2);\draw[flow] (ax2)--(ag);\draw[flow] (ag)--(y);
\draw[flow] (bu)--(bg);\draw[flow] (bg)--(bx1);\draw[flow] (bx1)--(bx2);\draw[flow] (bx2)--(y);
\node[note,above=5mm of $(ax1)!0.5!(ax2)$] {$X_2\circ X_1=\id$ before $G$};
\node[note,below=5mm of $(bx1)!0.5!(bx2)$] {$X_2\circ X_1=\id$ after $G$};
\end{tikzpicture}
\caption{A shifted identity block.  The correct place to test the construction is the first backward step at which the two one-step decompositions differ.  There it creates a nontrivial null pair, so either pointwise signature recovery identifies the axes or the construction lies on an exceptional resonance set.}
\label{fig:shifted-block}
\end{figure}

For a Softplus identity block, the last layer would give a one-step representation of an affine map,
\begin{equation}
 u=b+\sum_j v_j\sig(g_j(u)).
 \label{eq:softplus-affine-identity}
\end{equation}
The left side has zero Hessian.  On a private regular pole patch of a minimal sign-class,
\[
 D^2[\sig(g_j)]
 =\sig''(g_j)\,dg_j\otimes dg_j+\sig'(g_j)D^2g_j
\]
contains a singular contribution that cannot be canceled by terms analytic on that patch.  Hence an affine output
is possible only if the term is unused or constant, if another term shares the same divisor, or if sign-opposite
terms collapse through~\eqref{eq:two-sign-identities}.  These are exactly the roles of usedness, nonconstancy,
reducedness, pole separation, and sign-noncancellation.  Where MLR holds, the same conclusion is obtained more
structurally: the outer Laurent binomial factors can match only through $g$ versus $\pm g$; the negative case is
then removed by sign-noncancellation.

For ReLU, the affine identity has no derivative jump.  A used fresh facet of $g_j$ contributes the rank-one jump
\[
 v_j\nabla g_j^{\top}.
\]
It can disappear only if another unit has a coincident facet and an exactly cancelling jump.  Same-side
sign-scale reduction and sign-noncancellation rule out the basic opposite-pair construction.  Deep ReLU blocks
also permit ReLU-inactive-region and split-facet constructions, which is why one local facet is not sufficient and MFR uses a
coherent atlas of facets across inherited cells.

The first-mismatch formulation is the most general.  At the first backward zippering step where the two networks
differ, the already aligned next representation has two equal one-step expansions.  Pointwise MLR or MFR makes
that null pair trivial.  If the inserted block is genuinely nontrivial, it must therefore leave the recovery set.
The generic null-net theorem allows this on a proper lower-dimensional exceptional set, but excludes a full-dimensional family
of such shifted nontrivial collisions.  Quantitative margins keep the realized networks away from the corresponding
degeneracy or resonance sets.

\subsubsection{LayerNorm can trade a branch term against the skip}
\label{subsec:null-layernorm-transfer}

A pre-LayerNorm residual block has one null-pair mechanism that is absent from an ordinary affine--nonlinear expansion.  LayerNorm is unchanged by adding a tokenwise shift in the all-ones feature direction, while both activations used in this paper satisfy
\[
 \varphi(t)-\varphi(-t)=t.
\]
A sign reversal of a current MLP argument can therefore leave a term \(v(g\circ\operatorname{LN})\), and a sample-dependent shift of the residual stream can sometimes absorb it.  Proposition~\ref{prop:layernorm-skip-transfer} gives an exact one-unit example: the residual-sublayer outputs agree, but the two raw residual-input families are not affinely related.

The example shows why ordinary branch minimality is not quite enough for raw residual zippering.  Minimality can identify the current sign-class, yet the opposite orientation may leave a term that is linear in normalized coordinates and nonlinear in the raw residual state.  This is also why Lemma~\ref{lem:transformer-shrinkability} permits a term to enter the skip only after it has been shown to be raw-residual-affine.  This nonaffine transfer is the new null-pair mechanism.  The later check that the remaining affine map is injective is different: it is automatic for an untouched identity skip with no affine transfer, stable under small affine transfers, and generically fails only on a rank-defect set by Lemma~\ref{lem:transformer-affine-rank}.

The construction is nevertheless special rather than a free architectural symmetry.  Its readout must point in a LayerNorm-invisible direction, the competitor must obey an exact sign relation, and the residual input must move in just the corresponding invisible direction.  Proposition~\ref{prop:transformer-transfer-rank} writes these requirements as finite equations.  If, after all competitor variables are allowed to move, those equations still impose at least one condition on the source network, then the networks admitting the transfer are lower-dimensional.  Proposition~\ref{prop:generic-transformer-residual-peeling} handles any additional same-output transfer not listed in advance.  As elsewhere in the atlas, exact exceptional identities are allowed, but they cannot fill an open family of regular minimal networks.

\subsubsection{Softplus postactivation-shift symmetries}

A more subtle Softplus identity shifts an \emph{activated} feature rather than canceling a sign pair.  For $c>0$,
define
\begin{equation}
 H_c(t)=\sig^{-1}(\sig(t)+c)
 =\log\!\bigl(e^c(1+e^t)-1\bigr).
 \label{eq:Hc-def}
\end{equation}
Then
\begin{equation}
 \sig(H_c(t))=\sig(t)+c.
 \label{eq:Hc-shift}
\end{equation}
Let $V=[v_1\ \cdots\ v_k]$ be the downstream coefficient matrix and choose a nonzero vector $q$ with $Vq=0$.
For base arguments $u_j$ and constants $c_j$, set
\[
 g_j^{(\tau)}=H_{c_j+\tau q_j}(u_j).
\]
Then
\begin{align}
 Y_\tau
 &=b+\sum_jv_j\sig(g_j^{(\tau)})\\
 &=b+\sum_jv_j\bigl(\sig(u_j)+c_j+\tau q_j\bigr)
 =Y_0.
 \label{eq:shift-symmetry-family}
\end{align}
Thus $\tau$ is a continuous output-preserving coordinate even when every individual $v_j$ is nonzero and every
argument varies.

\begin{figure}[H]
\centering
\begin{tikzpicture}[node distance=8mm and 10mm]
\node[block] (u) {$u_j$};
\node[block,right=of u] (h) {$H_{c_j+\tau q_j}$};
\node[block,right=of h] (s) {$\sig$};
\node[block,right=of s] (v) {$V$};
\node[block,right=of v] (y) {$Y_\tau$};
\draw[flow] (u)--(h);\draw[flow] (h)--(s);\draw[flow] (s)--(v);\draw[flow] (v)--(y);
\node[note,below=6mm of s] (shiftnote)
  {$\sig(g_j^{(\tau)})=\sig(u_j)+c_j+\tau q_j$};

\node[note,right=3mm of shiftnote] (nullnote)
  {$Vq=0$};
\end{tikzpicture}
\caption{A postactivation-shift symmetry.  The hidden activated coordinates move in a downstream nullspace while the represented function remains fixed.}
\label{fig:postactivation-shift-symmetry}
\end{figure}
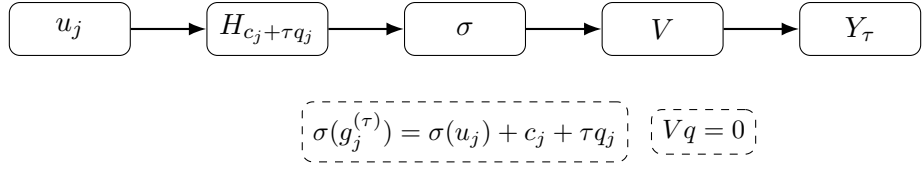

Differentiating~\eqref{eq:shift-symmetry-family} gives
\[
 DF(\vartheta)\,\partial_\tau=0.
\]
Hence qualitative Jacobian regularity fails.  This example is useful because it shows that intrinsic minimality
alone does not remove every continuous redundancy.  It does not threaten the regular generic null-net theorem:
the entire continuous symmetry is removed when one passes to the regular part of the realization chart.

\subsubsection{Nested Softplus identities}

Softplus also admits identities that insert a nominal nonlinear stage while changing an activated feature only by
a constant.  A basic example is
\begin{equation}
 \sig\!\left(a+\sig(t-a+\sig(a))\right)=\sig(a)+\sig(t).
 \label{eq:nested-softplus-identity}
\end{equation}
Equivalently, if
\[
 \widehat t=a+\sig(t-a+\sig(a)),
\]
then the extra stage satisfies $\sig(\widehat t)-\sig(t)=\sig(a)$, which the next bias can absorb.  Such an
identity is not generated by a generic perturbation of all weights and biases: it uses exact unit coefficients and
linked offsets.  In multiplicative form,
\[
 1+e^{\widehat t}=e^{\sig(a)}(1+e^t),
\]
so the two outer divisors coincide and their ratio is a constant unit.  This is precisely a non-MLR
factorization resonance.

The identity can be nested or shifted across several layers.  It may also coexist with another branch that makes
the whole network output genuinely depth-essential, so global output-depth minimality alone is not enough to
remove every nested subbranch.  What matters for the generic theorem is that the exact linked-weight equations
define a lower-dimensional resonance set.  Where MLR holds, the factorization is impossible except for the trivial
$g\leftrightarrow\pm g$ cases, and general position on collision strata prevents the nested-identity set from supporting a
full-dimensional collision stratum.

\subsubsection{A regular discrete Softplus resonance}
\label{subsec:discrete-softplus-resonance}

Not every exceptional nontrivial collision is Jacobian-singular or globally depth-reducible.  Let
\[
 \kappa=\sig(\eta).
\]
Then
\begin{equation}
 \sig\!\left(\eta-\sig(t)\right)
 +\sig\!\left(\eta-\sig(\kappa-t)\right)
 =\kappa.
 \label{eq:softplus-diamond}
\end{equation}
To verify this, write
\[
 Q_1=1+e^{\eta-\sig(t)}
 =\frac{1+e^t+e^\eta}{1+e^t},
\]
and, since $e^\kappa=1+e^\eta$,
\[
 Q_2=1+e^{\eta-\sig(\kappa-t)}
 =\frac{(1+e^\eta)(1+e^t)}{1+e^t+e^\eta}.
\]
Thus $Q_1Q_2=1+e^\eta=e^\kappa$, and taking the real logarithm gives~\eqref{eq:softplus-diamond}.

This identity yields two ordinary dense same-depth networks with equal outputs and different intermediate axes.
For input $(x,y)$, choose $n,m\ne0$, generic $\eta,\gamma,\beta$, and an irrational $0<\alpha<1$.  Let
\begin{align}
 z_1^A&=(x,y),\\
 z_2^A&=\bigl(\eta-\sig(x),\ \gamma+\alpha\sig(y)\bigr),\\
 z_3^A&=\beta+n\sig(z_{2,1}^A)+m\sig(z_{2,2}^A),
 \label{eq:resonance-network-A}
\end{align}
and
\begin{align}
 z_1^B&=(\kappa-x,y),\\
 z_2^B&=\bigl(\eta-\sig(\kappa-x),\ \gamma+\alpha\sig(y)\bigr),\\
 z_3^B&=\beta+n\kappa-n\sig(z_{2,1}^B)+m\sig(z_{2,2}^B).
 \label{eq:resonance-network-B}
\end{align}
Equation~\eqref{eq:softplus-diamond} gives
\[
 z_3^A(x,y)=z_3^B(x,y),
\]
while generically
\[
 z_{2,1}^A\ne z_{2,1}^B,
 \qquad
 z_{2,1}^A\ne -z_{2,1}^B.
\]
The common $y$-branch can be chosen to carry genuinely top-depth analytic structure; the irrational exponent
$\alpha$ is a convenient way to prevent the entire output from being explained by the elementary $x$-branch
identity alone.  Moreover, the individual realization Jacobians can remain full rank: this is a discrete
cross-family nontrivial collision, not a continuous symmetry.

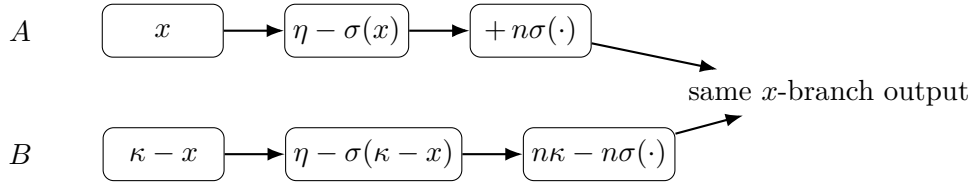
\begin{figure}[H]
\centering
\begin{tikzpicture}[node distance=7mm and 8mm]
\node (a) {$A$};
\node[block,right=of a] (ax) {$x$};
\node[block,right=of ax] (ag) {$\eta-\sig(x)$};
\node[block,right=of ag] (ao) {$+\,n\sig(\cdot)$};
\node[below=11mm of a] (b) {$B$};
\node[block,right=of b] (bx) {$\kappa-x$};
\node[block,right=of bx] (bg) {$\eta-\sig(\kappa-x)$};
\node[block,right=of bg] (bo) {$n\kappa- n\sig(\cdot)$};
\node[right=15mm of $(ao)!0.5!(bo)$] (y) {same $x$-branch output};
\draw[flow] (ax)--(ag);\draw[flow] (ag)--(ao);\draw[flow] (ao)--(y);
\draw[flow] (bx)--(bg);\draw[flow] (bg)--(bo);\draw[flow] (bo)--(y);
\end{tikzpicture}
\caption{The regular discrete Softplus resonance.  The identity is exact, but it requires a special reflected input, exact $-1$ coefficients, zero cross-couplings, and linked biases.}
\label{fig:softplus-discrete-resonance}
\end{figure}

Why does this not contradict the generic null-net theorem?  A fully dense $2$--$2$--$2$--$1$ network has $15$
parameters.  The displayed nontrivial-collision family has $12$ free parameters: six for two independent first-layer affine
forms, three for $(\eta,\gamma,\alpha)$, and three for $(n,m,\beta)$.  It imposes, in particular, the exact
second-layer conditions
\[
 W_{2,11}=-1,
 \qquad
 W_{2,12}=0,
 \qquad
 W_{2,21}=0.
\]
Thus its projection into an individual dense-network chart is codimension at least three.  The example is regular
and potentially depth-essential, but it lies on an integer-weight and sparsity resonance set.  The generic
theorem is designed to permit exactly this sort of exceptional nontrivial collision while ruling out an open family of them. Actually this type of example (which ChatGPT-5.5-Pro spent several hours finding!) is the motivation behind the whole genericity approach we take.  

\subsubsection{ReLU positive homogeneity and changes on ReLU-inactive regions}

The exact positive-rescaling symmetry
\begin{equation}
 (v,g)\mapsto(\alpha^{-1}v,\alpha g),
 \qquad \alpha>0,
 \label{eq:relu-scale-symmetry}
\end{equation}
is unavoidable because $\rhoR(\alpha g)=\alpha\rhoR(g)$.  It is removed by the ReLU normalization convention before
Jacobian regularity or null-pair uniqueness is stated.

A genuinely different obstruction is that ReLU erases everything an argument does while it remains negative.
For $\lambda>0$, let
\begin{equation}
 g(x)=x,
 \qquad
 h_\lambda(x)=x-\lambda\rhoR(-x-1)
 =\begin{cases}
 (1+\lambda)x+\lambda,&x<-1,\\
 x,&x\ge-1.
 \end{cases}
 \label{eq:inactive-surgery}
\end{equation}
For $x<-1$, $h_\lambda(x)<0$; for $x\ge-1$, $h_\lambda(x)=x$.  Hence
\[
 \rhoR(h_\lambda)=\rhoR(g)
\]
although $h_\lambda\not\sim_\pm g$.  The two arguments share the exposed zero facet at $x=0$ but differ on an
inactive inherited cell.

\begin{figure}[H]
\centering
\begin{tikzpicture}[x=1cm,y=.7cm]
\draw[->] (-2.4,0)--(2.4,0) node[right] {$x$};
\draw[->] (0,-2.7)--(0,2.5) node[above] {argument};
\draw[thick] (-2,-2)--(2,2) node[above right] {$g=x$};
\draw[thick,dashed] (-2,-3)--(-1,-1)--(2,2) node[below right] {$h_\lambda$};
\draw[dotted] (-1,-2.5)--(-1,1.6);
\node[bad] at (1.75,-1.7) {same ReLU output\\and same zero facet};
\node[align=center,font=\small] at (-1.55,1.3) {different only on\\the inactive side};
\end{tikzpicture}
\caption{Why one exposed facet is not enough for ReLU.  MFR records enough labelled facets across inherited cells to recover the complete normalized CPWA argument rather than only its visible active boundary.}
\label{fig:inactive-surgery-unified}
\end{figure}
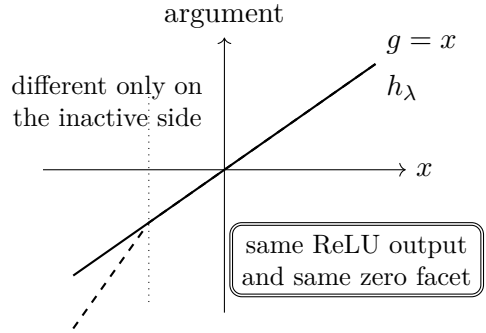

If $\lambda$ is a free coordinate in one realization chart,~\eqref{eq:inactive-surgery} is a continuous fiber and
Jacobian regularity after fixing the standard matching convention fails.  A discrete version may instead connect two separate regular families.  In that
case the task exposes too little of the argument for one-facet reasoning, while the MFR hypothesis explicitly
requires a task-complete multi-cell signature.  Under the signature secant-derivative and general-position assumptions on collision strata,
ambiguous ReLU-inactive-region nontrivial collisions can survive only on lower-dimensional strata.

\subsubsection{Split-facet matching and special activation-pattern cases}

A final ReLU-specific failure is \emph{split matching}.  Suppose one deep argument $g$ has marked facets
$F_1,F_2,F_3$ in three inherited cells.  Equality of represented kink sets might match $F_1$ to a competitor
$h_1$, $F_2$ to $h_2$, and $F_3$ to $h_3$, even though no single competitor is globally sign-scale equivalent to
$g$.

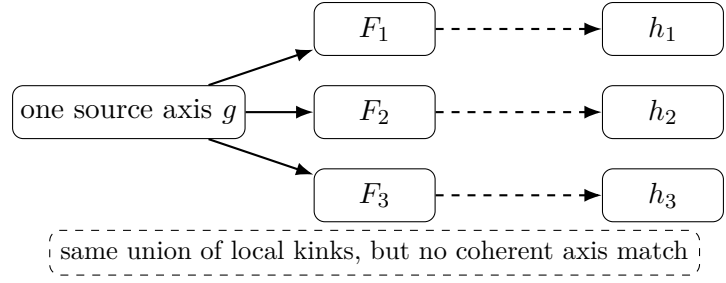
\begin{figure}[H]
\centering
\begin{tikzpicture}[node distance=7mm and 9mm]
\node[block] (g) {one source axis $g$};
\node[block,right=of g,yshift=11mm] (f1) {$F_1$};
\node[block,right=of g] (f2) {$F_2$};
\node[block,right=of g,yshift=-11mm] (f3) {$F_3$};
\node[block,right=22mm of f1] (h1) {$h_1$};
\node[block,right=22mm of f2] (h2) {$h_2$};
\node[block,right=22mm of f3] (h3) {$h_3$};
\draw[flow] (g)--(f1);\draw[flow] (g)--(f2);\draw[flow] (g)--(f3);
\draw[flow,dashed] (f1)--(h1);\draw[flow,dashed] (f2)--(h2);\draw[flow,dashed] (f3)--(h3);
\node[note,below=12mm of f2] {same union of local kinks, but no coherent axis match};
\end{tikzpicture}
\caption{Split-facet matching.  Coherent MFR requires all marked facets of one source argument to be explained by one competitor argument.}
\label{fig:split-facet}
\end{figure}

The coherence condition excludes this possibilty.  Generically, a prescribed split assignment imposes simultaneous proportionality equations for several cellwise affine germs, together with activation-pattern incidence equations. On a fixed ReLU activation-pattern stratum these are polynomial or piecewise-analytic constraints.  The split-signature result (Proposition \ref{prop:generic-split}) shows that, when the required rank condition holds at one point, the projected split-matching set has codimension at least one.  The ReLU general-position condition on collision strata then prevents a union of split assignments from filling a full-dimensional collision stratum.  

\begin{remark}
That this split-facet matching situation exists is totally reasonable, but the fact it has to be ruled out by the clobberingly strong coherence condition, rather than the elegant resonant Laurent-nonprimitive resonance characterization that works for softplus, makes us think we do not fully understand how best to handle the ReLU case.  Now, this type of case is non-generic, and the dimensionality analysis of exceptional sets is very clean and activation-independent, so ``all ends well''.  But it would be nice if coherence could get cleaned up in future work -- perhaps some form of discrete algebraic geometry will enable this. 
\end{remark}

\subsection{Other architecture types and activation functions}
\label{sec:extensions}

The arguments above were written for affine--ReLU and affine--Softplus networks because those
models make the nonlinear signatures especially transparent.  The reusable part of the theory is
more general.  For a new activation or architecture, one needs four ingredients: the natural symmetry convention, a task-visible nonlinear signature for adjacent weak--strong alignment, a pointwise one-step identifiability theorem, and a genericity argument showing that ambiguous full-dimensional collision strata are lower-dimensional. Once those ingredients are available, the compactness, inverse-continuity, and Jacobian-regularity arguments give the asymptotic and soft versions with little change.

\begin{table}[H]
\centering
\small
\begin{tabular}{@{}L{.15\textwidth}L{.24\textwidth}L{.27\textwidth}L{.26\textwidth}@{}}
\toprule
Setting & Natural signature/symmetry & Likely exact theory & Main additional difficulty\\
\midrule
CNNs & ReLU/Softplus signatures at channel-position axes; convolutional weight sharing &
Affine--nonlinear proofs apply to convolutional preactivation maps & Pooling, subsampling, and the choice between channel-level and channel-position axes\\
RNNs & Unrolled kink/ridge signatures with parameters tied across time &
Feedforward theorem applies to each fixed unrolled time and layer & Weak maps and genericity must respect time-tying and long-horizon stability\\
Tanh & Meromorphic poles; signed permutations & Close to an existing full null-net theory & Quantitative saturation and conditioning\\
GELU & Gaussian ridge or quadratic-exponential phase; sign identity & Adjacent theorem should be direct; deep theorem needs a new recovery module & Nonmonotone inverse branches and deep phase matching\\
ELU & Transition jump jet plus negative-branch exponential character & Hybrid ReLU/smooth proof & Negative-only translation symmetry and saturation\\
\bottomrule
\end{tabular}
\caption{High-level extension map.  The common collision geometry and zippering induction remain; the nonlinear signature mechanism changes.}
\label{tab:extensions}
\end{table}

\subsubsection{CNNs}

Convolutional layers are affine maps with weight sharing.  Therefore, before pooling or subsampling,
a convolutional ReLU or Softplus block is still an affine--coordinatewise-nonlinear--affine block.
The adjacent weak--strong argument applies to the channel-position preactivation axes: a used
ReLU channel-position creates a crossed kink trace, and a used Softplus channel-position creates a
curvature ridge.  Weight sharing does not remove these signatures; it ties the parameters generating
spatially translated copies of them.

The main bookkeeping issue is the level at which axes are counted.  Before spatial pooling, the
natural axes are channel-position coordinates.  After global or local averaging, several positions of
the same channel may be pooled into a coarser channel-level feature, so the theorem should either be
applied before pooling or be restated for the pooled channel representation.  Max-pooling is
piecewise affine and introduces its own switching surfaces; it can be included by refining the ReLU
cell decomposition, but it is cleaner to place the comparison immediately before or after the pooling
operation.

\subsubsection{RNNs}

Unrolling an RNN over time turns it into a feedforward network with tied parameters.  For a
ReLU or Softplus recurrent update, the weak--strong theorem applies at a fixed time and layer of
the unrolled computation exactly as it does for a feedforward affine--nonlinear block.  The hidden
state coordinates at that time have the usual kink or ridge signatures, and a pair of adjacent weak
maps forces the used coordinates to align.

The additional issue is compatibility across time.  The same recurrent weights are reused at every
unrolled step, so a comparison map chosen at one time cannot be arbitrary if one wants a single
time-consistent equivalence over a long trajectory.  For zippering, the corresponding compactness
and genericity statements must be made inside the tied-weight parameter chart.  This restriction
does not change the local nonlinear signature mechanism, but it can make the allowed family of
weak maps and null pairs smaller and more structured.

\subsubsection{Tanh networks}

For \(\tau(t)=\tanh t\), the unavoidable unit symmetry is
\[
 (v,g)\sim(-v,-g),
\]
so the allowed hidden-unit matching uses permutations and coordinatewise sign flips rather than positive scales.  Tanh is meromorphic with a regular complex pole lattice.  Affine-ridge weak--strong alignment can therefore be proved by matching pole signatures, and deep one-step null relations can be attacked by an activation-specific null-network classification.  In fact, tanh-type activations are substantially closer than Softplus to settings where a complete regular null-net theorem is available.  A tanh version of the present draft would likely replace MLR/MFR by a signed-pole recovery theorem and then reuse the common generic null-net and zippering sections almost verbatim.

The main difficulty is quantitative rather than exact.  Since
\[
 \tau'(t)=\operatorname{sech}^2(t),
\]
large preactivations are strongly saturated.  Exact identifiability may survive, while the smallest realization Jacobian singular value becomes very small.  Soft zippering therefore needs an explicit nonsaturation or Jacobian conditioning margin.

\subsubsection{GELU networks}

For exact GELU,
\[
 \gamma(t)=t\Phi(t),
\]
where \(\Phi\) is the Gaussian distribution function.  It satisfies
\[
 \gamma(t)-\gamma(-t)=t,
\]
so an affine sign-noncancellation condition again has a natural role.  For adjacent affine arguments, derivatives of GELU produce Gaussian ridge functions.  Restriction to generic lines and Fourier decay should identify affine arguments up to sign, giving a weak--strong theorem by a mechanism analogous to the Softplus affine-ridge proof.

Deep zippering is less immediate.  The natural analytic signature is a quadratic-exponential phase
of the form
\[
 \exp\!\left(-g^2/2\right),
\]
and one would need a quadratic-exponential analogue of MLR showing that persistent equal phases force \(h=\pm g\) outside a resonance set.  Exact GELU is also not globally injective, so part of its range has two inverse branches.  A complete theorem must either restrict counted arguments to a common monotone branch, prove that branch-swapping compositions cannot occur at the same network complexity, or treat persistent branch nontrivial collisions by general position on collision strata.  The common generic null-net framework is well suited to the last option, but the activation-specific recovery module is genuinely new.  The widely used tanh approximation to GELU
has different complex singularities and should be analyzed as a separate activation rather than silently identified with exact GELU.

\subsubsection{ELU networks}

For parameter \(\alpha>0\),
\[
 \eta_\alpha(t)=
 \begin{cases}
 t,&t>0,\\
 \alpha(e^t-1),&t\le0 .
 \end{cases}
\]
ELU combines a ReLU-like transition hypersurface with an exponential negative branch.  If
\(\alpha\ne1\), the first derivative jumps at zero; for \(\alpha=1\), the first derivative matches but the second derivative jumps.  An adjacent weak--strong theorem can therefore use the first nonvanishing transition jet.  Matching consecutive jet coefficients should fix both orientation and scale for a crossed affine argument.

Crossing is important because a unit confined to the negative branch has a translation--coefficient
symmetry.  If \(g\) and \(g+c\) remain negative on the task set, then an appropriate rescaling of the
downstream coefficient and a bias correction preserve the represented function.  A zero crossing anchors the offset and removes this symmetry.  A deep ELU recovery module would probably combine coherent multi-transition-jet signatures with exponential character recovery on negative cells.  The remaining generic collision geometry is the same as here.  Soft bounds must also account for negative-branch saturation, since \(\eta_\alpha'(t)=\alpha e^t\) can be arbitrarily small.

\subsection{RSA and Weak-Strong Equivalence}
\label{sec:rsa-weak-strong}
In this work so far, we've focused on relating two basic metrics for comparing networks: linear regression (weak alignment) and privileged axes (strong alignment).  So far, we have ignored Representational Similarity Analysis (RSA) \citep{kriegeskorte2008}. The underlying concept behind RSA is that networks are judged to be similar insofar as they induce similar ``representational geometries'' -- that is,  matrices of pairwise distances between inputs. The basic technical method by which RSA implements this idea is to fix an evaluation set of inputs, extract for each model to be compared the model feature vectors on these inputs, and then compute the num\_inputs-by-num\_inputs feature correlation matrix.  The RSMs for the different networks are then flattened into vectors and correlated with each other. When this correlation is high, networks are judged similar, and otherwise, not.  RSA is one of the most popular metrics for comparing neural networks to brain data, in part because it does not require data to fit parameters (unlike linear regression). 

In this section, we ask: how does RSA fit into the weak--strong equivalence picture? In doing so, we rely heavily on the program laid out in \citep{theiss2026parameter}, synthesizing their results into the framework of this paper and slightly extending their reach.  In what follows, we summarize and reformulate in our notation the main result of Theiss et al; strengthen the Theiss results in a few ways to support canonicalization; prove weak--strong equivalence for canonicalized RSA; develop a soft weak--strong alignment formula relating ridge error to canonical RSA distance; and show how sample bias can interact with all these issues. 

\subsubsection{RSM-RSA and the Theiss decomposition}
\label{sec:rsa-theiss-summary}

Fix an evaluation input set $\Omega_P=\{x^1,\ldots,x^P\}\subseteq\Omega$.
For a representation with scalar feature vectors
$f_1,\ldots,f_d\in\R^P$ on these stimuli, write
\begin{equation}
 H=\begin{bmatrix}f_1^\top\\ \vdots\\ f_d^\top\end{bmatrix},
 \qquad
 R(H):=H^\top H=\sum_{j=1}^{d}f_jf_j^\top.
 \label{eq:rsa-rsm-def}
\end{equation}
The $P\times P$ matrix $R(H)$ is the representational similarity matrix (RSM): its $(\mu,\nu)$ entry is the inner product between the population responses to stimuli $x^\mu$ and $x^\nu$.

Let $\Pi_{\mathsf H}$ set the diagonal of a symmetric matrix to zero and subtract the mean of its off-diagonal entries.  When
$\Pi_{\mathsf H}M\neq0$ and $\Pi_{\mathsf H}N\neq0$, define the Pearson RSM-RSA score
\begin{equation}
 \rho_{\rm RSA}(M,N)
 :=\frac{\langle\Pi_{\mathsf H}M,\Pi_{\mathsf H}N\rangle_F}
 {\|\Pi_{\mathsf H}M\|_F\,\|\Pi_{\mathsf H}N\|_F}.
 \label{eq:rsa-pearson-score}
\end{equation}
This is Pearson correlation between the strict upper-triangular RSM entries, a standard form of RSA \citep{kriegeskorte2008}; Lemma~B.3 of \citet{theiss2026parameter} gives the displayed cosine representation.  Thus $\rho_{\rm RSA}$ is a similarity score, not an RSA distance.

\paragraph{Summary of Theiss \emph{et al.}}
Theiss et al. study representations connected by a ``catalog'' of exact, function-preserving parameter symmetries.  Although those parameter transformations have various forms, their action on hidden feature vectors reduces to three elementary operations:
\emph{addition} of redundant features, \emph{duplication} of existing features, and featurewise \emph{scaling}.  Their Proposition~4.2 shows that every representation captured in the orbit of one such symmetry can be written in a single add--duplicate--scale form relative to a base member of that orbit. Their Proposition~B.1 then separates the RSM into a \emph{task-linked} part, whose generators come from the computationally relevant irreducible features of the base member, and a symmetry-induced part contributed by redundant added features. Scaling and duplication can still change the positive weights on the task-linked generators. Their Proposition~C.3 ultimately shows that a minimum representation-norm or minimum weight-norm rule selects one RSM within a fixed orbit. The next lemma restates their RSM decomposition in the notation used here.

\begin{lemma}[Theiss RSM decomposition]
\label{lem:rsa-theiss-rsm-decomp}
Let $H^\star\in\R^{m\times P}$ have rows
$f_1^\top,\ldots,f_m^\top$.  Suppose $H$ belongs to a Theiss symmetry orbit of $H^\star$, so that
\begin{equation}
 H=\operatorname{diag}(\alpha)D_\nu
 \begin{bmatrix}H^\star\\ U\end{bmatrix},
 \label{eq:rsa-theiss-add-duplicate-scale}
\end{equation}
where the rows $u_1^\top,\ldots,u_K^\top$ of $U$ are symmetry-added features, $D_\nu$ duplicates rows, and $\alpha$ contains the scales of the resulting physical coordinates.  Then
\begin{equation}
 R(H)
 =\sum_{j=1}^{m}\gamma_j f_jf_j^\top
 +\sum_{k=1}^{K}\eta_k u_ku_k^\top,
 \qquad \gamma_j,\eta_k>0.
 \label{eq:rsa-theiss-rsm-decomp}
\end{equation}
After applying $\Pi_{\mathsf H}$,
\begin{equation}
 \begin{aligned}
 g:=\Pi_{\mathsf H}R(H)
 &=\underbrace{\sum_{j=1}^{m}\gamma_jq_j}_{\substack{\text{task-linked}\\\text{essential-core part}}}
 +\underbrace{a}_{\substack{\text{symmetry-induced}\\\text{added-feature part}}},\\
 q_j&:=\Pi_{\mathsf H}(f_jf_j^\top),
 \qquad
 a:=\sum_{k=1}^{K}\eta_k\Pi_{\mathsf H}(u_ku_k^\top).
 \end{aligned}
 \label{eq:rsa-orbit-core-plus-addition}
\end{equation}
Thus scaling and duplication reweight the original RSM generators $q_j$, whereas feature addition can introduce new generators.
\end{lemma}

\begin{proof}
Put
\[
 M:=\begin{bmatrix}H^\star\\ U\end{bmatrix}.
\]
Equation~\eqref{eq:rsa-theiss-add-duplicate-scale} gives
\begin{align*}
 R(H)
 &=H^\top H\\
 &=M^\top D_\nu^\top
   \operatorname{diag}(\alpha^{\odot2})D_\nu M.
\end{align*}
Every row of a duplication matrix contains exactly one entry equal to one.  Consequently, two distinct columns of $D_\nu$ never have a nonzero entry in the same row.  The weighted Gram matrix of $D_\nu$ is therefore diagonal:
\begin{equation*}
 D_\nu^\top\operatorname{diag}(\alpha^{\odot2})D_\nu
 =\operatorname{diag}(D_\nu^\top\alpha^{\odot2}).
\end{equation*}
Let $\zeta:=D_\nu^\top\alpha^{\odot2}$.  Every entry of $\zeta$ is positive, because it is the sum of the squared nonzero scales of all physical copies of one row of $M$.  If $m_r^\top$ is row $r$ of $M$, then
\[
 R(H)=M^\top\operatorname{diag}(\zeta)M
 =\sum_{r=1}^{m+K}\zeta_r m_rm_r^\top.
\]
The first $m$ rows of $M$ are the $f_j^\top$ and the remaining $K$ rows are the $u_k^\top$.  Renaming the corresponding entries of $\zeta$ as $\gamma_j$ and $\eta_k$ proves \eqref{eq:rsa-theiss-rsm-decomp}.  Applying the linear map $\Pi_{\mathsf H}$ term by term proves \eqref{eq:rsa-orbit-core-plus-addition}.
\end{proof}

Following \citet{theiss2026parameter}, we call the first sum in \eqref{eq:rsa-orbit-core-plus-addition} the \emph{task-linked} (or task-relevant) part: its generators come from the irreducible feature directions of the base implementation, whereas $a$ comes from features introduced only by function-preserving symmetries and is redundant to the represented layer function.  The adjective ``task-linked'' does not mean that the weights $\gamma_j$ are fixed by the task---scaling and duplication leave those weights free.  We call the underlying task-linked directions the \emph{essential core}.  At this point, however, this is only the candidate essential core of one cataloged orbit.  The rows of $H^\star$ are chosen from one irreducible representative, and Theiss et al. do not claim that every same-function realization lies in its orbit; a given function might in principle have two different irreducible representatives.

The next two subsections turn this orbit-relative decomposition into a function-level statement.  The first proves an exact uniqueness result within the symmetries cataloged by Theiss et al: once the paper's existing minimality conditions are imposed, addition, duplication, and scaling do not produce a second minimal axis dictionary inside the same Theiss orbit.  The second asks whether the same function can nevertheless have another minimal realization outside that orbit, and shows that such alternatives are generically exceptional.  After these two steps, the first term of \eqref{eq:rsa-orbit-core-plus-addition} be interpreted as a weighting of \emph{the} essential core of the represented function.

\subsubsection{Uniqueness within a Theiss symmetry orbit}
\label{sec:rsa-structural-core}

Theiss et al. begin with one irreducible representative and describe all representations obtainable from it through their catalog of exact parameter symmetries.  Their RSM decomposition therefore already suggests a distinguished set of original feature directions, but only relative to that chosen orbit.  We first show that the existing minimality notions of this paper make that orbit-relative core exact and unambiguous.

\begin{definition}[Minimal realization and minimal axis dictionary]
\label{def:rsa-intrinsic-minimal-fiber}
A \emph{minimal realization} is a one-step expansion
\[
 G=\bigl(b,\{(v_j,g_j)\}_{j=1}^{k}\bigr),
 \qquad
 Y_G^\varphi=b+\sum_{j=1}^{k}v_j\varphi(g_j),
\]
that satisfies the earlier activation-specific minimality condition: Softplus minimality from Definition~10, or ReLU EFF-minimality from Definition~12.  

The \emph{minimal axis dictionary} of $G$ is the unordered set
\begin{equation}
 \mathcal J_\varphi(G)
 :=\{[(v_j,g_j)]_\varphi:1\le j\le k\},
 \label{eq:rsa-essential-core}
\end{equation}
where Softplus identifies terms only up to permutation, while ReLU also identifies
$(v,g)\sim(\alpha^{-1}v,\alpha g)$ for $\alpha>0$.
\end{definition}

Thus $G$ is the complete one-step formula, including its bias and readout vectors, whereas $\mathcal J_\varphi(G)$ contains only the nonlinear axis/readout terms that remain in a minimal formula.  Write $\mathcal O_{\rm cat}(G^\star)$ for the set of realizations obtained from $G^\star$ by finite compositions of the exact symmetries cataloged by Theiss et al.  The next proposition says that minimality gives one axis dictionary per such orbit.

\begin{proposition}[Orbitwise uniqueness under the Theiss symmetries]
\label{prop:rsa-theiss-primitives-reduction}
Let $G^\star$ be a minimal ReLU or Softplus one-step realization.  If
$\widetilde G\in\mathcal O_{\rm cat}(G^\star)$ is also minimal, then
\begin{equation}
 \mathcal J_\varphi(\widetilde G)
 =\mathcal J_\varphi(G^\star).
 \label{eq:rsa-theiss-orbitwise-unique-core}
\end{equation}
Hence every cataloged Theiss orbit has one minimal axis dictionary, up to permutation and the positive ReLU scale gauge.
\end{proposition}

\begin{proof}
Every realization in the cataloged orbit is obtained by composing feature permutation, duplication, addition, and scaling.  It is therefore enough to check that none of these operations can create a second minimal dictionary.

\smallskip
\noindent\emph{Permutation.}
Reordering the summands changes neither the represented function nor the unordered set in \eqref{eq:rsa-essential-core}.

\smallskip
\noindent\emph{Duplication and zero groups.}
Suppose $q$ physical units share the same scalar argument $g$.  Their total contribution is
\[
 \sum_{r=1}^{q}v_r\varphi(g)
 =\left(\sum_{r=1}^{q}v_r\right)\varphi(g).
\]
Thus the copies represent one nonlinear term with aggregate readout $\sum_rv_r$.  A minimal formula contains only that aggregate term; if the aggregate readout is zero, it contains no term at all.  Duplication and zero groups therefore do not change the minimal dictionary.

\smallskip
\noindent\emph{Constant additions.}
If $g(x)\equiv c$ on the task set, then $v\varphi(g(x))=v\varphi(c)$ is independent of $x$.  It is absorbed into the bias, so it contributes no nonlinear axis to a minimal formula.

\smallskip
\noindent\emph{Aligned and opposite groups.}
Both activations considered here satisfy
$\varphi(t)-\varphi(-t)=t$.  Hence
\begin{align*}
 v_+\varphi(g)+v_-\varphi(-g)
 &=v_+\varphi(g)+v_-\bigl(\varphi(g)-g\bigr)\\
 &=(v_++v_-)\varphi(g)-v_-g.
\end{align*}
The first term is one nonlinear sign-class; the second is affine in the lower-layer argument and belongs to the affine part of the expansion.  If $v_++v_-=0$, the nonlinear term disappears.  The linear and constant groups in the Theiss catalog are built from these same aligned/opposite and constant identities, and likewise leave at most one nonlinear class after minimality is imposed.

\smallskip
\noindent\emph{Positive ReLU scaling.}
For $\alpha>0$, positive homogeneity gives
\[
 v\rhoR(g)=(\alpha^{-1}v)\rhoR(\alpha g).
\]
These are two representatives of the same ReLU axis/readout class, which is why the equivalence relation in \eqref{eq:rsa-essential-core} identifies them.

Each cataloged primitive therefore preserves the minimal dictionary, and so does every finite composition of them.  This proves \eqref{eq:rsa-theiss-orbitwise-unique-core}.
\end{proof}

Proposition~\ref{prop:rsa-theiss-primitives-reduction} says that the core suggested by the first sum in \eqref{eq:rsa-orbit-core-plus-addition} is uniquely determined \emph{within a fixed Theiss orbit}.  Raw members of that orbit may still have very different RSMs---duplication changes multiplicity, scaling changes the weights of rank-one terms, and a canceling added feature contributes geometry even though it does not change the function---but all of them return to the same minimal axis dictionary.

The remaining question is larger in scope.  The same represented function might, in principle, have another minimal realization that does not belong to the cataloged orbit of $G^\star$.  The next subsection shows that such a second minimal core is generically absent.

\subsubsection{Generic uniqueness beyond the Theiss catalog}
\label{sec:rsa-full-fiber}

Fix hidden-unit order and, for ReLU, the positive-scale normalization used in Appendix~\S A.5.  For a represented function $f$, let
\begin{equation}
 \mathcal M_\varphi(f)
 :=\left\{G_\vartheta:
 G_\vartheta\text{ is normalized and minimal, and }
 Y_{G_\vartheta}^\varphi=f\right\}.
 \label{eq:rsa-intrinsic-minimal-fiber}
\end{equation}
This is simply the set of all normalized minimal one-step formulas for $f$; in geometric terminology it is a fiber over $f$, but no special use of that terminology is needed below.  Proposition~\ref{prop:rsa-theiss-primitives-reduction} says that each cataloged Theiss orbit contributes at most one normalized dictionary to this set.  We now ask whether $\mathcal M_\varphi(f)$ can contain a second point outside that orbit.

There are two ways such nonuniqueness could occur.  A family of alternatives might vary continuously with the original realization, in which case it is another continuous symmetry.  Or the alternative might be a separate solution, possibly with a different width or parameterization type.  Jacobian regularity handles the first case.

\begin{proposition}[Continuous same-function families are Jacobian-singular]
\label{prop:rsa-continuous-missing-symmetry}
Suppose a nonconstant differentiable curve $t\mapsto\vartheta(t)$ of normalized minimal realizations represents the same function for every $t$:
\[
 \mathcal F(\vartheta(t))=\mathcal F(\vartheta(0)).
\]
At every point where $\dot\vartheta(t)\neq0$, the realization derivative has a nonzero kernel vector.  Thus a continuous same-function symmetry must either be included in the gauge quotient or lie outside the Jacobian-regular locus.
\end{proposition}

\begin{proof}
Differentiate the identity
$\mathcal F(\vartheta(t))=\mathcal F(\vartheta(0))$
at a point, relabeled $t=0$, at which
$\dot\vartheta(0)\neq0$.  The chain rule gives
\[
 D\mathcal F(\vartheta(0))\dot\vartheta(0)=0.
\]
Thus the nonzero tangent vector $\dot\vartheta(0)$ lies in the kernel of the realization derivative.  If the curve is a persistent structural symmetry, it belongs in the gauge quotient, just like hidden-unit permutation and positive ReLU scaling.  If it has not been quotiented, the curve is contained in the Jacobian-singular locus.
\end{proof}

Consequently, after the ordinary gauges have been fixed, a Jacobian-regular realization cannot lie on an unaccounted continuous same-function path.  The only remaining possibility is a separate minimal solution.  Because such a solution may have another width, we allow a finite collection of normalized minimal-realization families $\{\Theta_\tau^\varphi\}_{\tau\in\mathfrak T}$, with realization maps
\[
 \mathcal F_\tau^\varphi(\vartheta):=Y_{G_\vartheta}^\varphi.
\]
Here a chart is simply one smooth normalized parameter family of minimal one-step formulas.

\begin{definition}[General position across minimal-realization families]
\label{def:rsa-full-fiber-general-position}
The collection $\{\Theta_\tau^\varphi\}_{\tau\in\mathfrak T}$ is in \emph{general position across minimal-realization families} if the following holds.  Take any two families and any smooth family of pairs
$(\vartheta,\widetilde\vartheta)$ satisfying
\[
 \mathcal F_\tau^\varphi(\vartheta)
 =\mathcal F_{\widetilde\tau}^\varphi(\widetilde\vartheta).
\]
If the source parameters $\vartheta$ fill an open subset of the source family, then the pair family contains at least one point at which the existing activation-specific one-step identifiability theorem applies: Softplus MLR, or coherent ReLU MFR.
\end{definition}

This is the cross-family version of Definition~22.  The earlier condition says that a full-dimensional family of same-function pairs inside one chart cannot be trapped entirely in the exceptional set where the one-step signatures fail to identify the terms.  Definition~\ref{def:rsa-full-fiber-general-position} asks for the same thing when the two minimal formulas may have different widths or belong to different parameter families.

The condition is natural because an open family of cross-family coincidences would be a robust reparameterization rule, not an accidental equality.  One would expect such a rule to come from one of four sources:
\begin{enumerate}[label=(\roman*)]
\item a genuine gauge that should be quotiented;
\item a systematic nonminimal construction, such as duplication or null-feature insertion;
\item an equivalence forced by the architecture, such as hard parameter tying or a built-in factorization;
\item a task patch too small to expose the nonlinear signatures needed for one-step identifiability.
\end{enumerate}
The first possibility is removed by normalization and Proposition~\ref{prop:rsa-continuous-missing-symmetry}; Proposition~\ref{prop:rsa-theiss-primitives-reduction} handles the second for the systematic Theiss constructions.  Once those are removed, an ordinary independently parameterized architecture has no evident mechanism that would make a genuinely different minimal solution track every perturbation of the source network exactly.  The local Softplus and ReLU genericity arguments support the same expectation because their failures are described by resonance equations or vanishing rank minors rather than by open sets.

This remains an assumption, not a theorem for every architecture.  It can fail when hard parameter tying, grouped or multi-branch factorizations, fixed bottlenecks, or additional residual or attention gauges force alternative descriptions.  With this definition in mind, however, we can establish generic uniqueness. 

\begin{theorem}[Generic uniqueness of the core]
\label{thm:rsa-generic-full-fiber}
Assume the normalized minimal-realization families are Jacobian-regular and satisfy Definition~\ref{def:rsa-full-fiber-general-position}.  Fix a source family $\Theta_\tau^\varphi$ of dimension $d_\tau$, and write
$f_\vartheta:=Y_{G_\vartheta}^\varphi$.  Then, outside a subset of Hausdorff dimension at most $d_\tau-1$,
\begin{equation}
 G_{\widetilde\vartheta}\in\mathcal M_\varphi(f_\vartheta)
 \quad\Longrightarrow\quad
 \widetilde\vartheta=\vartheta.
 \label{eq:rsa-full-fiber-identifiability}
\end{equation}
Hence the exceptional set has empty interior and measure zero.  Before fixing permutation and positive ReLU scale, the minimal realizations of a generic source function form one activation-gauge orbit.
\end{theorem}

\begin{proof}
Fix a source family $\Theta_\tau^\varphi$ and a possible target family $\Theta_{\widetilde\tau}^\varphi$.  The corresponding same-function collision set is
\begin{equation*}
 \mathcal C_{\tau,\widetilde\tau}^\varphi
 :=\{(\vartheta,\widetilde\vartheta):
 \mathcal F_\tau^\varphi(\vartheta)
 =\mathcal F_{\widetilde\tau}^\varphi(\widetilde\vartheta)\}.
\end{equation*}
Call a pair \emph{off-diagonal} when the two normalized realizations are different.  By the stratified setup of Appendix~\S A.6, the off-diagonal collision set is a countable locally finite union of smooth strata.  Let $S$ be one such stratum and let
\[
 \pi_1:S\to\Theta_\tau^\varphi,
 \qquad
 \pi_1(\vartheta,\widetilde\vartheta)=\vartheta,
\]
be the source projection.

\smallskip
\noindent\emph{Step 1: the source projection is locally injective on tangent spaces.}
Take
$(0,\dot{\widetilde\vartheta})\in T_{(\vartheta,\widetilde\vartheta)}S$.
Choose a differentiable curve
$t\mapsto(\vartheta(t),\widetilde\vartheta(t))$ in $S$ with that tangent at $t=0$.  Along the curve,
\[
 \mathcal F_\tau^\varphi(\vartheta(t))
 =\mathcal F_{\widetilde\tau}^\varphi(\widetilde\vartheta(t)).
\]
Differentiating at $t=0$ gives
\[
 D\mathcal F_\tau^\varphi(\vartheta)\dot\vartheta
 =D\mathcal F_{\widetilde\tau}^\varphi(\widetilde\vartheta)
  \dot{\widetilde\vartheta}.
\]
The source tangent is zero, so the left side vanishes.  Jacobian regularity makes
$D\mathcal F_{\widetilde\tau}^\varphi(\widetilde\vartheta)$ injective; hence
$\dot{\widetilde\vartheta}=0$.  Therefore
\[
 \ker(D\pi_1|_{TS})=\{0\}.
\]
The tangent map $D\pi_1|_{TS}$ is injective into the $d_\tau$-dimensional source tangent space, and consequently
\begin{equation}
 \dim S\le d_\tau.
 \label{eq:rsa-cross-chart-dim-first}
\end{equation}

\smallskip
\noindent\emph{Step 2: equality in \eqref{eq:rsa-cross-chart-dim-first} is impossible.}
Suppose $\dim S=d_\tau$.  Then $D\pi_1|_{TS}$ is an isomorphism at every regular point.  The inverse function theorem implies that, locally, $S$ is the graph of a smooth same-function partner over an open subset of the source family.  Thus there are an open set $V\subseteq\Theta_\tau^\varphi$ and a smooth map
$\Psi:V\to\Theta_{\widetilde\tau}^\varphi$ such that
\[
 \mathcal F_{\widetilde\tau}^\varphi(\Psi(\vartheta))
 =\mathcal F_\tau^\varphi(\vartheta)
 \qquad\text{for every }\vartheta\in V.
\]
The source members of this collision family fill an open set.  Definition~\ref{def:rsa-full-fiber-general-position} supplies one point at which the appropriate one-step identifiability theorem applies.  At that point, equality of the represented functions forces the two minimal realizations to agree after permutation, and after positive coordinatewise scaling in the ReLU case.  Those gauges have already been fixed, so the two normalized parameter values are equal.  This contradicts that $S$ is off-diagonal.  Hence every off-diagonal stratum satisfies
\begin{equation*}
 \dim S\le d_\tau-1.
\end{equation*}

\smallskip
\noindent\emph{Step 3: project to the exceptional source set.}
The smooth map $\pi_1$ is locally Lipschitz, so Hausdorff dimension cannot increase under projection:
\[
 \dim_{\rm H}\pi_1(S)
 \le\dim S
 \le d_\tau-1.
\]
Taking the countable locally finite union over all off-diagonal strata and all target families preserves the same Hausdorff-dimension bound.  The resulting exceptional subset of the $d_\tau$-dimensional source family has empty interior and Lebesgue measure zero.

Outside that set, $G_\vartheta$ has no different normalized minimal realization of the same function, which proves \eqref{eq:rsa-full-fiber-identifiability}.  Restoring permutation and positive ReLU scale turns the single normalized point into one activation-gauge orbit.
\end{proof}

\begin{corollary}[Generic unique-core principle]
\label{cor:rsa-generic-unique-core}
For a generic source realization covered by Theorem~\ref{thm:rsa-generic-full-fiber}, every minimal realization of its function $f$ has the same minimal axis dictionary.  We denote this common dictionary by
$\mathcal J_\varphi(f)$ and call it the \emph{essential core of $f$}.  More explicitly, any two minimal realizations of $f$ have the same number of terms and, after permutation,
\[
 \widetilde g_j=g_j,
 \qquad
 \widetilde v_j=v_j
 \quad\text{for Softplus},
\]
whereas for ReLU there are $\alpha_j>0$ such that
\begin{equation}
 \widetilde g_j=\alpha_jg_j,
 \qquad
 \alpha_j\widetilde v_j=v_j.
 \label{eq:rsa-relu-unique-core}
\end{equation}
On the evaluation set $\Omega_P$, the nonzero projected rank-one terms of these axes determine the function-intrinsic set of \emph{privileged-axis RSM directions}
\begin{equation}
 \mathcal D_\varphi(f)
 :=\{\R_{>0}\,\Pi_{\mathsf H}(f_jf_j^\top):
 [(v_j,g_j)]_\varphi\in\mathcal J_\varphi(f)\}.
 \label{eq:rsa-function-intrinsic-directions}
\end{equation}
Thus the first sum in \eqref{eq:rsa-orbit-core-plus-addition} is, generically, a positive weighting of a uniquely determined set of RSM directions.
\end{corollary}

\begin{proof}
Take two minimal realizations of the same generic function $f$ and normalize their hidden-unit order and, for ReLU, their positive scales.  If the normalized realizations were different, either would give an off-diagonal same-function partner forbidden by Theorem~\ref{thm:rsa-generic-full-fiber}.  Hence the normalized realizations are equal.

Undoing normalization leaves only the activation's unavoidable gauge.  Softplus has no nontrivial positive-scale gauge.  ReLU has
$(v,g)\sim(\alpha^{-1}v,\alpha g)$, which gives \eqref{eq:rsa-relu-unique-core}.  Positive rescaling multiplies $\Pi_{\mathsf H}(f_jf_j^\top)$ by a positive scalar without changing its ray, yielding \eqref{eq:rsa-function-intrinsic-directions}.  Lemma~\ref{lem:rsa-theiss-rsm-decomp} then identifies the first sum in \eqref{eq:rsa-orbit-core-plus-addition} as a positive weighting of representatives of these directions.
\end{proof}

\paragraph{Why the symmetry catalog need not be complete.}
Theorem~\ref{thm:rsa-generic-full-fiber} ranges over all normalized minimal-realization families, not only over realizations generated by the Theiss catalog.  The following corollary is therefore a scope statement, not a second identifiability theorem: it shows explicitly that an off-catalog minimal partner is already included in the theorem's exceptional set.

\begin{corollary}[Off-catalog collisions are part of the generic exceptional set]
\label{cor:rsa-missing-symmetry-generic}
For source and target families, define
\begin{equation}
 \mathcal C_{{\rm miss},\tau,\widetilde\tau}^\varphi
 :=\left\{(\vartheta,\widetilde\vartheta):
 \begin{array}{l}
 Y_{G_\vartheta}^\varphi=Y_{G_{\widetilde\vartheta}}^\varphi,\\
 G_\vartheta,G_{\widetilde\vartheta}\text{ are normalized and minimal},\\
 G_{\widetilde\vartheta}\notin\mathcal O_{\rm cat}(G_\vartheta)
 \end{array}\right\}.
 \label{eq:rsa-missing-symmetry-set}
\end{equation}
Under Theorem~\ref{thm:rsa-generic-full-fiber}, if the source family has dimension $d_\tau$, then
\begin{equation}
 \dim_{\rm H}\pi_1\left(
 \bigcup_{\widetilde\tau}
 \mathcal C_{{\rm miss},\tau,\widetilde\tau}^\varphi
 \right)
 \le d_\tau-1.
 \label{eq:rsa-missing-symmetry-codim}
\end{equation}
Thus generic uniqueness of the essential core does not require a complete symmetry catalog.
\end{corollary}

\begin{proof}
Every pair in \eqref{eq:rsa-missing-symmetry-set} is an off-diagonal normalized minimal same-function pair.  It is therefore contained in the collision set analyzed in Theorem~\ref{thm:rsa-generic-full-fiber}, and its source projection is contained in the theorem's exceptional source set.  The dimension bound follows because a subset cannot have larger Hausdorff dimension than the set containing it.
\end{proof}

\begin{remark}[Relation to the Vla\v{c}i\'c--B\"olcskei program]
An exhaustive null-network classification would enumerate every exact realization-preserving identity and prove global identifiability after all such symmetries have been quotiented.  The generic strategy used here asks for less.  Known Theiss redundancies collapse to one minimal dictionary by Proposition~\ref{prop:rsa-theiss-primitives-reduction}; any missing continuous symmetry is another gauge or is Jacobian-singular by Proposition~\ref{prop:rsa-continuous-missing-symmetry}; and any remaining separate, off-catalog minimal partner is confined to the lower-dimensional exceptional set of Corollary~\ref{cor:rsa-missing-symmetry-generic}.  This is the same trade made by the generic null-net theorem earlier in the paper: generic uniqueness replaces exhaustive enumeration when the intended conclusion only requires a well-defined generic core.
\end{remark}

\subsubsection{The essential-core cone and raw RSA}
\label{sec:rsa-raw-core-geometry}

Corollary~\ref{cor:rsa-generic-unique-core} makes the positive part of the Theiss decomposition truly well-defined.  Let $f$ be a generic represented function, choose one nonzero sampled feature vector $f_j$ from each class in $\mathcal J_\varphi(f)$, and put
\[
 q_j:=\Pi_{\mathsf H}(f_jf_j^\top).
\]
A different positive representative of the same ReLU class only rescales $q_j$.  It therefore leaves unchanged the set of all nonnegative combinations of these directions.  Define
\begin{equation}
 Q_f:\R^m\to\mathsf H,
 \qquad
 Q_f\gamma:=\sum_{j=1}^{m}\gamma_jq_j,
 \label{eq:rsa-Q-map}
\end{equation}
and the \emph{essential-core cone}
\begin{equation}
 \mathcal C_f
 :=\operatorname{cone}\{q_1,\ldots,q_m\}
 =\{Q_f\gamma:\gamma_j\ge0\}.
 \label{eq:rsa-core-cone}
\end{equation}
Here a cone simply means the set of all nonnegative linear combinations of its generators.  It contains all possible positive weightings of the function's privileged-axis RSM directions.

\begin{proposition}[The function-intrinsic core cone]
\label{prop:rsa-common-cone}
For a generic function $f$, the cone $\mathcal C_f$ is independent of the chosen minimal realization.  Every raw realization obtained from a minimal realization of $f$ through the cataloged Theiss operations has projected RSM
\begin{equation}
 g:=\Pi_{\mathsf H}R(H)
 =Q_f\gamma+a,
 \qquad \gamma_j>0,
 \label{eq:rsa-unique-core-plus-symmetry}
\end{equation}
where $Q_f\gamma\in\mathcal C_f$ is the weighted essential-core geometry and $a$ is the projected contribution of redundant added features.  Scaling and duplication move the core geometry within $\mathcal C_f$; an added feature enlarges the cone exactly when its projected rank-one term lies outside $\mathcal C_f$.

Moreover, if exact weak alignment holds on both sides of a nonlinear step and the one-step minimality and identifiability hypotheses apply, the two compared networks determine the same privileged-axis directions and hence the same cone.
\end{proposition}

\begin{proof}
Corollary~\ref{cor:rsa-generic-unique-core} makes the essential classes, and therefore their positive RSM rays, independent of the minimal realization.  Choosing different representatives merely rescales the generators and leaves their cone unchanged.  Lemma~\ref{lem:rsa-theiss-rsm-decomp} gives \eqref{eq:rsa-unique-core-plus-symmetry}; Proposition~\ref{prop:rsa-theiss-primitives-reduction} ensures that the cataloged duplicate, zero, constant, and aligned/opposite constructions return to the same minimal classes.

Changing scale or duplication count changes only $\gamma$, so $Q_f\gamma$ remains in $\mathcal C_f$.  If an added feature has projected contribution $r_u:=\Pi_{\mathsf H}(uu^\top)$, the enlarged cone is $\operatorname{cone}(\mathcal C_f\cup\{r_u\})$.  This equals $\mathcal C_f$ when $r_u\in\mathcal C_f$ and is strictly larger otherwise.

Finally, exact weak alignment at the two adjacent layers gives equality of the transported one-step functions.  The exact weak--strong and one-step identifiability results match the task-used minimal axes up to the activation gauge.  Their privileged-axis RSM directions, and hence their cones, are therefore equal.
\end{proof}

\begin{proposition}[RSA inside the essential-core cone]
\label{prop:rsa-equality-criterion}
Suppose two representations have no added-feature geometry, so
$g_A=Q_f\gamma_A$ and $g_B=Q_f\gamma_B$.  Then if 
$K_{ij}:=\langle q_i,q_j\rangle_F$ is the feature Gram matrix, 
\begin{equation}
 \rho_{\rm RSA}(A,B)
 =\frac{\gamma_A^\top K\gamma_B}
 {\sqrt{\gamma_A^\top K\gamma_A}
  \sqrt{\gamma_B^\top K\gamma_B}}.
 \label{eq:rsa-weight-gram-formula}
\end{equation}
The score is one exactly when the two core geometries lie on the same positive ray in the cone.
\end{proposition}

\begin{proof}
By the definition of Pearson RSM--RSA as a cosine in projected-RSM space, $$\rho_{\rm RSA}(A,B)
 =\cos(Q_f\gamma_A,Q_f\gamma_B).$$ Expanding the Frobenius inner products gives
\[
 \langle Q_f\gamma_A,Q_f\gamma_B\rangle_F
 =\sum_{i,j}\gamma_{Ai}\gamma_{Bj}\langle q_i,q_j\rangle_F
 =\gamma_A^\top K\gamma_B,
\]
and similarly
$\|Q_f\gamma_A\|_F^2=\gamma_A^\top K\gamma_A$, which proves \eqref{eq:rsa-weight-gram-formula}.  Equality in Cauchy--Schwarz gives
\begin{equation}
 Q_f\gamma_B=c\,Q_f\gamma_A
 \quad\text{for some }c>0.
\end{equation}
If the $q_j$ are linearly independent, equality of the two weighted sums forces equality of their coefficients up to the same scalar $c$.
\end{proof}

The positive conclusion is that the represented function now determines a well-defined cone of core geometries.  Different scale and multiplicity choices select different points in that cone, and RSA equals one precisely when the selected points lie on the same ray.  This immediately exposes the remaining raw-RSA ambiguity.

\begin{proposition}[Raw RSA is not forced by exact weak alignment]
\label{prop:rsa-no-raw-soft-bound}
There is no universal lower bound on raw RSM-RSA that tends to one as adjacent weak-alignment errors tend to zero, even when the two reduced privileged-axis sets agree exactly.  This is a direct consequence of the same scale, duplication, and feature-addition freedom identified by Theiss et al.
\end{proposition}

\begin{proof}
Choose two nonzero feature vectors $f_1,f_2\in\R^P$ such that
\[
 q_1:=\Pi_{\mathsf H}(f_1f_1^\top),
 \qquad
 q_2:=\Pi_{\mathsf H}(f_2f_2^\top)
\]
are not collinear.  Let network $A$ contain one copy of each feature.  Let network $B_M$ contain $M$ identical copies of $f_1$ and one copy of $f_2$.  Split the original downstream readout vector of $f_1$ equally among the $M$ copies.  The sum of the $M$ output contributions is unchanged, so $A$ and $B_M$ represent the same one-step function.  Their minimal axis dictionaries are therefore the same $\{f_1,f_2\}$.  Their weak error and reduced strong-axis error can therefore both be zero.

Their raw RSMs are
\[
 R_A=f_1f_1^\top+f_2f_2^\top,
 \qquad
 R_{B_M}=Mf_1f_1^\top+f_2f_2^\top,
\]
and hence
\[
 g_A=q_1+q_2,
 \qquad
 g_{B_M}=Mq_1+q_2.
\]
Because
$M^{-1}g_{B_M}=q_1+M^{-1}q_2\to q_1$, continuity of normalization gives
\[
 \frac{g_{B_M}}{\|g_{B_M}\|_F}
 \longrightarrow
 \frac{q_1}{\|q_1\|_F}.
\]
Therefore
\[
 \rho_{\rm RSA}(A,B_M)
 \longrightarrow
 \frac{\langle q_1+q_2,q_1\rangle_F}
 {\|q_1+q_2\|_F\,\|q_1\|_F}.
\]
The limit is strictly less than one because $q_1+q_2$ is not a positive multiple of $q_1$.  Thus a sequence with identically zero weak error has raw RSA bounded away from one, which rules out any universal weak-error-only lower bound converging to one.

Duplication already proves the claim.  Theiss feature addition gives still more freedom: if an added projected rank-one term lies outside the old cone, it creates a new RSM direction.  Their Proposition~B.4 shows that, as such generators are added, the attainable RSA scores against a fixed reference form nested intervals that can become very wide.  Hence exact agreement of the essential axes does not determine a raw RSM.
\end{proof}

Proposition~\ref{prop:rsa-no-raw-soft-bound} is a no-go result for weak--strong equivalence for \emph{raw} RSA, not for the core.  Exact weak alignment, together with the usual minimality and identifiability hypotheses, determines the privileged-axis dictionary and the cone $\mathcal C_f$.  It does not determine the weights $\gamma$ or the redundant term $a$.  A canonical RSA construction must therefore either remove this freedom or select a preferred point within it.

\subsubsection{Two canonical RSM constructions}
\label{sec:rsa-canonical-core}

There are two natural ways to turn the now-well-defined essential core into one RSM.  The \emph{axis quotient} removes the remaining gauge: it gives each essential axis class one equal, unit-normalized vote.  This is the simpler construction and is appropriate when the goal is to compare which privileged axes are present, independently of their scale, physical multiplicity, or downstream strength.  \emph{Minimum-norm balancing} instead selects a particular representative of the ReLU scale orbit.  This is closer to the gauge-selection strategy of \citet{theiss2026parameter}: it retains a gauge-invariant weighting based on how costly an axis is to generate and how strongly it is read out.  The potential advantage is that two axes need not count equally when one plays a much larger role in the computation.  The price is additional structure: one must choose a norm-based cost and common Euclidean coordinates in which that cost is evaluated.

\paragraph{Axis quotient: gauge removal.}

\begin{definition}[Axis-quotiented RSM and normalized RSA]
\label{def:rsa-axis-quotient}
For each nonzero sampled feature vector $f_j\in\R^P$ associated with an essential axis class, define
\[
 P[f_j]:=\frac{f_jf_j^\top}{\|f_j\|_2^2}.
\]
If the essential core has $m$ classes, define
\begin{equation}
 R_{\rm ax}(G)
 :=\frac1m\sum_{j=1}^{m}P[f_j],
 \qquad
 \rho_{\rm ax}(G,\widetilde G)
 :=\rho_{\rm RSA}(R_{\rm ax}(G),R_{\rm ax}(\widetilde G)).
 \label{eq:rsa-axis-quotient}
\end{equation}
The axis quotient gives one equal, scale-free contribution to each essential class, regardless of physical multiplicity.
\end{definition}

Proposition~\ref{prop:rsa-theiss-primitives-reduction} explains why the cataloged raw redundancies do not survive in this construction: duplicates represent one minimal class, zero groups disappear, constants enter the bias, and aligned/opposite groups leave at most one nonlinear class.  The normalization in $P[f_j]$ removes the remaining positive ReLU gauge.  Under Corollary~\ref{cor:rsa-generic-unique-core}, every minimal realization of the same generic function gives the same axis-quotiented RSM.  A positive rescaling leaves each projector unchanged because
\[
 \frac{(\alpha f)(\alpha f)^\top}{\|\alpha f\|_2^2}
 =\frac{ff^\top}{\|f\|_2^2}.
\]
Hence all minimal realizations produce the same average of projectors.
The axis-quotiented RSM is thus well defined and may therefore write it as $R_{\rm ax}(f)$.

The key point is that $R_{\rm ax}$ does support Weak--Strong Equivalence:
\begin{theorem}[Exact weak alignment implies axis-quotiented RSA equality]
\label{thm:rsa-axis-quotient-exact}
Suppose exact weak alignment holds at layers $\ell$ and $\ell+1$, and the induced one-step pair satisfies the minimality and identifiability hypotheses of Corollary~\ref{thm:pointwise-null-common}.  Then
\[
 R_{\rm ax}(A)=R_{\rm ax}(B)
 \qquad\text{and}\qquad
 \rho_{\rm ax}(A,B)=1
\]
whenever the common projected axis-quotiented RSM is nonzero.  The same conclusion holds at every layer recovered by exact zippering.
\end{theorem}

\begin{proof}
Exact weak alignment at the two adjacent layers gives an equality of transported one-step functions by Lemma~1.  Corollary~\ref{thm:pointwise-null-common} then gives a bijection of the minimal axes.  After permutation, their sampled postactivation feature vectors satisfy
\[
 f_j^B=\alpha_jf_j^A,
 \qquad \alpha_j>0,
\]
with $\alpha_j=1$ for Softplus.  Therefore $P[f_j^B]=P[f_j^A]$ for every matched pair, and summing the projectors gives
$R_{\rm ax}(A)=R_{\rm ax}(B)$.  Identical nonzero projected RSMs have Pearson correlation one.  Exact zippering supplies the same axis relation at every recovered layer.
\end{proof}

\paragraph{Minimum-norm balance: gauge selection.}
The axis quotient deliberately discards all relative weighting among essential axes.  The minimum-representation-norm and minimum-weight-norm rules of Theiss et al. make a different choice: they select the ReLU scale at which the cost of producing an axis is balanced against the norm of the readout that uses it.  The resulting RSM is more complicated, but it can retain information about computational importance.  Its weights depend jointly on the size or cost of a generated feature and on the downstream readout that uses it, rather than giving every essential class one equal vote.  This makes balancing attractive when the goal is not merely to compare the inventory of privileged axes, but to compare a norm-selected implementation of the computation.  Unlike the axis quotient, however, it depends on the chosen MR or MW cost and on the Euclidean coordinates in which those norms are measured.

\begin{definition}[Generation cost and minimum-norm balance]
\label{def:rsa-generation-cost}
For a ReLU minimal-core term $(v_j,g_j)$, a \emph{generation cost} is a positive number $s_j=s(g_j)$ satisfying
\begin{equation}
 s(\alpha g_j)=\alpha^2s(g_j),
 \qquad \alpha>0.
 \label{eq:rsa-scale-covariant-cost}
\end{equation}
It measures the squared norm assigned to producing the axis before the outgoing readout is applied.  The two choices used by Theiss et al. are
\begin{align}
 s_j^{\rm MR}&=\|h_j\|_2^2,
 &h_j&:=\bigl(\rhoR(g_j(x^1)),\ldots,\rhoR(g_j(x^P))\bigr)^\top,
 \label{eq:rsa-MR-generation-cost}\\
 s_j^{\rm MW}&=\|w_j\|_2^2+b_j^2,
 &g_j(x)&=w_j^\top x+b_j.
 \label{eq:rsa-MW-generation-cost}
\end{align}
The first is the feature-generation part of the minimum representation-norm parameterization (MRNP) objective; the second is the corresponding part of the minimum weight-norm parameterization (MWNP) objective.

Along the ReLU gauge
$(v_j,g_j)\mapsto(\alpha_j^{-1}v_j,\alpha_jg_j)$, define
\begin{equation}
 J_s(\alpha)
 :=\sum_{j=1}^{m}
 \left(\alpha_j^2s_j+\alpha_j^{-2}\|v_j\|_2^2\right).
 \label{eq:rsa-balance-objective}
\end{equation}
A minimal core is \emph{balanced} when its scales minimize $J_s$.
All generation costs and readout norms are evaluated in the Euclidean coordinates in which the supplied one-step expansion is written.
\end{definition}

Making an axis larger increases its generation cost but decreases the norm of its inverse-scaled readout; making it smaller does the opposite.  Balance selects the unique point where these two pressures agree.

\begin{theorem}[Balanced essential core and canonical RSM]
\label{thm:rsa-balanced-full-fiber}
Let $f$ be a represented ReLU function whose minimal realizations form one permutation-and-positive-scale orbit, as Theorem~\ref{thm:rsa-generic-full-fiber} guarantees generically.  Let
$G=(b,\{(v_j,g_j)\}_{j=1}^{m})$ be any minimal representative, with $s_j>0$ and $v_j\neq0$.  Then \eqref{eq:rsa-balance-objective} has a unique scale minimizer, with effective RSM weights
\begin{equation}
 \gamma_j^\star=(\alpha_j^\star)^2
 =\frac{\|v_j\|_2}{\sqrt{s_j}}.
 \label{eq:rsa-balanced-gamma}
\end{equation}
The resulting canonical RSM
\begin{equation}
 R_s^{\rm bal}(f)
 :=\sum_{j=1}^{m}\gamma_j^\star h_jh_j^\top
 \label{eq:rsa-balanced-rsm}
\end{equation}
is independent of the representative chosen from the ReLU scale orbit.
\end{theorem}

\begin{proof}
By hypothesis, every minimal realization of $f$ belongs to the same permutation-and-positive-scale orbit, so it is enough to minimize $J_s$ along that orbit. 

The objective separates over axes:
\[
 J_s(\alpha)=\sum_{j=1}^{m}J_j(\alpha_j),
 \qquad
 J_j(\alpha_j)=\alpha_j^2s_j+
 \alpha_j^{-2}\|v_j\|_2^2.
\]
Set $\gamma=\alpha_j^2>0$.  Then
\[
 F_j(\gamma)=\gamma s_j+\gamma^{-1}\|v_j\|_2^2.
\]
Because $s_j>0$ and $v_j\neq0$,
\[
 \lim_{\gamma\downarrow0}F_j(\gamma)=+\infty,
 \qquad
 \lim_{\gamma\uparrow\infty}F_j(\gamma)=+\infty,
\]
so a minimum exists in $(0,\infty)$.  Differentiating gives
\[
 F_j'(\gamma)=s_j-\gamma^{-2}\|v_j\|_2^2,
 \qquad
 F_j''(\gamma)=2\gamma^{-3}\|v_j\|_2^2>0.
\]
The positive second derivative makes $F_j$ strictly convex, hence its critical point is the unique minimizer.  Solving
$F_j'(\gamma)=0$ yields
\[
 \gamma_j^\star=\frac{\|v_j\|_2}{\sqrt{s_j}},
\]
which proves \eqref{eq:rsa-balanced-gamma}.  Since the full objective is a sum of independent strictly convex terms, the vector of optimal scales is unique.

Under scale $\alpha_j$, positive homogeneity gives the post-ReLU sample vector $\alpha_jh_j$.  Its RSM contribution is
\[
 (\alpha_jh_j)(\alpha_jh_j)^\top
 =\alpha_j^2h_jh_j^\top
 =\gamma_jh_jh_j^\top.
\]
Substituting the optimal weights and summing proves \eqref{eq:rsa-balanced-rsm}.

It remains to show that the result does not depend on the initial representative of the ReLU scale orbit.  Let another representative be
\[
 \widetilde h_j=\beta_jh_j,
 \qquad
 \widetilde v_j=\beta_j^{-1}v_j,
 \qquad
 \widetilde s_j=\beta_j^2s_j,
 \qquad \beta_j>0.
\]
Its optimal effective weight is
\begin{align*}
 \widetilde\gamma_j^\star
 &=\frac{\|\widetilde v_j\|_2}{\sqrt{\widetilde s_j}}\\
 &=\frac{\beta_j^{-1}\|v_j\|_2}
         {\beta_j\sqrt{s_j}}\\
 &=\beta_j^{-2}\gamma_j^\star.
\end{align*}
Therefore
\[
 \widetilde\gamma_j^\star
 \widetilde h_j\widetilde h_j^\top
 =\beta_j^{-2}\gamma_j^\star
  (\beta_jh_j)(\beta_jh_j)^\top
 =\gamma_j^\star h_jh_j^\top.
\]
Every rank-one contribution is unchanged, and so is their sum.  The single-orbit premise therefore makes this representative-independent matrix common to every minimal realization of $f$.
\end{proof}

Theorem~\ref{thm:rsa-balanced-full-fiber} extends Theiss Proposition~C.3 from a fixed cataloged orbit to the generically unique essential core.  Its extra content is that the selected RSM no longer depends on assuming in advance that the chosen Theiss orbit exhausts all minimal realizations of the function.

\begin{corollary}[Balanced RSM as a computationally weighted axis quotient]
\label{cor:rsa-weighted-axis-quotient}
Under Theorem~\ref{thm:rsa-balanced-full-fiber}, the MR-balanced RSM is
\begin{equation}
 R_{\rm BMR}(f)
 =\sum_{j=1}^{m}
 \|v_j\|_2\|h_j\|_2
 \frac{h_jh_j^\top}{\|h_j\|_2^2},
 \label{eq:rsa-BMR-weighted-projector}
\end{equation}
and the MW-balanced RSM is
\begin{equation}
 R_{\rm BMW}(f)
 =\sum_{j=1}^{m}
 \frac{\|v_j\|_2\|h_j\|_2^2}
 {\sqrt{\|w_j\|_2^2+b_j^2}}
 \frac{h_jh_j^\top}{\|h_j\|_2^2}.
 \label{eq:rsa-BMW-weighted-projector}
\end{equation}
Thus $R_{\rm ax}$ gives one equal vote per essential axis, whereas the balanced RSM gives one gauge-invariant vote weighted by the axis's readout strength relative to its generation cost.
\end{corollary}

\begin{proof}
For MR balance,
$s_j=\|h_j\|_2^2$, so Theorem~\ref{thm:rsa-balanced-full-fiber} gives
$\gamma_j^\star=\|v_j\|_2/\|h_j\|_2$.  Therefore
\begin{align*}
 \gamma_j^\star h_jh_j^\top
 &=\frac{\|v_j\|_2}{\|h_j\|_2}h_jh_j^\top\\
 &=\|v_j\|_2\|h_j\|_2
   \frac{h_jh_j^\top}{\|h_j\|_2^2}.
\end{align*}
Summing proves \eqref{eq:rsa-BMR-weighted-projector}.

For MW balance,
$\gamma_j^\star=\|v_j\|_2/
\sqrt{\|w_j\|_2^2+b_j^2}$.  Rewriting
$h_jh_j^\top$ as
$\|h_j\|_2^2(h_jh_j^\top/\|h_j\|_2^2)$ gives \eqref{eq:rsa-BMW-weighted-projector}.  Gauge invariance follows from the representative-independence proved in Theorem~\ref{thm:rsa-balanced-full-fiber}.
\end{proof}

The preceding corollary makes the comparison with the axis quotient explicit.  Both are averages of the same normalized rank-one projectors.  The axis quotient gives each projector weight $1/m$; the balanced construction assigns a gauge-invariant weight determined by readout strength and generation cost.  The latter can therefore retain a notion of computational importance that the equal-vote quotient intentionally removes.

\begin{definition}[Common-coordinate one-step comparison]
\label{def:rsa-common-coordinate-comparison}
Let $G_A$ and $G_B$ be adjacent one-step expansions, and let $E_\ell,E_{\ell+1}$ be their weak-comparison maps.  Using the output action $EG$ and input pullback $E^\ast G$ from Definition~\ref{def:one-step-expansion}, define
\begin{equation}
 G_A^{\rm com}:=E_{\ell+1}G_A,
 \qquad
 G_B^{\rm com}:=E_\ell^{\ast}G_B.
 \label{eq:rsa-common-coordinate-pair}
\end{equation}
These are the two one-step formulas written with the same layer-$\ell$ input coordinate and the same layer-$(\ell+1)$ output coordinate.
\end{definition}

\begin{theorem}[Exact weak alignment forces equality of balanced RSMs in common coordinates]
\label{thm:rsa-minnorm-equality}
Suppose exact weak alignment holds at layers $\ell$ and $\ell+1$.  Let $(G_A^{\rm com},G_B^{\rm com})$ be the common-coordinate pair in Definition~\ref{def:rsa-common-coordinate-comparison}.  If this pair satisfies the minimality and identifiability hypotheses of Corollary~\ref{thm:pointwise-null-common}, then, for either $s=s^{\rm MR}$ or $s=s^{\rm MW}$,
\begin{equation}
 R_s^{\rm bal}(G_A^{\rm com})
 =R_s^{\rm bal}(G_B^{\rm com}).
 \label{eq:rsa-common-coordinate-balanced-equality}
\end{equation}
Consequently their Pearson RSM--RSA score is one whenever the common projected RSM is nonzero.
\end{theorem}

\begin{proof}
Write the two original expansions and comparison maps as
\[
 G_A=\bigl(b_A,\{(v_j^A,g_j^A)\}_{j=1}^{m_A}\bigr),
 \qquad
 G_B=\bigl(b_B,\{(v_i^B,g_i^B)\}_{i=1}^{m_B}\bigr),
\]
\[
 E_r(z)=T_rz+a_r,
 \qquad r\in\{\ell,\ell+1\}.
\]
Unpacking Definition~\ref{def:rsa-common-coordinate-comparison} gives
\begin{align*}
 G_A^{\rm com}
 &=\left(T_{\ell+1}b_A+a_{\ell+1},
 \{(T_{\ell+1}v_j^A,g_j^A)\}_{j=1}^{m_A}\right),\\
 G_B^{\rm com}
 &=\left(b_B,
 \{(v_i^B,g_i^B\circ E_\ell)\}_{i=1}^{m_B}\right).
\end{align*}
For $z\in C$, exact weak alignment and Lemma~1 give
\begin{align*}
 Y_{G_B^{\rm com}}^{\rhoR}(z)
 &=Y_{G_B}^{\rhoR}(E_\ell z)\\
 &=E_{\ell+1}\bigl(Y_{G_A}^{\rhoR}(z)\bigr)\\
 &=Y_{G_A^{\rm com}}^{\rhoR}(z).
\end{align*}
Hence
\begin{equation*}
 Y_{G_A^{\rm com}}^{\rhoR}(z)
 =Y_{G_B^{\rm com}}^{\rhoR}(z),
 \qquad z\in C.
\end{equation*}
The two expansions therefore represent exactly the same function of the same input variable and take values in the same output Euclidean space.  Corollary~\ref{thm:pointwise-null-common} applies directly to this equality.

For brevity, write
\[
 \widehat v_j^A:=T_{\ell+1}v_j^A,
 \qquad
 \widehat g_i^B:=g_i^B\circ E_\ell.
\]
Corollary~\ref{thm:pointwise-null-common} implies that $m_A=m_B=:m$ and, after a permutation of the $B$ terms, there are positive scales $\lambda_j$ such that
\begin{equation}
 \widehat g_j^B=\lambda_j g_j^A,
 \qquad
 \lambda_jv_j^B=\widehat v_j^A.
 \label{eq:rsa-weak-identification-pair}
\end{equation}
Thus the matching of all counted axes and readouts is a conclusion of exact weak alignment plus one-step identifiability; it is not an additional assumption.

Let $h_j^A$ and $\widehat h_j^B$ be the post-ReLU sample vectors generated by $g_j^A$ and $\widehat g_j^B$ on the common input sample.  Positive homogeneity and \eqref{eq:rsa-weak-identification-pair} give
\begin{equation}
 \widehat h_j^B
 =\rhoR(\widehat g_j^B)
 =\rhoR(\lambda_jg_j^A)
 =\lambda_jh_j^A.
 \label{eq:rsa-common-coordinate-feature-scale}
\end{equation}
The readout relation gives
\begin{equation}
 v_j^B=\lambda_j^{-1}\widehat v_j^A,
 \qquad
 \|v_j^B\|_2
 =\lambda_j^{-1}\|\widehat v_j^A\|_2.
 \label{eq:rsa-common-coordinate-readout-scale}
\end{equation}
Both readout vectors in this equation lie in the common $B$-side output coordinates by construction.

We next compute the generation costs in the common input coordinates.  For MRNP,
\[
 s_{Aj}^{\rm MR}=\|h_j^A\|_2^2,
 \qquad
 s_{Bj}^{{\rm MR},{\rm com}}=\|\widehat h_j^B\|_2^2
 =\lambda_j^2s_{Aj}^{\rm MR}.
\]
For MWNP, write the original $B$-side affine argument as
\[
 g_i^B(y)=(w_i^B)^\top y+\beta_i^B.
\]
Since $E_\ell(z)=T_\ell z+a_\ell$, its pullback appearing in $G_B^{\rm com}$ is
\begin{align*}
 \widehat g_i^B(z)
 &=g_i^B(E_\ell z)\\
 &=(w_i^B)^\top(T_\ell z+a_\ell)+\beta_i^B\\
 &=(T_\ell^\top w_i^B)^\top z
   +\bigl((w_i^B)^\top a_\ell+\beta_i^B\bigr).
\end{align*}
Therefore its common-coordinate MW generation cost is
\begin{equation}
 s_{Bi}^{{\rm MW},{\rm com}}
 =\|T_\ell^\top w_i^B\|_2^2
 +\bigl((w_i^B)^\top a_\ell+\beta_i^B\bigr)^2.
 \label{eq:rsa-common-coordinate-MW-cost}
\end{equation}
Now write
\[
 g_j^A(z)=(w_j^A)^\top z+\beta_j^A,
 \qquad
 \widehat g_j^B(z)=(\widehat w_j^B)^\top z+\widehat\beta_j^B.
\]
The first relation in \eqref{eq:rsa-weak-identification-pair} implies
\[
 (\widehat w_j^B,\widehat\beta_j^B)
 =\lambda_j(w_j^A,\beta_j^A).
\]
Using \eqref{eq:rsa-common-coordinate-MW-cost}, this yields
\[
 s_{Bj}^{{\rm MW},{\rm com}}
 =\|\widehat w_j^B\|_2^2+(\widehat\beta_j^B)^2
 =\lambda_j^2\bigl(\|w_j^A\|_2^2+(\beta_j^A)^2\bigr)
 =\lambda_j^2s_{Aj}^{\rm MW}.
\]
Thus, for either balancing rule,
\begin{equation}
 s_{Bj}^{\rm com}=\lambda_j^2s_{Aj}.
 \label{eq:rsa-common-coordinate-generation-scale}
\end{equation}

The per-axis balancing calculation of Theorem~\ref{thm:rsa-balanced-full-fiber} gives the effective weight
\[
 \gamma_j^\star=\frac{\|v_j\|_2}{\sqrt{s_j}}.
\]
Define
\[
 \gamma_{Aj}^\star
 :=\frac{\|\widehat v_j^A\|_2}{\sqrt{s_{Aj}}},
 \qquad
 \gamma_{Bj}^\star
 :=\frac{\|v_j^B\|_2}{\sqrt{s_{Bj}^{\rm com}}}.
\]
Using \eqref{eq:rsa-common-coordinate-readout-scale} and \eqref{eq:rsa-common-coordinate-generation-scale},
\begin{align*}
 \gamma_{Bj}^\star
 &=\frac{\lambda_j^{-1}\|\widehat v_j^A\|_2}
         {\sqrt{\lambda_j^2s_{Aj}}}\\
 &=\lambda_j^{-2}\gamma_{Aj}^\star.
\end{align*}
Combining this identity with \eqref{eq:rsa-common-coordinate-feature-scale}, the matched $B$-side balanced contribution is
\begin{align*}
 \gamma_{Bj}^\star
 \widehat h_j^B(\widehat h_j^B)^\top
 &=\lambda_j^{-2}\gamma_{Aj}^\star
   (\lambda_jh_j^A)(\lambda_jh_j^A)^\top\\
 &=\gamma_{Aj}^\star
   h_j^A(h_j^A)^\top.
\end{align*}
The equality holds term by term after the identified permutation.  Summing over $j=1,\ldots,m$ proves \eqref{eq:rsa-common-coordinate-balanced-equality}.  If the common projected RSM is nonzero, its Pearson correlation with itself is one.
\end{proof}

The axis quotient and the balanced construction therefore solve different versions of the same problem.  The quotient is the cleaner consequence of strong-axis identification and is insensitive to how heavily each axis participates in the computation.  Minimum-norm balance is closer to the Theiss implementation-selection program: it chooses a particular gauge and preserves a norm-dependent weighting of computational importance, at the cost of a more elaborate and coordinate-dependent comparison.

\subsubsection{Soft weak alignment, canonical RSA, and ridge prediction}
\label{sec:rsa-soft-ridge}

Theorem~\ref{thm:rsa-axis-quotient-exact} says that exact weak alignment determines the axis-quotiented RSM.  As before, we would like to have a soft version of this statement.   To achieve this, we need one elementary fact: nearby projected RSMs have nearby cosine similarity:

\begin{lemma}[Cosine stability for projected RSMs]
\label{lem:rsa-cosine-stability}
Let $x$ and $y$ be nonzero elements of a Frobenius inner-product space, with
$\min\{\|x\|_F,\|y\|_F\}\ge\beta>0$.  Then
\begin{equation}
 1-\frac{\langle x,y\rangle_F}{\|x\|_F\|y\|_F}
 \le\frac{2\|x-y\|_F^2}{\beta^2}.
 \label{eq:rsa-basic-cosine-stability}
\end{equation}
More specifically, if
$g_A=g^\star+e_A$ and $g_B=g^\star+e_B$, and
$r:=\max\{\|e_A\|_F,\|e_B\|_F\}<\|g^\star\|_F$, then
\begin{equation}
 \rho_{\rm RSA}(g_A,g_B)
 \ge
 1-\frac{2\|e_A-e_B\|_F^2}
 {\bigl(\|g^\star\|_F-r\bigr)^2}.
 \label{eq:rsa-core-residual-bound}
\end{equation}
\end{lemma}

\begin{proof}
For nonzero vectors $x$ and $y$ in any Hilbert space,
\begin{equation}
 1-\frac{\langle x,y\rangle}{\|x\|\|y\|}
 =\frac12\left\|
 \frac{x}{\|x\|}-\frac{y}{\|y\|}
 \right\|^2.
 \label{eq:rsa-cosine-normalized-distance}
\end{equation}
We first bound the distance between the normalized vectors.  Assume without loss of generality that $\|x\|\le\|y\|$.  By the triangle inequality,
\begin{align*}
 \left\|\frac{x}{\|x\|}-\frac{y}{\|y\|}\right\|
 &\le
 \left\|\frac{x-y}{\|x\|}\right\|
 +\left\|y\left(\frac1{\|x\|}-\frac1{\|y\|}\right)\right\|\\
 &=\frac{\|x-y\|}{\|x\|}
 +\frac{|\|y\|-\|x\||}{\|x\|}\\
 &\le\frac{2\|x-y\|}{\|x\|},
\end{align*}
where the last line uses the reverse triangle inequality.  Since
$\|x\|=\min\{\|x\|,\|y\|\}\ge\beta$, substitution into \eqref{eq:rsa-cosine-normalized-distance} gives
\[
 1-\cos(x,y)
 \le\frac12\left(\frac{2\|x-y\|}{\beta}\right)^2
 =\frac{2\|x-y\|^2}{\beta^2},
\]
which proves \eqref{eq:rsa-basic-cosine-stability}.

For the second statement, put
$x=g_A=g^\star+e_A$ and
$y=g_B=g^\star+e_B$.  Then
$x-y=e_A-e_B$.  The reverse triangle inequality gives
\[
 \|g_A\|_F\ge\|g^\star\|_F-\|e_A\|_F
 \ge\|g^\star\|_F-r,
\]
and the same lower bound holds for $g_B$.  Apply the first part with
$\beta=\|g^\star\|_F-r$ to obtain \eqref{eq:rsa-core-residual-bound}.
\end{proof}

Recall the proof of Appendix Theorem~\ref{thm:common-asymp-ws}.  For a threshold $\theta<1$, it lets $B_\theta$ be the number of task-required axes that fail the best injective $\theta$-match, and proves
\begin{equation}
 B_\theta
 \le
 \frac{e_{\rm weak}^2}{\kappa(\theta)^2},
 \qquad
 e_{\rm weak}
 :=\|W_{\ell+1}^B\|_{\op}\omega_\ell+\omega_{\ell+1}.
 \label{eq:rsa-recalled-bad-axis-bound}
\end{equation}
\begin{theorem}[Soft weak alignment implies an axis-quotiented RSA bound]
\label{thm:rsa-soft-axis-bound}
Fix a compact comparison family satisfying Appendix Theorem~\ref{thm:common-asymp-ws}.  Suppose the two reduced representations have the same $m$ task-required axes, and compute the soft-axis score and $R_{\rm ax}$ from the same empirical axis vectors on $\Omega_P$.  If both projected axis-quotiented RSMs have Frobenius norm at least $\beta>0$, then for every $\theta\in(0,1)$,
\begin{equation}
 \rho_{\rm ax}(A,B)
 \ge
 1-\frac{2}{\beta^2}
 \left[
 \frac{2\sqrt2(1-\theta)}
 {\sqrt{1-(1-\theta)^2}}
 +\frac{\sqrt2}{m}
  \frac{e_{\rm weak}^2}{\kappa(\theta)^2}
 \right]^2.
 \label{eq:rsa-soft-axis-correlation-bound}
\end{equation}
\end{theorem}

\begin{proof}
The recalled weak--strong estimate \eqref{eq:rsa-recalled-bad-axis-bound} controls how many axes can fail the threshold.  We first show how one successful threshold match controls one normalized rank-one projector.

Let $u,v\in\R^P$ be a good matched pair.  By the definition of the soft-axis score, there is an allowed nonzero scalar $\alpha$ such that
\begin{equation}
 \frac{\|u-\alpha v\|_2}
 {\|u\|_2+|\alpha|\|v\|_2}
 \le\delta,
 \qquad
 \delta:=1-\theta<1.
 \label{eq:rsa-soft-pair-ratio}
\end{equation}
Set
$a=\|u\|_2$, $b=|\alpha|\|v\|_2$, and let $c$ be the cosine between $u$ and $\alpha v$.  Then
\[
 \|u-\alpha v\|_2^2=a^2+b^2-2abc.
\]
Squaring \eqref{eq:rsa-soft-pair-ratio} gives
\begin{equation*}
 a^2+b^2-2abc\le\delta^2(a+b)^2.
\end{equation*}
Also,
\[
 (a+b)^2-\|u-\alpha v\|_2^2=2ab(1+c).
\]
Using \eqref{eq:rsa-soft-pair-ratio} and $1+c\le2$,
\[
 4ab\ge2ab(1+c)
 \ge(1-\delta^2)(a+b)^2,
\]
so
\begin{equation}
 \frac{(a+b)^2}{ab}
 \le\frac4{1-\delta^2}.
 \label{eq:rsa-ab-ratio-bound}
\end{equation}
Now
\begin{align*}
 1-c
 &=\frac{\|u-\alpha v\|_2^2-(a-b)^2}{2ab}\\
 &\le\frac{\|u-\alpha v\|_2^2}{2ab}\\
 &\le\frac{\delta^2(a+b)^2}{2ab}\\
 &\le\frac{2\delta^2}{1-\delta^2},
\end{align*}
where the last line uses \eqref{eq:rsa-ab-ratio-bound}.

The normalized projector of a nonzero vector $w$ is
$P[w]=ww^\top/\|w\|_2^2$.  A direct calculation gives
\[
 \|P[u]-P[v]\|_F^2
 =2(1-c^2)
 \le4(1-c).
\]
Combining with the previous bound yields
\begin{equation*}
 \|P[u]-P[v]\|_F
 \le\frac{2\sqrt2\delta}{\sqrt{1-\delta^2}}
 =\frac{2\sqrt2(1-\theta)}{\sqrt{1-(1-\theta)^2}}.
\end{equation*}
The sign of $\alpha$ is irrelevant because $P[\alpha v]=P[v]$.

Choose the injective matching used in Appendix Theorem~\ref{thm:common-asymp-ws}.  There are $m-B_\theta$ good pairs and $B_\theta$ bad pairs.  For a bad pair, each normalized projector has Frobenius norm one, and
\[
 \|P[u]-P[v]\|_F^2
 =2\bigl(1-\langle\widehat u,\widehat v\rangle^2\bigr)
 \le2,
\]
so the distance is at most $\sqrt2$.  Averaging the projector differences and using the triangle inequality gives
\begin{align*}
 \|R_{\rm ax}(A)-R_{\rm ax}(B)\|_F
 &\le\frac1m\left[
 (m-B_\theta)\frac{2\sqrt2(1-\theta)}{\sqrt{1-(1-\theta)^2}}
 +B_\theta\sqrt2\right]\\
 &\le\frac{2\sqrt2(1-\theta)}{\sqrt{1-(1-\theta)^2}}
 +\frac{\sqrt2 B_\theta}{m}\\
 &\le\frac{2\sqrt2(1-\theta)}{\sqrt{1-(1-\theta)^2}}
 +\frac{\sqrt2}{m}
  \frac{e_{\rm weak}^2}{\kappa(\theta)^2},
\end{align*}
where the final line uses \eqref{eq:rsa-recalled-bad-axis-bound}.  Orthogonal projection is nonexpansive, so the same bound holds after applying $\Pi_{\mathsf H}$.

Let
$x=\Pi_{\mathsf H}R_{\rm ax}(A)$ and
$y=\Pi_{\mathsf H}R_{\rm ax}(B)$.  By assumption, both norms are at least $\beta$.  Apply Lemma~\ref{lem:rsa-cosine-stability} to the preceding RSM-distance bound.  This gives exactly \eqref{eq:rsa-soft-axis-correlation-bound}.
\end{proof}

\begin{corollary}[Asymptotic weak alignment implies normalized RSA convergence]
\label{cor:rsa-asymptotic-axis-convergence}
Under the hypotheses of Theorem~\ref{thm:rsa-soft-axis-bound}, if
$\omega_\ell\to0$ and $\omega_{\ell+1}\to0$ along a sequence in the compact family, then
\[
 \rho_{\rm ax}(A,B)\longrightarrow1.
\]
\end{corollary}

\begin{proof}
The adjacent weak error $e_{\rm weak}$ in \eqref{eq:rsa-recalled-bad-axis-bound} tends to zero because the compact family uniformly bounds $\|W_{\ell+1}^B\|_{\op}$.  Fix $\theta<1$.  Then the integer $B_\theta$ is bounded above by a quantity tending to zero, so $B_\theta=0$ for all sufficiently large members of the sequence.  At those indices, the projected axis-quotient distance is at most
\[
 \frac{2\sqrt2(1-\theta)}{\sqrt{1-(1-\theta)^2}}.
\]
Given any tolerance, first choose $\theta$ sufficiently close to one that this quantity is small enough, and then choose the sequence index large enough that $B_\theta=0$.  Lemma~\ref{lem:rsa-cosine-stability} then makes the RSA loss arbitrarily small.  Hence $\rho_{\rm ax}\to1$.
\end{proof}

\paragraph{A local ridge--canonical-RSA prediction.}
The soft theorem describes a one-sided frontier rather than a deterministic regression curve (see Fig. \ref{fig:rsa-ridge-sampling-schematic}A).  A simple local form is useful for empirical interpretation.  Let
$R_{\rm ridge}^2$ be the held-out squared correlation of a theorem-compatible linear comparison: the map is injective on the reduced core, or the comparison has been made effectively bidirectional or rank-matched.  Write its normalized prediction residual as
\begin{equation}
 R_{\rm ridge}^2=1-\varepsilon^2.
 \label{eq:rsa-ridge-residual}
\end{equation}

\begin{corollary}[Local ridge--canonical-RSA relation]
\label{cor:rsa-ridge-canonical}
Let $c_m$ and $c_\star$ be the projected canonical RSMs of a comparison model and a target.  Suppose that, on a locally identifiable compact family,
\[
 \|c_m-c_\star\|_F
 \le L_{\rm can}c_{\rm w}\varepsilon,
 \qquad
 \min\{\|c_m\|_F,\|c_\star\|_F\}\ge\beta>0,
\]
where $c_{\rm w}\varepsilon$ bounds the theorem-compatible adjacent weak defect and $L_{\rm can}$ converts that defect into canonical-RSM error.  Then
\begin{equation}
 1-\rho_{\rm canonical}(m,\star)
 \le
 C\bigl(1-R_{\rm ridge}^2\bigr),
 \qquad
 C:=\frac{2L_{\rm can}^2c_{\rm w}^2}{\beta^2}.
 \label{eq:rsa-ridge-canonical-bound}
\end{equation}
Here $\rho_{\rm canonical}$ may be the axis-quotiented score, or the balanced score when its weights vary Lipschitz-continuously on the family.
\end{corollary}

\begin{proof}
By definition,
\[
 \rho_{\rm canonical}(m,\star)
 =\frac{\langle c_m,c_\star\rangle_F}
 {\|c_m\|_F\|c_\star\|_F}.
\]
Lemma~\ref{lem:rsa-cosine-stability} gives
\[
 1-\rho_{\rm canonical}(m,\star)
 \le\frac{2\|c_m-c_\star\|_F^2}{\beta^2}.
\]
The assumed local weak-to-canonical estimate therefore implies
\begin{align*}
 1-\rho_{\rm canonical}(m,\star)
 &\le\frac{2L_{\rm can}^2c_{\rm w}^2}{\beta^2}
       \varepsilon^2\\
 &=C\bigl(1-R_{\rm ridge}^2\bigr),
\end{align*}
where the last line uses \eqref{eq:rsa-ridge-residual}.
\end{proof}

Equation~\eqref{eq:rsa-ridge-canonical-bound} becomes especially transparent when the canonical and raw RSMs are written side by side.  Define
\begin{equation}
 c_m:=\Pi_{\mathsf H}R_{\rm can}(m),
 \qquad
 g_m:=\Pi_{\mathsf H}R(H_m)=c_m+s_m,
 \label{eq:rsa-model-core-symmetry}
\end{equation}
where $R_{\rm can}$ is the axis-quotiented or balanced essential RSM.  Then
\begin{equation}
 \rho_{\rm canonical}(m,\star)
 =\cos(c_m,c_\star),
 \qquad
 \rho_{\rm raw}(m,\star)
 =\cos(c_m+s_m,c_\star+s_\star).
 \label{eq:rsa-canonical-raw-explicit}
\end{equation}
The residual $s_m$ has two conceptually different pieces:
\begin{equation}
 s_m
 =\underbrace{\bigl(g_m^{\rm core}-c_m\bigr)}_{\substack{\text{noncanonical scale}\\\text{and multiplicity}}}
 +\underbrace{a_m}_{\substack{\text{redundant added-feature}\\\text{geometry}}}.
 \label{eq:rsa-symmetry-residual-decomposition}
\end{equation}
The ridge bound controls the first cosine in \eqref{eq:rsa-canonical-raw-explicit} through the difference $c_m-c_\star$.  It does not control $s_m-s_\star$.  Hence canonical RSA has the one-sided prediction
\[
 1-\rho_{\rm canonical}
 \lesssim C(1-R_{\rm ridge}^2),
\]
while raw RSA need not follow any single curve.  At exact weak equivalence, the canonical cores can coincide, $c_m=c_\star$, so $\rho_{\rm canonical}=1$.  Theiss scaling and duplication can nevertheless change the first term of \eqref{eq:rsa-symmetry-residual-decomposition}, and redundant feature addition can change the second.  The same privileged axes can therefore support a broad horizontal range of raw RSA values at $R_{\rm ridge}^2=1$.

\begin{figure}[t]
\centering
\begin{tikzpicture}[
  x=6.15cm,
  y=4.70cm,
  >=Latex,
  font=\sffamily,
  axis/.style={->,line width=0.75pt},
  frontier/.style={line width=0.65pt,dash pattern=on 4pt off 3pt},
  point/.style={circle,fill=black,draw=none,inner sep=0pt,minimum size=2.7pt},
  panel title/.style={font=\sffamily\bfseries\normalsize,align=center},
  axis label/.style={font=\sffamily\small},
  annotation/.style={font=\sffamily\scriptsize,align=center,fill=white,inner sep=1.2pt}
]

\begin{scope}
  \draw[axis] (0,0) -- (1.035,0);
  \draw[axis] (0,0) -- (0,1.035);

  \node[panel title] at (0.515,1.105) {(A) canonical RSA};
  \node[axis label,rotate=90] at (-0.090,0.510) {Ridge similarity};
  \node[axis label] at (0.515,-0.105) {RSA similarity};

  \draw[frontier]
    (0.035,0.190)
      .. controls (0.315,0.675) and (0.715,0.880) ..
    (0.985,0.925);

  \node[annotation,anchor=south] at (0.765,0.965) {canonical frontier};
  \draw[line width=0.35pt] (0.765,0.945) -- (0.765,0.900);

  \foreach \x/\y in {
    0.045/0.095,
    0.085/0.145,
    0.120/0.210,
    0.165/0.280,
    0.220/0.245,
    0.255/0.355,
    0.305/0.420,
    0.345/0.385,
    0.395/0.505,
    0.445/0.555,
    0.485/0.515,
    0.555/0.635,
    0.605/0.590,
    0.655/0.700,
    0.715/0.665,
    0.765/0.775,
    0.825/0.750,
    0.875/0.835,
    0.940/0.825
  }{\node[point] at (\x,\y) {};}

  \node[annotation] at (0.675,0.245)
    {high ridge forces\\high canonical RSA};
\end{scope}

\begin{scope}[xshift=8.12cm]
  \draw[axis] (0,0) -- (1.035,0);
  \draw[axis] (0,0) -- (0,1.035);

  \node[panel title] at (0.515,1.105) {(B) raw RSA};
  \node[axis label,rotate=90] at (-0.090,0.510) {Ridge similarity};
  \node[axis label] at (0.515,-0.105) {RSA similarity};

  \draw[frontier]
    (0.035,0.170)
      .. controls (0.315,0.625) and (0.715,0.845) ..
    (0.985,0.920);

  \foreach \x/\y in {
    0.045/0.075,
    0.085/0.155,
    0.155/0.105,
    0.205/0.225,
    0.255/0.145,
    0.075/0.345,
    0.145/0.435,
    0.235/0.385,
    0.305/0.305,
    0.385/0.455,
    0.455/0.335,
    0.555/0.465,
    0.665/0.355,
    0.770/0.505,
    0.875/0.605,
    0.115/0.555,
    0.215/0.625,
    0.305/0.525,
    0.405/0.665,
    0.505/0.545,
    0.605/0.690,
    0.700/0.575,
    0.790/0.705,
    0.910/0.675,
    0.195/0.735,
    0.340/0.805,
    0.475/0.765,
    0.585/0.865,
    0.695/0.825,
    0.810/0.925,
    0.910/0.855,
    0.980/0.945
  }{\node[point] at (\x,\y) {};}

  \node[annotation] at (0.810,0.235)
    {horizontal spread\\from raw symmetries\\or biased sampling};
\end{scope}

\end{tikzpicture}

\caption{Schematic model-by-model scatter plots.  \textbf{(A)} The canonical ridge--RSA relation is one-sided: sufficiently high ridge similarity forces high canonical RSA, so the comparison models lie below an increasing frontier.  \textbf{(B)} Raw symmetries or unequal spatial sampling can reweight the observed geometry without comparably changing ridge predictivity, producing additional horizontal spread around the underlying trend.}
\label{fig:rsa-ridge-sampling-schematic}
\end{figure}
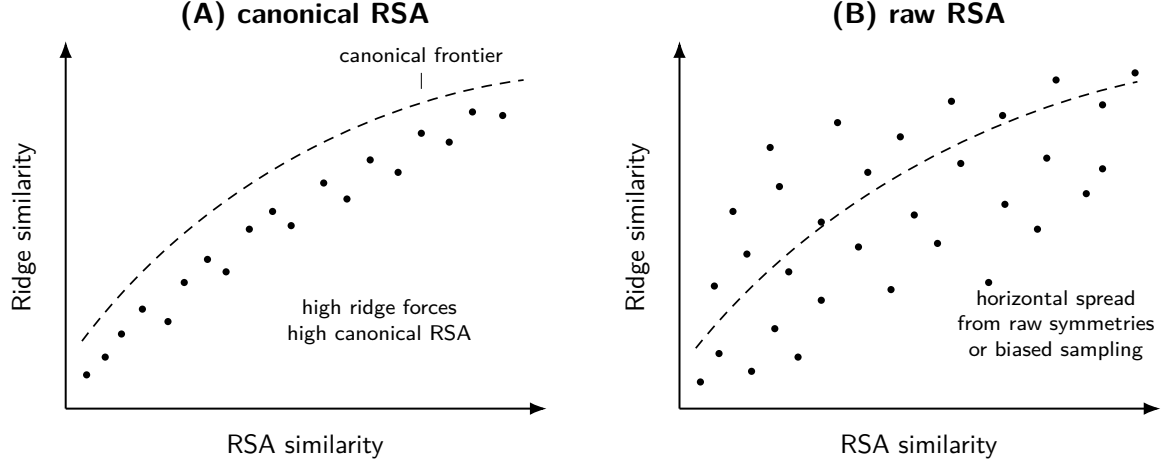

\paragraph{When should models lie near the canonical frontier?}
The constant $C$ in \eqref{eq:rsa-ridge-canonical-bound} should be
understood as a worst-case upper envelope for the local comparison family,
not as the slope that every empirical model is expected to attain.  A sharper
local description follows from ordinary smoothness.  Let $\mathcal M$ be a
finite-dimensional local family of comparison models with the target $\star$
as an interior point, and write
\[
  F(m):=1-R_{\rm ridge}^2(m,\star),
  \qquad
  c(m):=c_m.
\]
Assume that $F$ is $C^2$ and $c$ is $C^1$ near $\star$.  This is the relevant
local regularity condition: it holds when the fitted linear comparison and its
normalized prediction residual vary smoothly with the model, for example on a
constant-rank least-squares family with a locally unique fitted map.  It can
fail at rank changes or when model selection causes a regularization parameter
to jump.

In the present model--target setup, the target compared with itself has zero
prediction residual, so $F(\star)=0$; moreover $F\geq0$.  Hence $\star$ is a
local minimum of $F$ and $DF(\star)=0$.  For any smooth two-sided path
$m(t)\in\mathcal M$ with $m(0)=\star$ and tangent
$h:=\dot m(0)\in T_\star\mathcal M$, Taylor's theorem therefore gives
\begin{equation}
  1-R_{\rm ridge}^2(m(t),\star)
  =a_{\rm ridge}(h)t^2+o(t^2),
  \qquad
  a_{\rm ridge}(h)
  :=\frac12 D^2F(\star)[h,h]\geq0.
  \label{eq:rsa-local-ridge-expansion}
\end{equation}
Thus the quadratic ridge expansion is a consequence of local smoothness, not
an additional assumption.  Strict positivity is a nondegeneracy property of
the direction $h$: if $a_{\rm ridge}(h)=0$, ridge is locally insensitive to
that direction to second order.

Let
\[
  v:=Dc(\star)[h],
\]
so that the projected canonical RSM satisfies
\[
  c_{m(t)}=c_\star+t v+o(t).
\]
Write $v=v_\parallel+v_\perp$, where $v_\parallel$ is parallel to
$c_\star$ and $v_\perp$ is Frobenius-orthogonal to it.  The local cosine
expansion gives
\begin{equation}
  1-\rho_{\rm canonical}(m(t),\star)
  =
  \frac{\|v_\perp\|_F^2}
       {2\|c_\star\|_F^2}\,t^2
  +o(t^2).
  \label{eq:rsa-local-cosine-expansion}
\end{equation}
The local weak-to-canonical estimate used in
Corollary~\ref{cor:rsa-ridge-canonical} further implies that any direction
with $v\neq0$ must have $a_{\rm ridge}(h)>0$.  Indeed, substituting the local
expansions into
\[
  \|c_{m(t)}-c_\star\|_F
  \leq L_{\rm can}c_{\rm w}
      \sqrt{1-R_{\rm ridge}^2(m(t),\star)}
\]
and dividing by $|t|$ gives
\[
  \|v\|_F
  \leq L_{\rm can}c_{\rm w}\sqrt{a_{\rm ridge}(h)}.
\]
Thus a direction that changes the canonical RSM to first order cannot be
invisible to ridge at quadratic order.  Directions with
$a_{\rm ridge}(h)=0$ necessarily also have $v=0$ and require a higher-order
analysis.

For every direction with $a_{\rm ridge}(h)>0$,
\begin{equation}
  \lim_{t\to0}
  \frac{1-\rho_{\rm canonical}(m(t),\star)}
       {1-R_{\rm ridge}^2(m(t),\star)}
  =
  \lambda(h)
  :=
  \frac{\|v_\perp\|_F^2}
       {2\|c_\star\|_F^2a_{\rm ridge}(h)}.
  \label{eq:rsa-local-frontier-slope}
\end{equation}
The sharp local envelope for the family is therefore governed by
\[
  C_{\rm sharp}
  :=\sup_{\substack{h\in T_\star\mathcal M\\
                    a_{\rm ridge}(h)>0}}
       \lambda(h),
  \qquad
  C_{\rm sharp}\leq C.
\]
Literal equality with the proof constant $C$ is not the empirical prediction.
Rather, models lie near the observed frontier when their dominant directions
of variation have $\lambda(h)$ close to $C_{\rm sharp}$.  If a model
collection varies mainly along one common direction---for example, through
training progress, bottleneck strength, or gradual degradation of one
task-relevant feature family---then near the target it should approximately
satisfy
\[
  1-\rho_{\rm canonical}
  \approx
  \lambda(h)\bigl(1-R_{\rm ridge}^2\bigr),
\]
producing an approximately straight frontier ending at $(1,1)$.  A
heterogeneous model collection explores several directions with different
values of $\lambda(h)$ and should instead fill out a wedge below that
frontier.

The distinction between $v_\parallel$ and $v_\perp$ has a direct
interpretation.  The perpendicular component changes the relative pattern of
canonical stimulus similarities---making some stimulus pairs more similar and
others less similar---and therefore lowers canonical RSA.  The parallel
component mainly rescales the entire canonical-RSM pattern, which cosine
similarity largely ignores.  Models can therefore fall well below the
frontier when ridge error is produced mainly by such radial changes, by
components removed by $\Pi_{\mathsf H}$, or by axis-level errors that cancel
when combined into the canonical RSM.  Pure Theiss symmetries are different
again: they leave $c_m=c_\star$ and hence
$\rho_{\rm canonical}=1$, while changing the raw residual $s_m$.  They
therefore generate the horizontal spread in raw RSA described by
Proposition~\ref{prop:rsa-no-raw-soft-bound}, rather than the nontrivial
canonical frontier.

\subsubsection{Sampling stability of RSA}
\label{sec:rsa-spatial-sampling}
The preceding results compare complete model layers.  Neural experiments do
not.  Electrode arrays record units from a limited set of locations, while
fMRI forms spatially local weighted averages and has uneven sensitivity and
coverage across cortex.  These effects are especially important when tuning is
itself spatially organized, as it is in many real brain areas, or for example, in Topographical Deep Artificial Neural Networks (TDANNs), where units are placed on simulated cortical sheets and nearby units are encouraged to have similar response  profiles~\citep{margalit2024tdann}.  A spatially localized sample from such a layer is therefore not merely smaller than the full layer: it is biased toward the functions represented in that part of the sheet. Figure~\ref{fig:rsa-ridge-sampling-schematic} summarizes the resulting prediction.  Canonical RSA retains the one-sided ridge frontier developed above, whereas raw RSA can be spread horizontally by unequal spatial sampling.

\paragraph{Spatial sampling produces a measurement-side version of the
Theiss ambiguity.}
Let $H^N\in\R^{D_N\times P}$ be the complete response matrix of layer
$N\in\{A,B\}$ on $P$ stimuli, and let $S_N$ describe what the experiment
records from that layer.  The measured responses and RSM are
\begin{equation}
 Y_N=S_NH^N,
 \qquad
 R_N^{\rm obs}=Y_N^\top Y_N
              =H^{N\top}S_N^\top S_NH^N.
 \label{eq:rsa-spatial-observation}
\end{equation}
For simple unit selection, unequal recording gain, or repeated sampling of
similar units, $S_N^\top S_N$ is diagonal.  If $f_j^N\in\R^P$ is the response
vector of unit $j$, then
\begin{equation}
 R_N^{\rm obs}
 =\sum_j c_{Nj}f_j^Nf_j^{N\top},
 \qquad c_{Nj}\geq0.
 \label{eq:rsa-spatial-diagonal}
\end{equation}
This is exactly the same RSM form that appears in the Theiss analysis: the
observed geometry is a weighted sum of rank-one feature contributions.  The
source of the freedom is different---Theiss changes the network implementation
while leaving observation fixed, whereas spatial sampling changes the
observation while leaving the network fixed---but the concrete effects on the
RSM closely parallel one another:

\begin{center}
\small
\begin{tabularx}{0.94\linewidth}{@{}p{0.43\linewidth}X@{}}
\toprule
Spatial measurement effect & Theiss-like effect on the RSM \\
\midrule
Unequal gain or unequal inclusion probability
&
Rescaling a feature contribution
\\
Dense sampling of one functional patch
&
Giving those feature directions excess multiplicity, like duplication
\\
A region represented in one sample but absent from the other
&
Adding rank-one terms on only one side
\\
Failure to cover a region
&
Giving its feature directions zero weight
\\
Local voxel pooling
&
Creating mixtures of nearby features, which is more general than simple
scaling or duplication
\\
\bottomrule
\end{tabularx}
\end{center}

Thus spatial bias is a \emph{measurement-side analogue of the Theiss
ambiguity}: the same underlying computation and the same privileged axes can
yield different raw RSMs because the experiment assigns them different
effective scales, multiplicities, and mixtures.

\paragraph{Topography turns spatial bias into functional bias.}
If $m$ units are sampled uniformly without replacement from a layer of width
$D$, then
\begin{equation}
 \mathbb E[R^{\rm obs}]
 =\frac{m}{D}H^\top H.
 \label{eq:rsa-uniform-sampling}
\end{equation}
The factor $m/D$ does not affect Pearson RSM--RSA, so uniform sampling is
unbiased in direction, although a finite sample still adds variability.  If
unit $j$ is instead included with probability $p_j$, then
\begin{equation}
 \mathbb E[R^{\rm obs}]
 =\sum_j p_j f_jf_j^\top,
 \label{eq:rsa-biased-sampling}
\end{equation}
which is generally not proportional to the full-layer RSM.  In a topographic
layer, $p_j$ is correlated with tuning.  Recording more units from the same
patch therefore reduces uncertainty about that patch's RSM, but does not make
the sample representative of the whole layer.

A simple example shows how large the effect can be.  Suppose a layer contains
$K$ spatial patches dominated by orthogonal, equal-norm RSM directions
$q_1,\ldots,q_K$.  The whole layer has geometry
$g_{\rm full}=\sum_{k=1}^Kq_k$, while an array confined to one patch has
geometry $g_{\rm patch}=q_1$.  Then
\begin{equation}
 \cos(g_{\rm full},g_{\rm patch})=\frac{1}{\sqrt K}.
 \label{eq:rsa-spatial-patch-example}
\end{equation}
The full and sampled systems have not acquired different computations; the
sample has simply changed the relative representation of the same spatially
organized feature families.

\paragraph{How electrophysiology and fMRI enter the thought experiment.}
For electrophysiology, a useful first approximation is to select units from a
few localized neighborhoods of a TDANN-like sheet.  Different array placements
then sample different mixtures of tuning types.  For fMRI, each measured
channel is better modeled as a local weighted average of nearby units, together
with uneven spatial coverage and sensitivity.  TopoLM, for example, simulates
the pooling component by Gaussian smoothing over nearby model units before
analysis~\citep{rathi2025topolm}; real fMRI additionally has spatially varying
BOLD sensitivity, vascular filtering, distortion, and
dropout~\citep{kriegeskorte2010voxel,weiskopf2007optimized}.  The relevant
model comparison is therefore not only between two complete layers, but
between the layers after applying the spatial selection and smoothing
associated with the experiment.

\paragraph{RSA is specifically sensitive to sampling weights.}
Uniform and biased sampling have qualitatively different effects.  Uniform
sampling introduces variability that decreases as more units are recorded.
Biased spatial sampling instead converges to a systematically reweighted
representation.  Ridge can absorb some such target-side reweighting, whereas
raw RSA generally cannot.

Suppose the complete representations are exactly linearly related,
$H^B=TH^A$, but only the target is spatially sampled.  Then
\begin{equation}
Y_B=S_BH^B=S_BTH^A,
\label{eq:rsa-spatial-ridge-absorb}
\end{equation}
so a one-way ridge map from the complete source to the sampled target can
remain exact.  Reverse or bidirectional ridge is more demanding, because it
also asks whether the sampled target retains every source
direction~\citep{thobani2025iatc,muzellec2025reverse}.

RSA can change even when both linear directions remain exact.  If two measured
representations differ only by an invertible diagonal reweighting,
$Y_B=DY_A$, then each can linearly predict the other, but generically
\begin{equation}
R_B^{\rm obs}
=Y_A^\top D^\top DY_A
\not\propto
Y_A^\top Y_A
=R_A^{\rm obs}.
\label{eq:rsa-bidirectional-reweighting}
\end{equation}
Thus raw RSA depends not only on which feature directions are present, but
also on how strongly the experiment weights them through gain, multiplicity,
and spatial coverage.  Figure~\ref{fig:rsa-ridge-sampling-schematic}(B)
illustrates the resulting variation in RSA among representations with similar
ridge scores.

\paragraph{What normalization can and cannot repair.}
Axis quotienting can remove unequal multiplicity among feature classes that
were observed.  It cannot recover a class that was never sampled or undo an
unknown spatial mixture.  Known inclusion probabilities can be corrected by
inverse-probability weighting, although often at the cost of increased
variance.  Crossvalidated estimators such as crossnobis can remove
finite-trial noise bias, but still estimate the geometry of the sampled
population rather than that of the complete cortical sheet.

\paragraph{Predictions.}
For independently trained topographic networks, this analysis predicts:
\begin{enumerate}[label=(\roman*)]
\item complete equivalent populations should have high ridge predictivity and
high RSA;
\item under uniform sampling, RSA should approach the complete-population
score as the number of measured units grows;
\item localized sampling can retain a systematic RSA bias even when many units
are recorded;
\item a broadly sampled source can predict a localized target well by one-way
ridge despite reduced raw RSA; and
\item reverse or bidirectional ridge should reveal missing directions, whereas
raw RSA can also change because directions present on both sides receive
different weights~\citep{thobani2025iatc,muzellec2025reverse}.
\end{enumerate}

Thus, uniform subsampling mainly adds finite-sample variability, while
persistent spatial bias changes the RSM obtained even with unlimited data.
In a topographic system, this measurement bias acts like feature scaling or
duplication and can separate raw RSA from ridge despite close alignment of the
underlying feature directions.

\subsection{CKA and Weak-Strong Equivalence}
\label{sec:cka-weak-strong}
Centered kernel alignment (CKA) is another widely used method for comparing representations \citep{Kornblith2019CKA}.  Like raw RSA, raw linear CKA is sensitive to unequal feature scaling, duplication, redundant features, and biased measurement, so exact weak or strong alignment need not imply a high raw score.  After removing those freedoms, however, exact and soft weak alignment imply exact and asymptotically perfect canonical CKA.  We then show why raw RSA and raw CKA usually covary, why high linear CKA implies high scaled-Procrustes and ridge similarity, and how sampling can separate the empirical scores.

\subsubsection{Linear CKA and the centered essential-core cone}
\label{sec:cka-centered-cone}

Keep the evaluation set $\Omega_P=\{x^1,\ldots,x^P\}$ and representation orientation of \eqref{eq:rsa-rsm-def}.  Let
\begin{equation}
 C_P:=I_P-\frac1P\mathbf 1\mathbf 1^\top,
 \qquad
 K(H):=C_PR(H)C_P=(HC_P)^\top(HC_P).
 \label{eq:cka-centered-kernel}
\end{equation}
Thus $K(H)$ is the Gram matrix of the sample-centered population responses.  When both centered Gram matrices are nonzero, define
\begin{equation}
 \rho_{\rm CKA}(H_A,H_B)
 :=\frac{\langle K(H_A),K(H_B)\rangle_F}
 {\|K(H_A)\|_F\,\|K(H_B)\|_F},
 \qquad
 d_{\rm CKA}:=1-\rho_{\rm CKA}.
 \label{eq:cka-definition}
\end{equation}
This is linear CKA.  Equivalently, with $\bar H_N:=H_NC_P$,
\begin{equation}
 \rho_{\rm CKA}(H_A,H_B)
 =\frac{\|\bar H_A\bar H_B^\top\|_F^2}
 {\|\bar H_A\bar H_A^\top\|_F\,
  \|\bar H_B\bar H_B^\top\|_F}.
 \label{eq:cka-factor-form}
\end{equation}
The comparison with \eqref{eq:rsa-pearson-score} is therefore explicit:
\begin{equation}
 \rho_{\rm CKA}
 =\cos\!\bigl(C_PR_AC_P,C_PR_BC_P\bigr),
 \qquad
 \rho_{\rm RSA}
 =\cos\!\bigl(\Pi_{\mathsf H}R_A,\Pi_{\mathsf H}R_B\bigr).
 \label{eq:cka-rsa-projections}
\end{equation}
Both scores globally normalize a projected RSM.  CKA's projection removes translations of the represented point cloud but retains its complete centered inner-product geometry; the RSA projection additionally removes the diagonal and the off-diagonal mean.

\begin{remark}[Task-domain CKA]
\label{rem:cka-task-domain}
The finite-sample formulas have a direct task-domain form.  Let
$\mathcal H=L_0^2(\Omega)$ and let
$X_N:\R^{d_N}\to\mathcal H$ send $a$ to the centered scalar feature
$a^\top(z_N-\E z_N)$.  Then the centered kernel operator is the finite-rank operator $K_N=X_NX_N^\ast$, and linear CKA is the Hilbert--Schmidt cosine of $K_A$ and $K_B$.  Every factorization argument below is unchanged with matrices replaced by these finite-rank operators.  A finite empirical CKA score certifies only the sampled representations unless the evaluation set is determining for the task patch.  For the task-domain version, equality holds almost everywhere; when the task measure has full support on the patch, continuity of the network representations upgrades it to pointwise equality.  The exact weak--strong and zippering corollaries below use one of these task-determining interpretations.
\end{remark}

The Theiss decomposition and the generic uniqueness analysis in Appendix~\S A.12 are metric-independent until the final RSM projection is chosen.  Applying double centering to Lemma~\ref{lem:rsa-theiss-rsm-decomp} gives
\begin{equation}
 K(H)
 =\sum_{j=1}^{m}\gamma_j\bar f_j\bar f_j^\top
 +\sum_{k=1}^{K}\eta_k\bar u_k\bar u_k^\top,
 \qquad
 \bar f_j:=C_Pf_j,
 \quad
 \bar u_k:=C_Pu_k,
 \label{eq:cka-theiss-decomposition}
\end{equation}
with the same positive coefficients $\gamma_j,\eta_k$.  Thus scaling and duplication reweight centered essential-axis contributions, while a redundant added feature can introduce a new centered-kernel direction.

\begin{proposition}[The centered essential-core cone]
\label{prop:cka-centered-core-cone}
Let $f$ be a generic represented function, so that Corollary~\ref{cor:rsa-generic-unique-core} supplies one essential axis dictionary.  Choose one sampled feature vector $f_j$ from each of its $m$ classes and set
\begin{equation}
 k_j:=(C_Pf_j)(C_Pf_j)^\top,
 \qquad
 \mathcal C_f^{\rm CKA}:=\operatorname{cone}\{k_1,\ldots,k_m\}.
 \label{eq:cka-core-cone}
\end{equation}
Then $\mathcal C_f^{\rm CKA}$ is independent of the chosen minimal realization.  Every representation in a cataloged Theiss orbit has
\begin{equation}
 K(H)=\sum_{j=1}^{m}\gamma_jk_j+a,
 \qquad \gamma_j>0,
 \label{eq:cka-core-plus-addition}
\end{equation}
where $a$ is the centered contribution of redundant added features.  If $a_A=a_B=0$ and
$G_{ij}:=\langle k_i,k_j\rangle_F=(\bar f_i^\top\bar f_j)^2$, then
\begin{equation}
 \rho_{\rm CKA}(A,B)
 =\frac{\gamma_A^\top G\gamma_B}
 {\sqrt{\gamma_A^\top G\gamma_A}\,
  \sqrt{\gamma_B^\top G\gamma_B}},
 \label{eq:cka-core-weight-formula}
\end{equation}
and the score equals one exactly when the two weighted centered-core geometries lie on the same positive ray.
\end{proposition}

\begin{proof}
The orbitwise reduction, full-fiber generic uniqueness, and representative-independence arguments are exactly those of Proposition~\ref{prop:rsa-theiss-primitives-reduction}, Theorem~\ref{thm:rsa-generic-full-fiber}, and Corollary~\ref{cor:rsa-generic-unique-core}; only the final linear map on the RSM changes from $\Pi_{\mathsf H}$ to $M\mapsto C_PMC_P$.  Equation~\eqref{eq:cka-theiss-decomposition} therefore gives \eqref{eq:cka-core-plus-addition}.  Expanding the Frobenius cosine as in Proposition~\ref{prop:rsa-equality-criterion} gives \eqref{eq:cka-core-weight-formula}, and equality follows from equality in Cauchy--Schwarz.
\end{proof}

The centered cone is function-intrinsic, but a raw representation still chooses arbitrary positive weights inside it.  The next result is the CKA counterpart of Proposition~\ref{prop:rsa-no-raw-soft-bound}; for CKA, duplication alone can make the score tend all the way to zero.

\begin{proposition}[Raw linear CKA is not forced by exact weak alignment]
\label{prop:cka-no-raw-bound}
There is no universal lower bound on raw linear CKA that tends to one as adjacent weak-alignment errors tend to zero, even when the represented one-step function and reduced essential core agree exactly.  Indeed, there are exactly weakly aligned same-function pairs for which raw linear CKA tends to zero.
\end{proposition}

\begin{proof}
Choose centered orthonormal vectors $u_1,u_2\in\R^P$ with $P\ge4$, and put $q_i=u_iu_i^\top$.  Adding a sufficiently large constant multiple of $\mathbf 1$ gives positive sampled preactivation vectors $g_i=u_i+c\mathbf 1$ with $C_Pg_i=u_i$.  To realize them concretely, take $P$ affinely independent task inputs and interpolate the prescribed values $g_i(x^\mu)$ by affine functionals.  The centered rank-one terms satisfy
$\langle q_1,q_2\rangle_F=0$ and
$\|q_1\|_F=\|q_2\|_F=1$.

Start from the reduced one-step function
$F=v_1\varphi(g_1)+v_2\varphi(g_2)$.  Let $A_M$ contain $M$ copies of $g_1$ and one copy of $g_2$, splitting $v_1$ equally among the copies.  Let $B_M$ contain one copy of $g_1$ and $M$ copies of $g_2$, splitting $v_2$ equally among those copies.  Both implementations compute exactly $F$, and minimal reduction returns the same two-axis core.

Write the physical population responses as
\[
 z_A(x)=g_1(x)a_1+g_2(x)a_2,
 \qquad
 z_B(x)=g_1(x)b_1+g_2(x)b_2,
\]
where
$a_1=(\mathbf 1_M,0)$,
$a_2=e_{M+1}$,
$b_1=e_1$, and
$b_2=(0,\mathbf 1_M)$.  The pairs $(a_1,a_2)$ and $(b_1,b_2)$ are linearly independent, so a linear map sending $a_i$ to $b_i$ extends to an invertible map on $\R^{M+1}$.  Thus the hidden representations are exactly weakly aligned, while the next-layer outputs are identical.

Their centered Gram matrices are
\[
 K(A_M)=Mq_1+q_2,
 \qquad
 K(B_M)=q_1+Mq_2.
\]
Consequently,
\begin{equation}
 \rho_{\rm CKA}(A_M,B_M)
 =\frac{2M}{M^2+1}\longrightarrow0.
 \label{eq:cka-duplication-counterexample}
\end{equation}
This rules out any weak-error-only lower bound for raw CKA.  For activations with a function-preserving scale gauge, such as ReLU, featurewise rescaling produces the same reweighting mechanism; redundant feature addition supplies still more centered-kernel directions through \eqref{eq:cka-core-plus-addition}.
\end{proof}

\subsubsection{Canonical CKA: weak alignment implies CKA}
\label{sec:cka-canonical}

The direct CKA analog of the canonical RSA results is obtained by removing the same scale, multiplicity, and redundant-addition freedoms.  No new core-identifiability argument is needed: the canonical RSMs of Appendix~\S A.12.5 can simply be double-centered and compared by the CKA cosine.

\begin{definition}[Axis-quotiented and balanced CKA]
\label{def:cka-canonical}
For the axis-quotiented RSM of Definition~\ref{def:rsa-axis-quotient}, define
\begin{equation}
 K_{\rm ax}(G):=C_PR_{\rm ax}(G)C_P,
 \qquad
 \rho_{\rm CKA,ax}(G,\widetilde G)
 :=\frac{\langle K_{\rm ax}(G),K_{\rm ax}(\widetilde G)\rangle_F}
 {\|K_{\rm ax}(G)\|_F\,\|K_{\rm ax}(\widetilde G)\|_F}.
 \label{eq:cka-axis-quotient}
\end{equation}
For either balanced RSM $R_s^{\rm bal}$ from Theorem~\ref{thm:rsa-balanced-full-fiber}, define
$K_s^{\rm bal}:=C_PR_s^{\rm bal}C_P$ and its CKA score analogously.
\end{definition}

\begin{theorem}[Exact canonical CKA under weak alignment]
\label{thm:cka-canonical-exact}
Under the hypotheses of Theorem~\ref{thm:rsa-axis-quotient-exact},
\begin{equation}
 K_{\rm ax}(A)=K_{\rm ax}(B),
 \qquad
 \rho_{\rm CKA,ax}(A,B)=1,
 \label{eq:cka-canonical-exact}
\end{equation}
whenever the common centered kernel is nonzero.  Under the hypotheses of Theorem~\ref{thm:rsa-minnorm-equality}, the corresponding balanced CKA score is also one in the common coordinates.  Both conclusions hold at every layer recovered by exact zippering.
\end{theorem}

\begin{proof}
Apply the linear map $M\mapsto C_PMC_P$ to the canonical RSM equalities in Theorems~\ref{thm:rsa-axis-quotient-exact} and~\ref{thm:rsa-minnorm-equality}.  Equality of the nonzero centered kernels makes their CKA cosine one.
\end{proof}

\begin{theorem}[Soft canonical CKA under weak alignment]
\label{thm:cka-canonical-soft}
Under the hypotheses of Theorem~\ref{thm:rsa-soft-axis-bound}, let
\begin{equation}
 D_{\rm ax}(\theta)
 :=\frac{2\sqrt2(1-\theta)}
 {\sqrt{1-(1-\theta)^2}}
 +\frac{\sqrt2}{m}\frac{e_{\rm weak}^2}{\kappa(\theta)^2}.
 \label{eq:cka-soft-Dax}
\end{equation}
If
$\min\{\|K_{\rm ax}(A)\|_F,\|K_{\rm ax}(B)\|_F\}
\ge\beta_{\rm CKA}>0$, then
\begin{equation}
 \rho_{\rm CKA,ax}(A,B)
 \ge1-\frac{2D_{\rm ax}(\theta)^2}{\beta_{\rm CKA}^2}.
 \label{eq:cka-soft-bound}
\end{equation}
Consequently, adjacent weak errors tending to zero imply
$\rho_{\rm CKA,ax}\to1$.  Soft zippering likewise implies upstream axis-quotiented CKA convergence after substituting the bad-axis bound from Theorem~\ref{thm:common-soft-zip} into \eqref{eq:cka-soft-Dax}.
\end{theorem}

\begin{proof}
The proof of Theorem~\ref{thm:rsa-soft-axis-bound} gives
$\|R_{\rm ax}(A)-R_{\rm ax}(B)\|_F\le D_{\rm ax}(\theta)$.  Double centering is a nonexpansive orthogonal projection in Frobenius space, so the same bound holds for the centered kernels; Lemma~\ref{lem:rsa-cosine-stability} then yields \eqref{eq:cka-soft-bound}.  The asymptotic and zippering conclusions follow from the same integer bad-axis argument as Corollary~\ref{cor:rsa-asymptotic-axis-convergence}.
\end{proof}

Balanced CKA retains the computationally weighted votes of Corollary~\ref{cor:rsa-weighted-axis-quotient}, while axis-quotiented CKA gives one scale-free vote per essential class.  On compact families where the balanced weights vary continuously, the local argument of Corollary~\ref{cor:rsa-ridge-canonical} gives the corresponding soft balanced-CKA bound.

\subsubsection{How RSA, CKA, Procrustes, and ridge are related}
\label{sec:rsa-cka-relationship}

Fix a target representation $\star$ and compare a collection of models $m$ to it, as in Corollary~\ref{cor:rsa-ridge-canonical}.  RSA and linear CKA both start from the same stimulus-by-stimulus RSM, but process it differently.  Ridge is more flexible because it fits an affine map and can absorb changes that RSA and CKA retain.  This section makes those statements precise.

\paragraph{Canonical RSA and canonical CKA agree near perfect alignment.}
Let $R_m^{\rm can}$ denote either the axis-quotiented RSM or a balanced RSM, using the same canonicalization for every model and the target.  In the balanced case, assume as in Corollary~\ref{cor:rsa-ridge-canonical} that the selected weights vary smoothly on the local comparison family.  Put
\[
 r_m:=\Pi_{\mathsf H}R_m^{\rm can},
 \qquad
 k_m:=C_PR_m^{\rm can}C_P.
\]
Then
\[
 \rho_{\rm RSA,can}(m,\star)=\cos(r_m,r_\star),
 \qquad
 \rho_{\rm CKA,can}(m,\star)=\cos(k_m,k_\star).
\]
Thus the two scores start from the same canonical RSM and differ only in the final operation applied to it.  Under exact weak alignment both scores equal one by Theorems~\ref{thm:rsa-axis-quotient-exact} and~\ref{thm:cka-canonical-exact}; under soft weak alignment both approach one.

The following local result says more: near a fixed target, small canonical RSA error and small canonical CKA error are comparable up to constants, provided the two scores are insensitive to the same local changes.

\begin{proposition}[Local comparison of canonical RSA and CKA]
\label{prop:rsa-cka-canonical-local}
Let $R_\star^{\rm can}$ lie in a finite-dimensional smooth local family of canonical RSMs, and suppose
\[
 r_\star:=\Pi_{\mathsf H}R_\star^{\rm can}\ne0,
 \qquad
 k_\star:=C_PR_\star^{\rm can}C_P\ne0.
\]
For a local perturbation $\Delta$, define
\begin{align*}
 Q_{\rm RSA}(\Delta)
 &:=\frac{1}{2\|r_\star\|_F^2}
 \left\|
 \Pi_{\mathsf H}\Delta
 -\frac{\langle\Pi_{\mathsf H}\Delta,r_\star\rangle_F}
 {\|r_\star\|_F^2}r_\star
 \right\|_F^2,\\
 Q_{\rm CKA}(\Delta)
 &:=\frac{1}{2\|k_\star\|_F^2}
 \left\|
 C_P\Delta C_P
 -\frac{\langle C_P\Delta C_P,k_\star\rangle_F}
 {\|k_\star\|_F^2}k_\star
 \right\|_F^2.
\end{align*}
If $R_m^{\rm can}=R_\star^{\rm can}+\Delta_m$, then
\begin{align}
 1-\rho_{\rm RSA,can}(m,\star)
 &=Q_{\rm RSA}(\Delta_m)+o(\|\Delta_m\|_F^2),
 \label{eq:rsa-cka-local-rsa}\\
 1-\rho_{\rm CKA,can}(m,\star)
 &=Q_{\rm CKA}(\Delta_m)+o(\|\Delta_m\|_F^2).
 \label{eq:rsa-cka-local-cka}
\end{align}
Let $T$ be the tangent space of the family at the target.  If $Q_{\rm RSA}$ and $Q_{\rm CKA}$ vanish on the same directions in $T$, then there are constants
$0<\lambda_-\le\lambda_+<\infty$ such that
\begin{equation}
 \lambda_-Q_{\rm RSA}(\Delta)
 \le Q_{\rm CKA}(\Delta)
 \le \lambda_+Q_{\rm RSA}(\Delta),
 \qquad \Delta\in T.
 \label{eq:rsa-cka-canonical-local}
\end{equation}
\end{proposition}

\begin{proof}
For nonzero $u$ and a small perturbation $e$,
\[
 1-\cos(u+e,u)
 =\frac{1}{2\|u\|_F^2}
 \left\|e-\frac{\langle e,u\rangle_F}{\|u\|_F^2}u\right\|_F^2
 +o(\|e\|_F^2).
\]
Apply this first with $u=r_\star$ and $e=\Pi_{\mathsf H}\Delta_m$, and then with $u=k_\star$ and $e=C_P\Delta_mC_P$.  This gives \eqref{eq:rsa-cka-local-rsa}--\eqref{eq:rsa-cka-local-cka}.  After removing the directions on which both quadratic forms vanish, their square roots are norms on a finite-dimensional space, so they bound one another up to constants, giving \eqref{eq:rsa-cka-canonical-local}.
\end{proof}

This is only a local statement.  RSA and CKA need not agree globally because their final operations discard different information.  For every $a$,
\[
 C_P(\mathbf 1a^\top+a\mathbf 1^\top)C_P=0,
\]
whereas $\Pi_{\mathsf H}$ generally retains part of this term.  Conversely, $\Pi_{\mathsf H}D=0$ for every diagonal matrix $D$, whereas $C_PDC_P$ is generally nonzero.  Canonicalization removes scale, multiplicity, and redundant-feature differences, but it does not make RSA and CKA identical.

\paragraph{Raw RSA and raw CKA use the same feature weights.}
The raw relationship is clearest when the essential axes are fixed and there are no redundant added features.  Let $f_1,\ldots,f_m$ represent the common essential classes and define
\[
 q_j:=\Pi_{\mathsf H}(f_jf_j^\top),
 \qquad
 k_j:=(C_Pf_j)(C_Pf_j)^\top,
\]
\[
 (G_{\rm RSA})_{ij}:=\langle q_i,q_j\rangle_F,
 \qquad
 (G_{\rm CKA})_{ij}:=\langle k_i,k_j\rangle_F.
\]

\begin{proposition}[Raw RSA and CKA use the same scale and multiplicity weights]
\label{prop:rsa-cka-shared-weights}
Suppose the comparison model and target have no added-feature terms and have raw effective weights $\gamma_m,\gamma_\star\in\R_{>0}^m$ on the same essential classes.  Then
\begin{align}
 \rho_{\rm RSA,raw}(m,\star)
 &=\frac{\gamma_m^\top G_{\rm RSA}\gamma_\star}
 {\sqrt{\gamma_m^\top G_{\rm RSA}\gamma_m}\,
  \sqrt{\gamma_\star^\top G_{\rm RSA}\gamma_\star}},
 \label{eq:rsa-cka-raw-rsa}\\
 \rho_{\rm CKA,raw}(m,\star)
 &=\frac{\gamma_m^\top G_{\rm CKA}\gamma_\star}
 {\sqrt{\gamma_m^\top G_{\rm CKA}\gamma_m}\,
  \sqrt{\gamma_\star^\top G_{\rm CKA}\gamma_\star}}.
 \label{eq:rsa-cka-raw-cka}
\end{align}
In particular, the same scale and multiplicity vectors enter both scores.  If
$G_{\rm CKA}=aG_{\rm RSA}$ for some $a>0$ on the coefficient subspace explored by the model family, then the two raw scores are identical throughout that family.
\end{proposition}

\begin{proof}
Equation~\eqref{eq:rsa-cka-raw-rsa} is Proposition~\ref{prop:rsa-equality-criterion}; equation~\eqref{eq:rsa-cka-raw-cka} is Proposition~\ref{prop:cka-centered-core-cone}.  Both use the same coefficients $\gamma_m,\gamma_\star$; only the final matrix $G_{\rm RSA}$ or $G_{\rm CKA}$ differs.  Multiplying that matrix by a positive scalar does not change the normalized score.
\end{proof}

This explains why raw RSA and raw CKA should usually covary rather than behave independently.  If a feature is duplicated or rescaled, both metrics receive the same changed feature weight.  They can still rank models differently because $G_{\rm RSA}$ and $G_{\rm CKA}$ process those shared weights differently.  Added redundant features increase the disagreement because RSA and CKA retain different parts of the same added rank-one terms.

\paragraph{High linear CKA implies high Procrustes and ridge similarity.}
Near-perfect linear CKA means that the two centered Gram matrices are close after one global rescaling.  The next lemma converts this into a concrete statement about the representations themselves: after one global scale and an orthogonal map, their centered activations are close.

Normalize the centered empirical factors by
\[
 X:=P^{-1/2}C_PH_A^\top,
 \qquad
 Y:=P^{-1/2}C_PH_B^\top,
 \qquad
 K_A^c:=XX^\top,
 \qquad
 K_B^c:=YY^\top.
\]
The factor $P^{-1}$ in $K_N^c=P^{-1}K(H_N)$ does not change CKA.

\begin{lemma}[High CKA implies low scaled-Procrustes error]
\label{lem:cka-procrustes-bound}
Assume $d_A\le d_B$, let
$c=\rho_{\rm CKA}(H_A,H_B)$, and set
\[
 \tau^2:=\frac{\|K_B^c\|_F}{\|K_A^c\|_F},
 \qquad
 r:=\operatorname{rank}(K_B^c-\tau^2K_A^c)\le d_A+d_B.
\]
Then some $Q\in\R^{d_A\times d_B}$ with $QQ^\top=I_{d_A}$ makes
\[
 E_{\rm Proc}(z):=\tau Q^\top(z-\mu_A)+\mu_B
\]
have empirical alignment error
\begin{equation}
 \omega_P(E_{\rm Proc})
 \le\Gamma(c)
 :=(2r)^{1/4}\|K_B^c\|_F^{1/2}(1-c)^{1/4}.
 \label{eq:cka-procrustes-bound}
\end{equation}
The same statement holds on the task domain with Hilbert--Schmidt norms and finite operator rank.  On a fixed-rank family whose nonzero Gram eigenvalues are uniformly bounded below,
$\Gamma(c)=O((1-c)^{1/2})$.
\end{lemma}

\begin{proof}
The rectangular Procrustes identity and the Powers--St\o rmer inequality give
\[
 \min_{QQ^\top=I}\|Y-\tau XQ\|_F^2
 \le\|K_B^c-\tau^2K_A^c\|_\ast
 \le\sqrt r\,\|K_B^c-\tau^2K_A^c\|_F.
\]
By the choice of $\tau$,
\[
 \|K_B^c-\tau^2K_A^c\|_F^2
 =2\|K_B^c\|_F^2(1-c).
\]
Combining the displays and restoring the means gives
\eqref{eq:cka-procrustes-bound}.  The fixed-rank improvement follows from local Lipschitz stability of the matrix square root when the nonzero eigenvalues stay bounded away from zero.
\end{proof}

CKA compares centered Gram matrices; Procrustes compares the centered activations after choosing the best scale and orthogonal map.  If $\widehat K_N:=K(H_N)/\|K(H_N)\|_F$, then
\[
 \|\widehat K_A-\widehat K_B\|_F^2
 =2\bigl(1-\rho_{\rm CKA}(H_A,H_B)\bigr),
\]
so CKA near one means that the normalized centered Gram matrices are close.  The lemma then shows that the activations themselves are close after the best scale and orthogonal map.  The two scores are not identical away from one: if $s_i$ are the singular values of the centered cross-covariance, CKA weights shared modes through $\sum_i s_i^2$, while Procrustes uses $\sum_i s_i$.

Ridge is more flexible than either metric because it can fit any affine map, including anisotropic reweightings.  Therefore high ridge similarity need not imply high raw CKA.  The converse does hold near one: the scaled-Procrustes map above is one possible ridge map, so high raw linear CKA forces high ridge prediction.

\begin{corollary}[Consequences for ridge]
\label{cor:ridge-cka-relations}
Use the fixed-target/model-zoo setup and the normalized prediction-error convention
\[
 R_{\rm ridge}^2=1-\varepsilon_{\rm ridge}^2
\]
from Corollary~\ref{cor:rsa-ridge-canonical}.

\begin{enumerate}[label=(\roman*)]
\item Let
\[
 K_m^{\rm can}:=C_PR_m^{\rm can}C_P,
 \qquad
 K_\star^{\rm can}:=C_PR_\star^{\rm can}C_P
\]
be the centered canonical kernels.  Suppose, on a locally identifiable compact family,
\[
 \|K_m^{\rm can}-K_\star^{\rm can}\|_F
 \le L_{\rm CKA}c_{\rm w}\varepsilon_{\rm ridge},
 \qquad
 \min\{\|K_m^{\rm can}\|_F,\|K_\star^{\rm can}\|_F\}
 \ge\beta_{\rm CKA}>0.
\]
Then
\begin{equation}
 1-\rho_{\rm CKA,can}(m,\star)
 \le
 C_{\rm CKA}\bigl(1-R_{\rm ridge}^2\bigr),
 \qquad
 C_{\rm CKA}:=
 \frac{2L_{\rm CKA}^2c_{\rm w}^2}{\beta_{\rm CKA}^2}.
 \label{eq:ridge-canonical-cka-frontier}
\end{equation}
Thus high theorem-compatible ridge prediction forces high canonical CKA locally.

\item Let
$c_{\rm raw}:=\rho_{\rm CKA,raw}(m,\star)$, orient the comparison from model to target as in Lemma~\ref{lem:cka-procrustes-bound}, and suppose
\[
 Y_\star:=P^{-1/2}C_PH_\star^\top,
 \qquad
 \|Y_\star\|_F\ge\beta_Y>0.
\]
Then the best affine ridge fit satisfies
\begin{equation}
 1-R_{\rm ridge}^2
 \le\frac{\Gamma(c_{\rm raw})^2}{\beta_Y^2}.
 \label{eq:raw-cka-ridge-frontier}
\end{equation}
Consequently, on a regular fixed-rank family,
\begin{equation}
 1-R_{\rm ridge}^2
 \le C_{\rm CR}\bigl(1-c_{\rm raw}\bigr),
 \label{eq:raw-cka-ridge-fixed-rank}
\end{equation}
whereas the global rank-changing bound gives
\begin{equation}
 1-R_{\rm ridge}^2
 \le C_{\rm CR}^{\rm glob}
 \sqrt{1-c_{\rm raw}}.
 \label{eq:raw-cka-ridge-global}
\end{equation}
There is no converse lower bound on raw CKA in terms of ridge: Proposition~\ref{prop:cka-no-raw-bound} gives pairs with $R_{\rm ridge}^2=1$ but raw CKA tending to zero.
\end{enumerate}
\end{corollary}

\begin{proof}
For part (i), apply Lemma~\ref{lem:rsa-cosine-stability} to the centered-kernel bound and use
$\varepsilon_{\rm ridge}^2=1-R_{\rm ridge}^2$.
For part (ii), ridge optimizes over all affine maps, while the scaled-Procrustes map of Lemma~\ref{lem:cka-procrustes-bound} is one member of that class.  Its normalized error is at most
$\Gamma(c_{\rm raw})/\|Y_\star\|_F$, proving
\eqref{eq:raw-cka-ridge-frontier}.  The fixed-rank and global forms follow from the two bounds in that lemma.  The final statement is the duplication construction of Proposition~\ref{prop:cka-no-raw-bound}.
\end{proof}

Applying part (ii) at both adjacent layers and then Corollary~\ref{cor:rsa-ridge-canonical} also gives a one-way relation from raw CKA to canonical RSA.  If
\[
 \delta_{\rm CKA}
 :=\max_{r\in\{\ell,\ell+1\}}
 \bigl(1-\rho_{\rm CKA,raw}^{(r)}(m,\star)\bigr),
\]
then, under the same local identifiability assumptions,
\begin{equation}
 1-\rho_{\rm RSA,can}^{(\ell)}(m,\star)
 \le \widetilde C_{\rm RC}\,\delta_{\rm CKA}
 \label{eq:raw-cka-canonical-rsa-fixed-rank}
\end{equation}
on regular fixed-rank families, while the global rank-changing estimate has right-hand side
$\widetilde C_{\rm RC}^{\rm glob}\sqrt{\delta_{\rm CKA}}$.
At CKA equal to one, the Procrustes error is zero.  These implications rely on the Gram-factor structure of \emph{linear} CKA; nonlinear-kernel CKA equal to one need not imply an affine map between the original representations.

\paragraph{What this predicts in empirical model comparisons.}
The main conclusions are:
\begin{itemize}
\item After canonicalization, weak alignment makes both RSA and CKA approach one; near-perfect alignment, their errors are comparable when they are sensitive to the same local changes.
\item Raw RSA and raw CKA usually covary because the same feature scales and multiplicities enter both, but they can rank models differently because they process the resulting RSM differently.
\item High raw linear CKA forces high Procrustes and ridge similarity.  High ridge similarity does not force high raw CKA, and can coexist with a wide range of raw RSA scores.
\item Under the assumptions of Corollary~\ref{cor:rsa-ridge-canonical}, high theorem-compatible ridge prediction also forces high canonical RSA and canonical CKA locally.
\item High raw CKA at adjacent layers can be passed through the Procrustes map and the weak--strong results to force high canonical RSA upstream; low raw CKA by itself proves little.
\item Scale, duplication, redundant features, and biased neural sampling can change raw RSA and CKA even when the task-relevant representation is shared.
\end{itemize}

These distinctions also qualify the interpretation of \citet{bo2024functional}.  Their comparison pools trained and randomized networks in the same model zoo.  Trained and untrained networks are \emph{not} two minimal solutions to the same hard task, so this comparison lies outside the main contravariance regime.  Training produces large representational changes that raw RSA, CKA, and Procrustes retain, while ridge can ignore some of them by linearly reweighting and mixing features.  Because trained-versus-randomized differences also strongly affect behavior, this setup can make RSA, CKA, and Procrustes appear more behaviorally informative than ridge.  This is consistent with our theory, rather than evidence that those metrics are generally preferable.

\subsubsection{Sampling and empirical metric disagreement}
\label{sec:cka-sampling}

Raw CKA is sensitive to sampling.  Uniform sampling introduces error that
shrinks with sample size, while biased sampling converges to a systematically
reweighted representation.  Ridge can absorb some target-side reweighting,
whereas raw RSA and CKA retain it.

\paragraph{Unit and spatial sampling.}
Let $Y_N=S_NH_N$ be the measured representation.  Its centered CKA kernel is
\begin{equation}
K_N^{\rm obs}
=C_PH_N^\top S_N^\top S_NH_NC_P.
\label{eq:cka-observed-sampling}
\end{equation}
Thus unequal inclusion probabilities or gains reweight feature contributions.
Localized sampling overrepresents some tuning classes and misses others, while
voxel pooling mixes them.  In a topographic representation, spatial sampling
bias is therefore also functional sampling bias.

\begin{lemma}[Unit subsampling]
\label{lem:cka-unit-subsampling}
Write the complete centered kernel as
\[
K(H)=\sum_{j=1}^d k_j,
\qquad
k_j=\bar f_j\bar f_j^\top,
\qquad
\bar f_j=C_Pf_j.
\]
If $q$ units are sampled independently with probabilities
$p=(p_1,\ldots,p_d)$, let
\[
\widehat K_q=\frac1q\sum_{a=1}^q k_{J_a},
\qquad
K_p=\sum_{j=1}^d p_jk_j.
\]
Then
\begin{equation}
\E\widehat K_q=K_p,
\qquad
\E\|\widehat K_q-K_p\|_F^2
=\frac1q
\left(
\sum_{j=1}^d p_j\|k_j\|_F^2-\|K_p\|_F^2
\right).
\label{eq:cka-unit-subsampling-law}
\end{equation}
Uniform sampling therefore converges to $K(H)/d$, which gives the same CKA
score as the complete kernel, while biased sampling converges to a differently
weighted kernel.
\end{lemma}

\begin{proof}
The expectation is immediate, and independence gives the displayed
$1/q$ variance.
\end{proof}

Repeatedly sampling one tuning class therefore has the same effect as
increasing its multiplicity.  Known sampling probabilities can be corrected
by inverse-probability weighting, but an unobserved feature class cannot be
recovered.  RSA and CKA may also respond differently to the same bias because
they process the sampled Gram matrix differently.

\paragraph{Stimulus sampling.}
CKA is likewise sensitive to which stimuli are sampled.  If $S_I$ selects a
stimulus subset, then
\begin{equation}
K_{N,I}
=C_pS_IH_N^\top H_NS_I^\top C_p.
\label{eq:cka-stimulus-subsampling}
\end{equation}
Independent samples from the same task distribution converge to the
population score.  A biased stimulus set instead converges to the score for a
different task distribution.  Because centering is recomputed on the subset,
$K_{N,I}$ is not generally a principal submatrix of the complete centered
kernel.

\paragraph{Repeated measurement.}
Repeated measurements remove trial-noise bias, not sampling bias.  If
$\widetilde H_N^{(r)}=H_N+\Xi_N^{(r)}$ with independent mean-zero noise, then
for $r\ne s$,
\[
\E\!\left[
C_P\widetilde H_N^{(r)\top}
\widetilde H_N^{(s)}C_P
\right]
=K(H_N).
\]
Cross-repeat estimators therefore recover the clean kernel in expectation,
but cannot recover units or stimuli that were never sampled.

Thus uniform sampling produces shrinking error, while persistent sampling
bias changes the limiting Gram matrix.  Raw CKA can consequently disagree
with ridge, RSA, or Procrustes even when the complete representations satisfy
the canonical relationships derived above.

\end{document}